\pdfoutput=1
\documentclass[10pt]{article}
\usepackage[letterpaper,margin=0.875in]{geometry}
\usepackage[round]{natbib}
\usepackage{amsmath,amssymb,amsthm}
\usepackage{booktabs}
\usepackage{graphicx}
\usepackage{algorithm}
\usepackage{algpseudocode}
\usepackage[hidelinks]{hyperref}
\hypersetup{pdftitle={Exact and Certified Data Shapley for Weighted Nearest Neighbor Regression},pdfauthor={Zongye Lyu, Lihua Peng}}

\newtheorem{theorem}{Theorem}
\newtheorem{lemma}{Lemma}
\newtheorem{corollary}{Corollary}
\newtheorem{proposition}{Proposition}
\theoremstyle{definition}
\newtheorem{definition}{Definition}
\theoremstyle{remark}
\newtheorem{remark}{Remark}

\newcommand{\val}{\mathrm{val}}

\newenvironment{restate}[1]{\par\medskip\noindent\textbf{#1.}\ \itshape}{\par\medskip}
\newcommand{\appsection}[1]{\section{#1}%
  \setcounter{equation}{0}\setcounter{theorem}{0}\setcounter{lemma}{0}\setcounter{corollary}{0}%
  \setcounter{proposition}{0}\setcounter{remark}{0}\setcounter{definition}{0}}
\newcommand{\NK}{N_K}
\newcommand{\yhat}{\hat{y}}

\begin{document}

\title{Exact and Certified Data Shapley for Weighted Nearest Neighbor Regression}
\author{%
  Zongye Lyu\thanks{Equal contribution.}\\
  Faculty of Information Technology, Monash University\\
  Melbourne, VIC 3800, Australia\\
  \texttt{lyuzongye@gmail.com}
  \and
  Lihua Peng\footnotemark[1]\\
  Faculty of Business and Economics, The University of Melbourne\\
  Melbourne, VIC 3010, Australia\\
  \texttt{lihuap@student.unimelb.edu.au}}
\date{}
\maketitle

\begin{abstract}
Data Shapley can be computed exactly and efficiently for unweighted $K$-nearest-neighbor (KNN) models, the basis of the popular KNN-Shapley. For weighted KNN, however, efficient exact algorithms are known only for hard-label classification. In regression, the prediction is a weighted average whose normalization term changes with every subset. The best exact algorithm then takes $O(N^K)$ time for $N$ training points. We show that this normalization term is not an obstacle. A subset affects the prediction only through two numbers: the total weight of its $K$ nearest neighbors and their total weighted target. Once the weights are discretized, subsets with equal total weight share their normalization term and can be summed as a group. This gives an exact algorithm that is near-linear in $N$ and polynomial in $K$ and the number of weight levels, for real targets under the squared loss. The same idea gives exact values for soft-label classifiers in the same time, and speeds up the earlier exact algorithm for hard labels by a factor of $N$. For continuous weights, we give a deterministic approximation with a certified error bound, without discretizing the weights. On the downside, we prove that some exact values take exponentially many bits to write down, so the dependence on weight precision cannot be removed in general. In experiments on real datasets with up to $100{,}000$ training points and with $10$ or $100$ validation points, we compare the exact values with Monte Carlo estimates computed in the same time. Monte Carlo ranks the points less and less like the exact values as the training set grows. Two of its runs also disagree on a large share of the top $10\%$ of points.
\end{abstract}

\section{Introduction}
\label{sec:intro}

Data valuation assigns each training point a score for its contribution to a learned model. The scores are used to
find mislabeled or noisy data and to divide the revenue of a jointly trained model among data owners
\citep{ghorbani2019data,jia2019towards,agarwal2019marketplace}. The Shapley value \citep{shapley1953value} is the most
widely studied score, because it is the only one that satisfies a standard set of fairness axioms.

\textbf{KNN-Shapley.} For a general learning algorithm, the exact Shapley value takes a number of model trainings
exponential in the number of training points $N$. For the $K$-nearest-neighbor (KNN) model it is cheap:
\citet{jia2019knn} showed that the exact values of all $N$ points of an unweighted KNN model take $O(N\log N)$ time,
for classification and for regression with the squared loss.
Valuing data by the Shapley value of a KNN model, often on features from a pretrained network, is known as
KNN-Shapley, and it has become one of the most practical data valuation methods
\citep{wang2023threshold,wang2024efficient}.

\begin{table}[t]
\caption{Exact Data Shapley values of KNN models: the known algorithms (top) and ours (bottom). Each time is for all
$N$ values of one query, ours after sorting. $N$ counts the training points, $K$ the neighbors and $W$ the weight
levels. For continuous weights the values are within $\varepsilon$, with $L$ quadrature nodes, polylogarithmic in
$1/\varepsilon$ and the range of the weights; $P$ counts the pairs of a total weight and a total weighted target when
the targets are on a grid. The first row covers classification and the squared loss; the hardness holds for integer
targets written in binary and $K=N$.}
\label{tab:landscape}
\centering
\small
\setlength{\tabcolsep}{4pt}
\begin{tabular}{@{}lll@{}}
\toprule
KNN model and utility & Time & Source \\
\midrule
Unweighted & $O(N\log N)$ & Jia et al. \\
Weighted, hard label & $O(WK^2N^2)$ & Wang et al. \\
Weighted, any utility & $O(N^K)$ & Jia et al. \\
\midrule
Regression, squared loss & $O(WK^2N)$ & Theorem~\ref{thm:sq} \\
\quad continuous weights & $O(NKL)$ & Theorem~\ref{thm:cont} \\
Soft label & $O(WK^2N)$ & Corollary~\ref{cor:soft} \\
Weighted, hard label & $O(WK^2N)$ & Corollary~\ref{cor:hard} \\
Regression, other losses & $O(NKP)$ & Section~\ref{sec:pair} \\
\quad absolute loss & \#P-hard & Theorem~\ref{thm:abs} \\
\bottomrule
\end{tabular}
\end{table}

\textbf{Weighted KNN.} In weighted KNN, which lets nearer neighbors count more, exact values have remained
expensive. In regression, the prediction for a query from a subset $S$ of the training data is the weighted average
\begin{equation}
\yhat(S)=\frac{\sum_{z_j\in \NK(S)} w_j\,y_j}{\sum_{z_j\in \NK(S)} w_j}
\label{eq:pred}
\end{equation}
of the targets of the $K$ nearest neighbors $\NK(S)$ of the query in $S$, with weights $w_j$ that depend on the
distance to the query. Its denominator, the \emph{normalization term}, changes with $S$. For regression, and for soft-label classifiers, whose utility is the weight share of the correct class
among the $K$ nearest neighbors, the only exact algorithm enumerates the possible sets of $K$ nearest neighbors, in
$O(N^K)$ time \citep{jia2019knn}; \citet{wang2024efficient} call this impractical ``even for modest $K$, such as 5''.
For hard-label classifiers with weights on a grid, \citet{wang2024efficient} reduce the time to $O(N^2)$ by counting
the subsets that give each value of a point's marginal contribution. They set aside the soft-label classifier, because
with its normalization term ``there is little chance'' that two subsets give the same marginal contribution.
\citet{zhang2026banzhafknn} call weighted KNN regression ``a challenging open problem'' whose ``main obstacle is the
normalization term of weights'', and \citet{zhang2025shapley} avoid the term by valuing a variant of the game,
obtained by duplicating points; its values are not the Shapley values of the original game. In short, the exact
algorithms polynomial in $K$ avoid the normalization term.

\textbf{This paper.} We show that the normalization term is not an obstacle to exact values. The earlier algorithms
count the subsets with equal marginal contributions, but the Shapley value does not need these counts, only the
utilities of many subsets added up. Our key observation is that such a sum is easy within a group of subsets whose $K$
nearest neighbors have the same total weight, because all of them have the same normalization term. With the weights
on a grid of $W$ levels, as in \citet{wang2024efficient}, there are few such groups. For the squared loss, the
utilities of a group add up from three numbers: its size and the sums of the weighted target and of its square
(Section~\ref{sec:setup}). Hence the exact values follow from a table of these groups. A second step builds one table from all $N$
training points and takes each point out of it in turn, by subtracting the subsets that contain the point; this brings
the running time to near-linear in $N$.
Table~\ref{tab:landscape} places our results among the known ones.

\textbf{Exact values in near-linear time (Section~\ref{sec:exact}).} For weights on $W$ levels, real-valued targets
and the squared loss, we compute the exact Data Shapley values of all $N$ training points in $O(WK^2N)$ time once the
points are sorted by distance. To our knowledge, this is the first exact algorithm polynomial in $K$ for weighted KNN
regression with real-valued targets, and the first algorithm that computes the exact Shapley values of a weighted KNN
model in time linear in $N$ after sorting. The same construction gives exact values for the soft-label classifier, for any number of classes, and removes a factor of $N$ from the
exact algorithm of \citet{wang2024efficient} for hard labels.

\textbf{Continuous weights without a grid (Section~\ref{sec:approx}).} Kernel weights are real numbers, and earlier
work rounds them to a grid, which changes the values. For the squared loss the grid is not needed: we compute every
value within any tolerance $\varepsilon$, deterministically, in time linear in $N$ after sorting and polylogarithmic in
$1/\varepsilon$. The squared loss divides by
$A^2$, with $A$ the normalization term. The key is to write $1/A^2$ as an integral of $t\,e^{-tA}$ over $t$: for a
fixed $t$, $e^{-tA}$ is a product of one factor per neighbor, so the subsets need no grouping by total weight.

\textbf{Limits of exact computation (Section~\ref{sec:lower}).} Our exact algorithms are polynomial in the number of
weight levels, which is exponential in the number of bits of a weight. For $K=N$ this cannot be avoided: some exact
values have denominators of exponentially many bits. For the absolute loss, our exact algorithm also needs the targets
on a grid, and its time grows with the number of target levels. Here the problem itself is hard: with integer targets
written in binary, computing the value is \#P-hard, even without weights, when $K$ grows as a power of $N$.

\textbf{Experiments (Section~\ref{sec:exp}).} Our experiments show three things. (1) The measured running time of both
algorithms grows linearly with $N$: at $K=5$, $100{,}000$ synthetic training points take $8.5$ seconds per query. (2)
Given the exact algorithm's time, Monte Carlo estimation over random permutations falls further behind the exact
ranking as the training set grows: with $10$ validation points, Kendall's $\tau$ falls from $0.93$ at $150$ points to
$0.34$ at $100{,}000$, and with $100$ it falls as well (Figure~\ref{fig:results}(b)). (3) Rounding kernel weights to
$8$ levels changes the ranking, which the algorithm for continuous weights avoids.

\section{Preliminaries}
\label{sec:prelim}

\begin{figure*}[t]
\centering
\includegraphics[width=0.93\textwidth]{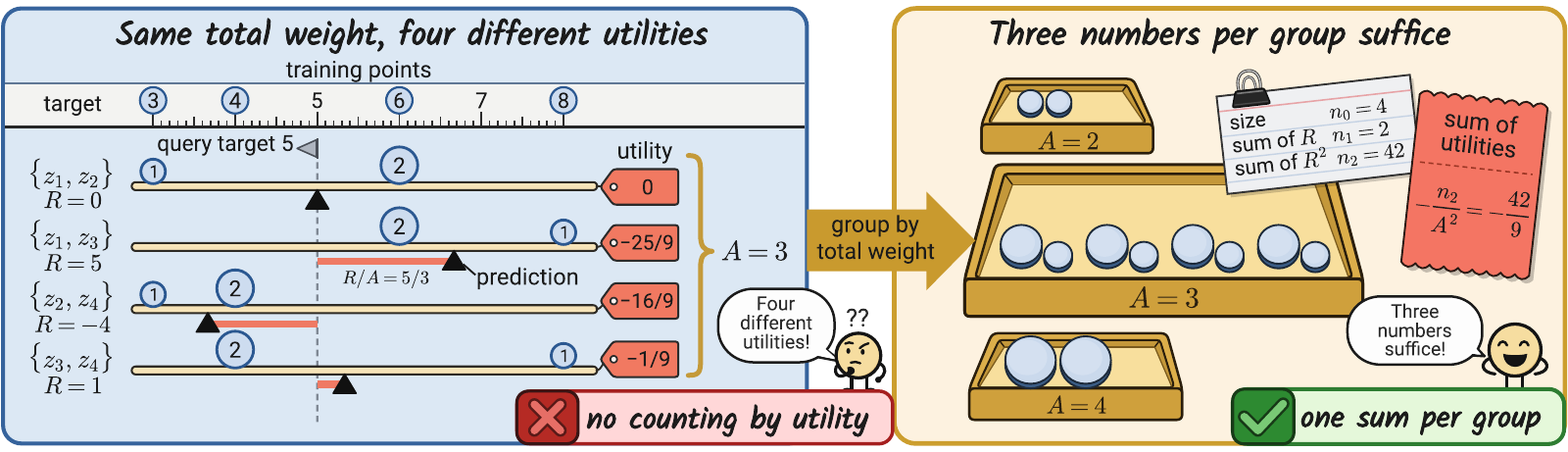}
\caption{Four training points with weights $(2,1,1,2)$ and targets $(6,3,8,4)$, a query with target $5$, and $K=2$.
The points are numbered by their distance to the query, which the drawing does not show. \textbf{Left:} each black triangle marks the prediction of a pair, the
balance point of its two disks. The four pairs have the same total weight $A=3$ and four different predictions, hence
four different utilities: subsets cannot be counted by their utility. $R$ is a pair's total weighted target, measured
from the query's target, and the prediction lies $R/A$ from that target. \textbf{Right:} grouped by total weight, the
utilities of a group sum to $-n_2/A^2$. A group keeps its size $n_0$, the sum $n_1$ of $R$ and the sum $n_2$ of $R^2$.}
\label{fig:idea}
\end{figure*}

\textbf{Data Shapley.} Let $D=\{z_i=(x_i,y_i)\}_{i=1}^{N}$ be the training set and let a utility function $v$ map
every subset $S\subseteq D$ to a score. The Shapley value of $z_i$ is its average marginal contribution,
\begin{equation}
\phi_i=\sum_{S\subseteq D\setminus\{z_i\}}c_N(|S|)\,\bigl[v(S\cup\{z_i\})-v(S)\bigr],
\label{eq:shapley}
\end{equation}
with the coefficients $c_N(s)=\frac1N\binom{N-1}{s}^{-1}$. In words, the change in utility when $z_i$ joins a subset of
the other points is averaged with a weight that depends only on the size of the subset.

\textbf{Weighted KNN regression.} We value the training points with respect to one query, the validation point
$z^{(\val)}=(x^{(\val)},y^{(\val)})$, and sort them by their distance to $x^{(\val)}$, so that $z_1$ is the nearest;
ties are broken by index. Each $z_i$ has a weight $w_i>0$ that depends only on its distance to $x^{(\val)}$, for
example the Gaussian kernel $w_i=\exp(-d(x_i,x^{(\val)})^2/h^2)$. For a subset $S$, let $\NK(S)$ be its $\min(K,|S|)$
nearest points. The prediction $\yhat(S)$ is the weighted average~\eqref{eq:pred}, with a fixed default
$\yhat(\emptyset)=y_{\emptyset}$, and the utility is the negative loss
\begin{equation}
v(S)=-\ell\bigl(\yhat(S),\,y^{(\val)}\bigr),
\label{eq:utility}
\end{equation}
where $\ell$ is the squared or the absolute loss. The normalization term is $\sum_{z_j\in\NK(S)}w_j$; in the unweighted
case it depends on $|S|$ only.

As in prior work \citep{jia2019knn,wang2024efficient}, a running time counts the arithmetic operations that
compute all $N$ values $\phi_1,\dots,\phi_N$ for one query. The value for a validation set is the average of
the values for its points, by linearity of the Shapley value.
Monte Carlo estimation samples random orderings of $D$ and averages the marginal contribution of
each point to the points before it \citep{castro2009polynomial,ghorbani2019data}; its error bounds hold only with high
probability.

\section{The Normalization Term Is Not an Obstacle}
\label{sec:setup}

Why does counting fail with the normalization term, and what can replace it? Figure~\ref{fig:idea} answers both.

\textbf{Discretized weights.} As \citet{wang2024efficient} do for classification, we take the weights on a grid:
$w_i=a_i\,\delta_w$ with an integer $a_i\in\{1,\dots,W\}$, where $\delta_w>0$ is the weight unit and $W$ the number of
weight levels ($W=2^b$ for $b$-bit weights). The unit cancels in the weighted average~\eqref{eq:pred}, so only the
integers $a_i$ matter. The targets $y_i$ stay arbitrary real numbers. Section~\ref{sec:approx} treats weights that are
not on a grid.

\textbf{Two numbers determine the prediction.} A subset enters the prediction only through two numbers of its $K$
nearest neighbors. For a set $C$ of training points write
\begin{equation}
A(C)=\sum_{z_j\in C}a_j,\qquad R(C)=\sum_{z_j\in C}r_j,
\label{eq:AR}
\end{equation}
where $r_j=a_j\,(y_j-y^{(\val)})$. $A(C)$ is the total weight of $C$, and $R(C)$ is its total weighted target, measured
from the target of the query. For a nonempty subset $S$ the prediction error is
$\yhat(S)-y^{(\val)}=R(\NK(S))/A(\NK(S))$, so under the squared loss
\begin{equation}
v(S)=-\frac{R(\NK(S))^2}{A(\NK(S))^2}.
\label{eq:vAR}
\end{equation}
The denominator is the squared normalization term, in units of $\delta_w$.

\textbf{Counting equal marginal contributions fails.} The exact algorithms for unweighted and for hard-label KNN count
the subsets that give each value of a point's marginal contribution. This works because there the marginal
contribution takes only a few values. With the normalization term it fails. In Figure~\ref{fig:idea}, the four pairs of
points with total weight $3$ have four different predictions and hence four different utilities, and with real-valued
targets almost no two sets of neighbors share a utility. This is the obstacle that \citet{wang2024efficient} describe.

\textbf{Summing over groups of equal total weight.} The Shapley value, however, does not need to know which subsets
share a marginal contribution. It is a weighted sum of utilities, so it only needs the utilities of many subsets added
up. Our key observation is that this sum is easy within a \emph{group} of sets that have the same total weight, because
their normalization term is the same number. For example, the four pairs of Figure~\ref{fig:idea} all have $A=3$, and
by~\eqref{eq:vAR} their utilities sum to $-(0+25+16+1)/9$: the squares of $R$ are added up and divided once by the
common $A^2$. The sum needs one number from the group, and not its members.

\textbf{Three numbers per group.} The algorithm of Section~\ref{sec:exact} adds one further point to every set of a
group, and for this a group keeps two more numbers. For a group $\mathcal{G}$ of sets with the same total weight
$\omega$, let
\begin{equation}
n_0=|\mathcal{G}|,\quad n_1=\sum_{C\in\mathcal{G}}R(C),\quad n_2=\sum_{C\in\mathcal{G}}R(C)^2
\label{eq:three}
\end{equation}
be its size, its sum of $R$ and its sum of $R^2$. When a point $z_j$ is added to every set of the group, each $R(C)$
grows by $r_j$, and expanding $(R(C)+r_j)^2$ shows that the three numbers $n=(n_0,n_1,n_2)$ become
\begin{equation}
n^{+j}=\bigl(n_0,\;n_1+r_jn_0,\;n_2+2r_jn_1+r_j^2n_0\bigr),
\label{eq:shift}
\end{equation}
while every set now has the total weight $\omega+a_j$. The utilities of the enlarged sets therefore sum to
\begin{equation}
-\frac{n_2+2r_jn_1+r_j^2n_0}{(\omega+a_j)^2}.
\label{eq:group}
\end{equation}
In words, within a group the normalization term is a constant, and three numbers tell us all that the squared loss
asks of the group, also after a point has been added to its sets.

With discretized weights there are few groups: sets of at most $K$ points take at most $KW+1$ total weights.

\section{Exact Values in Near-Linear Time}
\label{sec:exact}

This section turns the observation of Section~\ref{sec:setup} into an algorithm in three steps: the Shapley value as
two sums over subsets of one size (Section~\ref{sec:twosums}); a table of groups that gives each sum, and with one table
per point an algorithm quadratic in $N$ (Section~\ref{sec:sq}); and one table shared by all points, which removes the
factor $N$ (Section~\ref{sec:shared}). Section~\ref{sec:other} treats classification.

\subsection{The Value as Two Sums}
\label{sec:twosums}

We first ask when adding a training point $z_i$ to a subset $S\subseteq D\setminus\{z_i\}$ changes the utility at
all.

\begin{lemma}
\label{lem:cases}
Let $m$ be the number of points of $S$ that are nearer to the query than $z_i$.
\begin{enumerate}\setlength{\itemsep}{0pt}\setlength{\parskip}{0pt}\setlength{\topsep}{2pt}
\item If $m\ge K$, then $v(S\cup\{z_i\})=v(S)$.
\item If $|S|\le K-1$, then $\NK(S)=S$ and $\NK(S\cup\{z_i\})=S\cup\{z_i\}$.
\item If $|S|\ge K$ and $m\le K-1$, let $z_e$ be the $K$-th nearest point of $S$ and $C$ the $K-1$ points of $S$
nearer than $z_e$. Then $e>i$, $\NK(S)=C\cup\{z_e\}$ and $\NK(S\cup\{z_i\})=C\cup\{z_i\}$.
\end{enumerate}
\end{lemma}

In words, $z_i$ changes the utility only if it becomes one of the $K$ nearest neighbors. Either there is still room and
$z_i$ \emph{joins} the neighbors (case~2), or the neighborhood is full and $z_i$ \emph{replaces} its farthest member
$z_e$ (case~3). For example, take $z_2$ in Figure~\ref{fig:idea}, where $K=2$. It joins the empty set and each single
point. In every larger subset it replaces the second-nearest member, which is $z_3$ if the subset contains $z_1$ and
$z_3$, and $z_4$ otherwise.

In the replace case the points of $S$ beyond $z_e$ play no role: the marginal contribution of $z_i$ is
$v(C\cup\{z_i\})-v(C\cup\{z_e\})$, whatever lies beyond $z_e$. The subsets that share $C$ and $z_e$ can therefore be
taken together, and their Shapley coefficients add up to a closed form, $c_e(K)=\frac1e\binom{e-1}{K}^{-1}$: the
coefficient of a subset of size $K$ when only the $e$ points $z_1,\dots,z_e$ are present. Let
$D_{<e}=\{z_1,\dots,z_{e-1}\}$ be the points nearer than $z_e$.

\begin{proposition}
\label{prop:phi}
For every utility that depends on $S$ only through $\NK(S)$, the Shapley value of $z_i$ is
\begin{align}
&\phi_i=\sum_{s=0}^{K-1}c_N(s)\sum_{S}\bigl[v(S\cup\{z_i\})-v(S)\bigr]\notag\\
&\;+\sum_{e>\max(i,K)}c_e(K)\sum_{C}\bigl[v(C\cup\{z_i\})-v(C\cup\{z_e\})\bigr],
\label{eq:phi}
\end{align}
where $S$ ranges over the subsets of $D\setminus\{z_i\}$ with $s$ points and $C$ over those of
$D_{<e}\setminus\{z_i\}$ with $K-1$ points.
\end{proposition}

The first sum is the join case and the second is the replace case. Both have the same shape: each inner sum adds up
the utilities of all subsets of one size of a set of points, after one further point ($z_i$ or $z_e$) has been added to
each, or with no point added. Section~\ref{sec:setup} showed how to add up such utilities group by group.

\subsection{A Table of Groups}
\label{sec:sq}

To add up these utilities without listing the subsets, we keep a table of the three numbers of every group. For a set
$G$ of training points, a size $s$ and
a total weight $\omega$, let $Q_G[s,\omega]$ be the three numbers~\eqref{eq:three} of the group of all $s$-subsets of
$G$ with total weight $\omega$. By~\eqref{eq:group}, for a size $s\le K-1$ and a point $z_j$ outside $G$,
\begin{equation}
\sum_{\substack{C\subseteq G\\|C|=s}}v(C\cup\{z_j\})
=-\sum_{\omega}\frac{n_2+2r_jn_1+r_j^2n_0}{(\omega+a_j)^2},
\label{eq:read}
\end{equation}
where $(n_0,n_1,n_2)=Q_G[s,\omega]$ and $\omega$ runs over the total weights of the $s$-subsets of $G$. With no point
added, the sum is $-\sum_\omega n_2/\omega^2$ for $s\ge1$. So each inner sum of Proposition~\ref{prop:phi} is read off
one row of a table.

The table is built one point at a time. The table of the empty set has $Q_\emptyset[0,0]=(1,0,0)$ and zeros
elsewhere. An $s$-subset of $G\cup\{z_j\}$ with total weight $\omega$ either avoids $z_j$, or it is an $(s-1)$-subset of
$G$ with total weight $\omega-a_j$ to which $z_j$ has been added. Hence
\begin{equation}
Q_{G\cup\{z_j\}}[s,\omega]=Q_G[s,\omega]+Q_G[s-1,\omega-a_j]^{+j},
\label{eq:update}
\end{equation}
where $n^{+j}$ is the shift~\eqref{eq:shift}. Proposition~\ref{prop:phi} uses the sizes $s\le K-1$ only, so a table
has $K$ rows and at most $(K-1)W+1$ columns, and adding a point costs $O(K\cdot KW)$ operations.

\textbf{A direct algorithm.} The two steps so far already give an exact algorithm, but it builds one table for every
training point. For $z_i$, Proposition~\ref{prop:phi} needs the table of $D\setminus\{z_i\}$ for its first sum and,
for every $e$, row $K-1$ of the table of $D_{<e}\setminus\{z_i\}$ for its second. These sets grow by one point as $e$
increases, so we keep one table for $z_i$ and add the other points to it from the nearest to the farthest. Just before
$z_e$ is added, the table is that of $D_{<e}\setminus\{z_i\}$ and gives the term of $z_e$ in the second sum; after the
last point it is the table of $D\setminus\{z_i\}$ and gives the first sum. This takes $N$ updates per point,
$O(WK^2N^2)$ operations in all, the running time that \citet{wang2024efficient} obtain for hard labels.

\subsection{One Table for All Points}
\label{sec:shared}

The $N$ tables of the direct algorithm are nearly the same, because each is built from all points but one. The second
idea is to build a table once, with every point in it, and to take a point out when it is needed. Taking a point out
is the update~\eqref{eq:update} read backwards.

\begin{lemma}
\label{lem:remove}
Let $G$ be a set of training points, $z_i\in G$ and $G'=G\setminus\{z_i\}$. Then $Q_{G'}[0,\omega]=Q_G[0,\omega]$,
and for $s\ge1$
\begin{equation}
Q_{G'}[s,\omega]=Q_G[s,\omega]-Q_{G'}[s-1,\omega-a_i]^{+i}.
\label{eq:remove}
\end{equation}
\end{lemma}

In words, the subsets that avoid $z_i$ are all subsets minus those that contain $z_i$, and those are the smaller
subsets that avoid $z_i$, each with $z_i$ added. For example, the pairs of the four points of Figure~\ref{fig:idea}
that avoid $z_2$ are the six pairs minus the three pairs that contain $z_2$, and these three are the single points
$z_1$, $z_3$ and $z_4$, each with $z_2$ added. Row $s$ of the smaller table follows from its row $s-1$, so taking a
point out costs as much as adding one. For the first sum of Proposition~\ref{prop:phi}, $z_i$ needs the table of
$D\setminus\{z_i\}$; it follows from the table of all $N$ points by Lemma~\ref{lem:remove}, so one table, built once,
serves every point.

\textbf{The second sum.} The second sum is harder, because each of its terms needs a table of its own. A term belongs
to a point $z_e$ that $z_i$ can replace, that is, $e>\max(i,K)$. It takes the sets of $K-1$ points nearer than $z_e$,
other than $z_i$, and compares their utility with $z_i$ added and with $z_e$ added. Both are read from row $K-1$ of the
table of these points, $D_{<e}\setminus\{z_i\}$, so the direct algorithm needs one table per pair of points.

\textbf{Shared tables.} What saves the factor $N$ is that these tables are shared. Since $z_i$ is nearer than $z_e$,
the table of $D_{<e}\setminus\{z_i\}$ is the table of $D_{<e}$ with $z_i$ taken out, and the table of $D_{<e}$ is the
same for all points. In Figure~\ref{fig:shared}, where $N=6$ and $K=2$, point $z_3$ needs the tables of $D_{<4}$,
$D_{<5}$ and $D_{<6}$, and point $z_4$ the last two. Taking a point out of a table, by~\eqref{eq:remove}, and reading a
sum from it, by~\eqref{eq:read}, are both linear in the entries of the table. Applying them to each table and adding
up the results therefore gives the same as applying them once to the sum of the tables. So we add up the tables first,
each multiplied by its coefficient $c_e(K)$,
\begin{equation}
\bar Q_i=\sum_{e>\max(i,K)}c_e(K)\,Q_{D_{<e}},
\label{eq:Qbar}
\end{equation}
and take $z_i$ out of this sum once; the terms $v(C\cup\{z_i\})$ of the second sum are read from the result. Only the
range of $e$ depends on $i$, so all the sums $\bar Q_i$ come from two passes over the points in order of distance: the
first adds up all the tables, and the second subtracts the tables with $e\le i$ as it reaches $z_i$.

\textbf{The replaced point.} The terms $v(C\cup\{z_e\})$ need one more step, because by~\eqref{eq:read} they are read
with the weight and target of the added point $z_e$, which changes from table to table. For $z_3$ in
Figure~\ref{fig:shared}, they put $z_4$ into the sets of the table of $D_{<4}$, $z_5$ into those of $D_{<5}$ and $z_6$
into those of $D_{<6}$. We therefore add $z_e$ to every set of the table of $D_{<e}$ first, by the
shift~\eqref{eq:shift}. This gives a second family of tables, one for each $e$. Like the first, it does not depend on
$i$, so it is summed and passed over in the same way (Appendix~\ref{app:exact}). Each point then needs a fixed number of
table operations instead of $N$, and the factor $N$ of the direct algorithm is gone.

\begin{figure}[t]
\centering
\includegraphics[width=0.55\textwidth]{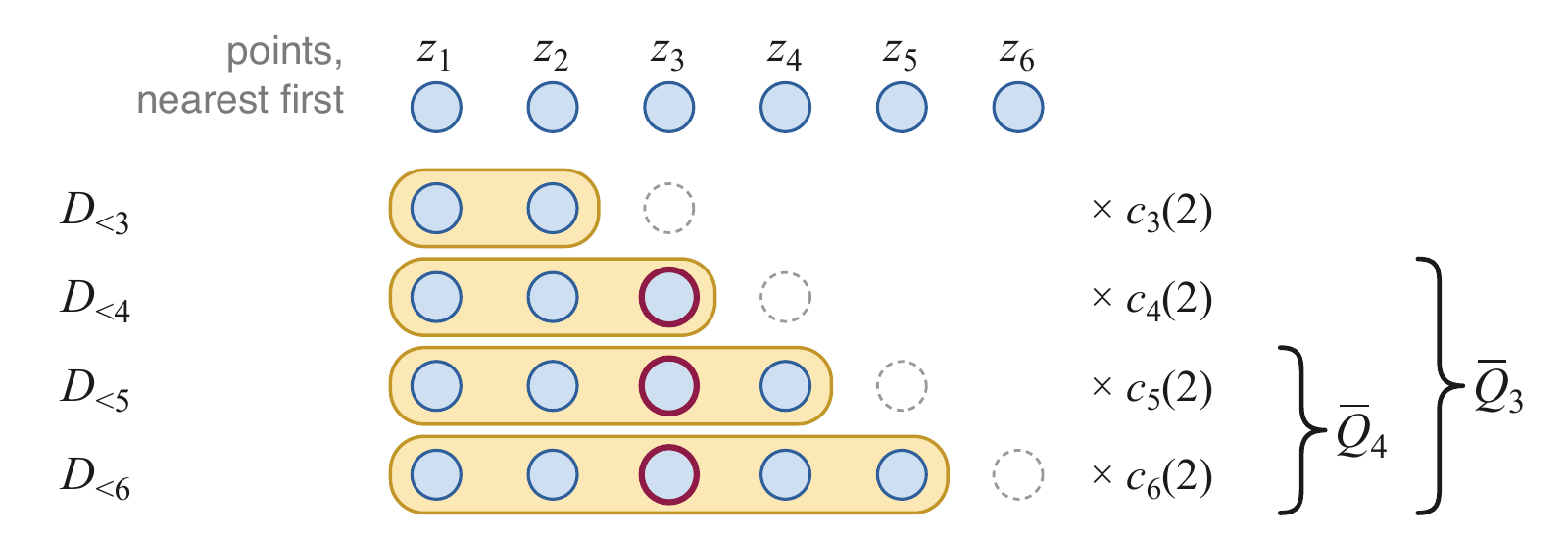}
\caption{One running sum of tables serves every point ($N=6$, $K=2$). The table of $D_{<e}$ holds the points nearer
than $z_e$ (dashed) and is multiplied by $c_e(2)$. Point $z_3$ (ringed) needs the tables that contain it, and the
brace marks their weighted sum, $\bar Q_3$; point $z_4$ needs the same sum without its first table.}
\label{fig:shared}
\end{figure}

\begin{theorem}
\label{thm:sq}
For weights on $W$ levels, real-valued targets and the squared loss, the exact Shapley values of all $N$ training
points are computed in $O(WK^2N)$ time and $O(WK^2+N)$ space after sorting by distance.
\end{theorem}

Each pass takes $O(WK^2)$ operations per point. With $K$ and $W$ small, as in practice, the
algorithm is linear in $N$ after sorting, like the algorithm for unweighted KNN. In floating point, the subtraction of
Lemma~\ref{lem:remove} loses accuracy as $K$ grows. Our implementation therefore treats the $2K$ nearest points by the
direct algorithm, at $O(WK^4)$ operations that do not grow with $N$, and agrees with the direct algorithm to rounding
error for $K$ up to $50$ (Appendix~\ref{app:exact}).

\textbf{Relation to earlier derivations.} Like that of \citet{wang2024efficient}, our derivation starts at the case
distinction of Lemma~\ref{lem:cases}, and then sums over groups instead of counting. In databases, the Shapley value
of an average, whose denominator also changes with the subset, is computed by a dynamic program over subset sizes
\citep{standke2025tractability}, and closed forms for the average and the variance use the count, the sum and the
sum of squares \citep{khani2026closed}; there are no weights or nearest neighbors there. Lemma~\ref{lem:remove} is an
old device \citep{mann1962values}, and for unweighted KNN, tables shared by all points give the Banzhaf value in
linear time \citep[see also \citealp{lyu2023improving}]{zhang2026banzhafknn}; new here is the use of both with
weights, where the weight of the replaced neighbor enters the normalization term (Appendix~\ref{app:related}).

\subsection{Classification}
\label{sec:other}

The squared loss entered the derivation only through what a group keeps; Proposition~\ref{prop:phi} and
Lemma~\ref{lem:remove} hold for every utility of the $K$ nearest neighbors, and Appendix~\ref{app:exact} treats the
utilities that are polynomial in $R$.

\textbf{Soft labels.}\label{sec:soft} The soft-label utility that \citet{jia2019knn} define and
\citet{wang2024efficient} set aside is the weight share of the correct class among the $K$ nearest neighbors,
\begin{equation}
v(S)=\frac{1}{A(\NK(S))}\sum_{z_j\in\NK(S)}a_j\,\mathbf{1}[y_j=y^{(\val)}],
\label{eq:soft}
\end{equation}
with $v(\emptyset)=0$. It has the same normalization term, and it is linear in its numerator: with
$r_j=a_j\mathbf{1}[y_j=y^{(\val)}]$ the numerator is $R$, so a group needs only its size $n_0$ and its sum $n_1$.

\begin{corollary}
\label{cor:soft}
For weights on $W$ levels and any number of classes, the exact Shapley values of all $N$ training points under
the soft-label utility~\eqref{eq:soft} are computed in $O(WK^2N)$ time after sorting.
\end{corollary}

\textbf{Hard labels.}\label{sec:hard} \citet{wang2024efficient} value the hard-label classifier with two classes,
whose utility is $1$ if the query's class has at least the total weight of the other class among the $K$ nearest
neighbors, and $0$ otherwise. With the weights of the other class counted negatively, this utility depends on the
signed total weight alone, and a group needs only its size. Their algorithm builds such counts for every training
point separately; one shared table removes this factor of $N$.

\begin{corollary}
\label{cor:hard}
For the hard-label classifier with two classes and weights on $W$ levels, the exact Shapley values of all $N$
training points are computed in $O(WK^2N)$ time after sorting, instead of $O(WK^2N^2)$.
\end{corollary}

Appendix~\ref{app:exact} also covers the Brier score and, by changing only $c_N(s)$ and $c_e(K)$, other semivalues.

\section{Continuous Weights}
\label{sec:approx}

Kernel weights are real numbers, and rounding them to a grid changes the values (Section~\ref{sec:exp}). For their
hard-label utility, \citet{wang2024efficient} round the weights, prove that rounding never reverses the order of two
values, and study the size of the change only empirically. The table of Section~\ref{sec:exact} cannot hold real weights,
because it has one column per total weight. This section gives an algorithm that needs no grid; here
$A(C)=\sum_{z_j\in C}w_j$ and $r_j=w_j\,(y_j-y^{(\val)})$ are formed with the weights themselves.

\textbf{The idea.} Our idea is to take the total weight out of the table. The identity
\begin{equation}
\frac{1}{A^2}=\int_0^\infty t\,e^{-tA}\,dt
\label{eq:laplace}
\end{equation}
writes the factor $1/A^2$ of the squared loss as an integral over an auxiliary variable $t$. For every fixed $t$, the
factor $e^{-tA(C)}=\prod_{z_j\in C}e^{-tw_j}$ is a product over the members of $C$: adding a point $z_j$ to a set
multiplies it by $e^{-tw_j}$, whatever its total weight. So the subsets need not be grouped by their total
weight: for each size, the table keeps the three numbers of Section~\ref{sec:setup} summed over all subsets of that
size, each weighted by $e^{-tA(C)}$. Proposition~\ref{prop:phi}, Lemma~\ref{lem:remove} and the two passes of
Section~\ref{sec:shared} carry over, at $O(K)$ operations per point for each $t$. The identity is classical
\citep{cressie1981moment}; we know of no earlier work that computes a Shapley value through it. It remains to integrate over $t$, which a
Gauss--Legendre quadrature does with $L$ nodes fixed before the run, within $\varepsilon$ (Appendix~\ref{app:approx}).

\begin{theorem}
\label{thm:cont}
Consider the squared loss, real weights in $[w_{\min},w_{\max}]$ with $w_{\min}>0$, and targets with
$|y_j-y^{(\val)}|\le B$. For every $0<\varepsilon\le B^2$, all $N$ values are computed within additive error
$\varepsilon$ using $O(NKL)$ arithmetic operations and evaluations of the exponential function, in exact real
arithmetic, after sorting by distance. Here
$L=O\bigl(\lambda\,(\lambda+\log\tfrac{Kw_{\max}}{w_{\min}})\bigr)$ with $\lambda=1+\log\tfrac{B^2}{\varepsilon}$.
\end{theorem}

In words, the algorithm is linear in $N$ after sorting, the tolerance and the range of the weights enter only through
their logarithms, and the error bound holds on every instance, not only with high probability.

\textbf{Exact and floating-point arithmetic.} Theorem~\ref{thm:cont} is proved for exact real arithmetic, where the
quadrature is the only source of error. A bound that also covers rounding errors is proved for a variant that is
quadratic in $N$ (Appendix~\ref{app:approx}, which also covers the soft-label utility and bounds the error of rounding
the weights). For the linear algorithm in the double precision that we run, the error is measured, not
bounded. It stayed within the tolerance in all $72{,}000$ runs of our check against exhaustive enumeration (six on each
of $12{,}000$ instances), and up to $100{,}000$ points its rounding stays far below the tolerance
(Appendix~\ref{app:more:enum}).

\section{Limits of Exact Computation}
\label{sec:lower}

Theorem~\ref{thm:sq} is polynomial in the number of weight levels $W$, which is exponential in the number of bits of a
weight. Could an exact algorithm be polynomial in the number of bits instead? This section shows in which sense it
cannot, for the weights and, under the absolute loss, for the targets.

\begin{theorem}[Exact values can be exponentially long]
\label{thm:bits}
For every $n\ge18$ there is an instance with $N=n+1$ training points, integer weights of at most $n$ bits, targets
in $\{0,1\}$ and $K=N$, for which the Shapley value of one point under the squared loss, in lowest terms, has a
denominator of at least $2^{n-1}$ bits.
\end{theorem}

In words, writing the exact value down already takes time proportional to $W=2^{n-1}$ on these instances, so no
algorithm that writes exact values out in binary runs in time polynomial in $\log W$ (Appendix~\ref{app:lower} builds the instances). The theorem limits the exact value, not its approximation: on a Turing
machine the value within $2^{-n}$ takes time polynomial in $n$ (Corollary~\ref{app:approx:cor:bits}).

\textbf{The absolute loss needs the targets on a grid.}\label{sec:pair} For the absolute loss no fixed number of sums
per group suffices (Appendix~\ref{app:exact}). With the targets on a grid too, we group subsets by the pair of their
total weight and total weighted target, which fixes the prediction, and the two
passes give all values in $O(NKP)$ time after sorting, with $P$ the number of pairs (Theorem~\ref{thm:pair}). This is
polynomial in the number of target levels; the next theorem shows it cannot be polynomial in their number of bits.

\begin{theorem}
\label{thm:abs}
For unweighted KNN regression ($w_j=1$ for all $j$) with the absolute loss, integer targets written in binary and
$K=N$, computing the Shapley value of a training point is \#P-hard. Consequently, unless
$\mathrm{FP}=\#\mathrm{P}$, no deterministic algorithm computes the value in time polynomial in $N$ and the number
of bits of the targets, or approximates it within $2^{-(N+2\log_2N+1)}$ in such time.
\end{theorem}

Here $\mathrm{FP}=\#\mathrm{P}$ would mean that every counting problem in \#P is solvable in polynomial time
(Appendix~\ref{app:lower} gives the reduction and extends it to $K$ that grows as a power of $N$). Unless that holds, the target grid is needed, and no deterministic algorithm polynomial
in $\log(1/\varepsilon)$ exists for the absolute loss.

\begin{figure*}[t]
\centering
\includegraphics[width=0.98\textwidth]{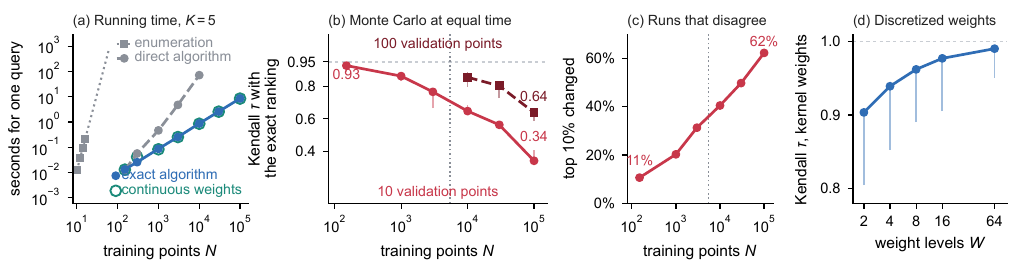}
\caption{\textbf{(a)} Seconds per query on synthetic instances ($W=8$; continuous weights within $10^{-6}B^2$; enumeration extrapolated beyond $N=16$, dotted). \textbf{(b)} Kendall's $\tau$ of Monte Carlo with the
exact values at the exact algorithm's time, with $10$ and with $100$ validation points (for $100$, one
split per dataset: $9$, $3$ and $3$ cases at the three sizes); bars: the middle half of the cases. \textbf{(c)} Share of one run's top
10\% points, at $3{,}000$ permutations, that are not in another run's. In (b) and (c): medians over cases; left of the
dotted line, the eight smaller datasets on a laptop; right of it, the larger datasets on cluster nodes. \textbf{(d)} $\tau$ of the values on $W$ levels with the values for the kernel weights ($30$ validation points; mean
of $40$ cases; bars reach the lowest case).}
\label{fig:results}
\end{figure*}

\section{Numerical Experiments}
\label{sec:exp}

We ask three questions of our algorithms. Does their running time grow linearly with $N$? Given the exact
algorithm's time, does Monte Carlo estimation reach the exact ranking, and how does that change as $N$ grows? Does
rounding the weights matter? We answer the first on synthetic instances and the other two on seventeen regression
datasets, all but one from OpenML \citep{vanschoren2014openml}: eight smaller and nine larger
ones. A \emph{case} is one dataset with one random split into training and validation points; we report medians or
means over cases. In each case we shift the targets of 10\% of the training points by three standard deviations,
for a detection task. The weights are a Gaussian kernel of the distance to the query. Unless stated otherwise we use $K=5$ and $W=8$ weight levels, as \citet{wang2024efficient} do, and $10$
validation points. Monte Carlo is the permutation sampling of Section~\ref{sec:prelim}, run on the same utility and skipping every point that cannot enter the $K$ nearest
neighbors. Appendix~\ref{app:exp} gives the details. Every algorithm agrees with exhaustive enumeration on
$12{,}000$ small random instances (Appendix~\ref{app:more:enum}).

\textbf{Running time.} On synthetic instances, the running time of both of our algorithms grows linearly with $N$
(Figure~\ref{fig:results}(a)), and it still does up to $1{,}000{,}000$ points (fitted exponents in
Appendix~\ref{app:more:timing}). At $N=100{,}000$ and $K=5$, each takes $8.5$~s per query. At $N=10{,}000$ the direct
algorithm of Section~\ref{sec:sq} takes $71.5$~s, and the exact algorithm $0.87$~s. The enumeration of neighbor sets
\citep{jia2019knn} grows as $N^{6.1}$ at $K=5$.

\textbf{Monte Carlo against exact values.} Given the exact algorithm's time, Monte Carlo does not reach the exact
ranking, and it falls further behind as the training set grows (Figure~\ref{fig:results}(b); Table~\ref{tab:mc} in
Appendix~\ref{app:more:mc}). With $10$ validation points, Kendall's $\tau$ with the exact values falls from $0.93$ at
$150$ training points to $0.34$ at $100{,}000$. Its runs also disagree with each other: two runs of $3{,}000$ permutations
differ on $11\%$ to $62\%$ of the top 10\% points as $N$ grows (Figure~\ref{fig:results}(c)). More validation points
narrow the gap but do not stop it from widening: with $100$ instead of $10$, $\tau$ is higher in every case, and it
still falls, to $0.64$ at $100{,}000$ points.

\textbf{Scope of the comparison.} At a fixed $3{,}000$ permutations, where timing plays no role, $\tau$ falls too, to
$0.45$ at $100{,}000$ points. Up to $3{,}000$ points, the exact algorithm's lead over
Monte Carlo comes from its near-linear running time: given the time of the quadratic direct algorithm instead, Monte
Carlo comes close ($\tau$ of $0.94$ and $0.95$ at $1{,}000$ and $3{,}000$ points). Of three other samplers that we
compared up to $8{,}000$ points, Owen sampling \citep{okhrati2021multilinear} did best, and we ran it at every size up
to $100{,}000$ points. Within the exact algorithm's time, its median $\tau$ is $0.973$ at $150$ points but stays below
$0.95$ from $1{,}000$ points on. Downstream, in finding the shifted targets Monte Carlo does as well, and when the
most or the least valuable points are removed before refitting, the exact values do not separate from it
(Appendix~\ref{app:more:detection}). What the exact values add is
the exact ranking and the set of the most valuable points.

\textbf{Discretized weights.} Rounding the kernel weights to $8$ levels changes the ranking that the values for the
kernel weights give: $\tau$ between the two is $0.962$ on average, $0.890$ in the lowest case
(Figure~\ref{fig:results}(d); Appendix~\ref{app:more:levels}). Rounding is an error of its own, which the
algorithm for continuous weights avoids, at a cost growing with the range of the weights.

\section{Conclusion}
\label{sec:conclusion}

The normalization term does not prevent efficient exact Data Shapley values for weighted KNN: grouping subsets by the
total weight of their $K$ nearest neighbors gives them in time linear in $N$ after sorting, and for kernel weights an
integral for $1/A^2$ gives them to any tolerance (proved for exact arithmetic). Given the same time, Monte Carlo falls further behind them as
the training set grows. Exact computation is limited by the precision of the weights and, for the
absolute loss, by the problem itself. Our algorithms need one run per query, and the exact ones grow with $K^2$.

\section*{Use of Generative AI}
In this work, we used generative AI tools to help implement the algorithms and the code that cleans and reformats
the data and runs the experiments, to help propose and refine hypotheses, to help develop the theory, formulate the
theorems, provide the ingredients of their proofs and write the proofs, and to help interpret results. We have not
used generative AI tools to generate synthetic data sets or for translation, and qualitative or thematic data
analysis is not applicable to this work. We also used generative AI tools to edit prose, create figures, search the
literature and run the experiment pipeline. We have reviewed all AI-assisted work: the study's design and structure,
the review and revision of all content and the verification of every result remain with the authors, and every
reported number was re-derived from the raw result files. We take responsibility for the final content of this work,
including text, claims or artifacts produced with the aid of generative AI.

\bibliography{references}

\clearpage
\appendix

\renewcommand{\theequation}{\thesection.\arabic{equation}}
\renewcommand{\thetheorem}{\thesection.\arabic{theorem}}
\renewcommand{\thelemma}{\thesection.\arabic{lemma}}
\renewcommand{\thecorollary}{\thesection.\arabic{corollary}}
\renewcommand{\theproposition}{\thesection.\arabic{proposition}}
\renewcommand{\theremark}{\thesection.\arabic{remark}}
\renewcommand{\thedefinition}{\thesection.\arabic{definition}}
\makeatletter
\@ifundefined{theHequation}{}{%
  \renewcommand{\theHequation}{\thesection.\arabic{equation}}%
  \renewcommand{\theHtheorem}{\thesection.\arabic{theorem}}%
  \renewcommand{\theHlemma}{\thesection.\arabic{lemma}}%
  \renewcommand{\theHcorollary}{\thesection.\arabic{corollary}}%
  \renewcommand{\theHproposition}{\thesection.\arabic{proposition}}%
  \renewcommand{\theHremark}{\thesection.\arabic{remark}}%
  \renewcommand{\theHdefinition}{\thesection.\arabic{definition}}}
\makeatother
\setlength{\belowcaptionskip}{4pt}

\appsection{Related Work}\label{app:related}
\textbf{Data valuation and the Shapley value.} \citet{ghorbani2019data} and \citet{jia2019towards} proposed the
Shapley value \citep{shapley1953value} as a measure of the contribution of a training point. For a general model it
is estimated by Monte Carlo sampling of permutations \citep{castro2009polynomial,maleki2013bounding}, which gives
error bounds that hold with high probability. Other notions of value relax the axioms or change the weighting of
subset sizes, such as Beta Shapley \citep{kwon2022beta} and Data Banzhaf \citep{wang2023banzhaf}, or leave the
game-theoretic frame, such as Data-OOB \citep{kwon2023dataoob}, which Appendix~\ref{app:more} includes as a
baseline for detecting noisy targets.

\textbf{KNN-Shapley.} \citet{jia2019knn} showed that for unweighted KNN the Shapley value is computed exactly in
$O(N\log N)$ time, for classification and for a regression utility that divides by the constant $K$; for weighted
KNN they gave an $O(N^K)$ algorithm. \citet{wang2023noteknn} gave closed forms for the unweighted soft-label
utility, \citet{wang2023threshold} a private variant based on a distance threshold. \citet{wang2024efficient}
computed the exact value of weighted hard-label classifiers with discretized weights in $O(WK^2N^2)$ time by
counting subsets according to one sum of signed weights, and gave a deterministic approximation in $O(WK^2N^{1.5})$
time. They set aside the soft-label utility because of its normalization term. \citet{zhang2025shapley} avoid the
term by valuing a variant of the game obtained by duplicating points, in near-linear time; their values are not the
Shapley values of the original game, and they name regression as future work. \citet{zhang2026banzhafknn} study
the Banzhaf value of KNN classifiers. For weighted KNN they prove it \#P-hard and give an exact algorithm that is
quadratic in $N$, which the shared table of Corollary~\ref{cor:hard} with the Banzhaf coefficients makes linear in
$N$ after sorting (checked on $8{,}000$ instances); for unweighted KNN they compute the values of all points in $O(NK^2)$ time from tables shared by
all points, because without weights one need not know which neighbor a point pushes out. \citet{lyu2023improving}
share tables among all points in the same way for an additive retrieval utility. For the regression utility,
\citet{zhang2026banzhafknn} state that ``it remains a challenging open problem to analytically compute the Banzhaf
or Shapley values''. The present paper treats that utility and the soft-label one, and it computes the weighted
values, hard-label ones included, in time linear in $N$ after sorting.

\textbf{Counting algorithms for the Shapley value.} Computing Shapley-type indices by counting subsets according to
an integer sum goes back to the power indices of weighted voting games
\citep{mann1962values,matsui2000survey}. \citet{bhagat2014budgeted} formulate the principle in general: if the
value of a game is computed by one pass over the players with a state of few values, the Shapley value is computed
in time polynomial in the number of players and of states; they also note that an additive bound on the value
function passes to the Shapley value, a step that Appendix~\ref{app:approx} uses. For KNN, fixing the point that is
the $K$-th nearest neighbor of a subset is used by \citet{karlas2020nearest,karlas2024data} and by
\citet{wang2024efficient}; Lemma~\ref{lem:cases} is this case distinction. In databases,
\citet{standke2025tractability} compute the Shapley value of a fact for the average aggregate, whose denominator
also changes with the subset, by a dynamic program that accounts for the size of the subset, and
\citet{khani2026closed} give closed forms for the average and the variance from the count, the sum and the sum of
squares of the values. There the aggregate is an unweighted average over all answers, with no selection of nearest
neighbors and no weights. What Section~\ref{sec:exact} adds to this line is that, for a utility that is a
polynomial in the numerator of a weighted ratio, the numerator need not be part of the state: its sums over each
group of equal total weight suffice, so the targets stay real and the running time does not depend on them.

\textbf{One table for all players.} Removing a player from a table of counts, as Lemma~\ref{lem:remove} does, is
as old as the exact computation of power indices: \citet{mann1962values} divide out the factor of a player to pass
from one player's index to the next. \citet{kurz2016computing} states the recursion and the resulting cost for all
players, \citet{kakimura2026computing} give it as an alternating sum, and \citet{uno2012efficient} reaches the same
cost with forward and backward tables. In feature attribution, the UNWIND step of TreeSHAP
\citep{lundberg2018consistent} and the polynomial division of \citet{yu2022linear} are the same device. A weighted
voting game depends on a coalition only through its total weight and size, so a single table serves all players.
A KNN utility also depends on the order of the points. Our algorithm applies the removal to the tables of the
nearer points and takes the sum over the replaced neighbor on these tables, which is what makes one pass serve all
training points when the weights differ.

\textbf{Integral representations.} The representation of the Shapley value as an integral of the multilinear
extension is due to \citet{owen1972multilinear}; it is commonly used to build sampling estimators, and recent work
integrates it exactly by Gauss--Legendre rules for product games and tree ensembles
\citep{mohammadi2026quadrashap,wettenstein2026quadrature}. The second algorithm of Appendix~\ref{app:approx} uses
it in this way. The identity~\eqref{eq:laplace} for a negative power of a sum is classical
\citep{cressie1981moment}. We are not aware of earlier work that computes a Shapley value or another semivalue of a
ratio utility through it.

\textbf{Hardness of Shapley values.} The Shapley value is \#P-hard for weighted majority games
\citep{deng1994complexity,faliszewski2009complexity}, and so is the Banzhaf value \citep{prasad1990np}, from which
\citet{zhang2026banzhafknn} derive the hardness of the Banzhaf value of the hard-label KNN classifier. Hardness
proofs for the Shapley value often solve a family of instances indexed by a growing parameter and invert the matrix of Shapley
coefficients \citep{aziz2014shapley,livshits2021shapley}. Theorem~\ref{thm:abs} concerns a different kind of game:
the value of a subset is a loss of the average of its members' targets, and the reduction reads a count of subsets
from the second difference of the absolute value. A kink of this kind also drives the hardness of the Shapley value
in bankruptcy games \citep{aziz2013computation,yamazaki2025bankruptcy} and of an average of rectified values in
databases \citep{standke2025tractability}. Theorem~\ref{thm:bits} concerns the length of the exact value, which for
weighted majority games is polynomial.

\appsection{Proofs for Section~\ref{sec:exact}}\label{app:exact}
This appendix proves the results of Section~\ref{sec:exact}, and first Lemma~\ref{lem:cases}, on which they rest.
\begin{itemize}
\item Sections~\ref{app:exact:cases} to~\ref{app:exact:replace} prove Lemma~\ref{lem:cases}, the two identities
  behind the table, and Proposition~\ref{prop:phi}, which writes the Shapley value as two sums for every utility
  that depends on the $K$ nearest neighbors only (Proposition~\ref{app:exact:prop:twosums} states it for every
  $K$).
\item Section~\ref{app:exact:direct} evaluates the two sums for the squared loss
  (Proposition~\ref{app:exact:prop:phisq}) and gives the direct algorithm, Algorithm~\ref{app:exact:alg:direct},
  which takes $O(WK^2N^2)$ time.
\item Section~\ref{app:exact:sq} proves Lemma~\ref{lem:remove} and Theorem~\ref{thm:sq}: by removing a point from a
  shared table, Algorithm~\ref{app:exact:alg:linear} takes $O(WK^2N)$ time.
\item Section~\ref{app:exact:rational} shows that both algorithms are exact for rational targets and bounds the
  length of the value. Section~\ref{app:exact:float} treats floating-point arithmetic.
\item Section~\ref{app:exact:soft} proves Corollary~\ref{cor:soft}. Sections~\ref{app:exact:counts}
  to~\ref{app:exact:pair} prove Corollary~\ref{cor:hard} and Theorem~\ref{thm:pair} with tables of counts.
\item Section~\ref{app:exact:semi} treats other semivalues. Section~\ref{app:exact:abs} shows why no fixed number
  of sums per group serves the absolute loss.
\end{itemize}

\textbf{Notation review.} The training points $z_1,\dots,z_N$ are sorted by their distance to the query, and ties
are broken by index. So ``nearer'' is a strict order: $z_j$ is nearer than $z_{j'}$ exactly when $j<j'$, and
$z_{j'}$ is then farther than $z_j$. The set $D_{<e}=\{z_1,\dots,z_{e-1}\}$ holds the points nearer than $z_e$, and
$\NK(S)$ holds the $\min(K,|S|)$ nearest points of $S$. The weights
are $w_j=a_j\delta_w$ with integers $a_j\in\{1,\dots,W\}$, and $A(C)$, $R(C)$ and $r_j$ are as
in~\eqref{eq:AR}. An $s$-subset is a subset with $s$ points. The Shapley coefficient is
$c_N(s)=\frac1N\binom{N-1}{s}^{-1}=s!\,(N-1-s)!/N!$ for $0\le s\le N-1$, and for $e>K$ we write
\begin{equation}
c_e(K)=\frac{K!\,(e-1-K)!}{e!},
\label{app:exact:eq:ce}
\end{equation}
which is the same coefficient for a subset of size $K$ among $e$ points. A sum over the sizes $s=1,\dots,K-1$
stops at $s=N-1$ if $K>N$, because a set of $N-1$ points has no larger subsets.

Most of the argument uses one property of the utility only: it depends on a subset through its $K$ nearest
neighbors,
\begin{equation}
v(S)=v(\NK(S))\qquad\text{for every }S\subseteq D.
\label{app:exact:eq:nk}
\end{equation}
The utility~\eqref{eq:utility} has this property for every loss, because the prediction~\eqref{eq:pred} is computed
from $\NK(S)$ and $\NK(\NK(S))=\NK(S)$. The utilities of Sections~\ref{app:exact:soft} to~\ref{app:exact:pair} have
it as well.

\subsection{The Cases of a Marginal Contribution}
\label{app:exact:cases}

Fix a training point $z_i$ and a subset $S\subseteq D\setminus\{z_i\}$.

\begin{restate}{Lemma~\ref{lem:cases}}
Let $m$ be the number of points of $S$ that are nearer to the query than $z_i$.
\begin{enumerate}
\item If $m\ge K$, then $v(S\cup\{z_i\})=v(S)$.
\item If $|S|\le K-1$, then $\NK(S)=S$ and $\NK(S\cup\{z_i\})=S\cup\{z_i\}$.
\item If $|S|\ge K$ and $m\le K-1$, let $z_e$ be the $K$-th nearest point of $S$ and $C$ the $K-1$ points of $S$
nearer than $z_e$. Then $e>i$, $\NK(S)=C\cup\{z_e\}$ and $\NK(S\cup\{z_i\})=C\cup\{z_i\}$.
\end{enumerate}
\end{restate}

\begin{proof}
A point of a set with at least $K$ points is one of its $K$ nearest points exactly when fewer than $K$ points of the
set are nearer than it.

\emph{Case 1.} Both $S$ and $S\cup\{z_i\}$ have at least $m\ge K$ points. The point $z_i$ has $m\ge K$ nearer
points in $S\cup\{z_i\}$, so it is not in $\NK(S\cup\{z_i\})$. Every point of $\NK(S)$ is nearer than $z_i$, because
a point farther than $z_i$ has at least $m\ge K$ nearer points in $S$. Adding $z_i$ therefore leaves the number of
nearer points of a point of $\NK(S)$ unchanged, and it does not lower this number for the other points of $S$.
Hence $\NK(S\cup\{z_i\})=\NK(S)$, and $v(S\cup\{z_i\})=v(S)$ by~\eqref{app:exact:eq:nk}.

\emph{Case 2.} Here $\min(K,|S|)=|S|$, so $\NK(S)=S$. The set $S\cup\{z_i\}$ has $|S|+1\le K$ points, so
$\NK(S\cup\{z_i\})=S\cup\{z_i\}$.

\emph{Case 3.} Since $|S|\ge K$, the $K$-th nearest point $z_e$ of $S$ exists, exactly $K-1$ points of $S$ are
nearer than $z_e$, and $\NK(S)=C\cup\{z_e\}$. If $e<i$, then $z_e$ and the $K-1$ points of $C$ would all be nearer
than $z_i$, which contradicts $m\le K-1$. Since $z_i\notin S$, we have $e\ne i$, and so $e>i$. Now consider
$S\cup\{z_i\}$. The $K$ points of $C\cup\{z_i\}$ are nearer than $z_e$: those of $C$ by definition, and $z_i$
because $i<e$. Every other point of $S\cup\{z_i\}$ is $z_e$ or a point of $S$ farther than $z_e$, and so it has at
least $K$ nearer points. Hence $\NK(S\cup\{z_i\})=C\cup\{z_i\}$.
\end{proof}

The three conditions of the lemma do not overlap, and they leave no subset out.

\begin{lemma}[The three cases split the subsets]
\label{app:exact:lem:split}
Every $S\subseteq D\setminus\{z_i\}$ satisfies exactly one of the three conditions of Lemma~\ref{lem:cases}:
$m\ge K$; or $|S|\le K-1$; or $|S|\ge K$ and $m\le K-1$. Consequently,
\begin{equation}
\phi_i=\sum_{\substack{S\subseteq D\setminus\{z_i\}\\ |S|\le K-1}}c_N(|S|)\bigl[v(S\cup\{z_i\})-v(S)\bigr]
\;+\sum_{\substack{S\subseteq D\setminus\{z_i\}\\ |S|\ge K,\ m\le K-1}}c_N(|S|)\bigl[v(S\cup\{z_i\})-v(S)\bigr].
\label{app:exact:eq:split}
\end{equation}
\end{lemma}

\begin{proof}
The $m$ points that are counted belong to $S$, so $m\le|S|$. If $|S|\le K-1$, then $m\le K-1$ as well, and only the
second condition holds. If $|S|\ge K$, the second condition fails, and exactly one of $m\ge K$ and $m\le K-1$ holds.
This proves the first claim. For the second, split the sum~\eqref{eq:shapley} by the three conditions. The subsets
with $m\ge K$ contribute nothing, by case~1 of Lemma~\ref{lem:cases}.
\end{proof}

In the first sum of~\eqref{app:exact:eq:split} the point $z_i$ joins the neighbors. In the second sum it replaces the
$K$-th nearest neighbor.

\subsection{Sums over a Group}
\label{app:exact:group}

A group is a family of sets of training points that have the same total weight. The next lemma is the
identity~\eqref{eq:group}. Three numbers give the sum of the squared prediction errors over a group, before and
after one point is added to each of its sets. For a triple $n=(n_0,n_1,n_2)$, a total weight $\omega$ and a
training point $z_j$ we write
\begin{equation*}
\sigma_j(n,\omega)=\frac{n_2+2r_jn_1+r_j^2n_0}{(\omega+a_j)^2},
\end{equation*}
which is the last number of the triple $n^{+j}$ of~\eqref{eq:shift}, divided by the square of the total weight
after $z_j$ is added.

\begin{lemma}[Sum over a group]
\label{app:exact:lem:group}
Let $\mathcal{G}$ be a finite family of sets of training points with the same total weight $\omega$, and let
$n=(n_0,n_1,n_2)$ be its three numbers~\eqref{eq:three}. Let $z_j$ be a training point that lies in no set of
$\mathcal{G}$. Then
\begin{equation}
\sum_{C\in\mathcal{G}}\frac{R(C\cup\{z_j\})^2}{A(C\cup\{z_j\})^2}
=\frac{n_2+2r_jn_1+r_j^2n_0}{(\omega+a_j)^2}=\sigma_j(n,\omega).
\label{app:exact:eq:groupsum}
\end{equation}
If $\omega\ge1$, then also $\sum_{C\in\mathcal{G}}R(C)^2/A(C)^2=n_2/\omega^2$.
\end{lemma}

\begin{proof}
Let $C\in\mathcal{G}$. Since $z_j\notin C$, the sums~\eqref{eq:AR} give $A(C\cup\{z_j\})=\omega+a_j$ and
$R(C\cup\{z_j\})=R(C)+r_j$. The denominator is the same for all sets of the group, and
\begin{equation*}
\sum_{C\in\mathcal{G}}\bigl(R(C)+r_j\bigr)^2
=\sum_{C\in\mathcal{G}}R(C)^2+2r_j\sum_{C\in\mathcal{G}}R(C)+r_j^2\,|\mathcal{G}|
=n_2+2r_jn_1+r_j^2n_0 .
\end{equation*}
This is~\eqref{app:exact:eq:groupsum}. The last claim holds because $A(C)=\omega$ for every $C\in\mathcal{G}$.
\end{proof}

Multiplied by $-1$, the identity~\eqref{app:exact:eq:groupsum} is~\eqref{eq:group}. The lemma is a statement about the
sums $A$ and $R$. It becomes a statement about utilities through~\eqref{eq:vAR}: if the sets $C\cup\{z_j\}$ have at
most $K$ points, they are their own $K$ nearest neighbors, and $-\sigma_j(n,\omega)$ is the sum of their utilities.

\textbf{The table.} Recall that $Q_G[s,\omega]$ holds the three numbers of the family of $s$-subsets of $G$ with
total weight $\omega$. An entry with $s<0$ or $\omega<0$ is read as $(0,0,0)$, since no set has such a size or
such a total weight.

\begin{lemma}[Update of the table]
\label{app:exact:lem:update}
The table of the empty set has $Q_\emptyset[0,0]=(1,0,0)$ and $(0,0,0)$ in every other entry. Let $G$ be a set of
training points and $z_j\notin G$. For all $s\ge0$ and $\omega\ge0$,
\begin{equation}
Q_{G\cup\{z_j\}}[s,\omega]=Q_G[s,\omega]+\bigl(n_0,\;n_1+r_jn_0,\;n_2+2r_jn_1+r_j^2n_0\bigr)
\quad\text{with }(n_0,n_1,n_2)=Q_G[s-1,\omega-a_j],
\label{app:exact:eq:tableupdate}
\end{equation}
which is the update~\eqref{eq:update} with the triple~\eqref{eq:shift} written out. In particular, row $s=0$
never changes.
\end{lemma}

\begin{proof}
The only subset of the empty set is the empty set. It has size $0$, total weight $0$ and $R(\emptyset)=0$. Its
group therefore has the three numbers $(1,0,0)$, and all other groups are empty.

For the update, split the $s$-subsets of $G\cup\{z_j\}$ with total weight $\omega$ into two kinds. Those that avoid
$z_j$ are the $s$-subsets of $G$ with total weight $\omega$, and their three numbers are $Q_G[s,\omega]$. Those
that contain $z_j$ are the sets $C\cup\{z_j\}$, where $C$ is an $(s-1)$-subset of $G$ with total weight
$\omega-a_j$. The map $C\mapsto C\cup\{z_j\}$ is one-to-one, so there are $n_0$ sets of the second kind. Since
$R(C\cup\{z_j\})=R(C)+r_j$, their sum of $R$ is $n_1+r_jn_0$, and their sum of $R^2$ is $n_2+2r_jn_1+r_j^2n_0$ as
in the proof of Lemma~\ref{app:exact:lem:group}. Adding the numbers of the two kinds
gives~\eqref{app:exact:eq:tableupdate}. For $s=0$ the second kind is empty, because $Q_G[-1,\cdot]$ is zero.
\end{proof}

\textbf{Size of the table.} Only the rows $s\le K-1$ are used. Row $s$ of the update reads row $s-1$ only, so these
rows can be maintained without the others. An $s$-subset has a total weight between $s$ and $sW$, because every
$a_j$ lies in $\{1,\dots,W\}$. In the rows $s\le K-1$ the nonzero entries therefore lie in the columns
$0\le\omega\le(K-1)W$, and in a row $s\ge1$ they lie in the columns $\omega\ge1$. A table is stored as $K$ rows of
$(K-1)W+1$ entries.

\subsection{The Replace Case}
\label{app:exact:replace}

We now describe the subsets of the second sum of~\eqref{app:exact:eq:split}. Fix $z_i$, and let $\mathcal{S}_i$ be
the family of the subsets $S\subseteq D\setminus\{z_i\}$ with $|S|\ge K$ and $m\le K-1$. For $e>i$, the set
$D_{<e}\setminus\{z_i\}$ holds the points nearer than $z_e$, other than $z_i$. Let $\mathcal{T}_i$ be the set of the
triples $(C,z_e,F)$ in which $e\in\{i+1,\dots,N\}$, $C$ is a $(K-1)$-subset of $D_{<e}\setminus\{z_i\}$, and $F$ is
any subset of $\{z_{e+1},\dots,z_N\}$, the points farther than $z_e$.

\begin{lemma}[One-to-one correspondence]
\label{app:exact:lem:bijection}
Map a subset $S\in\mathcal{S}_i$ to the triple $(C,z_e,F)$ in which $z_e$ is the $K$-th nearest point of $S$, $C$
is the set of the $K-1$ points of $S$ nearer than $z_e$, and $F$ is the set of the points of $S$ farther than
$z_e$. This map is a one-to-one correspondence between $\mathcal{S}_i$ and $\mathcal{T}_i$. Its inverse maps
$(C,z_e,F)$ to $C\cup\{z_e\}\cup F$, and $|S|=K+|F|$.
\end{lemma}

\begin{proof}
Let $S\in\mathcal{S}_i$. By case~3 of Lemma~\ref{lem:cases}, $e>i$. The points of $C$ are nearer than $z_e$, and
they differ from $z_i$ because $z_i\notin S$. So $C$ is a $(K-1)$-subset of $D_{<e}\setminus\{z_i\}$, and
$F\subseteq\{z_{e+1},\dots,z_N\}$ by its definition. The triple therefore lies in $\mathcal{T}_i$. Every point of
$S$ is nearer than $z_e$, equal to $z_e$, or farther than $z_e$, so $S$ is the disjoint union $C\cup\{z_e\}\cup F$.
Hence $S$ is recovered from its triple, and $|S|=K+|F|$.

Conversely, let $(C,z_e,F)\in\mathcal{T}_i$ and put $S=C\cup\{z_e\}\cup F$. The three parts are disjoint, since the
points of $C$ are nearer than $z_e$ and those of $F$ are farther. The point $z_i$ is not in $S$: it is not in $C$
by the choice of $C$, it differs from $z_e$, and it is not in $F$ because $i<e$. The set has $|S|=K+|F|\ge K$
points. Exactly the $K-1$ points of $C$ are nearer than $z_e$. So $z_e$ is the $K$-th nearest point of $S$, $C$ is
the set of its points nearer than $z_e$, and $F$ is the set of its points farther than $z_e$. Finally, $z_e$ and
the points of $F$ are farther than $z_i$. The points of $S$ nearer than $z_i$ therefore lie in $C$, which gives
$m\le K-1$. So $S\in\mathcal{S}_i$, and its triple is $(C,z_e,F)$. The two maps are inverse to each other.
\end{proof}

A triple exists only if $D_{<e}\setminus\{z_i\}$ has at least $K-1$ points. Since this set has $e-2$ points, this
means $e\ge K+1$, so the index $e$ ranges over $e>\max(i,K)$.

Two facts make the correspondence useful. First, by case~3 of Lemma~\ref{lem:cases} and
by~\eqref{app:exact:eq:nk}, the marginal contribution of $z_i$ to $S=C\cup\{z_e\}\cup F$ is
\begin{equation}
v(S\cup\{z_i\})-v(S)=v(C\cup\{z_i\})-v(C\cup\{z_e\}),
\label{app:exact:eq:marg}
\end{equation}
which depends on $C$ and $z_e$ but not on $F$. Second, the Shapley coefficient $c_N(|S|)=c_N(K+|F|)$ depends on $F$
through its size only. The sum of the coefficients over all $F$ has a closed form.

\begin{lemma}[Closed form over the farther points]
\label{app:exact:lem:closed}
For $K+1\le e\le N$,
\begin{equation}
\sum_{F\subseteq\{z_{e+1},\dots,z_N\}}c_N(K+|F|)
=\sum_{j=0}^{N-e}\binom{N-e}{j}\,c_N(K+j)=c_e(K).
\label{app:exact:eq:closed}
\end{equation}
\end{lemma}

\begin{proof}
The set $\{z_{e+1},\dots,z_N\}$ has $N-e$ points and hence $\binom{N-e}{j}$ subsets of size $j$, which gives the
first equality. For the second we use the Beta integral: $\int_0^1p^{k}(1-p)^{l}\,dp=k!\,l!/(k+l+1)!$ for integers
$k,l\ge0$. With $k=s$ and $l=N-1-s$ it gives $c_N(s)=\int_0^1p^s(1-p)^{N-1-s}\,dp$ for $0\le s\le N-1$. Put
$f=N-e$, the number of farther points. For $0\le j\le f$ we have $K+j\le K+N-e\le N-1$, so this form applies to
every term:
\begin{align*}
\sum_{j=0}^{f}\binom{f}{j}c_N(K+j)
&=\int_0^1\sum_{j=0}^{f}\binom{f}{j}\,p^{K+j}(1-p)^{N-1-K-j}\,dp\\
&=\int_0^1p^{K}(1-p)^{e-1-K}\,\sum_{j=0}^{f}\binom{f}{j}\,p^{j}(1-p)^{f-j}\,dp
=\int_0^1p^{K}(1-p)^{e-1-K}\,dp .
\end{align*}
The second line uses $N-1-K-j=(e-1-K)+(f-j)$ with $e-1-K\ge0$, and then the binomial theorem, by which the inner
sum is $(p+(1-p))^f=1$. By the Beta integral with $k=K$ and $l=e-1-K$, the last integral is
$K!\,(e-1-K)!/e!=c_e(K)$.
\end{proof}

In words, summing over the points farther than $z_e$ removes them from the game. What remains is the coefficient
of a subset of size $K$ when only the $e$ points $z_1,\dots,z_e$ are present.

\begin{proposition}[The Shapley value as two sums]
\label{app:exact:prop:twosums}
Let $v$ be any utility with the property~\eqref{app:exact:eq:nk}. For every $i$,
\begin{align}
\phi_i={}&\sum_{s=0}^{\min(K,N)-1}c_N(s)\sum_{\substack{S\subseteq D\setminus\{z_i\}\\|S|=s}}
\bigl[v(S\cup\{z_i\})-v(S)\bigr]\notag\\
&+\sum_{e>\max(i,K)}c_e(K)\sum_{\substack{C\subseteq D_{<e}\setminus\{z_i\}\\|C|=K-1}}
\bigl[v(C\cup\{z_i\})-v(C\cup\{z_e\})\bigr].
\label{app:exact:eq:twosums}
\end{align}
\end{proposition}

\begin{proof}
The first sum of~\eqref{app:exact:eq:split} is the first sum of~\eqref{app:exact:eq:twosums} with the subsets
arranged by their size $s$, which is at most $N-1$. In the second sum of~\eqref{app:exact:eq:split}, replace every
$S\in\mathcal{S}_i$ by its triple (Lemma~\ref{app:exact:lem:bijection}) and use~\eqref{app:exact:eq:marg}:
\begin{equation*}
\sum_{S\in\mathcal{S}_i}c_N(|S|)\bigl[v(S\cup\{z_i\})-v(S)\bigr]
=\sum_{e>\max(i,K)}\;\sum_{\substack{C\subseteq D_{<e}\setminus\{z_i\}\\|C|=K-1}}
\biggl[\,\sum_{F\subseteq\{z_{e+1},\dots,z_N\}}c_N(K+|F|)\biggr]\bigl[v(C\cup\{z_i\})-v(C\cup\{z_e\})\bigr].
\end{equation*}
By Lemma~\ref{app:exact:lem:closed}, the sum over $F$ in the first bracket equals $c_e(K)$.
\end{proof}

\begin{restate}{Proposition~\ref{prop:phi}}
For every utility that depends on a subset $S$ only through $\NK(S)$, the Shapley value of $z_i$ is
\begin{align*}
&\phi_i=\sum_{s=0}^{K-1}c_N(s)\sum_{S}\bigl[v(S\cup\{z_i\})-v(S)\bigr]\\
&\;+\sum_{e>\max(i,K)}c_e(K)\sum_{C}\bigl[v(C\cup\{z_i\})-v(C\cup\{z_e\})\bigr],
\tag{\ref{eq:phi}}
\end{align*}
where $S$ ranges over the subsets of $D\setminus\{z_i\}$ with $s$ points and $C$ over those of
$D_{<e}\setminus\{z_i\}$ with $K-1$ points.
\end{restate}

\begin{proof}
For $K\le N$, $\min(K,N)-1=K-1$. For $K>N$, the terms with $s\ge N$ are empty, since $D\setminus\{z_i\}$ has
$N-1$ points, and so is the second sum. In both cases the statement is Proposition~\ref{app:exact:prop:twosums}.
\end{proof}

The second sum of~\eqref{app:exact:eq:twosums} has two features that the algorithms below use. The points farther
than $z_e$ no longer appear: they are summed into the coefficient $c_e(K)$. And $C$ ranges over all
$(K-1)$-subsets of $D_{<e}\setminus\{z_i\}$, so the points nearer than $z_i$ and the points between $z_i$ and $z_e$
need not be counted separately. For each pair of $z_i$ and $z_e$, one table over the subsets of
$D_{<e}\setminus\{z_i\}$ is therefore enough.

\subsection{The Two Sums for the Squared Loss and the Direct Algorithm}
\label{app:exact:direct}

For the squared loss, the table gives the inner sums of Proposition~\ref{prop:phi}: this is the
display~\eqref{eq:read} of Section~\ref{sec:sq}. The next proposition states the result in one formula.

\begin{proposition}[The two sums for the squared loss]
\label{app:exact:prop:phisq}
For weights on $W$ levels, real-valued targets and the squared loss, the Shapley value of $z_i$ is
\begin{align}
\phi_i={}&\tfrac1N\bigl[v(\{z_i\})-v(\emptyset)\bigr]\notag\\
&+\sum_{s=1}^{K-1}c_N(s)\sum_{\omega}\Bigl[\frac{n_2}{\omega^2}-\sigma_i(n,\omega)\Bigr]\notag\\
&+\sum_{e>\max(i,K)}c_e(K)\sum_{\omega}\bigl[\sigma_e(n',\omega)-\sigma_i(n',\omega)\bigr],
\label{app:exact:eq:phisq}
\end{align}
where $n=Q_{D\setminus\{z_i\}}[s,\omega]$ in the first sum and $n'=Q_{D_{<e}\setminus\{z_i\}}[K-1,\omega]$ in the
second.
\end{proposition}

The sums over $\omega$ in~\eqref{app:exact:eq:phisq} run over the total weights of the nonempty groups, that is, over the
$\omega$ with $n_0>0$ in the first sum and with $n'_0>0$ in the second.

\begin{proof}
The squared-loss utility has the property~\eqref{app:exact:eq:nk}, so Proposition~\ref{app:exact:prop:twosums}
applies. We evaluate the two sums of~\eqref{app:exact:eq:twosums} with~\eqref{eq:vAR}.

\emph{The point $z_i$ joins.} The term $s=0$ of the first sum is the empty subset, and $c_N(0)=1/N$. It gives
$\frac1N[v(\{z_i\})-v(\emptyset)]$. Now let $1\le s\le\min(K,N)-1$, and group the $s$-subsets of $D\setminus\{z_i\}$
by their total weight $\omega$, which is at least $s\ge1$. For a subset $S$ of the group, case~2 of
Lemma~\ref{lem:cases} and~\eqref{eq:vAR} give
\begin{equation*}
v(S\cup\{z_i\})-v(S)=-\frac{R(S\cup\{z_i\})^2}{A(S\cup\{z_i\})^2}+\frac{R(S)^2}{A(S)^2}.
\end{equation*}
The three numbers of the group are $n=Q_{D\setminus\{z_i\}}[s,\omega]$. By Lemma~\ref{app:exact:lem:group} with
$j=i$, the marginal contributions of the group sum to $n_2/\omega^2-\sigma_i(n,\omega)$. Summing over $\omega$, and
over $s$ with the coefficients $c_N(s)$, gives the first sum of~\eqref{app:exact:eq:phisq}.

\emph{The point $z_i$ replaces $z_e$.} Let $e>\max(i,K)$, and group the $(K-1)$-subsets $C$ of
$D_{<e}\setminus\{z_i\}$ by their total weight $\omega$. The sets $C\cup\{z_i\}$ and $C\cup\{z_e\}$ have $K$
points, so they are their own $K$ nearest neighbors, and by~\eqref{eq:vAR}
\begin{equation*}
v(C\cup\{z_i\})-v(C\cup\{z_e\})
=-\frac{R(C\cup\{z_i\})^2}{A(C\cup\{z_i\})^2}+\frac{R(C\cup\{z_e\})^2}{A(C\cup\{z_e\})^2}.
\end{equation*}
The three numbers of the group are $n'=Q_{D_{<e}\setminus\{z_i\}}[K-1,\omega]$, and neither $z_i$ nor $z_e$ lies in
a set of the group. By Lemma~\ref{app:exact:lem:group}, applied with $j=e$ and with $j=i$, these differences sum to
$\sigma_e(n',\omega)-\sigma_i(n',\omega)$ over the group. Summing over $\omega$, and over $e$ with the coefficients
$c_e(K)$, gives the second sum of~\eqref{app:exact:eq:phisq}.
\end{proof}

\textbf{The direct algorithm.} For each $z_i$, formula~\eqref{app:exact:eq:phisq} needs row $K-1$ of the tables
$Q_{D_{<e}\setminus\{z_i\}}$ for all $e>\max(i,K)$, and the rows $s\le K-1$ of the table $Q_{D\setminus\{z_i\}}$.
These tables form a chain: the set $D_{<e}\setminus\{z_i\}$ grows by one point as $e$ increases, and it ends at
$D\setminus\{z_i\}$. Algorithm~\ref{app:exact:alg:direct} therefore keeps one table per point $z_i$ and adds the
other points to it in order of distance, by the update~\eqref{eq:update}. Just before it adds a point $z_e$ with
$e>\max(i,K)$, it reads row $K-1$ for the second sum of~\eqref{app:exact:eq:phisq}. After the last point it reads all rows for
the first sum.

\begin{algorithm}[t]
\caption{The direct algorithm: exact Shapley values of weighted KNN regression with the squared loss, with one
table per point (Proposition~\ref{app:exact:prop:phisq}).}
\label{app:exact:alg:direct}
\renewcommand{\algorithmicrequire}{\textbf{Input:}}
\renewcommand{\algorithmicensure}{\textbf{Output:}}
\begin{algorithmic}[1]
\Require $K$; the integer weights $a_1,\dots,a_N\in\{1,\dots,W\}$ and the targets $y_1,\dots,y_N$ of the training
  points, in order of distance to the query; the target $y^{(\val)}$ of the query; the default prediction
  $y_{\emptyset}$
\Ensure $\phi_1,\dots,\phi_N$
\State $r_j\gets a_j\,(y_j-y^{(\val)})$ for $j=1,\dots,N$
\For{$i=1,\dots,N$}
  \State $Q\gets{}$table of the empty set, with rows $0\le s\le K-1$ and columns $0\le\omega\le(K-1)W$:
    $Q[0,0]=(1,0,0)$, and $(0,0,0)$ elsewhere
  \State\label{app:exact:line:empty} $\phi_i\gets\frac1N\bigl[(y_{\emptyset}-y^{(\val)})^2-(y_i-y^{(\val)})^2\bigr]$
    \Comment{the empty subset}
  \For{$e=1,\dots,N$ with $e\neq i$}\label{app:exact:line:loop}
    \If{$e>\max(i,K)$}\Comment{$z_i$ replaces $z_e$}
      \State\label{app:exact:line:replace} $\phi_i\gets\phi_i+c_e(K)\sum_{\omega=0}^{(K-1)W}
        \bigl[\sigma_e(n',\omega)-\sigma_i(n',\omega)\bigr]$, where $n'=Q[K-1,\omega]$
    \EndIf
    \State\label{app:exact:line:add} \Call{AddPoint}{$Q,z_e$}
  \EndFor
  \For{$s=1,\dots,\min(K,N)-1$}\Comment{$z_i$ joins $s$ neighbors}
    \State\label{app:exact:line:join} $\phi_i\gets\phi_i+c_N(s)\sum_{\omega=1}^{(K-1)W}
      \bigl[n_2/\omega^2-\sigma_i(n,\omega)\bigr]$, where $n=(n_0,n_1,n_2)=Q[s,\omega]$
  \EndFor
\EndFor
\State \Return $\phi_1,\dots,\phi_N$
\Statex
\Procedure{AddPoint}{$Q,z_j$}\Comment{the update~\eqref{eq:update}, in place}
  \For{$s=K-1,K-2,\dots,1$}\Comment{row $s-1$ is read before it changes}
    \For{$\omega=a_j,\dots,(K-1)W$}
      \State $(n_0,n_1,n_2)\gets Q[s-1,\omega-a_j]$
      \State $Q[s,\omega]\gets Q[s,\omega]+\bigl(n_0,\;n_1+r_jn_0,\;n_2+2r_jn_1+r_j^2n_0\bigr)$
    \EndFor
  \EndFor
\EndProcedure
\end{algorithmic}
\end{algorithm}

\begin{lemma}[Invariant of Algorithm~\ref{app:exact:alg:direct}]
\label{app:exact:lem:invariant}
Consider the iteration of Algorithm~\ref{app:exact:alg:direct} for a point $z_i$. Whenever the inner loop reaches
an index $e$ (line~\ref{app:exact:line:loop}), $Q[s,\omega]=Q_G[s,\omega]$ holds for all $0\le s\le K-1$ and
$0\le\omega\le(K-1)W$, with $G=D_{<e}\setminus\{z_i\}$. When the inner loop ends, the same holds with
$G=D\setminus\{z_i\}$.
\end{lemma}

\begin{proof}
We use induction over the passes of the inner loop. Before the first pass, $Q$ is the table of the empty set
(Lemma~\ref{app:exact:lem:update}). The first index is $e=1$, or $e=2$ if $i=1$, and $D_{<e}\setminus\{z_i\}$ is
empty in both cases. Now assume the claim when the loop reaches $e$. The pass ends with
\textsc{AddPoint}$(Q,z_e)$, which applies~\eqref{app:exact:eq:tableupdate} with $j=e$ to the rows $s=K-1,\dots,1$
in this order. When row $s$ is updated, row $s-1$ has not changed yet, so the entries that are read are those of
$Q_G$. The columns $\omega<a_e$ are left as they are, in agreement with~\eqref{app:exact:eq:tableupdate}, because
$Q_G[s-1,\omega-a_e]$ is zero there. The columns $\omega>(K-1)W$ hold zeros in the rows $s\le K-1$ and are not
stored. After the pass, $Q$ is therefore the table of $G\cup\{z_e\}=\{z_1,\dots,z_e\}\setminus\{z_i\}$. This is the
set of the claim for the next index, and it is $D\setminus\{z_i\}$ after the last pass.
\end{proof}

\textbf{Correctness.} By the lemma, line~\ref{app:exact:line:replace} is executed with
$Q[K-1,\omega]=Q_{D_{<e}\setminus\{z_i\}}[K-1,\omega]$, the $n'$ of~\eqref{app:exact:eq:phisq}, and it adds the term of the
second sum for this $e$. Line~\ref{app:exact:line:join} is executed with $Q=Q_{D\setminus\{z_i\}}$ and adds the
term of the first sum for this $s$. An empty group has the entry $(0,0,0)$ and adds zero, so a sum over all stored
columns equals the sum over the nonempty groups. Line~\ref{app:exact:line:join} skips the column $\omega=0$, which
is empty in a row $s\ge1$. Line~\ref{app:exact:line:empty} sets the first term of~\eqref{app:exact:eq:phisq}, since
$v(\{z_i\})=-(y_i-y^{(\val)})^2$ and $v(\emptyset)=-(y_{\emptyset}-y^{(\val)})^2$. Hence the algorithm returns the
value~\eqref{app:exact:eq:phisq}. The coefficients follow from the recurrences $c_N(0)=1/N$, $c_N(s)=c_N(s-1)\,s/(N-s)$ and
$c_{K+1}(K)=1/(K+1)$, $c_{e+1}(K)=c_e(K)\,(e-K)/(e+1)$.

\textbf{Running time and space.} We count arithmetic operations. One call of \textsc{AddPoint} updates the $K-1$
rows $s\ge1$. Each row has at most $(K-1)W+1$ entries, and an entry takes a constant number of additions and
multiplications. For one point $z_i$ the algorithm makes $N-1$ calls. It reads row $K-1$ at most $N-i$ times,
with a constant number of operations for each of its $(K-1)W+1$ entries. At the end it reads at most $(K-1)^2W$
entries once. Over all $i$ this is at most
\begin{align*}
&N(N-1)\,(K-1)\bigl((K-1)W+1\bigr)&&\text{entry updates and}\\
&\tfrac12N(N-1)\bigl((K-1)W+1\bigr)+N(K-1)^2W&&\text{entry reads},
\end{align*}
which is $O(WK^2N^2)$ operations. Computing the numbers $r_j$ and the coefficients adds $O(N+K)$ operations. The
algorithm stores one table of $K$ rows and $(K-1)W+1$ columns with three numbers per entry, the $N$ numbers $r_j$,
at most $N+K$ coefficients and the $N$ values, which is $O(WK^2+N)$ numbers. In the iteration for $z_i$, the first
$i-1$ passes build the table of $D_{<i}$ and read nothing. A second table can carry this table from one iteration
to the next, which saves half of the additions and changes neither order.

\textbf{Boundary cases.} If $K\ge N$, no subset of $D\setminus\{z_i\}$ has $K$ points. Cases~1 and~3 of
Lemma~\ref{lem:cases} do not occur, no index satisfies $e>\max(i,K)$, and the second sum of~\eqref{app:exact:eq:phisq} is
empty. The rows $s\ge N$ of the table stay zero, so the first sum runs over $s\le N-1$. This is why the last loop
of Algorithm~\ref{app:exact:alg:direct} stops at $\min(K,N)-1$. If $K=1$, the table has the single row $s=0$, which
holds $(1,0,0)$ at $\omega=0$ and never changes. The first sum is empty, and
$\sigma_j((1,0,0),0)=r_j^2/a_j^2=(y_j-y^{(\val)})^2$, so~\eqref{app:exact:eq:phisq} becomes
\begin{equation}
\phi_i=\tfrac1N\bigl[v(\{z_i\})-v(\emptyset)\bigr]
+\sum_{e>i}\frac{(y_e-y^{(\val)})^2-(y_i-y^{(\val)})^2}{e\,(e-1)} .
\label{app:exact:eq:Kone}
\end{equation}
The weights have cancelled, as they must: the prediction from a single neighbor is its target.

\subsection{Removing a Point: Proof of \texorpdfstring{Lemma~\ref{lem:remove} and
Theorem~\ref{thm:sq}}{the Removal Lemma and the Main Theorem}}
\label{app:exact:sq}

The direct algorithm builds one table for every point $z_i$. This section shows that one table serves all points.
Lemma~\ref{lem:remove} takes a point out of a table, and we use it in two ways. For the first sum
of~\eqref{app:exact:eq:phisq}, it takes $z_i$ out of the table of all points. For the second sum, it expresses the table of
$D_{<e}\setminus\{z_i\}$ through the table of $D_{<e}$, which does not depend on $i$, so that the sum over $e$ is
taken once for all points. The result is Algorithm~\ref{app:exact:alg:linear}, which proves the theorem.

\begin{restate}{Lemma~\ref{lem:remove}}
Let $G$ be a set of training points, $z_i\in G$ and $G'=G\setminus\{z_i\}$. Then $Q_{G'}[0,\omega]=Q_G[0,\omega]$,
and for $s\ge1$
\begin{equation*}
Q_{G'}[s,\omega]=Q_G[s,\omega]-Q_{G'}[s-1,\omega-a_i]^{+i}.
\tag{\ref{eq:remove}}
\end{equation*}
\end{restate}

Written out with the triple~\eqref{eq:shift}, the lemma subtracts
$\bigl(n_0,\;n_1+r_in_0,\;n_2+2r_in_1+r_i^2n_0\bigr)$ with $(n_0,n_1,n_2)=Q_{G\setminus\{z_i\}}[s-1,\omega-a_i]$.

\begin{restate}{Theorem~\ref{thm:sq}}
For weights on $W$ levels, real-valued targets and the squared loss, the exact Shapley values of all $N$ training
points are computed in $O(WK^2N)$ time and $O(WK^2+N)$ space after sorting by distance.
\end{restate}

\begin{proof}[Proof of Lemma~\ref{lem:remove}]
Apply Lemma~\ref{app:exact:lem:update} to the set $G\setminus\{z_i\}$ and to the point $z_i$, which is not in it.
Their union is $G$, so the update~\eqref{app:exact:eq:tableupdate} reads
\begin{equation*}
Q_G[s,\omega]=Q_{G\setminus\{z_i\}}[s,\omega]+\bigl(n_0,\;n_1+r_in_0,\;n_2+2r_in_1+r_i^2n_0\bigr)
\quad\text{with }(n_0,n_1,n_2)=Q_{G\setminus\{z_i\}}[s-1,\omega-a_i].
\end{equation*}
Solving for $Q_{G\setminus\{z_i\}}[s,\omega]$ gives~\eqref{eq:remove}. Row $0$ never changes in an update, so it is
the same in both tables.
\end{proof}

The lemma gives the smaller table row by row: row $s$ of $Q_{G\setminus\{z_i\}}$ follows from row $s$ of $Q_G$ and
from row $s-1$ of $Q_{G\setminus\{z_i\}}$. Removing a player from a table of counts by reversing the update is the
device of \citet{mann1962values} for weighted voting games \citep[see also][]{kurz2016computing,kakimura2026computing}.
Applied $s$ times, the lemma leads from row $0$ to row $s$. The next lemma states the result of these $s$ steps in
one formula, for the sum of any function of the two totals over the subsets.

\begin{lemma}[Removal in one formula]
\label{app:exact:lem:unrolled}
Let $G$ be a set of training points, $z_i\in G$ and $s\ge0$. For every real function $g$ of a total weight and a
total weighted target,
\begin{equation}
\sum_{\substack{C\subseteq G\setminus\{z_i\}\\|C|=s}}g\bigl(A(C),\,R(C)\bigr)
=\sum_{t=0}^{s}(-1)^t\sum_{\substack{C\subseteq G\\|C|=s-t}}g\bigl(A(C)+t\,a_i,\;R(C)+t\,r_i\bigr).
\label{app:exact:eq:unrolled}
\end{equation}
\end{lemma}

\begin{proof}
For a function $g$, write $E_s(g)$ for the left side of~\eqref{app:exact:eq:unrolled} and $F_s(g)$ for the sum of
$g(A(C),R(C))$ over the $s$-subsets $C$ of $G$. Let $g^{+}(\omega,\rho)=g(\omega+a_i,\rho+r_i)$ be the function $g$
shifted by $z_i$, and $g^{+t}(\omega,\rho)=g(\omega+t\,a_i,\rho+t\,r_i)$ the function shifted $t$ times. An
$s$-subset of $G$ either avoids $z_i$, or it is an $(s-1)$-subset of $G\setminus\{z_i\}$ plus $z_i$, and adding
$z_i$ adds $a_i$ and $r_i$ to the two totals. Hence
\begin{equation*}
F_s(g)=E_s(g)+E_{s-1}(g^{+})\quad\text{for }s\ge1,\qquad F_0(g)=E_0(g)=g(0,0).
\end{equation*}
The claim~\eqref{app:exact:eq:unrolled} is $E_s(g)=\sum_{t=0}^{s}(-1)^tF_{s-t}(g^{+t})$. We show it for all
functions $g$ by induction on $s$. For $s=0$ it is $E_0(g)=F_0(g)$. For $s\ge1$, the induction hypothesis for the
function $g^{+}$ gives $E_{s-1}(g^{+})=\sum_{t=0}^{s-1}(-1)^tF_{s-1-t}(g^{+(t+1)})$, and so
\begin{equation*}
E_s(g)=F_s(g)-E_{s-1}(g^{+})=F_s(g)+\sum_{t=1}^{s}(-1)^{t}F_{s-t}(g^{+t}).\qedhere
\end{equation*}
\end{proof}

In words, term $t$ adds $t$ copies of $z_i$ to the two totals of every $(s-t)$-subset of $G$, whether or not
$z_i$ is in it. Term $0$ is the sum over all $s$-subsets of $G$. It counts the subsets that contain $z_i$ too much,
and what each term counts too much, the next one removes. The proof uses of $A$ and $R$ only that they are sums
over the points of a set. The lemma therefore holds for functions of any such sums, which
Sections~\ref{app:exact:soft} and~\ref{app:exact:counts} use.

\textbf{The second sum through tables that depend on $e$ only.} For an integer $t\ge0$, a triple
$n=(n_0,n_1,n_2)$ and a total weight $\omega$ with $\omega+t\,a_i\ge1$, write
\begin{equation}
\sigma^{(t)}_i(n,\omega)=\frac{n_2+2t\,r_in_1+t^2r_i^2n_0}{(\omega+t\,a_i)^2}.
\label{app:exact:eq:sigmat}
\end{equation}
If $n$ holds the three numbers of a group $\mathcal{G}$ with total weight $\omega$, then
\begin{equation}
\sigma^{(t)}_i(n,\omega)=\sum_{C\in\mathcal{G}}\frac{(R(C)+t\,r_i)^2}{(A(C)+t\,a_i)^2},
\label{app:exact:eq:sigmatsum}
\end{equation}
by the computation in the proof of Lemma~\ref{app:exact:lem:group} with $t\,a_i$ and $t\,r_i$ in place of $a_j$
and $r_j$. This is the sum over the group after $t$ copies of $z_i$ are added to the two totals of each set. In
particular $\sigma^{(1)}_i=\sigma_i$ and $\sigma^{(0)}_i(n,\omega)=n_2/\omega^2$, and
$\sigma^{(t)}_i(n,\omega)$ is linear in $n$.

Next, for $s\ge0$ let $Q^{+}_e[s,\omega]$ be the three numbers of the family of the sets $C\cup\{z_e\}$ with total
weight $\omega$, where $C$ is an $s$-subset of $D_{<e}$. By the proof of Lemma~\ref{app:exact:lem:update},
\begin{equation}
Q^{+}_e[s,\omega]=\bigl(n_0,\;n_1+r_en_0,\;n_2+2r_en_1+r_e^2n_0\bigr)
\quad\text{with }(n_0,n_1,n_2)=Q_{D_{<e}}[s,\omega-a_e].
\label{app:exact:eq:Qplus}
\end{equation}
The tables $Q_{D_{<e}}$ and $Q^{+}_e$ depend on $e$ only. Let
\begin{equation}
\bar Q_i[s,\omega]=\sum_{e>\max(i,K)}c_e(K)\,Q_{D_{<e}}[s,\omega],
\qquad
\bar Q^{+}_i[s,\omega]=\sum_{e>\max(i,K)}c_e(K)\,Q^{+}_e[s,\omega]
\label{app:exact:eq:Qbar}
\end{equation}
be their sums over the points that $z_i$ can replace, taken entry by entry and component by component.

\begin{proposition}[The second sum from shared tables]
\label{app:exact:prop:suffix}
For every $i$, the second sum of~\eqref{app:exact:eq:phisq} equals
\begin{equation}
\sum_{t=0}^{K-1}(-1)^t\sum_{\omega}\Bigl[\sigma^{(t)}_i\bigl(\bar Q^{+}_i[K-1-t,\omega],\,\omega\bigr)
-\sigma^{(t+1)}_i\bigl(\bar Q_i[K-1-t,\omega],\,\omega\bigr)\Bigr],
\label{app:exact:eq:suffix}
\end{equation}
where the sums over $\omega$ run over the nonzero entries.
\end{proposition}

\begin{proof}
Fix $e>\max(i,K)$, and let $j$ be $e$ or $i$. As in the proof of Proposition~\ref{app:exact:prop:phisq},
$\sum_{\omega}\sigma_j(n',\omega)$ is the sum of $g_j(A(C),R(C))$ over the $(K-1)$-subsets $C$ of
$D_{<e}\setminus\{z_i\}$, with the function $g_j(\omega,\rho)=(\rho+r_j)^2/(\omega+a_j)^2$. Since $i<e$, the point
$z_i$ lies in $D_{<e}$, and Lemma~\ref{app:exact:lem:unrolled} with $G=D_{<e}$ and $s=K-1$ gives
\begin{equation*}
\sum_{\omega}\sigma_j(n',\omega)
=\sum_{t=0}^{K-1}(-1)^t\sum_{\substack{C\subseteq D_{<e}\\|C|=K-1-t}}
\frac{\bigl(R(C)+r_j+t\,r_i\bigr)^2}{\bigl(A(C)+a_j+t\,a_i\bigr)^2}.
\end{equation*}
For $j=i$, the inner sum runs over the $(K-1-t)$-subsets of $D_{<e}$ with $t+1$ copies of $z_i$ added. Grouping
them by total weight and using~\eqref{app:exact:eq:sigmatsum} turns it into
$\sum_{\omega}\sigma^{(t+1)}_i(Q_{D_{<e}}[K-1-t,\omega],\omega)$. For $j=e$ we have $z_e\notin C$, so
$R(C)+r_e=R(C\cup\{z_e\})$ and $A(C)+a_e=A(C\cup\{z_e\})$. The inner sum then runs over the sets $C\cup\{z_e\}$
with $t$ copies of $z_i$ added, and grouping these sets by total weight turns it into
$\sum_{\omega}\sigma^{(t)}_i(Q^{+}_e[K-1-t,\omega],\omega)$. All divisors are positive, because $(t+1)\,a_i\ge1$
and a set that contains $z_e$ has a total weight of at least $1$. Hence
\begin{equation*}
\sum_{\omega}\bigl[\sigma_e(n',\omega)-\sigma_i(n',\omega)\bigr]
=\sum_{t=0}^{K-1}(-1)^t\sum_{\omega}\Bigl[\sigma^{(t)}_i\bigl(Q^{+}_e[K-1-t,\omega],\omega\bigr)
-\sigma^{(t+1)}_i\bigl(Q_{D_{<e}}[K-1-t,\omega],\omega\bigr)\Bigr].
\end{equation*}
Multiply by $c_e(K)$ and sum over $e>\max(i,K)$. Since $\sigma^{(t)}_i(n,\omega)$ is linear in $n$, the sum over
$e$ can be taken on the tables, which gives~\eqref{app:exact:eq:suffix}.
\end{proof}

In~\eqref{app:exact:eq:suffix} the point $z_i$ enters through $a_i$, through $r_i$, and through the range
$e>\max(i,K)$ of the sums~\eqref{app:exact:eq:Qbar}. These are suffix sums of terms that depend on $e$ only: for
$i\ge K$, they lose the term of $e=i+1$ from $z_i$ to $z_{i+1}$.

\textbf{The algorithm.} Algorithm~\ref{app:exact:alg:linear} makes two passes over the points in order of distance.
In each pass it maintains the table $Q_{D_{<e}}$ by adding one point at a time. The first pass adds up the terms
$c_e(K)\,Q_{D_{<e}}$ and $c_e(K)\,Q^{+}_e$ of~\eqref{app:exact:eq:Qbar} over all $e>K$, and it ends with the table
of all points. The second pass adds up the same terms again as it goes. When it reaches $z_i$, it holds their sums
over $K<e\le i$, and the differences from the totals of the first pass are $\bar Q_i$ and $\bar Q^{+}_i$. It then
evaluates~\eqref{app:exact:eq:suffix} for the second sum of~\eqref{app:exact:eq:phisq}, and it takes $z_i$ out of the table
of all points by Lemma~\ref{lem:remove} for the first sum.

\begin{algorithm}[t]
\caption{Exact Shapley values in time linear in $N$ (Theorem~\ref{thm:sq}). Input, output and
\textsc{AddPoint} as in Algorithm~\ref{app:exact:alg:direct}. The tables $Q$, $Q^{\mathrm{all}}$, $\check Q$,
$\bar Q^{\mathrm{tot}}$ and $\bar Q^{\mathrm{pre}}$ have the columns $0\le\omega\le(K-1)W$, and
$\bar Q^{+,\mathrm{tot}}$ and $\bar Q^{+,\mathrm{pre}}$ have the columns $0\le\omega\le KW$.}
\label{app:exact:alg:linear}
\begin{algorithmic}[1]
\State $r_j\gets a_j\,(y_j-y^{(\val)})$ for $j=1,\dots,N$
\State $Q\gets{}$table of the empty set;\quad $\bar Q^{\mathrm{tot}},\bar Q^{+,\mathrm{tot}}\gets0$
  \Comment{pass 1}
\For{$e=1,\dots,N$}\label{app:exact:line:p1loop}
  \State \textbf{if} $e>K$ \textbf{then} \Call{Accumulate}{$\bar Q^{\mathrm{tot}},\bar Q^{+,\mathrm{tot}},Q,z_e$}
    \Comment{here $Q=Q_{D_{<e}}$}
  \State \Call{AddPoint}{$Q,z_e$}
\EndFor
\State $Q^{\mathrm{all}}\gets Q$ \Comment{the table of all points}
\State $Q\gets{}$table of the empty set;\quad $\bar Q^{\mathrm{pre}},\bar Q^{+,\mathrm{pre}}\gets0$
  \Comment{pass 2}
\For{$i=1,\dots,N$}
  \State\label{app:exact:line:prefix} \textbf{if} $i>K$ \textbf{then}
    \Call{Accumulate}{$\bar Q^{\mathrm{pre}},\bar Q^{+,\mathrm{pre}},Q,z_i$}
    \Comment{sums over $K<e\le i$}
  \State\label{app:exact:line:diff} $\bar Q_i\gets\bar Q^{\mathrm{tot}}-\bar Q^{\mathrm{pre}}$;\quad
    $\bar Q^{+}_i\gets\bar Q^{+,\mathrm{tot}}-\bar Q^{+,\mathrm{pre}}$ \Comment{sums over $e>\max(i,K)$}
  \State $\phi_i\gets\frac1N\bigl[(y_{\emptyset}-y^{(\val)})^2-(y_i-y^{(\val)})^2\bigr]$ \Comment{the empty subset}
  \For{$t=0,\dots,K-1$}\Comment{$z_i$ replaces a farther point}
    \State\label{app:exact:line:suffix} $\phi_i\gets\phi_i+(-1)^t\Bigl[\sum_{\omega=1}^{KW}
      \sigma^{(t)}_i\bigl(\bar Q^{+}_i[K-1-t,\omega],\omega\bigr)-\sum_{\omega=0}^{(K-1)W}
      \sigma^{(t+1)}_i\bigl(\bar Q_i[K-1-t,\omega],\omega\bigr)\Bigr]$
  \EndFor
  \State $\check Q[0,\omega]\gets Q^{\mathrm{all}}[0,\omega]$ for all $\omega$
    \Comment{$\check Q$ becomes the table of $D\setminus\{z_i\}$}
  \For{$s=1,\dots,\min(K,N)-1$}\Comment{$z_i$ joins $s$ neighbors}
    \State\label{app:exact:line:remove} $\check Q[s,\omega]\gets Q^{\mathrm{all}}[s,\omega]
      -\bigl(n_0,\;n_1+r_in_0,\;n_2+2r_in_1+r_i^2n_0\bigr)$ with $(n_0,n_1,n_2)=\check Q[s-1,\omega-a_i]$,
      for all $\omega$
    \State\label{app:exact:line:join2} $\phi_i\gets\phi_i+c_N(s)\sum_{\omega=1}^{(K-1)W}
      \bigl[n_2/\omega^2-\sigma_i(n,\omega)\bigr]$, where $n=(n_0,n_1,n_2)=\check Q[s,\omega]$
  \EndFor
  \State \Call{AddPoint}{$Q,z_i$}
\EndFor
\State \Return $\phi_1,\dots,\phi_N$
\Statex
\Procedure{Accumulate}{$\bar Q,\bar Q^{+},Q,z_e$}\Comment{adds the terms of $z_e$ in~\eqref{app:exact:eq:Qbar}}
  \For{$s=0,\dots,K-1$ and $\omega=0,\dots,(K-1)W$}
    \State $(n_0,n_1,n_2)\gets Q[s,\omega]$
    \State $\bar Q[s,\omega]\gets\bar Q[s,\omega]+c_e(K)\,(n_0,\,n_1,\,n_2)$
    \State $\bar Q^{+}[s,\omega+a_e]\gets\bar Q^{+}[s,\omega+a_e]
      +c_e(K)\,\bigl(n_0,\;n_1+r_en_0,\;n_2+2r_en_1+r_e^2n_0\bigr)$
  \EndFor
\EndProcedure
\end{algorithmic}
\end{algorithm}

\textbf{Correctness.} In both passes the table $Q$ equals $Q_{D_{<e}}$ when the loop reaches the index $e$. This
is Lemma~\ref{app:exact:lem:invariant} without an excluded point, and it has the same proof. A call of
\textsc{Accumulate} for $z_e$ therefore adds $c_e(K)\,Q_{D_{<e}}$ to its first table and, by~\eqref{app:exact:eq:Qplus},
$c_e(K)\,Q^{+}_e$ to its second. After the first pass, $\bar Q^{\mathrm{tot}}$ and $\bar Q^{+,\mathrm{tot}}$ hold the
sums over all $e>K$, and $Q^{\mathrm{all}}$ is the table of $D$. In the second pass, after
line~\ref{app:exact:line:prefix} for the point $z_i$, the tables $\bar Q^{\mathrm{pre}}$ and
$\bar Q^{+,\mathrm{pre}}$ hold the sums over $K<e\le i$. The differences in line~\ref{app:exact:line:diff} are
therefore the sums over $e>\max(i,K)$, which are the tables~\eqref{app:exact:eq:Qbar}. By
Proposition~\ref{app:exact:prop:suffix}, the loop over $t$ adds the second sum of~\eqref{app:exact:eq:phisq}. An entry that is
zero adds zero, because every divisor in the line is positive, so the sums over all stored columns equal the sums
over the nonzero entries. The line skips the column $\omega=0$ of $\bar Q^{+}_i$, which is empty. By
Lemma~\ref{lem:remove} with $G=D$ and by induction on $s$,
row $s$ of $\check Q$ is row $s$ of $Q_{D\setminus\{z_i\}}$ after line~\ref{app:exact:line:remove}, where an entry
of $\check Q$ in a column $\omega-a_i<0$ is read as $(0,0,0)$. Line~\ref{app:exact:line:join2} then adds the term of the first
sum of~\eqref{app:exact:eq:phisq} for this $s$, as line~\ref{app:exact:line:join} of Algorithm~\ref{app:exact:alg:direct}
does. Hence Algorithm~\ref{app:exact:alg:linear} returns the value~\eqref{app:exact:eq:phisq}.

\textbf{Running time and space.} Every table of the algorithm has $K$ rows and at most $KW+1$ columns. In each
pass the algorithm calls \textsc{AddPoint} $N$ times and \textsc{Accumulate} at most $N$ times, and a call takes a
constant number of operations for each entry of a table. In the second pass, for each point $z_i$,
line~\ref{app:exact:line:diff} and the loop over $t$ touch every entry of the four accumulated tables once, and
the loop over $s$ forms and reads at most $K-1$ rows. Each of the $2N$ steps of the two passes thus takes at most
a constant times $K(KW+1)$ operations, which is
\begin{equation*}
O\bigl(N\,K\,(KW+1)\bigr)=O(WK^2N)
\end{equation*}
operations in all. The numbers $r_j$ and the coefficients add $O(N+K)$ operations. The algorithm stores nine
tables, the numbers $r_j$, the coefficients and the values, which is $O(WK^2+N)$ numbers. This proves
Theorem~\ref{thm:sq}.

\textbf{Boundary cases.} If $K\ge N$, no index exceeds $K$. The accumulated tables stay zero, the loop over $t$
adds zero, and the loop over $s$ removes $z_i$ from the rows $s\le N-1$. If $K=1$, every table has the single row
$s=0$. The table $Q^{+}_e$ holds $(1,r_e,r_e^2)$ at $\omega=a_e$, the loop over $t$ has the single term $t=0$, and
the algorithm returns~\eqref{app:exact:eq:Kone}.

\subsection{Exact Arithmetic and Rational Targets}
\label{app:exact:rational}

Algorithms~\ref{app:exact:alg:direct} and~\ref{app:exact:alg:linear} apply only additions, subtractions,
multiplications and divisions to their inputs, and they divide by positive integers only. In exact real arithmetic,
where each of these operations is exact and costs one unit, their output is therefore the value $\phi_i$ itself.
The targets may be arbitrary real numbers, and no input is rounded. If the targets are rational, the same
operations on fractions return $\phi_i$ as an exact rational number. The next proposition bounds its length, and
the running time of both algorithms on fractions.

\begin{proposition}[Rational targets]
\label{app:exact:prop:rational}
Let $y_1,\dots,y_N$, $y^{(\val)}$ and $y_{\emptyset}$ be fractions with a common denominator $q\ge1$ and with
absolute values at most $2^{\kappa}$, and put $\mathcal{L}=\operatorname{lcm}(1,2,\dots,KW)$. Then $N!\,q^2\mathcal{L}^2\,\phi_i$ is an
integer for every $i$, and the numerator and the denominator of $\phi_i$ have $O(N\log N+\kappa+\log q+KW)$ bits.
Algorithms~\ref{app:exact:alg:direct} and~\ref{app:exact:alg:linear}, run on fractions in lowest terms, compute all
$N$ values in time polynomial in $N$, $K$, $W$, $\kappa$ and $\log q$.
\end{proposition}

\begin{proof}
\emph{Denominators.} Every $r_j=a_j(y_j-y^{(\val)})$ is an integer divided by $q$, and so is $R(C)$ for every set
$C$. In every entry of a table, $n_0$, $q\,n_1$ and $q^2n_2$ are therefore integers. So is
$q^2(n_2+2r_jn_1+r_j^2n_0)$, which by Lemma~\ref{app:exact:lem:group} is $q^2$ times a sum of squares
$(R(C)+r_j)^2$. A term of~\eqref{app:exact:eq:phisq} is a coefficient times such a number, divided by $q^2\omega^2$ or by
$q^2(\omega+a_j)^2$. The coefficients $c_N(s)$ and $c_e(K)$ are integers divided by $N!$, because $e!$ divides
$N!$. The divisors $\omega$ and $\omega+a_j$ are integers between $1$ and $(K-1)W+W=KW$, so their squares divide
$\mathcal{L}^2$. The first term of~\eqref{app:exact:eq:phisq} is an integer divided by $Nq^2$. Hence every term of~\eqref{app:exact:eq:phisq}, and
with it $\phi_i$, is an integer divided by $N!\,q^2\mathcal{L}^2$.

\emph{Length.} Put $\bar Y=\max\bigl(\max_j|y_j-y^{(\val)}|,\,|y_{\emptyset}-y^{(\val)}|\bigr)\le2^{\kappa+1}$. The
prediction for a nonempty subset is a weighted average of targets, so $-\bar Y^2\le v(S)\le0$ for every $S$, and
every marginal contribution lies between $-\bar Y^2$ and $\bar Y^2$. The coefficients in~\eqref{eq:shapley} are
positive, and they sum to $\sum_{s=0}^{N-1}\binom{N-1}{s}c_N(s)=1$. Hence $|\phi_i|\le\bar Y^2\le2^{2\kappa+2}$. By
the first step, $\phi_i=h/(N!\,q^2\mathcal{L}^2)$ for an integer $h$ with $|h|\le2^{2\kappa+2}N!\,q^2\mathcal{L}^2$. Here
$\log_2N!\le N\log_2N$, and $\log_2\mathcal{L}=O(KW)$ by Chebyshev's estimate $\log\operatorname{lcm}(1,\dots,n)=O(n)$. So $h$
and $N!\,q^2\mathcal{L}^2$ have $O(N\log N+\kappa+\log q+KW)$ bits, and reducing the fraction does not lengthen them.

\emph{Running time of Algorithm~\ref{app:exact:alg:direct}.} We bound the length of every number that the
algorithm forms. A group has at most $2^N$ sets, and each of them has
$|q\,R(C)|\le q\,A(C)\,\bar Y\le q\,KW\,2^{\kappa+1}$. The integers $n_0$, $q\,n_1$ and $q^2n_2$ of an entry are
thus at most $2^N(q\,KW\,2^{\kappa+1})^2$ in absolute value, and the products formed in the update are of the same
length up to a constant. Next, $\sigma_j(n,\omega)$ and $n_2/\omega^2$ are sums of $n_0$ squared prediction
errors, each at most $\bar Y^2$. The sums over $\omega$ in lines~\ref{app:exact:line:replace}
and~\ref{app:exact:line:join}, and their partial sums, are therefore at most $2^N\bar Y^2$ in absolute value, and
they are integers divided by $q^2\mathcal{L}^2$. Finally, the partial sums of $\phi_i$ add the first term
of~\eqref{app:exact:eq:phisq} and at most $N+K$ such sums, each multiplied by a coefficient that is at most $1$ and is an
integer divided by $N!$. They are integers divided by $N!\,q^2\mathcal{L}^2$, of absolute value at most
$(N+K+1)\,2^N\bar Y^2$. All these numbers have $O(N\log N+\kappa+\log q+KW)$ bits. An operation on two fractions
of this length, with the reduction to lowest terms by Euclid's algorithm, takes time polynomial in the length, and
the algorithm performs $O(WK^2N^2)$ operations.

\emph{Running time of Algorithm~\ref{app:exact:alg:linear}.} The tables $Q$ and $\check Q$ are tables of sets of
training points, so their entries are bounded as before. An entry of an accumulated table is a sum of at most $N$
entries of the tables $Q_{D_{<e}}$ or $Q^{+}_e$, each multiplied by a coefficient $c_e(K)$. Its three components,
multiplied by $N!$, by $N!\,q$ and by $N!\,q^2$, are therefore integers, of absolute value at most $N!\,q^2$ times
$N\,2^N(KW\,2^{\kappa+1})^2$, and the same holds for the differences in line~\ref{app:exact:line:diff}. In
line~\ref{app:exact:line:suffix}, a nonzero entry of row $K-1-t$ of $\bar Q^{+}_i$ belongs to sets of $K-t$
points, so $\omega\le(K-t)W$ and $\omega+t\,a_i\le KW$. A nonzero entry of row $K-1-t$ of $\bar Q_i$ belongs to
sets of $K-1-t$ points, so $\omega+(t+1)\,a_i\le KW$ as well. Every term of the line is therefore an integer
divided by $N!\,q^2\mathcal{L}^2$. By~\eqref{app:exact:eq:sigmatsum} it is a sum, over at most $N$ indices $e$ and at most
$2^N$ sets each, of $c_e(K)$ times a squared weighted average of the numbers $y_j-y^{(\val)}$, so it is at most
$N\,2^N\bar Y^2$ in absolute value. The line has at most $2K(KW+1)$ terms. All numbers that the algorithm forms
thus have $O(N\log N+\kappa+\log q+KW)$ bits again, and it performs $O(WK^2N)$ operations.
\end{proof}

If the numerators and the common denominator have at most $\kappa$ bits, the length is
$O(N\log N+\kappa+KW)$. If instead every target is a fraction with its own numerator and denominator of at most
$\kappa$ bits, then $\log_2q\le(N+2)\,\kappa$ and the length is $O(N\log N+N\kappa+KW)$. The term $KW$ comes from
adding fractions with the different denominators $(\omega+a_j)^2$; the entries of the table are much shorter. A
term linear in $W$ cannot be avoided: on the instances of Theorem~\ref{thm:bits} the reduced denominator has at
least $W$ bits.

\subsection{Floating-Point Arithmetic}
\label{app:exact:float}

Algorithm~\ref{app:exact:alg:direct} builds every table by the update~\eqref{eq:update} alone.
Algorithm~\ref{app:exact:alg:linear} also subtracts tables from one another, and in floating-point arithmetic a
subtraction of nearly equal numbers loses accuracy. This section says where that happens, describes the form of
Algorithm~\ref{app:exact:alg:linear} that our implementation uses, and reports how closely it agrees with
Algorithm~\ref{app:exact:alg:direct}. We prove no bound on the rounding error: what this section says about
accuracy is measured.

\begin{table}[t]
\caption{Agreement of Algorithm~\ref{app:exact:alg:linear} with Algorithm~\ref{app:exact:alg:direct} in double
precision: the largest absolute difference between their values, divided by the largest absolute value. Each row
is one random instance, with weights drawn uniformly from $\{1,\dots,W\}$ and standard normal targets. The fourth
column is the form that we implement, with the terms of $e\le2K$ summed directly. The last column uses
Lemma~\ref{lem:remove} for every $e$.}
\label{app:exact:tab:float}
\centering
\begin{tabular}{@{}rrrcc@{}}
\toprule
$K$ & $N$ & $W$ & $2K$ nearest points direct & lemma for every $e$ \\
\midrule
$5$ & $1000$ & $8$ & $1.9\times10^{-14}$ & $5.8\times10^{-14}$ \\
$5$ & $1000$ & $64$ & $5.1\times10^{-15}$ & $2.4\times10^{-14}$ \\
$10$ & $1000$ & $8$ & $1.0\times10^{-14}$ & $1.7\times10^{-12}$ \\
$20$ & $41$ & $8$ & $2.4\times10^{-15}$ & $1.8\times10^{-11}$ \\
$20$ & $60$ & $8$ & $2.1\times10^{-14}$ & $2.5\times10^{-11}$ \\
$20$ & $1000$ & $8$ & $2.1\times10^{-14}$ & $1.1\times10^{-10}$ \\
$30$ & $61$ & $8$ & $1.4\times10^{-14}$ & $1.5\times10^{-8}$ \\
$30$ & $90$ & $8$ & $1.8\times10^{-14}$ & $4.2\times10^{-8}$ \\
$30$ & $1000$ & $8$ & $9.9\times10^{-15}$ & $5.2\times10^{-8}$ \\
$50$ & $101$ & $8$ & $5.8\times10^{-15}$ & $1.2\times10^{-2}$ \\
$50$ & $150$ & $8$ & $3.6\times10^{-14}$ & $1.5\times10^{-2}$ \\
$50$ & $600$ & $8$ & $1.0\times10^{-13}$ & $3.1\times10^{-2}$ \\
\bottomrule
\end{tabular}
\end{table}

\textbf{Where accuracy is lost.} Algorithm~\ref{app:exact:alg:linear} subtracts in three places. The first is
the alternating sum over $t$ in line~\ref{app:exact:line:suffix}. It is the sum~\eqref{app:exact:eq:unrolled} for
$G=D_{<e}$ and $s=K-1$, as Proposition~\ref{app:exact:prop:suffix} uses it. Its left side is a sum over
$\binom{e-2}{K-1}$ sets, and its term $t$ is a sum over $\binom{e-1}{K-1-t}$ sets. Every set contributes a
nonnegative number of the size of a squared prediction error. From one term to the next, the number of sets
changes by the factor
\begin{equation*}
\binom{e-1}{K-2-t}\Big/\binom{e-1}{K-1-t}=\frac{K-1-t}{e-K+1+t}.
\end{equation*}
If $e\ge2K-2$, this factor is at most $1$ for every $t$, and the number of sets does not grow. If $e\le2K-3$, it
grows as long as $t<K-1-e/2$. The extreme case is $e=K+1$: the left side is a single set, and the largest term
has $\binom{K}{\lfloor K/2\rfloor}$ sets, which is about $2^K/\sqrt{K}$. The alternating sum then cancels sums over
that many sets against a result of one set, and double precision loses about $K$ bits. These counts are a guide
and not a bound, because the size of a term also depends on the targets.

The second place is the difference in line~\ref{app:exact:line:diff}. For a point $z_i$ far from the query, the
sums over $e>i$ are a small part of the totals. The difference then has the absolute accuracy of the totals, and
not the relative accuracy of what remains. The third place is the removal in line~\ref{app:exact:line:remove}.
Row $s$ of the table of all points counts $\binom{N}{s}$ sets, and the row that the lemma subtracts from it counts
the $\binom{N-1}{s-1}$ sets that contain $z_i$, which is the fraction $s/N$ of them.

\textbf{The form that we implement.} Our implementation splits the second sum of~\eqref{app:exact:eq:phisq} at $e=2K$. The
terms with $e\le2K$ involve the $2K$ nearest points only. It computes them by
line~\ref{app:exact:line:replace} of Algorithm~\ref{app:exact:alg:direct}, run on these $2K$ points; the
coefficients $c_e(K)$ do not depend on $N$. This takes $O(WK^2\,(2K)^2)=O(WK^4)$ operations. The terms with
$e>2K$ it computes by the two passes of Algorithm~\ref{app:exact:alg:linear}, in which \textsc{Accumulate} is now
called for $e>2K$ only. For these $e$ the factor above is at most $(K-1)/(K+2)$, so the number of sets decreases
from each term of the alternating sum to the next. If $N\le2K$, the implementation runs
Algorithm~\ref{app:exact:alg:direct} on all points. In exact arithmetic this form computes the same
value~\eqref{app:exact:eq:phisq}, and its running time is $O(WK^2N+WK^4)$. The removal in
line~\ref{app:exact:line:remove} is then used only when $N>2K$, where the fraction $s/N$ is below one half for
every $s\le K-1$. The difference in line~\ref{app:exact:line:diff} is left as it is.

\textbf{Measured agreement.} Table~\ref{app:exact:tab:float} compares both forms of
Algorithm~\ref{app:exact:alg:linear} with Algorithm~\ref{app:exact:alg:direct} in double precision. Its measure is
the largest difference divided by the largest value. It therefore reports errors relative to the largest value of
an instance, and not the relative error of a value that is much smaller. With the $2K$ nearest points summed
directly, the difference is at most $1.0\times10^{-13}$ of the largest value in every row, up to $K=50$. With
Lemma~\ref{lem:remove} used for every $e$, the difference grows with $K$ as described above. It is at most
$1.1\times10^{-10}$ up to $K=20$, and it reaches $5.2\times10^{-8}$ at $K=30$ and $3.1\times10^{-2}$ at $K=50$. Appendix~\ref{app:more:enum} compares both forms with extended precision up to $100{,}000$ points, and
Appendix~\ref{app:more:timing} reports how far the sum of the values misses $v(D)-v(\emptyset)$ at $1{,}000{,}000$
points.

\subsection{Proof of \texorpdfstring{Corollary~\ref{cor:soft}}{the Soft-Label Corollary}}
\label{app:exact:soft}

In this section $m$ denotes the number of classes, and a set of training points is written $X$.

\begin{restate}{Corollary~\ref{cor:soft}}
For weights on $W$ levels and any number of classes, the exact Shapley values of all $N$ training points under
the soft-label utility~\eqref{eq:soft} are computed in $O(WK^2N)$ time after sorting.
\end{restate}

The Brier score, the squared error of the predicted class shares, takes one run of Theorem~\ref{thm:sq} per class,
which is $O(m\,WK^2N)$ time for $m$ classes. Both claims follow from one statement about utilities that are polynomials in
a numerator. We prove this statement first, and then the two claims.

\textbf{Polynomial utilities.} Let $r_1,\dots,r_N$ be real numbers, $R(X)=\sum_{z_j\in X}r_j$ and $d\ge1$. A group
$\mathcal{G}$ now carries the $d+1$ numbers $n_k=\sum_{X\in\mathcal{G}}R(X)^k$, $0\le k\le d$, and $Q_G[s,\omega]$
holds these numbers for the $s$-subsets of $G$ with total weight $\omega$. For a real number $\rho$, the shift of
$n=(n_0,\dots,n_d)$ by $\rho$ is
\begin{equation}
n^{+\rho}=\bigl(n^{+\rho}_0,\dots,n^{+\rho}_d\bigr),\qquad
n^{+\rho}_k=\sum_{l=0}^{k}\binom{k}{l}\,\rho^{\,k-l}\,n_l .
\label{app:exact:eq:shift}
\end{equation}
For $d=2$, the shift by $r_j$ is the triple $n^{+j}$ of~\eqref{eq:shift}.

\begin{lemma}[Shift, update and removal with $d+1$ numbers]
\label{app:exact:lem:shift}
Let $\mathcal{G}$ be a group with the numbers $n$. Then $\sum_{X\in\mathcal{G}}(R(X)+\rho)^k=n^{+\rho}_k$ for every
real $\rho$ and every $k\le d$. Consequently, for $z_j\notin G$ the update is
$Q_{G\cup\{z_j\}}[s,\omega]=Q_G[s,\omega]+n^{+r_j}$ with $n=Q_G[s-1,\omega-a_j]$. For $z_i\in G$ and $s\ge1$ the
removal is $Q_{G\setminus\{z_i\}}[s,\omega]=Q_G[s,\omega]-n^{+r_i}$ with
$n=Q_{G\setminus\{z_i\}}[s-1,\omega-a_i]$, and row $0$ is the same in both tables.
\end{lemma}

\begin{proof}
The first claim is the binomial theorem, $(R(X)+\rho)^k=\sum_{l=0}^{k}\binom{k}{l}\rho^{k-l}R(X)^l$, summed over
the group. The update follows as in the proof of Lemma~\ref{app:exact:lem:update}: the $s$-subsets of
$G\cup\{z_j\}$ that contain $z_j$ are the sets $X\cup\{z_j\}$ with $X$ an $(s-1)$-subset of $G$, and
$R(X\cup\{z_j\})=R(X)+r_j$. The removal is the update for the set $G\setminus\{z_i\}$ and the point $z_i$, solved
for the smaller table as in the proof of Lemma~\ref{lem:remove}.
\end{proof}

\begin{proposition}[Utilities that are polynomials in the numerator]
\label{app:exact:prop:poly}
Consider weights on $W$ levels, and let $d\ge1$. Suppose that for every nonempty subset $S$
\begin{equation}
v(S)=\sum_{k=0}^{d}\gamma_k\bigl(A(\NK(S))\bigr)\,R(\NK(S))^k ,
\label{app:exact:eq:polyutil}
\end{equation}
where $\gamma_0,\dots,\gamma_d$ are real functions on $\{1,\dots,KW\}$, and let $v(\emptyset)$ be any number. Then
all $N$ Shapley values are computed exactly in $O(d^2\,WK^2N)$ time after sorting, by
Algorithm~\ref{app:exact:alg:linear} with the $d+1$ numbers per group. An evaluation of a function $\gamma_k$
counts as one operation.
\end{proposition}

\begin{proof}
For a group with the numbers $n$ and the total weight $\omega$, and for an integer $t\ge0$, write
\begin{equation}
\tau^{(t)}_i(n,\omega)=\sum_{k=0}^{d}\gamma_k(\omega+t\,a_i)\,n^{+t\,r_i}_k .
\label{app:exact:eq:taut}
\end{equation}
By Lemma~\ref{app:exact:lem:shift}, this is the sum of $\sum_k\gamma_k(A(X)+t\,a_i)\,(R(X)+t\,r_i)^k$ over the
sets $X$ of the group: the summed utility after $t$ copies of $z_i$ are added to the two totals of each set. It
takes the place of $-\sigma^{(t)}_i(n,\omega)$, and it is linear in $n$.

The utility has the property~\eqref{app:exact:eq:nk}, so Proposition~\ref{app:exact:prop:twosums} applies. Sets of
at most $K$ points are their own $K$ nearest neighbors. Grouping the sets in the two sums
of~\eqref{app:exact:eq:twosums} by their total weight, as in the proof of Proposition~\ref{app:exact:prop:phisq}, gives
\begin{align}
\phi_i={}&\tfrac1N\bigl[v(\{z_i\})-v(\emptyset)\bigr]
+\sum_{s=1}^{K-1}c_N(s)\sum_{\omega}\bigl[\tau^{(1)}_i(n,\omega)-\tau^{(0)}_i(n,\omega)\bigr]\notag\\
&+\sum_{e>\max(i,K)}c_e(K)\sum_{\omega}\bigl[\tau^{(1)}_i(n',\omega)-\tau^{(1)}_e(n',\omega)\bigr],
\label{app:exact:eq:template}
\end{align}
where $n$ and $n'$ are the entries of the same tables as in~\eqref{app:exact:eq:phisq}. The first sum is computed as in
Algorithm~\ref{app:exact:alg:linear}, with the removal of Lemma~\ref{app:exact:lem:shift} in
line~\ref{app:exact:line:remove}. For the second sum, define $\bar Q_i$ and $\bar Q^{+}_i$
by~\eqref{app:exact:eq:Qbar}, where $Q^{+}_e[s,\omega]$ is now the shift $n^{+r_e}$ of
$n=Q_{D_{<e}}[s,\omega-a_e]$. Lemma~\ref{app:exact:lem:unrolled} and the proof of
Proposition~\ref{app:exact:prop:suffix}, with $\tau$ in place of $-\sigma$, give
\begin{align}
&\sum_{e>\max(i,K)}c_e(K)\sum_{\omega}\bigl[\tau^{(1)}_i(n',\omega)-\tau^{(1)}_e(n',\omega)\bigr]\notag\\
&\qquad=\sum_{t=0}^{K-1}(-1)^t\sum_{\omega}\Bigl[\tau^{(t+1)}_i\bigl(\bar Q_i[K-1-t,\omega],\,\omega\bigr)
-\tau^{(t)}_i\bigl(\bar Q^{+}_i[K-1-t,\omega],\,\omega\bigr)\Bigr].
\label{app:exact:eq:templatesuffix}
\end{align}
All sums over $\omega$ are restricted to the columns in which a row can be nonzero: $s\le\omega\le sW$ in row $s$ of
a table of sets and of $\bar Q_i$, and $s+1\le\omega\le(s+1)W$ in row $s$ of $\bar Q^{+}_i$, whose sets contain
$z_e$ as well. In row $K-1-t$ this gives $K\le\omega+(t+1)\,a_i\le KW$ for $\bar Q_i$ and
$K\le\omega+t\,a_i\le KW$ for $\bar Q^{+}_i$, so the functions $\gamma_k$ are evaluated between $1$ and $KW$ only.
Algorithm~\ref{app:exact:alg:linear} computes these expressions with the update and the removal of
Lemma~\ref{app:exact:lem:shift} and with these ranges of $\omega$, and its proof of correctness applies. For the cost, the $O(d^2)$
products $\binom{k}{l}\rho^{k-l}$ are computed once for each shift $\rho$ that is applied to a row. A shifted
entry then takes $O(d^2)$ operations. This multiplies the operation count of Section~\ref{app:exact:sq} by
$O(d^2)$. Only additions, subtractions and multiplications are applied to the inputs and to the values of the
$\gamma_k$, so the result is exact.
\end{proof}

The squared loss is the case $d=2$ with $\gamma_2(\omega)=-1/\omega^2$ and $\gamma_0=\gamma_1=0$. More generally, a
loss that is a polynomial $\sum_{k=0}^{d}\ell_k\,(\yhat-y^{(\val)})^k$ in the prediction error, with real
coefficients $\ell_k$, is covered with $r_j=a_j(y_j-y^{(\val)})$ and $\gamma_k(\omega)=-\ell_k/\omega^k$. The direct
algorithm evaluates~\eqref{app:exact:eq:template} as it stands, in $O(d^2\,WK^2N^2)$ time. We checked the near-linear form for
$d=1$ and $d=2$ (Table~\ref{tab:app-enum}); for $d=4$ only a direct form was checked.

\textbf{The weight share.} From here on the targets are classes: each training point has one class
$y_j\in\{1,\dots,m\}$, and $y^{(\val)}$ is the class of the query. Put $r_j=a_j\mathbf{1}[y_j=y^{(\val)}]$, so that
$R(X)$ is the total weight of the points of $X$ that have the class of the query. The utility~\eqref{eq:soft} is $v(S)=R(\NK(S))/A(\NK(S))$ for a
nonempty $S$, and $v(\emptyset)=0$. \citet[Eq.~25]{jia2019knn} write this utility with the weights of the $K$ nearest
neighbors in $S$; \citet[Eq.~1]{wang2024efficient} write its normalization out, and~\eqref{eq:soft} is that form. This is the case $d=1$ of Proposition~\ref{app:exact:prop:poly}, with
$\gamma_1(\omega)=1/\omega$ and $\gamma_0=0$, and with
$\tau^{(t)}_i(n,\omega)=(n_1+t\,r_in_0)/(\omega+t\,a_i)$. The formula~\eqref{app:exact:eq:template} is the analogue
of~\eqref{app:exact:eq:phisq},
\begin{align}
\phi_i={}&\frac{\mathbf{1}[y_i=y^{(\val)}]}{N}
+\sum_{s=1}^{K-1}c_N(s)\sum_{\omega}\Bigl[\frac{n_1+r_in_0}{\omega+a_i}-\frac{n_1}{\omega}\Bigr]\notag\\
&+\sum_{e>\max(i,K)}c_e(K)\sum_{\omega}
\Bigl[\frac{n'_1+r_in'_0}{\omega+a_i}-\frac{n'_1+r_en'_0}{\omega+a_e}\Bigr],
\label{app:exact:eq:share}
\end{align}
where $(n_0,n_1)$ are the first two numbers of $Q_{D\setminus\{z_i\}}[s,\omega]$ and $(n'_0,n'_1)$ those of
$Q_{D_{<e}\setminus\{z_i\}}[K-1,\omega]$. The first term is $\frac1N[v(\{z_i\})-v(\emptyset)]$ with
$v(\{z_i\})=r_i/a_i$. The two numbers $(n_0,n_1)$ are those named in Section~\ref{sec:other}, and their update
and removal are the first two components of~\eqref{eq:update} and~\eqref{eq:remove}. By the proposition the time is $O(WK^2N)$ after
sorting. The classes enter only through the indicator $\mathbf{1}[y_j=y^{(\val)}]$, so the time does not depend on
their number. This proves the first claim.

\textbf{The Brier score.} For a class $c$ let $R_c(X)=\sum_{z_j\in X}a_j\mathbf{1}[y_j=c]$ be the total weight of
class $c$ in $X$. The soft-label classifier predicts for class $c$ the weight share
$p_c(S)=R_c(\NK(S))/A(\NK(S))$ if $S$ is nonempty, and a fixed default $p_c(\emptyset)$ otherwise. The Brier score
compares the shares with the indicators $\mathbf{1}[c=y^{(\val)}]$ of the correct class, and the utility is its
negative:
\begin{equation}
v(S)=-\sum_{c=1}^{m}\bigl(p_c(S)-\mathbf{1}[c=y^{(\val)}]\bigr)^2 .
\label{app:exact:eq:brier}
\end{equation}
Write $v=\sum_{c=1}^{m}v^{(c)}$ with $v^{(c)}(S)=-(p_c(S)-\mathbf{1}[c=y^{(\val)}])^2$. The share $p_c(S)$ is the
prediction~\eqref{eq:pred} of weighted KNN regression with the targets $\mathbf{1}[y_j=c]$ and the default
prediction $p_c(\emptyset)$. So $v^{(c)}$ is the squared-loss utility~\eqref{eq:utility} of this regression problem
when the target of the query is $\mathbf{1}[c=y^{(\val)}]$. By Theorem~\ref{thm:sq}, its Shapley values
$\phi^{(c)}_i$ are computed in $O(WK^2N)$ time after sorting, and Algorithm~\ref{app:exact:alg:linear} does so with
the numbers $r_j=a_j(\mathbf{1}[y_j=c]-\mathbf{1}[c=y^{(\val)}])$. The Shapley value~\eqref{eq:shapley} is linear in
the utility, so $\phi_i=\sum_{c=1}^{m}\phi^{(c)}_i$. The $m$ runs take $O(m\,WK^2N)$ time and the space of one run.
This proves the second claim.

In terms of the class weights, the term of class $c$ is a quadratic in the one numerator $R_c$. For a nonempty $S$
with $X=\NK(S)$ and $\omega=A(X)$ it equals $-(R_c(X)-\mathbf{1}[c=y^{(\val)}]\,\omega)^2/\omega^2$, and since
exactly one of the indicators is $1$,
\begin{equation}
v(S)=-\frac{1}{\omega^2}\sum_{c=1}^{m}R_c(X)^2+\frac{2\,R_{y^{(\val)}}(X)}{\omega}-1 .
\label{app:exact:eq:brierexp}
\end{equation}
The Brier utility is thus a sum over the classes of quadratics in one numerator each, plus a constant. No product
of two different numerators occurs, which is why the cost is linear in $m$.

By~\eqref{app:exact:eq:brierexp}, the $m$ runs can be merged into one run with $m+2$ numbers per group. For a group
$\mathcal{G}$ with total weight $\omega$ these are
\begin{equation*}
n_0=|\mathcal{G}|,\qquad
n_{1,c}=\sum_{X\in\mathcal{G}}R_c(X)\quad(c=1,\dots,m),\qquad
n_2=\sum_{X\in\mathcal{G}}\,\sum_{c=1}^{m}R_c(X)^2 .
\end{equation*}
Adding a point $z_j$ to a set $X$ changes the numerator of its own class only:
$R_c(X\cup\{z_j\})=R_c(X)+a_j\mathbf{1}[c=y_j]$, and hence
$\sum_cR_c(X\cup\{z_j\})^2=\sum_cR_c(X)^2+2a_jR_{y_j}(X)+a_j^2$. In the update and in the removal, the numbers that
are added or subtracted are therefore
\begin{equation}
\bigl(n_0,\;\;n_{1,c}+a_j\mathbf{1}[c=y_j]\,n_0\;\;(c=1,\dots,m),\;\;n_2+2a_jn_{1,y_j}+a_j^2n_0\bigr),
\label{app:exact:eq:brierupdate}
\end{equation}
where $(n_0,n_{1,1},\dots,n_{1,m},n_2)$ is the entry that~\eqref{eq:update} and~\eqref{eq:remove} name.
By~\eqref{app:exact:eq:brierexp}, the summed utility of a group after $t$ copies of $z_i$ are added to the totals
of each set is
\begin{equation}
\tau^{(t)}_i(n,\omega)=-\frac{n_2+2t\,a_in_{1,y_i}+t^2a_i^2n_0}{(\omega+t\,a_i)^2}
+\frac{2\bigl(n_{1,y^{(\val)}}+t\,a_i\mathbf{1}[y_i=y^{(\val)}]\,n_0\bigr)}{\omega+t\,a_i}-n_0 .
\label{app:exact:eq:briertau}
\end{equation}
With these numbers, \eqref{app:exact:eq:template} and~\eqref{app:exact:eq:templatesuffix} hold by the proof of
Proposition~\ref{app:exact:prop:poly}. An update changes $m+2$ numbers per entry and a read uses at most five, so
this run also takes $O(m\,WK^2N)$ time. It needs $O(m\,WK^2+N)$ space. Our implementation makes the $m$ runs instead.

\subsection{Tables of Counts}
\label{app:exact:counts}

Corollary~\ref{cor:hard} and Theorem~\ref{thm:pair} concern utilities that are not polynomials in a numerator. For
such utilities the table holds counts, and its index holds everything that the utility depends on. This section
states the two algorithms in that form, and the next two sections apply it.

\textbf{Setting.} Give every training point $z_j$ a vector of integers $\theta_j$ of a fixed dimension, and write
$\Theta(C)=\sum_{z_j\in C}\theta_j$ for the totals of a set $C$. Let $g$ be a real function of such totals, and
suppose that
\begin{equation}
v(S)=g\bigl(\Theta(\NK(S))\bigr)\qquad\text{for every nonempty }S ,
\label{app:exact:eq:gsetting}
\end{equation}
while $v(\emptyset)$ is any number. This utility has the property~\eqref{app:exact:eq:nk}. For a set $G$ of
training points, a size $s\ge0$ and a vector $\vartheta$, let $T_G[s,\vartheta]$ be the number of $s$-subsets $C$
of $G$ with $\Theta(C)=\vartheta$. An entry with $s<0$ is read as $0$.

\begin{lemma}[Update and removal of counts]
\label{app:exact:lem:countupdate}
The table of the empty set has $T_\emptyset[0,0]=1$ and $0$ in every other entry. For a set $G$ of training
points, a point $z_j\notin G$ and all $s\ge0$ and $\vartheta$,
\begin{equation}
T_{G\cup\{z_j\}}[s,\vartheta]=T_G[s,\vartheta]+T_G[s-1,\;\vartheta-\theta_j].
\label{app:exact:eq:countupdate}
\end{equation}
For $z_i\in G$ and $s\ge1$, $T_{G\setminus\{z_i\}}[s,\vartheta]=T_G[s,\vartheta]-T_{G\setminus\{z_i\}}[s-1,\vartheta-\theta_i]$,
and row $0$ is the same in both tables.
\end{lemma}

\begin{proof}
The empty set has one subset, of size $0$ and with the totals $0$. For the update, an $s$-subset of
$G\cup\{z_j\}$ with the totals $\vartheta$ either avoids $z_j$, and is then one of the $T_G[s,\vartheta]$ such
subsets of $G$, or it is $C\cup\{z_j\}$ for an $(s-1)$-subset $C$ of $G$. In the second case
$\Theta(C\cup\{z_j\})=\Theta(C)+\theta_j$, so $C$ has the totals $\vartheta-\theta_j$, and $C\mapsto C\cup\{z_j\}$
is one-to-one. The removal is the update for the set $G\setminus\{z_i\}$ and the point $z_i$, solved for the
smaller table.
\end{proof}

As in~\eqref{app:exact:eq:Qbar}, let
\begin{equation}
\bar T_i[s,\vartheta]=\sum_{e>\max(i,K)}c_e(K)\,T_{D_{<e}}[s,\vartheta],
\qquad
\bar T^{+}_i[s,\vartheta]=\sum_{e>\max(i,K)}c_e(K)\,T_{D_{<e}}[s,\vartheta-\theta_e].
\label{app:exact:eq:Tbar}
\end{equation}
The second table counts the sets $C\cup\{z_e\}$ by their totals, where $C$ is an $s$-subset of $D_{<e}$.

\begin{proposition}[Two passes on a table of counts]
\label{app:exact:prop:counts}
In the setting~\eqref{app:exact:eq:gsetting}, for every $i$,
\begin{align}
\phi_i={}&\tfrac1N\bigl[g(\theta_i)-v(\emptyset)\bigr]
+\sum_{s=1}^{K-1}c_N(s)\sum_{\vartheta}T_{D\setminus\{z_i\}}[s,\vartheta]\,
\bigl[g(\vartheta+\theta_i)-g(\vartheta)\bigr]\notag\\
&+\sum_{t=0}^{K-1}(-1)^t\sum_{\vartheta}\Bigl[\bar T_i[K-1-t,\vartheta]\;g\bigl(\vartheta+(t+1)\,\theta_i\bigr)
-\bar T^{+}_i[K-1-t,\vartheta]\;g\bigl(\vartheta+t\,\theta_i\bigr)\Bigr],
\label{app:exact:eq:countphi}
\end{align}
where the sums over $\vartheta$ run over the nonzero entries. All $N$ values are computed exactly in
$O(N\,K\,P_\Theta)$ operations after sorting, where $P_\Theta$ is the number of distinct totals $\Theta(C)$ over
sets $C$ of at most $K$ points, and an evaluation of $g$ counts as one operation.
\end{proposition}

\begin{proof}
\emph{The formula.} By Proposition~\ref{app:exact:prop:twosums}, $\phi_i$ is the sum of two sums. In the first,
the empty subset gives $\frac1N[g(\theta_i)-v(\emptyset)]$. An $s$-subset $S$ of $D\setminus\{z_i\}$ with
$1\le s\le K-1$ and with the totals $\vartheta$ satisfies $\NK(S)=S$ and $\NK(S\cup\{z_i\})=S\cup\{z_i\}$, so its
marginal contribution is $g(\vartheta+\theta_i)-g(\vartheta)$, and there are $T_{D\setminus\{z_i\}}[s,\vartheta]$
such subsets. This gives the first line of~\eqref{app:exact:eq:countphi}. In the second sum, a $(K-1)$-subset $C$
of $D_{<e}\setminus\{z_i\}$ contributes $g(\Theta(C)+\theta_i)-g(\Theta(C)+\theta_e)$. By
Lemma~\ref{app:exact:lem:unrolled} for the totals $\Theta$, with $G=D_{<e}$ and $s=K-1$, for $j=i$ and for $j=e$,
\begin{equation*}
\sum_{\substack{C\subseteq D_{<e}\setminus\{z_i\}\\|C|=K-1}}g\bigl(\Theta(C)+\theta_j\bigr)
=\sum_{t=0}^{K-1}(-1)^t\sum_{\substack{C\subseteq D_{<e}\\|C|=K-1-t}}g\bigl(\Theta(C)+\theta_j+t\,\theta_i\bigr).
\end{equation*}
For $j=i$ the inner sum is $\sum_{\vartheta}T_{D_{<e}}[K-1-t,\vartheta]\,g(\vartheta+(t+1)\theta_i)$. For $j=e$ it
is $\sum_{\vartheta}T_{D_{<e}}[K-1-t,\vartheta-\theta_e]\,g(\vartheta+t\,\theta_i)$, with $\vartheta$ the totals of
$C\cup\{z_e\}$. Multiplying by $c_e(K)$ and summing over $e>\max(i,K)$ gives the second line
of~\eqref{app:exact:eq:countphi}.

\emph{The algorithm.} It is Algorithm~\ref{app:exact:alg:linear} with the count in place of the three numbers.
Both passes maintain the table $T_{D_{<e}}$ by~\eqref{app:exact:eq:countupdate}. The first pass adds up the terms
of~\eqref{app:exact:eq:Tbar} over all $e>K$ and ends with the table $T_D$. The second pass adds them up again,
forms $\bar T_i$ and $\bar T^{+}_i$ as differences, evaluates the second line of~\eqref{app:exact:eq:countphi},
and obtains the rows of $T_{D\setminus\{z_i\}}$ from $T_D$ by the removal of
Lemma~\ref{app:exact:lem:countupdate} for the first line. The proof of correctness of
Algorithm~\ref{app:exact:alg:linear} applies with Lemma~\ref{app:exact:lem:countupdate} in place of
Lemmas~\ref{app:exact:lem:update} and~\ref{lem:remove}. The counts are integers, and the result is exact once $g$
is evaluated exactly.

\emph{The cost.} A row is stored as the list of its nonzero entries, each with its totals $\vartheta$, sorted by
$\vartheta$ in lexicographic order. Shifting all totals of a list by the same vector keeps it sorted. Adding a
multiple of one row to another, or subtracting a shifted row, is therefore a merge of two sorted lists and takes
time linear in their lengths. Every set counted in a row of $T_{D_{<e}}$, $T_D$ or $T_{D\setminus\{z_i\}}$ has at
most $K-1$ points, and every set counted in a row of $\bar T^{+}_i$ has at most $K$ points. So every row has at
most $P_\Theta$ entries. For each point, the two passes add the point, accumulate its terms, form the differences,
evaluate the second line of~\eqref{app:exact:eq:countphi} and remove the point from $T_D$. Each of these steps
merges or reads at most $2K$ rows. This is $O(K\,P_\Theta)$ operations per point and $O(N\,K\,P_\Theta)$ in all.
\end{proof}

The direct algorithm, with one table of counts per point, evaluates the two sums of
Proposition~\ref{app:exact:prop:twosums} as they stand and takes $O(N^2K\,P_\Theta)$ operations.

\subsection{Proof of \texorpdfstring{Corollary~\ref{cor:hard}}{the Hard-Label Corollary}}
\label{app:exact:hard}

\begin{restate}{Corollary~\ref{cor:hard}}
For the hard-label classifier with two classes and weights on $W$ levels, the exact Shapley values of all $N$
training points are computed in $O(WK^2N)$ time after sorting, instead of $O(WK^2N^2)$.
\end{restate}

\textbf{Setting.} The classifier has two classes. As \citet{wang2024efficient} do in their Eq.~(3), give $z_j$ the
signed weight $\tilde a_j=a_j$ if $y_j=y^{(\val)}$ and $\tilde a_j=-a_j$ otherwise, and let
$\tilde A(C)=\sum_{z_j\in C}\tilde a_j$ be the signed total weight of a set. The classifier is correct on the query
when, among the $K$ nearest neighbors, the class of the query has at least the weight of the other class. Their
utility is therefore
\begin{equation}
v(S)=g\bigl(\tilde A(\NK(S))\bigr),\qquad g(\vartheta)=\mathbf{1}[\vartheta\ge0].
\label{app:exact:eq:hard}
\end{equation}
A tie counts as correct, and the empty set, whose signed total is $0$, has the utility $1$.
\citet{zhang2026banzhafknn} use the other convention, in which a tie counts as wrong and the empty set has the
utility $0$; this is $g(\vartheta)=\mathbf{1}[\vartheta>0]$. The proof uses neither the form of $g$ nor the value
of $v(\emptyset)$, so it covers both conventions.

\begin{proof}
The utility~\eqref{app:exact:eq:hard} is of the form~\eqref{app:exact:eq:gsetting} with the single total
$\theta_j=\tilde a_j$. A set of at most $K$ points has a signed total weight between $-KW$ and $KW$, so
$P_\Theta\le2KW+1$. By Proposition~\ref{app:exact:prop:counts}, all $N$ values are computed in
$O(N\,K\,(2KW+1))=O(WK^2N)$ operations after sorting. The table has $K$ rows and at most $2KW+1$ columns, and it
can be stored as an array. The running time $O(WK^2N^2)$ that the corollary replaces is Theorem~9 of
\citet{wang2024efficient}.
\end{proof}

The proof does not use that the weights are positive, so a weight level $0$ may be included.

\textbf{More than two classes.} The corollary is for two classes. With more classes, whether the classifier is
correct depends on the total weight of every class among the $K$ nearest neighbors, and the setting of
Section~\ref{app:exact:counts} applies with one total per class. The number $P_\Theta$ of distinct vectors of
totals then grows with the number of classes; with $m$ classes it is at most $(KW+1)^{m}$.
\citet{wang2024efficient} observe the same growth for the direct extension of their algorithm to several classes,
$O(K^{1+m}N^2W^{m})$ in their Appendix~E.1, and change the utility for this reason. We did not implement this
case.

\subsection{Other Losses: Counting by the Pair of Totals}
\label{app:exact:pair}

This section states and proves the result of Section~\ref{sec:pair} for losses that are not polynomials in the
prediction. The weights are on $W$ levels as before, and the targets are on a grid as well: $y_j=b_j\,\delta_y$
with integers $b_j$ and a target unit $\delta_y>0$. For a set $C$ of training points,
$M(C)=\sum_{z_j\in C}a_jb_j$ is its total weighted target in units of the two grids.

\begin{theorem}[Other losses]
\label{thm:pair}
For discretized weights and targets and any loss of the prediction, all $N$ Shapley values are computed exactly in
$O(NKP)$ time after sorting, where $P$ is the number of distinct pairs $(A(C),M(C))$ over sets $C$ of at most $K$
points.
\end{theorem}

\textbf{Setting.} The weight unit cancels in~\eqref{eq:pred} and the target unit factors out: for a nonempty $S$,
\begin{equation}
\yhat(S)=\frac{\delta_w\,\delta_y\sum_{z_j\in\NK(S)}a_jb_j}{\delta_w\sum_{z_j\in\NK(S)}a_j}
=\delta_y\,\frac{M(\NK(S))}{A(\NK(S))} .
\label{app:exact:eq:pairpred}
\end{equation}
Let the loss $\ell$ be any real function of the prediction and of $y^{(\val)}$, and define
\begin{equation}
u(\omega,\mu)=-\ell\bigl(\delta_y\,\mu/\omega,\;y^{(\val)}\bigr)
\qquad\text{for integers }\omega\ge1\text{ and }\mu .
\label{app:exact:eq:pairu}
\end{equation}
By~\eqref{app:exact:eq:pairpred}, $v(S)=u\bigl(A(\NK(S)),M(\NK(S))\bigr)$ for a nonempty $S$, and
$v(\emptyset)=-\ell(y_{\emptyset},y^{(\val)})$. So the pair $(A,M)$ of the $K$ nearest neighbors determines the
utility. Nothing else about $\ell$ is used, and $y^{(\val)}$ need not lie on the target grid.

\begin{proof}
The utility is of the form~\eqref{app:exact:eq:gsetting} with the totals $\theta_j=(a_j,a_jb_j)$, for which
$\Theta(C)=(A(C),M(C))$, and with $g=u$. The table $T_G[s,\omega,\mu]$ counts the $s$-subsets of $G$ with the pair
$(\omega,\mu)$. The formula~\eqref{app:exact:eq:countphi} evaluates $u$ at pairs with a positive first component
only: in its first line $\omega\ge1$ because $s\ge1$, and in its second line the first component is the total
weight of a set that contains $z_e$ or $z_i$. The number $P_\Theta$ of Proposition~\ref{app:exact:prop:counts} is
the number $P$ of the theorem. Hence all $N$ values are computed exactly in $O(NKP)$ operations after sorting,
where an evaluation of $\ell$ counts as one operation. The tables take $O(KP)$ space, and the inputs, the
coefficients and the values take $O(N)$.
\end{proof}

\textbf{The size of $P$.} Let $b_{\min}$ and $b_{\max}$ be the smallest and the largest $b_j$, and let
$H=b_{\max}-b_{\min}+1$ be the number of target levels between them. A set $C$ of at most $K$ points has a total
weight $\omega=A(C)$ in $\{0,\dots,KW\}$, and $M(C)$ is an integer between $\omega\,b_{\min}$ and
$\omega\,b_{\max}$. For a given $\omega$ there are thus at most $\omega(H-1)+1$ values of $M$, and
\begin{equation}
P\le\sum_{\omega=0}^{KW}\bigl(\omega(H-1)+1\bigr)
\le\bigl(KW+1\bigr)\bigl(KW(H-1)+1\bigr)\le 2K^2W^2H .
\label{app:exact:eq:Pbound}
\end{equation}
The running time of Theorem~\ref{thm:pair} is therefore polynomial in $N$, $K$, $W$ and the number of target
levels. A table stored as a dense array over the pairs counted in~\eqref{app:exact:eq:Pbound} gives the same bound
with $2K^2W^2H$ in place of $P$.

\subsection{Other Semivalues}
\label{app:exact:semi}

A semivalue replaces the Shapley coefficients $c_N(s)$ by other weights of the subset sizes. This section shows
that the algorithms of this appendix compute every semivalue with the same tables: only two sets of coefficients
change.

\begin{proposition}[Semivalues]
\label{app:exact:prop:semi}
Let $\beta(0),\dots,\beta(N-1)$ be real numbers, and let
\begin{equation}
\phi^{\beta}_i=\sum_{S\subseteq D\setminus\{z_i\}}\beta(|S|)\,\bigl[v(S\cup\{z_i\})-v(S)\bigr],
\qquad
\beta_e(K)=\sum_{j=0}^{N-e}\binom{N-e}{j}\,\beta(K+j)\quad\text{for }K<e\le N.
\label{app:exact:eq:semi}
\end{equation}
For every utility with the property~\eqref{app:exact:eq:nk}, Proposition~\ref{app:exact:prop:twosums} holds for
$\phi^{\beta}_i$ with $\beta(s)$ in place of $c_N(s)$ for every $s\ge0$ and $\beta_e(K)$ in place of $c_e(K)$.
Consequently, every formula and every algorithm of this appendix computes $\phi^{\beta}_i$ after this
replacement, with the same tables and the same number of operations on them. The replacement includes the factor
$\frac1N$ of the first term of each formula and of the first assignment to $\phi_i$ in the algorithms: it is
$c_N(0)$, and it becomes $\beta(0)$.
\end{proposition}

\begin{proof}
Lemma~\ref{lem:cases} does not involve the coefficients. The subsets with $m\ge K$ have a marginal contribution of
zero, so~\eqref{app:exact:eq:split} holds with $\beta(|S|)$ in place of $c_N(|S|)$.
Lemma~\ref{app:exact:lem:bijection} and~\eqref{app:exact:eq:marg} concern sets only. In the proof of
Proposition~\ref{app:exact:prop:twosums}, the sum over $F$ becomes $\sum_{F}\beta(K+|F|)$, which is $\beta_e(K)$
because there are $\binom{N-e}{j}$ sets $F$ of size $j$. The tables, the sums over a group,
Lemma~\ref{lem:remove} and Lemma~\ref{app:exact:lem:unrolled} do not involve the coefficients either.
\end{proof}

\textbf{The coefficients in closed form.} For the Shapley value, $\beta(s)=c_N(s)$ and $\beta_e(K)=c_e(K)$ by
Lemma~\ref{app:exact:lem:closed}. For the Banzhaf value, $\beta(s)=2^{-(N-1)}$ for every $s$, and
\begin{equation}
\beta_e(K)=2^{-(N-1)}\sum_{j=0}^{N-e}\binom{N-e}{j}=2^{-(e-1)} .
\label{app:exact:eq:banzhaf}
\end{equation}
For the Beta Shapley value of \citet{kwon2022beta}, write $\alpha_1,\alpha_2>0$ for its two parameters, which
they call $\alpha$ and $\beta$, and $\mathrm{B}(x,y)=\int_0^1p^{x-1}(1-p)^{y-1}\,dp$ for the Beta function. Its
coefficients are
\begin{equation}
\beta(s)=\frac{\mathrm{B}(s+\alpha_2,\;N-1-s+\alpha_1)}{\mathrm{B}(\alpha_1,\alpha_2)},
\qquad
\beta_e(K)=\frac{\mathrm{B}(K+\alpha_2,\;e-1-K+\alpha_1)}{\mathrm{B}(\alpha_1,\alpha_2)} .
\label{app:exact:eq:betashapley}
\end{equation}
The first expression is their weight of a subset of size $s$. It says that
$\beta(s)=\int_0^1p^s(1-p)^{N-1-s}\,\pi(p)\,dp$ with the density
$\pi(p)=p^{\alpha_2-1}(1-p)^{\alpha_1-1}/\mathrm{B}(\alpha_1,\alpha_2)$. The computation in the proof of
Lemma~\ref{app:exact:lem:closed}, with $\pi(p)\,dp$ in place of $dp$, then gives
$\beta_e(K)=\int_0^1p^K(1-p)^{e-1-K}\,\pi(p)\,dp$, which is the second expression. For $\alpha_1=\alpha_2=1$ these
are the Shapley coefficients. In the same way, size weights of the form
$\beta(s)=\int_0^1p^s(1-p)^{N-1-s}\,d\xi(p)$, for a measure $\xi$ on $[0,1]$, have
$\beta_e(K)=\int_0^1p^K(1-p)^{e-1-K}\,d\xi(p)$.

\textbf{The cost of the coefficients.} The algorithms need $\beta(s)$ for $s<\min(K,N)$ and $\beta_e(K)$ for
$K<e\le N$. For the Banzhaf value these are powers of two. For Beta Shapley, the identity
$\mathrm{B}(x+1,y)=\mathrm{B}(x,y)\,x/(x+y)$ and its mirror image in $y$ give
\begin{align*}
\beta(0)&=\prod_{k=0}^{N-2}\frac{k+\alpha_1}{k+\alpha_1+\alpha_2},&
\beta(s)&=\beta(s-1)\,\frac{s-1+\alpha_2}{N-1-s+\alpha_1},\\
\beta_{K+1}(K)&=\prod_{k=0}^{K-1}\frac{k+\alpha_2}{k+\alpha_1+\alpha_2},&
\beta_{e+1}(K)&=\beta_e(K)\,\frac{e-1-K+\alpha_1}{e-1+\alpha_1+\alpha_2}.
\end{align*}
In both cases all coefficients take $O(N+K)$ operations, as for the Shapley value, and the running times of this
appendix are unchanged. For arbitrary size weights, the numbers $\beta_e(K)$ follow from the recurrence
$b(f,k)=b(f-1,k)+b(f-1,k+1)$ with $b(0,k)=\beta(k)$, which Pascal's rule gives for
$b(f,k)=\sum_{j=0}^{f}\binom{f}{j}\beta(k+j)$; then $\beta_e(K)=b(N-e,K)$. This takes $O(N^2)$ additions, which
exceeds the cost of the two passes when $N$ is larger than $WK^2$.

\subsection{Why a Fixed Number of Sums Does Not Serve the Absolute Loss}
\label{app:exact:abs}

Algorithms~\ref{app:exact:alg:direct} and~\ref{app:exact:alg:linear} and Proposition~\ref{app:exact:prop:poly}
obtain the summed utility of a group from its numbers $n_0,\dots,n_d$. They can do so because, among the sets with
a given total weight, the utility is one polynomial of degree at most $d$ in $R$. For the absolute loss the
prediction error is $|\yhat(S)-y^{(\val)}|=|R(\NK(S))|/A(\NK(S))$, so a set with total weight $\omega$ has the
utility $-|R|/\omega$. The next proposition shows that no polynomial of a fixed degree can take the place of the
absolute value, and that the numbers of a group do not determine its sum of $|R|$ in any other way either.

\begin{proposition}[The absolute value and sums of powers]
\label{app:exact:prop:abs}
Let $d\ge1$, and let $\rho_0<0<\rho_1<\dots<\rho_{d+1}$ be real numbers.
\begin{enumerate}
\item No polynomial of degree at most $d$ agrees with $|\rho|$ at all the points $\rho_0,\dots,\rho_{d+1}$.
\item If the $\rho_t$ are rational, there are two finite lists of numbers, taken from $\rho_0,\dots,\rho_{d+1}$
with repetitions, that have the same sum of $k$-th powers for every $k=0,\dots,d$ and different sums of absolute
values.
\end{enumerate}
Both claims also hold for $\rho_0>0>\rho_1>\dots>\rho_{d+1}$.
\end{proposition}

\begin{proof}
\emph{Claim 1.} Let $p$ be a polynomial of degree at most $d$ with $p(\rho_t)=|\rho_t|$ for all $t$. The
polynomial $p(\rho)-\rho$ has degree at most $d$, because $d\ge1$, and it vanishes at the $d+1$ positive points
$\rho_1,\dots,\rho_{d+1}$. A nonzero polynomial of degree at most $d$ has at most $d$ roots, so $p(\rho)=\rho$ for
all real $\rho$. Then $p(\rho_0)=\rho_0<0<|\rho_0|$, a contradiction.

\emph{Claim 2.} The $d+1$ linear equations $\sum_{t=0}^{d+1}c_t\,\rho_t^k=0$, $0\le k\le d$, in the $d+2$ unknowns
$c_0,\dots,c_{d+1}$ have a nonzero solution. Since the $\rho_t$ are rational, it can be taken with integer
entries. If $c_0$ were zero, the numbers $c_1,\dots,c_{d+1}$ would solve a Vandermonde system with the distinct
nodes $\rho_1,\dots,\rho_{d+1}$, and they would all be zero. So $c_0\neq0$. The equation for $k=1$, together with
$|\rho_t|=\rho_t$ for $t\ge1$, gives
\begin{equation*}
\sum_{t=0}^{d+1}c_t\,|\rho_t|=\sum_{t=0}^{d+1}c_t\bigl(|\rho_t|-\rho_t\bigr)=2\,c_0\,|\rho_0|\neq0 .
\end{equation*}
Let the first list contain $\rho_t$ exactly $\max(c_t,0)$ times, and the second list exactly $\max(-c_t,0)$ times.
For every $k\le d$, their sums of $k$-th powers differ by $\sum_tc_t\,\rho_t^k=0$. Their sums of absolute values
differ by $2c_0|\rho_0|$.

\emph{Exchanged signs.} Apply both claims to the numbers $-\rho_t$, and in claim~1 replace $p(\rho)$ by
$p(-\rho)$.
\end{proof}

\textbf{What the proposition implies.} Take unit weights and $K\ge2$. The sets of one point have total weight $1$,
and their values of $R$ are the numbers $r_j=y_j-y^{(\val)}$. Suppose that these numbers include
$\rho_0,\dots,\rho_{d+1}$ as in the proposition. By claim~1, no polynomial of degree at most $d$ gives the utility
$-|R|$ of these sets, so the sum over a group cannot be formed as in Proposition~\ref{app:exact:prop:poly}. By
claim~2, no other rule succeeds either. Take two training sets in which the numbers $r_j$ are the two lists of
claim~2, and add to each a point $z_i$ with $y_i=y^{(\val)}$, so that $r_i=0$. In both, the group of the one-point
sets of $D\setminus\{z_i\}$ has the same numbers $n_0,\dots,n_d$. Joining $z_i$ changes the utility of $\{z_j\}$
from $-|r_j|$ to $-|r_j|/2$, so the marginal contributions of the group sum to $\frac12\sum_j|r_j|$, which differs
between the two training sets. Since $d$ is arbitrary, no fixed number of sums per group serves every instance.

With targets on a grid, the values of $R$ among the sets with a given total weight correspond to the values of
$M$. If there are $\nu$ of them, a polynomial of degree $\nu-1$ interpolates the absolute value, and $\nu$ sums
suffice. But these sums carry the same information as the $\nu$ counts of Theorem~\ref{thm:pair}, to which an
invertible Vandermonde matrix relates them, so nothing is gained. Two numbers do suffice when all targets lie on
one side of $y^{(\val)}$: then all $r_j$ have one sign, and $|R|=\pm R$ is linear in $R$. The weight
share~\eqref{eq:soft} is of this kind. It equals $1-|\yhat(S)-1|$ for the prediction $\yhat(S)$ from the targets
$\mathbf{1}[y_j=y^{(\val)}]$, none of which exceeds the target $1$ of the query.

\textbf{What it does not imply.} The proposition rules out the method, not every algorithm. It concerns algorithms
that obtain the summed utility of a group from a fixed number of sums of powers of $R$, and it says nothing about
others. That no algorithm computes the value in time polynomial in $N$ and in the number of bits of the targets,
unless $\mathrm{FP}=\#\mathrm{P}$, is Theorem~\ref{thm:abs}, which holds already for unit weights.

\appsection{Proofs for Section~\ref{sec:approx}}\label{app:approx}

This appendix proves the statements of Section~\ref{sec:approx}. It has six sections.
\begin{itemize}
\item Section~\ref{app:approx:cont} proves Theorem~\ref{thm:cont}. It states the parameters and the error bound
explicitly, proves the bound, and describes the algorithm, which is linear in $N$.
\item Section~\ref{app:approx:soft} treats the soft-label utility~\eqref{eq:soft} with continuous weights in the same
way, and the Brier score.
\item Section~\ref{app:approx:second} gives a second algorithm for the values of Theorem~\ref{thm:cont}. It is
quadratic in $N$.
\item Section~\ref{app:approx:float} bounds the effect of floating-point rounding for the second algorithm and
reports measurements for the first.
\item Sections~\ref{app:approx:round} and~\ref{app:approx:otherloss} bound the error of discretizing instead:
rounding the weights for the squared loss, and rounding weights and targets for other losses.
\end{itemize}

\textbf{Two algorithms.} Both algorithms compute the Shapley values of the same approximated game, which is
defined in~\eqref{app:approx:eq:eta}. The error bound of Theorem~\ref{app:approx:thm:explicit} is a statement about
this game and holds for both. The operation count of Theorems~\ref{thm:cont} and~\ref{app:approx:thm:explicit} is
that of the first algorithm, Algorithm~\ref{app:approx:alg:linear}. Proposition~\ref{app:approx:prop:float} is a
statement about the second, Algorithm~\ref{app:approx:alg:second}, and Corollary~\ref{app:approx:cor:bits} is
proved with the second.

\textbf{Notation review.} The weights are real numbers, as in Section~\ref{sec:approx}: $A(C)=\sum_{z_j\in C}w_j$ is
the total weight of a set $C$ of training points, $r_j=w_j\,(y_j-y^{(\val)})$, and $R(C)=\sum_{z_j\in C}r_j$. The
training points are sorted by their distance to the query, and $\NK(S)$ is the set of the $\min(K,|S|)$ nearest
points of a subset $S$. For a nonempty $S$ the prediction error is $\yhat(S)-y^{(\val)}=R(\NK(S))/A(\NK(S))$. We
write $\phi_i[u]$ for the Shapley value of $z_i$ under a utility $u$, so that $\phi_i=\phi_i[v]$.

All parts compare the Shapley values of two utility functions that are close on every subset. The following
lemma is used each time; \citet[Lemma~2]{bhagat2014budgeted} state a one-sided form of it.

\begin{lemma}[Close utilities have close Shapley values]
\label{app:approx:lem:close}
Let $v$ and $\tilde v$ be two utility functions on the subsets of $D$ with $|\tilde v(S)-v(S)|\le\Delta$ for every
$S\subseteq D$. Then their Shapley values satisfy $|\tilde\phi_i-\phi_i|\le 2\Delta$ for every $i$.
\end{lemma}

\begin{proof}
By~\eqref{eq:shapley},
\begin{equation*}
\tilde\phi_i-\phi_i=\sum_{S\subseteq D\setminus\{z_i\}}c_N(|S|)\,
\Bigl(\bigl[\tilde v(S\cup\{z_i\})-v(S\cup\{z_i\})\bigr]-\bigl[\tilde v(S)-v(S)\bigr]\Bigr),
\end{equation*}
and each term in the large parentheses is at most $2\Delta$ in absolute value. The coefficients are nonnegative and
sum to one: $D\setminus\{z_i\}$ has $\binom{N-1}{s}$ subsets of size $s$, and $\binom{N-1}{s}\,c_N(s)=1/N$ for each
of the $N$ sizes $s=0,\dots,N-1$. Hence $|\tilde\phi_i-\phi_i|\le2\Delta$.
\end{proof}

\subsection{Proof of \texorpdfstring{Theorem~\ref{thm:cont}}{the Theorem for Continuous Weights}}
\label{app:approx:cont}

\begin{restate}{Theorem~\ref{thm:cont}}
Consider the squared loss, real weights in $[w_{\min},w_{\max}]$ with $w_{\min}>0$, and targets with
$|y_j-y^{(\val)}|\le B$. For every $0<\varepsilon\le B^2$, all $N$ values are computed within additive error
$\varepsilon$ using $O(NKL)$ arithmetic operations and evaluations of the exponential function, in exact real
arithmetic, after sorting by distance. Here
$L=O\bigl(\lambda\,(\lambda+\log\tfrac{Kw_{\max}}{w_{\min}})\bigr)$ with $\lambda=1+\log\tfrac{B^2}{\varepsilon}$.
\end{restate}

We first define the approximated game and state the theorem with explicit parameters
(Theorem~\ref{app:approx:thm:explicit}). Lemmas~\ref{app:approx:lem:transfer} to~\ref{app:approx:lem:eta} then
prove the error bound, and Lemmas~\ref{app:approx:lem:nodes} to~\ref{app:approx:lem:suffix} show how the values of
the approximated game are computed in time linear in $N$.

\textbf{Notation.} In the rest of this appendix $K'=\min(K,N)$ is the largest number of neighbors that a subset
can have, and $\varrho=2+\sqrt3$. We write $\varrho$ because $\rho$ denotes the rounding error in
Section~\ref{app:approx:round}. The numbers $w_{\min}$, $w_{\max}$ and $B$ may be any bounds with
$0<w_{\min}\le w_j\le w_{\max}$ and $|y_j-y^{(\val)}|\le B$ for all $j$. For a nonempty subset $S$, the set
$\NK(S)$ has between one and $K'$ points, and $|r_j|\le B\,w_j$. Hence
\begin{equation}
w_{\min}\le A(\NK(S))\le K'w_{\max}
\qquad\text{and}\qquad
|R(\NK(S))|\le B\,A(\NK(S)) .
\label{app:approx:eq:range}
\end{equation}

\textbf{Exact real arithmetic.} By exact real arithmetic we mean the following model of computation. A real number
is stored in one cell. An addition, subtraction, multiplication, division or comparison of two real numbers and an
evaluation of the exponential function are exact and cost one operation each. The nodes and weights of the
Gauss--Legendre rules defined next are available as exact real numbers. Section~\ref{app:approx:float} replaces this
model by floating-point arithmetic.

\textbf{Gauss--Legendre rules.} For $m\ge1$, the $m$-point Gauss--Legendre rule on $[-1,1]$ has the nodes
$\xi_1,\dots,\xi_m\in(-1,1)$, which are the zeros of the Legendre polynomial of degree $m$, and the weights
$\beta_1,\dots,\beta_m$. We use two standard facts \citep{trefethen2013approximation}. The rule is exact for every
polynomial $f$ of degree at most $2m-1$, that is, $\sum_{j}\beta_jf(\xi_j)=\int_{-1}^{1}f(\xi)\,d\xi$, and the
weights are positive. For $f=1$ the first fact shows that the weights sum to $2$. On an interval $[\tau,\tau']$ the rule has the
nodes $\tfrac12(\tau+\tau')+\tfrac12(\tau'-\tau)\,\xi_j$ and the weights $\tfrac12(\tau'-\tau)\,\beta_j$. It is
exact for the same polynomials, and its weights are positive and sum to $\tau'-\tau$.

\textbf{The rule in $t$.} Let $k_{\min}<k_{\max}$ be integers and $n\ge1$. The \emph{rule in $t$} applies the
$n$-point Gauss--Legendre rule to each of the pieces $[2^k,2^{k+1}]$, $k_{\min}\le k<k_{\max}$. On the piece
$[2^k,2^{k+1}]$ its nodes are $2^{k-1}(3+\xi_j)$ and its weights are $2^{k-1}\beta_j$, for $j=1,\dots,n$. We list
the nodes and weights of all pieces as $(t_\ell,\mu_\ell)$, $\ell=1,\dots,L$, with $L=n\,(k_{\max}-k_{\min})$. For a
total weight $A>0$ let
\begin{equation}
q(A)=\sum_{\ell=1}^{L}\mu_\ell\,t_\ell\,e^{-t_\ell A}
\label{app:approx:eq:q}
\end{equation}
be the value that the rule assigns to the integral~\eqref{eq:laplace}.

\textbf{The approximated game.} The algorithm replaces $1/A^2$ in the utility~\eqref{eq:vAR} by the value $q(A)$
of the rule. Let
\begin{equation}
\hat v(S)=\begin{cases}-R(\NK(S))^2\,q\bigl(A(\NK(S))\bigr)&\text{if }S\ne\emptyset,\\[2pt]
v(\emptyset)&\text{if }S=\emptyset,\end{cases}
\qquad
\eta=\max_{w_{\min}\le A\le K'w_{\max}}\bigl|A^2q(A)-1\bigr| .
\label{app:approx:eq:eta}
\end{equation}
Every nonempty subset is approximated, also a subset with a single point. Only the utility
$v(\emptyset)=-(y_\emptyset-y^{(\val)})^2$ of the empty subset is kept, because it is not of the
form~\eqref{eq:vAR}. The number $\eta$ is the largest relative error of the rule on the
integral~\eqref{eq:laplace}, over all total weights that a set of neighbors can have. We write
$\hat\phi_i=\phi_i[\hat v]$ for the Shapley values of the approximated game. The definitions of $q$, $\hat v$ and
$\eta$ make sense for any nodes $t_\ell>0$ and any real weights $\mu_\ell$, and
Lemmas~\ref{app:approx:lem:transfer}, \ref{app:approx:lem:nodes} and~\ref{app:approx:lem:pnodes} hold for every
such rule. Section~\ref{app:approx:float} uses them for a rule that is computed approximately.

\begin{theorem}[Explicit form of Theorem~\ref{thm:cont}]
\label{app:approx:thm:explicit}
Let $N\ge1$, $K\ge1$ and $K'=\min(K,N)$. Let $0<w_{\min}\le w_j\le w_{\max}$ and $|y_j-y^{(\val)}|\le B$ for all
$j$, and let $0<\varepsilon\le B^2$. Put $\varrho=2+\sqrt3$ and
\begin{equation}
k_{\min}=\biggl\lfloor\log_2\frac{\sqrt{\varepsilon/3}}{B\,K'w_{\max}}\biggr\rfloor,\qquad
k_{\max}=\biggl\lceil\log_2\frac{2\ln(6B^2/\varepsilon)}{w_{\min}}\biggr\rceil,\qquad
n=\biggl\lceil\frac12+\frac{\ln(198\,B^2/\varepsilon)}{2\ln\varrho}\biggr\rceil .
\label{app:approx:eq:params}
\end{equation}
Then $k_{\min}<k_{\max}$. For the rule in $t$ with $n$ nodes on each of the $k_{\max}-k_{\min}$ pieces, the values
$\hat\phi_i$ of the approximated game~\eqref{app:approx:eq:eta} satisfy for every $i$
\begin{equation}
|\hat\phi_i-\phi_i|
\le 2B^2\Bigl[\tfrac12\bigl(2^{k_{\min}}K'w_{\max}\bigr)^2+(1+x)\,e^{-x}+33\,\varrho^{\,1-2n}\Bigr]
\le\varepsilon,
\qquad x=2^{k_{\max}}w_{\min}.
\label{app:approx:eq:bound}
\end{equation}
In exact real arithmetic, once the points are sorted by distance, Algorithm~\ref{app:approx:alg:linear} computes
all $N$ values $\hat\phi_i$ with $O(NK'L)$ arithmetic operations and $NL$ evaluations of the exponential
function, where $L=n\,(k_{\max}-k_{\min})$ and
\begin{equation}
n\le\frac32+\frac{\ln(198\,B^2/\varepsilon)}{2\ln\varrho},\qquad
k_{\max}-k_{\min}\le\log_2\frac{K'w_{\max}}{w_{\min}}+\frac12\log_2\frac{3B^2}{\varepsilon}
+\log_2\Bigl(2\ln\frac{6B^2}{\varepsilon}\Bigr)+2 .
\label{app:approx:eq:counts}
\end{equation}
It stores $O(N+K')$ numbers besides the rule. The first inequality of~\eqref{app:approx:eq:bound} holds for all
integers $k_{\min}<k_{\max}$ and all $n\ge1$.
\end{theorem}

The three terms in the brackets of~\eqref{app:approx:eq:bound} come from cutting the integral over $t$ below
$2^{k_{\min}}$, from cutting it above $2^{k_{\max}}$, and from the Gauss--Legendre rules on the pieces. The
parameters~\eqref{app:approx:eq:params} make each term at most $\varepsilon/(6B^2)$. The bound depends on the
instance only through $K'$, $w_{\min}$, $w_{\max}$ and $B$, so it is known before the run. The operation count does
not include the evaluation of~\eqref{app:approx:eq:params}, which is done once.

\begin{proof}[Proof of Theorem~\ref{thm:cont} from Theorem~\ref{app:approx:thm:explicit}]
The error bound is~\eqref{app:approx:eq:bound}. For the operation count, write
$\lambda=1+\ln(B^2/\varepsilon)\ge1$; another base of the logarithm changes $\lambda$ by a constant factor. Then
$\ln(198\,B^2/\varepsilon)\le 6\lambda$ and $\ln(6B^2/\varepsilon)\le2\lambda$, so $n=O(\lambda)$
by~\eqref{app:approx:eq:counts}. In the bound on $k_{\max}-k_{\min}$, the term
$\tfrac12\log_2(3B^2/\varepsilon)$ is $O(\lambda)$, and so is
$\log_2(2\ln(6B^2/\varepsilon))\le2\ln(6B^2/\varepsilon)$, because $\log_2x\le x$ for $x>0$. Hence
$k_{\max}-k_{\min}=O\bigl(\lambda+\log(K'w_{\max}/w_{\min})\bigr)$ and
$L=O\bigl(\lambda\,(\lambda+\log(K'w_{\max}/w_{\min}))\bigr)$. With $K'\le K$, the $O(NK'L)$ arithmetic operations
and the $NL$ evaluations of the exponential function are together
$O\bigl(NK\,\lambda\,(\lambda+\log(Kw_{\max}/w_{\min}))\bigr)$.
\end{proof}

\subsubsection*{The error bound}

The error bound compares the approximated game with the true one. It does not depend on how the values of the
approximated game are computed.

\begin{lemma}[From the rule to the Shapley value]
\label{app:approx:lem:transfer}
For every $i$, $|\phi_i[\hat v]-\phi_i|\le2B^2\eta$.
\end{lemma}

\begin{proof}
The two utilities agree on the empty subset. Let $S$ be nonempty and write $A=A(\NK(S))$ and $R=R(\NK(S))$.
By~\eqref{eq:vAR} and~\eqref{app:approx:eq:range},
\begin{equation*}
|\hat v(S)-v(S)|=R^2\,\Bigl|q(A)-\frac{1}{A^2}\Bigr|=\frac{R^2}{A^2}\,\bigl|A^2q(A)-1\bigr|\le B^2\eta .
\end{equation*}
Lemma~\ref{app:approx:lem:close} with $\Delta=B^2\eta$ gives the claim.
\end{proof}

By Lemma~\ref{app:approx:lem:transfer}, the error of every value is at most $2B^2$ times the relative error $\eta$
of the rule in $t$ on the single function $t\,e^{-tA}$. The remaining lemmas bound $\eta$ for the rule on the
dyadic pieces. They start from the identity~\eqref{eq:laplace}, in a form that also gives the integral over a part
of the range. They then treat the two parts of the range that the rule leaves out, one piece, and the sum over the
pieces.

\begin{lemma}[The integral over $t$]
\label{app:approx:lem:laplace}
For $A>0$ and $0\le\tau<\tau'$,
\begin{equation}
A^2\int_{\tau}^{\tau'}t\,e^{-tA}\,dt=(1+\tau A)\,e^{-\tau A}-(1+\tau'A)\,e^{-\tau'A}.
\label{app:approx:eq:partial}
\end{equation}
In particular $\int_0^\infty t\,e^{-tA}\,dt=1/A^2$, which is~\eqref{eq:laplace}.
\end{lemma}

\begin{proof}
The derivative of $-(1+tA)\,e^{-tA}$ with respect to $t$ is $A^2\,t\,e^{-tA}$, which gives~\eqref{app:approx:eq:partial}.
For $\tau=0$ and $\tau'\to\infty$ the right-hand side tends to $1$.
\end{proof}

\begin{lemma}[The two tails]
\label{app:approx:lem:tails}
For $A>0$ and $\tau>0$,
\begin{equation*}
A^2\int_0^{\tau}t\,e^{-tA}\,dt\le\tfrac12\,(\tau A)^2
\qquad\text{and}\qquad
A^2\int_{\tau}^{\infty}t\,e^{-tA}\,dt=(1+\tau A)\,e^{-\tau A}.
\end{equation*}
Moreover, $(1+x)\,e^{-x}\le e^{-x/2}$ for every $x\ge3$.
\end{lemma}

\begin{proof}
Substituting $x=tA$ gives $A^2\int_0^\tau t\,e^{-tA}\,dt=\int_0^{\tau A}x\,e^{-x}\,dx\le\int_0^{\tau A}x\,dx$, which
is the first claim. The second is~\eqref{app:approx:eq:partial} for $\tau'\to\infty$. For the last claim, the function
$e^{x/2}-1-x$ is positive at $x=3$, since $e^{3/2}>4$, and its derivative $\tfrac12e^{x/2}-1$ is positive for
$x\ge3$. Hence $1+x\le e^{x/2}$ for $x\ge3$.
\end{proof}

For one piece we use the classical bound for the Gauss--Legendre rule on an analytic function. We include its proof
so that all constants are explicit. The Bernstein ellipse with parameter $\varrho$ is the curve
\begin{equation*}
\mathcal{E}_\varrho=\bigl\{\tfrac12(\zeta+\zeta^{-1}):\ \zeta\in\mathbb{C},\ |\zeta|=\varrho\bigr\},
\end{equation*}
and $T_k$ denotes the Chebyshev polynomial of degree $k$, defined by $T_k(\cos\theta)=\cos k\theta$. The next two
lemmas use only that $\varrho>1$.

\begin{lemma}[Chebyshev coefficients; {\citealp[Theorem~8.1]{trefethen2013approximation}}]
\label{app:approx:lem:cheb}
Let $f$ be an entire function with $|f|\le M$ on $\mathcal{E}_\varrho$. Then there are numbers $a_0,a_1,\dots$
with $|a_k|\le2M\varrho^{-k}$ such that $f(\xi)=\sum_{k\ge0}a_k\,T_k(\xi)$ for every $\xi\in[-1,1]$.
\end{lemma}

\begin{proof}
The function $F(\zeta)=f\bigl(\tfrac12(\zeta+\zeta^{-1})\bigr)$ is analytic for $\zeta\ne0$. It therefore has a
Laurent series $F(\zeta)=\sum_{k\in\mathbb{Z}}b_k\,\zeta^k$ that converges for every $\zeta\ne0$, with the
coefficients $b_k=\frac{1}{2\pi\mathrm{i}}\oint_{|\zeta|=\varrho}F(\zeta)\,\zeta^{-k-1}\,d\zeta$. On the circle
$|\zeta|=\varrho$ the argument of $f$ lies on $\mathcal{E}_\varrho$, so $|F|\le M$ there, and the integral gives
$|b_k|\le M\varrho^{-k}$. Since $F(\zeta)=F(1/\zeta)$ and the Laurent series is unique, $b_{-k}=b_k$. For
$\zeta=e^{\mathrm{i}\theta}$ we have $\tfrac12(\zeta+\zeta^{-1})=\cos\theta$ and $\zeta^k+\zeta^{-k}=2\cos k\theta$,
so
\begin{equation*}
f(\cos\theta)=b_0+\sum_{k\ge1}2\,b_k\cos k\theta=b_0+\sum_{k\ge1}2\,b_k\,T_k(\cos\theta).
\end{equation*}
The claim holds with $a_0=b_0$ and $a_k=2b_k$ for $k\ge1$.
\end{proof}

\begin{lemma}[Error of the $n$-point rule]
\label{app:approx:lem:gauss}
Let $f$ be an entire function with $|f|\le M$ on $\mathcal{E}_\varrho$, and let $n\ge1$. Then
\begin{equation*}
\Bigl|\int_{-1}^{1}f(\xi)\,d\xi-\sum_{j=1}^{n}\beta_j\,f(\xi_j)\Bigr|\le\frac{8M\,\varrho^{\,1-2n}}{\varrho-1}.
\end{equation*}
\end{lemma}

\begin{proof}
Let $f_n=\sum_{k\le2n-1}a_k\,T_k$ be the series of Lemma~\ref{app:approx:lem:cheb} cut after the degree $2n-1$. Since
$|T_k|\le1$ on $[-1,1]$, every $\xi\in[-1,1]$ satisfies
\begin{equation*}
|f(\xi)-f_n(\xi)|\le\sum_{k\ge2n}2M\varrho^{-k}=\frac{2M\,\varrho^{-2n}}{1-\varrho^{-1}}=\frac{2M\,\varrho^{\,1-2n}}{\varrho-1}.
\end{equation*}
The rule is exact for the polynomial $f_n$, so its error on $f$ equals its error on $f-f_n$. The integral of
$f-f_n$ over $[-1,1]$ is at most $2\max|f-f_n|$ in absolute value. So is the value of the rule on $f-f_n$, because
the weights $\beta_j$ are positive and sum to $2$. The error is therefore at most $4\max|f-f_n|$, which is the
claim.
\end{proof}

From here on $\varrho=2+\sqrt3$ again, and
\begin{equation}
\psi(\alpha)=\alpha^2\,e^{-\alpha/2}\qquad(\alpha\ge0).
\label{app:approx:eq:psi}
\end{equation}

\begin{lemma}[One dyadic piece]
\label{app:approx:lem:piece}
Let $A>0$, $\tau>0$ and $\alpha=\tau A$. The value of the $n$-point Gauss--Legendre rule on $[\tau,2\tau]$ for the
function $t\mapsto A^2\,t\,e^{-tA}$ differs from $A^2\int_\tau^{2\tau}t\,e^{-tA}\,dt$ by at most
\begin{equation*}
\frac{10}{1+\sqrt3}\,\psi(\alpha)\,\varrho^{\,1-2n}.
\end{equation*}
\end{lemma}

\begin{proof}
Substitute $t=\tfrac12\tau\,(3+\xi)$ with $\xi\in[-1,1]$. Then
\begin{equation}
A^2\int_\tau^{2\tau}t\,e^{-tA}\,dt=\int_{-1}^{1}G(\xi)\,d\xi,
\qquad
G(\xi)=\frac{\alpha^2}{4}\,(3+\xi)\,e^{-\alpha(3+\xi)/2}.
\label{app:approx:eq:G}
\end{equation}
The rule on $[\tau,2\tau]$ has the nodes $\tfrac12\tau\,(3+\xi_j)$ and the weights $\tfrac12\tau\,\beta_j$, so its
value for $A^2\,t\,e^{-tA}$ is $\sum_j\beta_j\,G(\xi_j)$. It remains to apply Lemma~\ref{app:approx:lem:gauss} to the entire
function $G$. For $\varrho=2+\sqrt3$ we have $\varrho^{-1}=2-\sqrt3$. A point $\xi=\tfrac12(\zeta+\zeta^{-1})$ of
$\mathcal{E}_\varrho$ therefore satisfies $|\xi|\le\tfrac12(\varrho+\varrho^{-1})=2$, and so its real part is at
least $-2$. Hence $|3+\xi|\le5$ and $|e^{-\alpha(3+\xi)/2}|\le e^{-\alpha/2}$ on $\mathcal{E}_\varrho$, that is,
$|G|\le\tfrac54\,\psi(\alpha)$. Lemma~\ref{app:approx:lem:gauss} with $M=\tfrac54\,\psi(\alpha)$ and $\varrho-1=1+\sqrt3$
gives the claim.
\end{proof}

\begin{lemma}[The sum over the dyadic pieces]
\label{app:approx:lem:dyadic}
For every $A>0$, $\sum_{k\in\mathbb{Z}}\psi(2^kA)\le9$.
\end{lemma}

\begin{proof}
The derivative $\psi'(\alpha)=\alpha\,(2-\alpha/2)\,e^{-\alpha/2}$ shows that $\psi$ increases on $[0,4]$ and
decreases on $[4,\infty)$. Let $\alpha_*$ be the one number of the form $2^kA$ in $[1,2)$. The numbers $2^kA$,
$k\in\mathbb{Z}$, are then $2^{j}\alpha_*$, $j\in\mathbb{Z}$, and we bound three groups of them.

For $j\le-1$ we use $\psi(\alpha)\le\alpha^2$: the sum is at most $\alpha_*^2\sum_{j\ge1}4^{-j}=\alpha_*^2/3<4/3$.

For $0\le j\le4$ the five numbers lie in $[1,2)$, $[2,4)$, $[4,8)$, $[8,16)$ and $[16,32)$. By the monotonicity of
$\psi$ their values are at most $\psi(2)$, $\psi(4)$, $\psi(4)$, $\psi(8)$ and $\psi(16)$, which are below $1.48$,
$2.17$, $2.17$, $1.18$ and $0.09$. The sum is at most $7.09$.

For $j\ge5$ the numbers are at least $32$. There $\alpha^2\le e^{\alpha/4}$: this holds at $\alpha=32$, where
$1024<e^{8}$, and the derivative $\tfrac14-2/\alpha$ of $\alpha/4-2\ln\alpha$ is positive for $\alpha>8$. Hence
$\psi(\alpha)\le e^{-\alpha/4}$, and the sum is at most $\sum_{j\ge0}e^{-8\cdot2^{j}}<0.001$.

The three bounds add up to less than $9$.
\end{proof}

\begin{lemma}[Relative error of the rule in $t$]
\label{app:approx:lem:eta}
Let $k_{\min}<k_{\max}$ be integers and $n\ge1$. The rule in $t$ on the pieces $[2^k,2^{k+1}]$,
$k_{\min}\le k<k_{\max}$, with $n$ nodes per piece satisfies, for every $A>0$,
\begin{equation}
\bigl|A^2q(A)-1\bigr|\le\tfrac12\bigl(2^{k_{\min}}A\bigr)^2+\bigl(1+2^{k_{\max}}A\bigr)\,e^{-2^{k_{\max}}A}
+33\,\varrho^{\,1-2n}.
\label{app:approx:eq:etaA}
\end{equation}
Consequently
$\eta\le\tfrac12\bigl(2^{k_{\min}}K'w_{\max}\bigr)^2+(1+x)\,e^{-x}+33\,\varrho^{\,1-2n}$ with
$x=2^{k_{\max}}w_{\min}$.
\end{lemma}

\begin{proof}
By Lemma~\ref{app:approx:lem:laplace}, $1=A^2\int_0^\infty t\,e^{-tA}\,dt$. We split this integral at $2^{k_{\min}}$, at
$2^{k_{\max}}$ and into the pieces between them, and we split $A^2q(A)$ into the contributions of the pieces. The
difference $1-A^2q(A)$ is then the sum of the two tails and of the errors of the rule on the pieces.
Lemma~\ref{app:approx:lem:tails} bounds the tails by the first two terms of~\eqref{app:approx:eq:etaA}.
Lemma~\ref{app:approx:lem:piece} with $\tau=2^k$ bounds the error on the piece $[2^k,2^{k+1}]$, and
Lemma~\ref{app:approx:lem:dyadic} bounds the sum over the pieces by
$\frac{10}{1+\sqrt3}\cdot9\cdot\varrho^{\,1-2n}\le33\,\varrho^{\,1-2n}$.

For the bound on $\eta$, let $w_{\min}\le A\le K'w_{\max}$. The first term of~\eqref{app:approx:eq:etaA} increases with
$A$. The second decreases with $A$, because the derivative of $(1+x)\,e^{-x}$ is $-x\,e^{-x}$.
\end{proof}

\begin{proof}[Proof of the error bound~\eqref{app:approx:eq:bound}]
By Lemmas~\ref{app:approx:lem:transfer} and~\ref{app:approx:lem:eta},
\begin{equation*}
|\hat\phi_i-\phi_i|\le2B^2\eta
\le2B^2\Bigl[\tfrac12\bigl(2^{k_{\min}}K'w_{\max}\bigr)^2+(1+x)\,e^{-x}+33\,\varrho^{\,1-2n}\Bigr]
\end{equation*}
for all integers $k_{\min}<k_{\max}$ and all $n\ge1$. This is the first inequality of~\eqref{app:approx:eq:bound}.

Now let the parameters be those of~\eqref{app:approx:eq:params}. The floor in $k_{\min}$ gives
$2^{k_{\min}}K'w_{\max}\le\sqrt{\varepsilon/3}/B$, so the first term in the brackets is at most
$\varepsilon/(6B^2)$. The ceiling in $k_{\max}$ gives $x\ge2\ln(6B^2/\varepsilon)$, and $x\ge2\ln6>3$ because
$\varepsilon\le B^2$. By Lemma~\ref{app:approx:lem:tails}, $(1+x)\,e^{-x}\le e^{-x/2}\le\varepsilon/(6B^2)$. The ceiling
in $n$ gives $(2n-1)\ln\varrho\ge\ln(198\,B^2/\varepsilon)$, so $33\,\varrho^{\,1-2n}\le33\,\varepsilon/(198\,B^2)=
\varepsilon/(6B^2)$. The three terms sum to at most $\varepsilon/(2B^2)$, which is the second inequality
of~\eqref{app:approx:eq:bound}.

The two numbers under the floor and the ceiling of $k_{\min}$ and $k_{\max}$ differ by
\begin{equation*}
\log_2\Bigl(2\ln\frac{6B^2}{\varepsilon}\cdot\sqrt{\frac{3B^2}{\varepsilon}}\cdot\frac{K'w_{\max}}{w_{\min}}\Bigr)
\ge\log_2\bigl(2\sqrt3\,\ln6\bigr)>2,
\end{equation*}
so $k_{\min}<k_{\max}$. The same difference, increased by $2$ for the floor and the ceiling, is the bound on
$k_{\max}-k_{\min}$ in~\eqref{app:approx:eq:counts}. The bound on $n$ follows from $\lceil x\rceil\le x+1$.
\end{proof}

\subsubsection*{The algorithm}

It remains to compute the Shapley values of the approximated game. The rule in $t$ is fixed from here on, and no
further approximation is made.

\textbf{One node at a time.} For $t>0$ and a subset $T\subseteq D$ let
\begin{equation}
g_t(T)=R(\NK(T))^2\,e^{-t\,A(\NK(T))},
\label{app:approx:eq:g}
\end{equation}
so that $g_t(\emptyset)=0$, and for a training point $z_i$ let
\begin{equation}
\Gamma_i(t)=\sum_{S\subseteq D\setminus\{z_i\}}c_N(|S|)\,\bigl[\,g_t(S)-g_t(S\cup\{z_i\})\,\bigr].
\label{app:approx:eq:Gamma}
\end{equation}
In words, $\Gamma_i(t)$ is the Shapley value of $z_i$ under the utility $-g_t$, which is the
utility~\eqref{eq:vAR} with $e^{-tA}$ in place of $1/A^2$.

\begin{lemma}[The values as a sum over the nodes]
\label{app:approx:lem:nodes}
For every $i$,
\begin{equation}
\hat\phi_i=-\frac{v(\emptyset)}{N}+\sum_{\ell=1}^{L}\mu_\ell\,t_\ell\,\Gamma_i(t_\ell).
\label{app:approx:eq:nodes}
\end{equation}
\end{lemma}

\begin{proof}
By~\eqref{app:approx:eq:q} and~\eqref{app:approx:eq:g}, $\sum_\ell\mu_\ell t_\ell\,g_{t_\ell}(T)$ equals
$-\hat v(T)$ for a nonempty $T$ and $0$ for $T=\emptyset$. Hence
\begin{equation*}
\sum_{\ell=1}^{L}\mu_\ell\,t_\ell\,\Gamma_i(t_\ell)
=\sum_{S\subseteq D\setminus\{z_i\}}c_N(|S|)\,\bigl[\hat v(S\cup\{z_i\})-\hat v(S)\,\mathbf{1}[S\ne\emptyset]\bigr].
\end{equation*}
The sum on the right is $\phi_i[\hat v]$ without the term $-c_N(0)\,\hat v(\emptyset)$ of the empty subset, so it
equals $\hat\phi_i+v(\emptyset)/N$.
\end{proof}

Both algorithms of this appendix evaluate~\eqref{app:approx:eq:nodes}. They compute the numbers
$\Gamma_i(t_\ell)$ for one node after the other, and they differ in how they sum over the subsets
in~\eqref{app:approx:eq:Gamma}. We fix $t>0$ until Lemma~\ref{app:approx:lem:suffix}.

\textbf{The triple of a set.} For a set $C$ of training points let
\begin{equation}
\chi(C)=\bigl(e^{-tA(C)},\;R(C)\,e^{-tA(C)},\;R(C)^2\,e^{-tA(C)}\bigr),
\label{app:approx:eq:chi}
\end{equation}
so that $\chi(\emptyset)=(1,0,0)$ and $g_t(T)$ is the last entry of $\chi(\NK(T))$. We write $[f]_2$ for the last
entry of a triple $f=(f_0,f_1,f_2)$. The three entries play the role of the three numbers of a group in
Section~\ref{sec:setup}, a count, a sum of $R$ and a sum of $R^2$. The difference is that a set now enters with the
factor $e^{-tA}$, so the sets need not be sorted by their total weight. For triples $f$ and $h=(h_0,h_1,h_2)$ and
a point $z_j$ we define
\begin{equation}
f^{+j}=e^{-tw_j}\,\bigl(f_0,\;f_1+r_jf_0,\;f_2+2r_jf_1+r_j^2f_0\bigr),
\qquad
\langle f,h\rangle=f_2h_0+2f_1h_1+f_0h_2 .
\label{app:approx:eq:ops}
\end{equation}
The first operation adds the point $z_j$ to a set, and the second joins two disjoint sets:
\begin{equation}
\chi(C\cup\{z_j\})=\chi(C)^{+j}\ \ \text{for }z_j\notin C,
\qquad
R(C\cup C')^2\,e^{-tA(C\cup C')}=\bigl\langle\chi(C),\chi(C')\bigr\rangle\ \ \text{for }C\cap C'=\emptyset .
\label{app:approx:eq:add}
\end{equation}
Both identities follow from $A(C\cup C')=A(C)+A(C')$ and $R(C\cup C')=R(C)+R(C')$ by expanding the square. The
first is the update~\eqref{eq:update} with the factor $e^{-tw_j}$ in place of the shift of the total weight by
$a_j$. The map $f\mapsto f^{+j}$ is linear and $\langle\cdot,\cdot\rangle$ is bilinear. The second operation is
used in Section~\ref{app:approx:second} only.

A direct computation shows that two operations of the first kind compose by adding the weights and the numbers
$r_j$:
\begin{equation}
\bigl(f^{+j}\bigr)^{+e}=e^{-t(w_j+w_e)}\,\bigl(f_0,\;f_1+(r_j+r_e)f_0,\;f_2+2(r_j+r_e)f_1+(r_j+r_e)^2f_0\bigr).
\label{app:approx:eq:compose}
\end{equation}
The operations for different points therefore commute. In the same way, applying $f\mapsto f^{+j}$ $\tau$ times
gives
\begin{equation}
f^{+\tau j}=e^{-\tau tw_j}\,\bigl(f_0,\;f_1+\tau r_jf_0,\;f_2+2\tau r_jf_1+\tau^2r_j^2f_0\bigr),
\qquad\tau=0,1,2,\dots,
\label{app:approx:eq:copies}
\end{equation}
which shifts the total weight by $\tau w_j$ and the total weighted target by $\tau r_j$, as if $\tau$ copies of
$z_j$ were added.

\textbf{Three sums per size.} For a set $G$ of training points and an integer $s\ge0$ let
\begin{equation}
M_G[s]=\sum_{C\subseteq G,\ |C|=s}\chi(C)
\label{app:approx:eq:M}
\end{equation}
be the triple of the sums of $e^{-tA(C)}$, of $R(C)\,e^{-tA(C)}$ and of $R(C)^2e^{-tA(C)}$ over the $s$-subsets
$C$ of $G$. So $M_G[0]=(1,0,0)$, and $M_G[s]=(0,0,0)$ for $s>|G|$ and for $s<0$. These are the three numbers per
size of Section~\ref{sec:approx}. The table $Q_G$ of Section~\ref{sec:exact} keeps such sums separately for every
total weight $\omega$. Here the factor $e^{-tA(C)}$ takes the place of the index $\omega$, and one triple per
size remains.

\begin{lemma}[Adding and removing a point]
\label{app:approx:lem:remove}
Let $G$ be a set of training points and $s\ge0$. If $z_j\notin G$, then
\begin{equation}
M_{G\cup\{z_j\}}[s]=M_G[s]+M_G[s-1]^{+j}.
\label{app:approx:eq:Madd}
\end{equation}
If $z_i\in G$, then $M_{G\setminus\{z_i\}}[0]=M_G[0]$ and, for $s\ge1$,
\begin{equation}
M_{G\setminus\{z_i\}}[s]=M_G[s]-M_{G\setminus\{z_i\}}[s-1]^{+i},
\qquad\text{hence}\qquad
M_{G\setminus\{z_i\}}[s]=\sum_{\tau=0}^{s}(-1)^{\tau}\,M_G[s-\tau]^{+\tau i}.
\label{app:approx:eq:Mremove}
\end{equation}
\end{lemma}

\begin{proof}
An $s$-subset of $G\cup\{z_j\}$ either avoids $z_j$, or it is an $(s-1)$-subset $C$ of $G$ together with $z_j$,
and then its triple is $\chi(C)^{+j}$ by~\eqref{app:approx:eq:add}. Summing over both kinds
gives~\eqref{app:approx:eq:Madd}, because $f\mapsto f^{+j}$ is linear. The first identity
of~\eqref{app:approx:eq:Mremove} is~\eqref{app:approx:eq:Madd} for the set $G\setminus\{z_i\}$ and the point
$z_i$, solved for the smaller set. Inserting it into itself $s$ times, with~\eqref{app:approx:eq:copies} and
$M_{G\setminus\{z_i\}}[0]=M_G[0]$, gives the second.
\end{proof}

The identities~\eqref{app:approx:eq:Madd} and~\eqref{app:approx:eq:Mremove} are the update~\eqref{eq:update} and
Lemma~\ref{lem:remove}, for one node and without the index $\omega$. In the alternating sum, the term with the
index $\tau$ collects the smaller subsets of $G$, each with $\tau$ copies of $z_i$ added.

\textbf{The two cases.} The next lemma is Proposition~\ref{prop:phi}, stated for every $K$ and for any function
of the subsets that depends on the $K$ nearest neighbors only. As in Section~\ref{sec:exact}, $D_{<e}=\{z_1,\dots,z_{e-1}\}$ is the set
of the points nearer than $z_e$, and $c_e(K)=K!\,(e-1-K)!/e!$ for $e>K$.

\begin{lemma}[Two sums]
\label{app:approx:lem:twosums}
Let $u$ be a function of the subsets of $D$ with $u(S)=u(\NK(S))$ for every $S$. Then, for every $i$,
\begin{align}
\phi_i[u]={}&\sum_{s=0}^{K'-1}c_N(s)\sum_{\substack{S\subseteq D\setminus\{z_i\}\\|S|=s}}
\bigl[u(S\cup\{z_i\})-u(S)\bigr]\notag\\
&+\sum_{e>\max(i,K)}c_e(K)\sum_{\substack{C\subseteq D_{<e}\setminus\{z_i\}\\|C|=K-1}}
\bigl[u(C\cup\{z_i\})-u(C\cup\{z_e\})\bigr].
\label{app:approx:eq:twocases}
\end{align}
\end{lemma}

\begin{proof}
Appendix~\ref{app:exact} proves this formula (Proposition~\ref{app:exact:prop:twosums}), and that proof uses of
the utility only that it depends on a subset through its $K$ nearest neighbors. We recall the steps.
Lemma~\ref{lem:cases} describes the sets $\NK(S)$ and $\NK(S\cup\{z_i\})$, so it holds for $u$. By its case~1, a
subset with at least $K$ points nearer than $z_i$ contributes nothing. The subsets with at most $K-1$ points give
the first sum; they have at most $N-1$ points. By case~3, every other subset $S$ consists of its $K$-th nearest
point $z_e$, where $e>i$, of the $K-1$ points $C$ of $S$ nearer than $z_e$, and of a set $F$ of points farther
than $z_e$. Conversely, every such $C$, $z_e$ and $F$ form a subset of this kind. Its marginal contribution is
$u(C\cup\{z_i\})-u(C\cup\{z_e\})$, and its coefficient $c_N(K+|F|)$ depends on $F$ through its size only. The sum
of the coefficients over all sets $F$ of points farther than $z_e$ is $c_e(K)$.
\end{proof}

\begin{lemma}[The two cases for one node]
\label{app:approx:lem:onenode}
For every $i$, $\Gamma_i(t)=J_i(t)+\Sigma_i(t)$ with
\begin{align}
J_i(t)&=\sum_{s=0}^{K'-1}c_N(s)\,\Bigl[M_{D\setminus\{z_i\}}[s]-M_{D\setminus\{z_i\}}[s]^{+i}\Bigr]_2,
\label{app:approx:eq:J}\\
\Sigma_i(t)&=\sum_{e>\max(i,K)}c_e(K)\,
\Bigl[M_{D_{<e}\setminus\{z_i\}}[K-1]^{+e}-M_{D_{<e}\setminus\{z_i\}}[K-1]^{+i}\Bigr]_2 .
\label{app:approx:eq:Sigma}
\end{align}
\end{lemma}

\begin{proof}
The utility $-g_t$ depends on a subset through its $K$ nearest neighbors, and $\phi_i[-g_t]=\Gamma_i(t)$
by~\eqref{app:approx:eq:Gamma}. We apply Lemma~\ref{app:approx:lem:twosums} to it. A set $S$ with at most $K-1$
points and the set $S\cup\{z_i\}$ are their own $K$ nearest neighbors. Hence $g_t(S)=[\chi(S)]_2$ and
$g_t(S\cup\{z_i\})=[\chi(S)^{+i}]_2$ by~\eqref{app:approx:eq:add}. The sum over the $s$-subsets $S$ of
$D\setminus\{z_i\}$ turns $\chi(S)$ into $M_{D\setminus\{z_i\}}[s]$, which gives $J_i(t)$. In the same way, let
$C$ be a set of $K-1$ points of $D_{<e}\setminus\{z_i\}$. The sets $C\cup\{z_e\}$ and $C\cup\{z_i\}$ have $K$
points, so $g_t(C\cup\{z_e\})=[\chi(C)^{+e}]_2$ and $g_t(C\cup\{z_i\})=[\chi(C)^{+i}]_2$, which gives
$\Sigma_i(t)$.
\end{proof}

In $J_i$ the point $z_i$ joins the neighbors, and in $\Sigma_i$ it replaces $z_e$. The lemma is
the formula of Proposition~\ref{app:exact:prop:phisq} for the utility $-g_t$: the sums over $\omega$ of
$n_2/\omega^2$ and of $\sigma_j(n,\omega)$ become the last entries of $M[s]$ and of $M[s]^{+j}$. The term $s=0$ of
$J_i$ is $-\tfrac1N\,r_i^2\,e^{-tw_i}$. After the sum over the nodes it is $\tfrac1N\,\hat v(\{z_i\})$.

\textbf{Removing a factor of $N$.} For $J_i$, the first identity of~\eqref{app:approx:eq:Mremove} gives the sums
of $D\setminus\{z_i\}$ from those of $D$, size by size, in $O(K')$ operations. For $\Sigma_i$, the second identity
expresses the sums of $D_{<e}\setminus\{z_i\}$ through those of $D_{<e}$, which all points share. We state the
result for the terms beyond an index $e_0\ge K$. The algorithm uses $e_0=K$, which gives all of $\Sigma_i$;
Section~\ref{app:approx:float} uses $e_0=2K$.

\begin{lemma}[Suffix sums]
\label{app:approx:lem:suffix}
Let $e_0\ge K$ and, for $0\le s\le K-1$,
\begin{equation}
U_i[s]=\sum_{e>\max(i,e_0)}c_e(K)\,M_{D_{<e}}[s]^{+e},
\qquad
V_i[s]=\sum_{e>\max(i,e_0)}c_e(K)\,M_{D_{<e}}[s].
\label{app:approx:eq:UV}
\end{equation}
Then the terms of~\eqref{app:approx:eq:Sigma} with $e>e_0$ sum to
\begin{equation}
\Sigma_i^{>e_0}(t)=\sum_{\tau=0}^{K-1}(-1)^{\tau}\,
\Bigl[U_i[K-1-\tau]^{+\tau i}-V_i[K-1-\tau]^{+(\tau+1)i}\Bigr]_2 .
\label{app:approx:eq:suffix}
\end{equation}
\end{lemma}

\begin{proof}
Let $e>\max(i,e_0)$. Then $z_i\in D_{<e}$, and~\eqref{app:approx:eq:Mremove} with $G=D_{<e}$ and $s=K-1$ gives
$M_{D_{<e}\setminus\{z_i\}}[K-1]=\sum_{\tau=0}^{K-1}(-1)^\tau M_{D_{<e}}[K-1-\tau]^{+\tau i}$. We apply the linear
maps $f\mapsto f^{+e}$ and $f\mapsto f^{+i}$ to this identity. By~\eqref{app:approx:eq:compose}
and~\eqref{app:approx:eq:copies}, the first commutes with $f\mapsto f^{+\tau i}$, and
$(f^{+\tau i})^{+i}=f^{+(\tau+1)i}$. The term of $z_e$ in~\eqref{app:approx:eq:Sigma} is therefore
\begin{equation*}
c_e(K)\sum_{\tau=0}^{K-1}(-1)^\tau\,
\Bigl[\bigl(M_{D_{<e}}[K-1-\tau]^{+e}\bigr)^{+\tau i}-M_{D_{<e}}[K-1-\tau]^{+(\tau+1)i}\Bigr]_2 .
\end{equation*}
The maps $f\mapsto f^{+\tau i}$ do not depend on $e$. Summing over $e>\max(i,e_0)$
gives~\eqref{app:approx:eq:suffix}.
\end{proof}

The sets $D_{<e}$ do not depend on $i$. So $U_i$ and $V_i$ are suffix sums over $e$ of terms that depend on $e$
only: for $i\ge e_0$ they arise from $U_{i+1}$ and $V_{i+1}$ by adding the term of $e=i+1$, and for $i<e_0$ they
equal $U_{e_0}$ and $V_{e_0}$. A pass over the points that adds them to the sums $M_{D_{<e}}$ in order of distance
by~\eqref{app:approx:eq:Madd} produces all these terms. Each point $z_i$ then reads
$\Sigma_i^{>e_0}$ from~\eqref{app:approx:eq:suffix} in $O(K)$ operations, with its own $w_i$ and $r_i$. This is
the step of Section~\ref{sec:shared} in which the sum over $e$ is taken on the tables (Algorithm~\ref{app:exact:alg:linear}),
for one node.

\begin{remark}[The removal identity does not change the error]
\label{app:approx:rem:exactidentity}
The identities~\eqref{app:approx:eq:Mremove} and~\eqref{app:approx:eq:suffix} are exact. They hold for every
node $t$, because they are identities between sums over subsets and use of the numbers $e^{-tw_j}$ only that
$e^{-tA(C)}$ is their product over $z_j\in C$. They are applied after the rule in $t$ is fixed, to a utility that
is already the approximated one. The numbers that the algorithm returns are therefore the values $\hat\phi_i$
themselves, and the error bound~\eqref{app:approx:eq:bound} is not affected. In particular, the terms with
$\tau\ge1$ belong to sets in which $z_i$ is counted several times. Their total weights can exceed $K'w_{\max}$,
where the rule in $t$ is not known to be accurate. This does no harm: these terms cancel exactly, and the error
bound needs the rule only at the total weights of actual sets of neighbors.
\end{remark}

\begin{algorithm}[t]
\caption{Data Shapley values for continuous weights within a tolerance $\varepsilon$, in time linear in $N$
(Theorems~\ref{thm:cont} and~\ref{app:approx:thm:explicit}). An array of sums holds one triple for each size
$s=0,\dots,K'-1$; the sums of the empty set are $(1,0,0)$ for $s=0$ and $(0,0,0)$ otherwise. Adding $z_j$ to an
array $M$ means $M[s]\gets M[s]+M[s-1]^{+j}$ for $s=K'-1,\dots,1$. The operations~\eqref{app:approx:eq:ops}
and~\eqref{app:approx:eq:copies} use the stored number $E_j$ for $e^{-tw_j}$ and its powers.}
\label{app:approx:alg:linear}
\begin{algorithmic}[1]
\Require weights $w_j$ and targets $y_j$ of the training points in order of distance; $y^{(\val)}$; $y_\emptyset$;
$K$; bounds $w_{\min}$, $w_{\max}$, $B$; tolerance $0<\varepsilon\le B^2$
\Ensure $\hat\phi_1,\dots,\hat\phi_N$ with $|\hat\phi_i-\phi_i|\le\varepsilon$ for every $i$
\State $K'\gets\min(K,N)$;\quad $r_j\gets w_j\,(y_j-y^{(\val)})$ for all $j$;\quad $k_{\min}$, $k_{\max}$, $n$
from~\eqref{app:approx:eq:params};\quad $(t_\ell,\mu_\ell)_{\ell=1}^{L}\gets$ the rule in $t$
\State $\hat\phi_i\gets(y_\emptyset-y^{(\val)})^2/N$ for all $i$ \Comment{the term $-v(\emptyset)/N$}
\For{$\ell=1,\dots,L$}
  \State $E_j\gets\exp(-t_\ell\,w_j)$ for all $j$ \label{app:approx:lin:exp}
  \State $M\gets$ the sums of the empty set;\quad $U[s],V[s]\gets(0,0,0)$ for $s=0,\dots,K'-1$
  \For{$e=1,\dots,N$} \Comment{first pass; here $M=M_{D_{<e}}$}
    \If{$e>K$}
      \State $U[s]\gets U[s]+c_e(K)\,M[s]^{+e}$ and $V[s]\gets V[s]+c_e(K)\,M[s]$ for $s=0,\dots,K-1$
      \label{app:approx:lin:total}
    \EndIf
    \State add $z_e$ to $M$ \label{app:approx:lin:add}
  \EndFor
  \State $\bar M\gets M$;\quad $M\gets$ the sums of the empty set \Comment{$\bar M=M_D$}
  \For{$i=1,\dots,N$} \Comment{second pass; here $M=M_{D_{<i}}$}
    \If{$i>K$}
      \State $U[s]\gets U[s]-c_i(K)\,M[s]^{+i}$ and $V[s]\gets V[s]-c_i(K)\,M[s]$ for $s=0,\dots,K-1$
      \label{app:approx:lin:suffix}
    \EndIf
    \State $\Gamma\gets0$ \Comment{now $U=U_i$ and $V=V_i$ for $e_0=K$}
    \If{$N>K$}
      \State $\Gamma\gets\sum_{\tau=0}^{K-1}(-1)^\tau\bigl[U[K-1-\tau]^{+\tau i}
      -V[K-1-\tau]^{+(\tau+1)i}\bigr]_2$ \Comment{$z_i$ replaces a neighbor~\eqref{app:approx:eq:suffix}}
      \label{app:approx:lin:far}
    \EndIf
    \State $f\gets(1,0,0)$;\quad $\Gamma\gets\Gamma+c_N(0)\,[f-f^{+i}]_2$ \label{app:approx:lin:join0}
    \For{$s=1,\dots,K'-1$} \Comment{$z_i$ joins; $f$ becomes $M_{D\setminus\{z_i\}}[s]$}
      \State $f\gets\bar M[s]-f^{+i}$;\quad $\Gamma\gets\Gamma+c_N(s)\,[f-f^{+i}]_2$ \label{app:approx:lin:join}
    \EndFor
    \State $\hat\phi_i\gets\hat\phi_i+\mu_\ell\,t_\ell\,\Gamma$ \label{app:approx:lin:sum}
    \State add $z_i$ to $M$
  \EndFor
\EndFor
\State \Return $\hat\phi_1,\dots,\hat\phi_N$
\end{algorithmic}
\end{algorithm}

\textbf{Pseudocode.} Algorithm~\ref{app:approx:alg:linear} evaluates~\eqref{app:approx:eq:nodes}. For each node
it computes the $N$ numbers $e^{-t_\ell w_j}$ once and then $\Gamma_i(t_\ell)$ for all $i$ by
Lemma~\ref{app:approx:lem:onenode} and by Lemma~\ref{app:approx:lem:suffix} with $e_0=K$. The first pass forms
the sums~\eqref{app:approx:eq:UV} over all $e>K$ and the sums $M_D$ of all points. The second pass goes through
the points again. When it reaches a point $z_i$ with $i>K$, it takes the term of $e=i$ out of $U$ and $V$, which
then hold $U_i$ and $V_i$. Line~\ref{app:approx:lin:far} computes $\Sigma_i$; if $N\le K$, no neighbor is ever
replaced and $\Sigma_i=0$. Lines~\ref{app:approx:lin:join0} and~\ref{app:approx:lin:join} add $J_i$
by~\eqref{app:approx:eq:J}, where the sums of $D\setminus\{z_i\}$ come from $M_D$
by~\eqref{app:approx:eq:Mremove}. Line~\ref{app:approx:lin:sum} forms the sum over the nodes.

\textbf{Cost and memory.} The coefficients $c_N(s)$ for $s<K'$ and $c_e(K)$ for $K<e\le N$ are computed once in
$O(N)$ operations, from $c_N(0)=1/N$, $c_N(s+1)=c_N(s)\,(s+1)/(N-1-s)$, $c_{K+1}(K)=1/(K+1)$ and
$c_{e+1}(K)=c_e(K)\,(e-K)/(e+1)$. For one node and one point, each of the
lines~\ref{app:approx:lin:total}, \ref{app:approx:lin:add}, \ref{app:approx:lin:suffix}
and~\ref{app:approx:lin:far} and the loop of line~\ref{app:approx:lin:join} takes $O(K')$ operations. In
lines~\ref{app:approx:lin:total}, \ref{app:approx:lin:suffix} and~\ref{app:approx:lin:far} we have $N>K$ and hence
$K'=K$, and the powers of $E_i$ in line~\ref{app:approx:lin:far} are formed by $K$ multiplications. One node
therefore takes $O(NK')$ operations, and the algorithm uses $O(NK'L)$ arithmetic operations in total. The
exponential function is evaluated only in line~\ref{app:approx:lin:exp}, $NL$ times. The algorithm stores the
numbers $r_j$, $E_j$ and $\hat\phi_i$ and the coefficients, which are $O(N)$ numbers, and four arrays of $K'$
triples. This completes the proof of Theorem~\ref{app:approx:thm:explicit}.

\subsection{Soft-Label Classification with Continuous Weights}
\label{app:approx:soft}

With real weights, the soft-label utility~\eqref{eq:soft} is the weight share of the query's class,
\begin{equation}
v(S)=\frac{R(\NK(S))}{A(\NK(S))}\ \ \text{for }S\ne\emptyset,\qquad v(\emptyset)=0,
\qquad\text{where now }r_j=w_j\,\mathbf{1}[y_j=y^{(\val)}] .
\label{app:approx:eq:share}
\end{equation}
It is linear in $R$. So $1/A$ takes the place of $1/A^2$, and the identity $1/A=\int_0^\infty e^{-tA}\,dt$ takes
the place of~\eqref{eq:laplace}. For a rule in $t$ with nodes $t_\ell$ and weights $\mu_\ell$ let
\begin{equation}
q_1(A)=\sum_{\ell=1}^{L}\mu_\ell\,e^{-t_\ell A},
\qquad
\hat v_1(S)=R(\NK(S))\,q_1\bigl(A(\NK(S))\bigr)\ \ \text{for }S\ne\emptyset,
\qquad\hat v_1(\emptyset)=0 .
\label{app:approx:eq:q1}
\end{equation}

\begin{corollary}[Weight share with continuous weights]
\label{app:approx:cor:share}
Let $0<w_{\min}\le w_j\le w_{\max}$ for all $j$, let $K'=\min(K,N)$ and $0<\varepsilon\le2$. Put
\begin{equation}
k_{\min}=\biggl\lfloor\log_2\frac{\varepsilon}{6\,K'w_{\max}}\biggr\rfloor,\qquad
k_{\max}=\biggl\lceil\log_2\frac{\ln(6/\varepsilon)}{w_{\min}}\biggr\rceil,\qquad
n=\biggl\lceil\frac12+\frac{\ln(36/\varepsilon)}{2\ln\varrho}\biggr\rceil .
\label{app:approx:eq:params1}
\end{equation}
Then $k_{\min}<k_{\max}$. For the rule in $t$ with $n$ nodes on each of the pieces $[2^k,2^{k+1}]$,
$k_{\min}\le k<k_{\max}$, the Shapley values of $\hat v_1$ are within $\varepsilon$ of the Shapley values of the
utility~\eqref{app:approx:eq:share}. In exact real arithmetic, once the points are sorted by distance, they are
computed with $O(NK'L)$ arithmetic operations and $NL$ evaluations of the exponential function, where
\begin{equation*}
L=n\,(k_{\max}-k_{\min})=O\Bigl(\lambda_1\bigl(\lambda_1+\log\tfrac{K'w_{\max}}{w_{\min}}\bigr)\Bigr),
\qquad\lambda_1=1+\log\tfrac{2}{\varepsilon}.
\end{equation*}
The algorithm is Algorithm~\ref{app:approx:alg:linear} with two sums per size in place of three.
\end{corollary}

The proof follows that of Theorem~\ref{app:approx:thm:explicit}. The only new ingredient is the sum over the
dyadic pieces for the function $e^{-tA}$ in place of $t\,e^{-tA}$.

\begin{lemma}[The sum over the dyadic pieces, for the share]
\label{app:approx:lem:dyadic1}
Let $\psi_1(\alpha)=\alpha\,e^{-\alpha/2}$. For every $A>0$, $\sum_{k\in\mathbb{Z}}\psi_1(2^kA)\le4$.
\end{lemma}

\begin{proof}
The derivative $\psi_1'(\alpha)=(1-\alpha/2)\,e^{-\alpha/2}$ shows that $\psi_1$ increases on $[0,2]$ and
decreases on $[2,\infty)$. As in the proof of Lemma~\ref{app:approx:lem:dyadic}, let $\alpha_*$ be the number of
the form $2^kA$ in $[1,2)$, so that the numbers $2^kA$ are $2^{j}\alpha_*$, $j\in\mathbb{Z}$. We bound four groups
of them.

For $j\le-3$ we use $\psi_1(\alpha)\le\alpha$ and $2^j\alpha_*<2^{j+1}$: the sum is at most
$\sum_{j\le-3}2^{j+1}=\tfrac12$.

For $j=-2$ and $j=-1$ the numbers lie in $[\tfrac14,\tfrac12)$ and $[\tfrac12,1)$. Their values are at most
$\psi_1(\tfrac12)<0.39$ and $\psi_1(1)<0.61$.

For $0\le j\le4$ the numbers lie in $[1,2)$, $[2,4)$, $[4,8)$, $[8,16)$ and $[16,32)$. Their values are at most
$\psi_1(2)$, $\psi_1(2)$, $\psi_1(4)$, $\psi_1(8)$ and $\psi_1(16)$, which are below $0.74$, $0.74$, $0.55$,
$0.15$ and $0.01$.

For $j\ge5$ the numbers are at least $32$. There $\alpha\le\alpha^2\le e^{\alpha/4}$, as in the proof of
Lemma~\ref{app:approx:lem:dyadic}. Hence $\psi_1(\alpha)\le e^{-\alpha/4}$, and the sum is at most
$\sum_{j\ge0}e^{-8\cdot2^{j}}<0.001$.

The four bounds add up to less than $3.7$.
\end{proof}

\begin{proof}[Proof of Corollary~\ref{app:approx:cor:share}]
\emph{From the rule to the Shapley value.} Since $0\le r_j\le w_j$, every set $C$ has $0\le R(C)\le A(C)$. Let
$S$ be nonempty and write $A=A(\NK(S))$ and $R=R(\NK(S))$. By~\eqref{app:approx:eq:range},
$w_{\min}\le A\le K'w_{\max}$, and therefore
\begin{equation*}
|\hat v_1(S)-v(S)|=R\,\Bigl|q_1(A)-\frac1A\Bigr|=\frac RA\,\bigl|A\,q_1(A)-1\bigr|\le\eta_1,
\qquad
\eta_1=\max_{w_{\min}\le A\le K'w_{\max}}\bigl|A\,q_1(A)-1\bigr| .
\end{equation*}
The two utilities agree on the empty subset. By Lemma~\ref{app:approx:lem:close}, their Shapley values differ by
at most $2\eta_1$.

\emph{The rule.} Let $A>0$. We have $A\int_0^\infty e^{-tA}\,dt=1$, and for $\tau>0$ the two tails are
$A\int_0^\tau e^{-tA}\,dt=1-e^{-\tau A}\le\tau A$ and $A\int_\tau^\infty e^{-tA}\,dt=e^{-\tau A}$. On a piece
$[\tau,2\tau]$, the substitution $t=\tfrac12\tau\,(3+\xi)$ gives, with $\alpha=\tau A$,
\begin{equation*}
A\int_\tau^{2\tau}e^{-tA}\,dt=\int_{-1}^{1}G_1(\xi)\,d\xi,
\qquad
G_1(\xi)=\frac{\alpha}{2}\,e^{-\alpha(3+\xi)/2},
\end{equation*}
and the value of the $n$-point rule on the piece for the function $A\,e^{-tA}$ is $\sum_j\beta_j\,G_1(\xi_j)$. As
in the proof of Lemma~\ref{app:approx:lem:piece}, the real part of a point of $\mathcal{E}_\varrho$ is at least
$-2$, so $|G_1|\le\tfrac12\,\psi_1(\alpha)$ on $\mathcal{E}_\varrho$. Lemma~\ref{app:approx:lem:gauss} then bounds
the error on the piece by $\frac{4}{1+\sqrt3}\,\psi_1(\alpha)\,\varrho^{\,1-2n}$. By
Lemma~\ref{app:approx:lem:dyadic1}, the errors of all pieces sum to at most
$\frac{16}{1+\sqrt3}\,\varrho^{\,1-2n}\le6\,\varrho^{\,1-2n}$. With the two tails, for every $A>0$,
\begin{equation}
\bigl|A\,q_1(A)-1\bigr|\le2^{k_{\min}}A+e^{-2^{k_{\max}}A}+6\,\varrho^{\,1-2n}.
\label{app:approx:eq:eta1}
\end{equation}

\emph{The parameters.} Let $w_{\min}\le A\le K'w_{\max}$. The floor in $k_{\min}$ gives
$2^{k_{\min}}A\le2^{k_{\min}}K'w_{\max}\le\varepsilon/6$. The ceiling in $k_{\max}$ gives
$2^{k_{\max}}A\ge2^{k_{\max}}w_{\min}\ge\ln(6/\varepsilon)$, so the second term of~\eqref{app:approx:eq:eta1} is at
most $\varepsilon/6$. The ceiling in $n$ gives $(2n-1)\ln\varrho\ge\ln(36/\varepsilon)$, so the third term is at
most $\varepsilon/6$. Hence $\eta_1\le\varepsilon/2$, and the Shapley values of $\hat v_1$ are within
$\varepsilon$. The two numbers under the floor and the ceiling of $k_{\min}$ and $k_{\max}$ differ by
\begin{equation*}
\log_2\Bigl(\frac6\varepsilon\,\ln\frac6\varepsilon\cdot\frac{K'w_{\max}}{w_{\min}}\Bigr)\ge\log_2(3\ln3)>1,
\end{equation*}
because $\varepsilon\le2$. So $k_{\min}<k_{\max}$, and
$k_{\max}-k_{\min}\le\log_2(K'w_{\max}/w_{\min})+\log_2(6/\varepsilon)+\log_2\ln(6/\varepsilon)+2$. With
$n\le\tfrac32+\ln(36/\varepsilon)/(2\ln\varrho)$ this gives the bound on $L$, as in the proof of
Theorem~\ref{thm:cont} from Theorem~\ref{app:approx:thm:explicit}.

\emph{The algorithm.} For $t>0$ let $g^{(1)}_t(T)$ be the middle entry $R\,e^{-tA}$ of the triple
$\chi(\NK(T))$. Then $\hat v_1(T)=\sum_\ell\mu_\ell\,g^{(1)}_{t_\ell}(T)$ for every $T$, and the Shapley value is
linear in the utility, so $\phi_i[\hat v_1]=\sum_\ell\mu_\ell\,\phi_i\bigl[g^{(1)}_{t_\ell}\bigr]$.
Lemmas~\ref{app:approx:lem:twosums}, \ref{app:approx:lem:onenode} and~\ref{app:approx:lem:suffix} apply to the
utility $g^{(1)}_t$ with the middle entry of a triple in place of the last: $\phi_i\bigl[g^{(1)}_t\bigr]$ is
$-J_i(t)-\Sigma_i(t)$ with the middle entries in~\eqref{app:approx:eq:J}, \eqref{app:approx:eq:Sigma}
and~\eqref{app:approx:eq:suffix}. The first two entries of $f^{+j}$ and of $f^{+\tau j}$ depend on the first two
entries of $f$ only. Algorithm~\ref{app:approx:alg:linear} therefore runs with the first two entries of every
triple. It starts from $\hat\phi_i=0$, reads the middle entry where it reads the last, and subtracts
$\mu_\ell\,\Gamma$ in line~\ref{app:approx:lin:sum} where it adds $\mu_\ell\,t_\ell\,\Gamma$. Its cost is the one
counted for Theorem~\ref{app:approx:thm:explicit}.
\end{proof}

\textbf{Brier score.} The Brier score over $m$ classes needs no new result. It is the sum, over the classes $c$,
of the squared-loss utility on the indicator targets $\mathbf{1}[y_j=c]$ (Appendix~\ref{app:exact:soft}). For
$0<\varepsilon\le m$, one run of Theorem~\ref{thm:cont} per class, with $B=1$ and the tolerance $\varepsilon/m$,
therefore gives its Shapley values within $\varepsilon$.

\subsection{A Second Algorithm, with Quadrature over the Inclusion Probability}
\label{app:approx:second}

This section gives a second way to compute the numbers $\Gamma_i(t)$ of~\eqref{app:approx:eq:Gamma}, and
by~\eqref{app:approx:eq:nodes} the same values $\hat\phi_i$. In place of the two cases of
Lemma~\ref{app:approx:lem:onenode}, it sums over the subsets by an integral over an inclusion probability $p$. It
is quadratic in $N$. Its advantage is that it never takes a point out of a sum, and the effect of floating-point
rounding on it can be bounded (Section~\ref{app:approx:float}).

\textbf{The integral over $p$.} Fix a training point $z_i$. For $p\in[0,1]$ and a function $f$ of the subsets of
$D\setminus\{z_i\}$ let
\begin{equation}
\mathbb{E}_p[f]=\sum_{S\subseteq D\setminus\{z_i\}}p^{|S|}\,(1-p)^{N-1-|S|}\,f(S)
\label{app:approx:eq:Ep}
\end{equation}
be the expectation of $f(S)$ when $S$ contains every point other than $z_i$ independently with probability $p$,
and for $t>0$ let
\begin{equation}
\Delta_i(p,t)=\mathbb{E}_p\bigl[\,g_t(S\cup\{z_i\})-g_t(S)\,\bigr].
\label{app:approx:eq:Delta}
\end{equation}
In words, $\Delta_i(p,t)$ is the expected change of $R^2e^{-tA}$ over the $K$ nearest neighbors when $z_i$ is
added to a random subset. The \emph{rule in $p$} is the Gauss--Legendre rule on $[0,1]$ with
$n_p=\lceil N/2\rceil$ nodes; we write $(p_h,\nu_h)$, $h=1,\dots,n_p$, for its nodes and weights. The first lemma
is the reason why the rule in $p$ causes no error.

\begin{lemma}[The integral over $p$ is computed exactly]
\label{app:approx:lem:beta}
For $0\le s\le N-1$, $c_N(s)=\int_0^1p^s(1-p)^{N-1-s}\,dp$. Consequently, for every function $f$ of the subsets of
$D\setminus\{z_i\}$,
\begin{equation}
\sum_{S\subseteq D\setminus\{z_i\}}c_N(|S|)\,f(S)=\int_0^1\mathbb{E}_p[f]\,dp=\sum_{h=1}^{n_p}\nu_h\,\mathbb{E}_{p_h}[f].
\label{app:approx:eq:owen}
\end{equation}
\end{lemma}

\begin{proof}
For integers $s,s'\ge0$ let $I(s,s')=\int_0^1p^s(1-p)^{s'}\,dp$. Then $I(s,0)=1/(s+1)$, and integration by parts
gives $I(s,s')=\frac{s'}{s+1}\,I(s+1,s'-1)$ for $s'\ge1$, because the boundary term vanishes. Applying this $s'$
times,
\begin{equation*}
I(s,s')=\frac{s'}{s+1}\cdot\frac{s'-1}{s+2}\cdots\frac{1}{s+s'}\cdot\frac{1}{s+s'+1}=\frac{s!\,s'!}{(s+s'+1)!}.
\end{equation*}
For $s'=N-1-s$ this is $s!\,(N-1-s)!/N!=c_N(s)$. Multiplying by $f(S)$ and summing over $S$ gives the first
equality of~\eqref{app:approx:eq:owen}, by the definition~\eqref{app:approx:eq:Ep}. Every term of~\eqref{app:approx:eq:Ep} is a
polynomial in $p$ of degree $N-1$, so $\mathbb{E}_p[f]$ is a polynomial of degree at most $N-1$. The rule in $p$ is
exact for polynomials of degree at most $2\lceil N/2\rceil-1\ge N-1$, which is the second equality.
\end{proof}

The representation of the Shapley value as the integral~\eqref{app:approx:eq:owen} is due to
\citet{owen1972multilinear}.

\begin{lemma}[The numbers $\Gamma_i(t)$ from the rule in $p$]
\label{app:approx:lem:pnodes}
For every $i$ and every $t>0$, $\Gamma_i(t)=-\sum_{h=1}^{n_p}\nu_h\,\Delta_i(p_h,t)$. Consequently,
\begin{equation}
\hat\phi_i=-\frac{v(\emptyset)}{N}-\sum_{\ell=1}^{L}\mu_\ell\,t_\ell\sum_{h=1}^{n_p}\nu_h\,\Delta_i(p_h,t_\ell),
\qquad i=1,\dots,N .
\label{app:approx:eq:output}
\end{equation}
\end{lemma}

\begin{proof}
Apply Lemma~\ref{app:approx:lem:beta} to the function $f(S)=g_t(S\cup\{z_i\})-g_t(S)$. The left-hand side
of~\eqref{app:approx:eq:owen} is $-\Gamma_i(t)$ by~\eqref{app:approx:eq:Gamma}, and the right-hand side is
$\sum_h\nu_h\,\Delta_i(p_h,t)$ by~\eqref{app:approx:eq:Delta}. The second claim then
is~\eqref{app:approx:eq:nodes}.
\end{proof}

It remains to compute the numbers $\Delta_i(p,t)$ for all $i$ together. We fix $p\in[0,1]$ and $t>0$ until
Lemma~\ref{app:approx:lem:diff}, and use the triples $\chi(C)$ and the two operations
of~\eqref{app:approx:eq:ops}.

\textbf{Prefix and suffix triples.} Let every training point be present independently with probability $p$. For
$0\le j\le N$ let $S_{\le j}$ be the set of the present points among $z_1,\dots,z_j$, and for $1\le j\le N+1$ let
$S_{\ge j}$ be the set of the present points among $z_j,\dots,z_N$. For a subset $S$ and $k\ge0$ we write $N_k(S)$
for the $\min(k,|S|)$ nearest points of $S$. We define the prefix triples and the suffix triples
\begin{align}
F_j[m]&=\mathbb{E}\bigl[\chi(S_{\le j})\,\mathbf{1}[\,|S_{\le j}|=m\,]\bigr],
&&0\le j\le N,\ \ 0\le m\le K'-1,
\label{app:approx:eq:F}\\
H_j[k]&=\mathbb{E}\bigl[\chi\bigl(N_k(S_{\ge j})\bigr)\bigr],
&&1\le j\le N+1,\ \ 0\le k\le K'.
\label{app:approx:eq:H}
\end{align}
In $F_j[m]$ the state $m$ is the number of neighbors taken among the first $j$ points. Since $m<K'$, all present
points are neighbors. In $H_j[k]$ the number $k$ counts the places among the $K'$ neighbors that are still open
when the pass reaches $z_j$.

\begin{lemma}[One pass in each direction]
\label{app:approx:lem:pass}
The prefix triples satisfy $F_0[0]=(1,0,0)$, $F_0[m]=(0,0,0)$ for $m\ge1$, and
\begin{equation}
F_j[m]=(1-p)\,F_{j-1}[m]+p\,F_{j-1}[m-1]^{+j},\qquad 1\le j\le N,\ \ 0\le m\le K'-1,
\label{app:approx:eq:Frec}
\end{equation}
where $F_{j-1}[-1]=(0,0,0)$. The suffix triples satisfy $H_{N+1}[k]=(1,0,0)$ for all $k$, $H_j[0]=(1,0,0)$ for
all $j$, and
\begin{equation}
H_j[k]=(1-p)\,H_{j+1}[k]+p\,H_{j+1}[k-1]^{+j},\qquad 1\le j\le N,\ \ 1\le k\le K'.
\label{app:approx:eq:Hrec}
\end{equation}
\end{lemma}

\begin{proof}
The set $S_{\le0}$ is empty, which gives $F_0$. The point $z_j$ is present with probability $p$, independently of
$S_{\le j-1}$. If it is absent, $S_{\le j}=S_{\le j-1}$. If it is present, $S_{\le j}$ is $S_{\le j-1}$ with one
more point, and $\chi(S_{\le j})=\chi(S_{\le j-1})^{+j}$ by~\eqref{app:approx:eq:add}. The expectation over the two cases
is~\eqref{app:approx:eq:Frec}.

For $k=0$ and for $j=N+1$ the set $N_k(S_{\ge j})$ is empty, which gives the initial values of $H$. Let
$1\le j\le N$ and $k\ge1$. If $z_j$ is absent, $N_k(S_{\ge j})=N_k(S_{\ge j+1})$. If it is present, it is the
nearest point of $S_{\ge j}$, so $N_k(S_{\ge j})$ consists of $z_j$ and of $N_{k-1}(S_{\ge j+1})$, and its triple
is $\chi(N_{k-1}(S_{\ge j+1}))^{+j}$. The expectation over the two cases is~\eqref{app:approx:eq:Hrec}.
\end{proof}

\begin{lemma}[Splitting at a point]
\label{app:approx:lem:split}
Let $T\subseteq D$ and $0\le j\le N$. Let $T_{\le j}$ be the set of the points of $T$ among $z_1,\dots,z_j$, let
$T_{>j}$ be the set of the other points of $T$, and let $m=\min(K',|T_{\le j}|)$. Then $\NK(T)$ is the disjoint
union of $\NK(T_{\le j})$ and $N_{K'-m}(T_{>j})$.
\end{lemma}

\begin{proof}
Every point of $T_{\le j}$ is nearer to the query than every point of $T_{>j}$. A list of $T$ in order of distance
therefore begins with the points of $T_{\le j}$, and $\NK(T)$ consists of the first $\min(K',|T|)$ points of the
list. The first $m$ of them are $\NK(T_{\le j})$. If $m=K'$, there are no others. If $m<K'$, then $m=|T_{\le j}|$,
and the remaining $\min(K'-m,|T_{>j}|)$ points are the nearest points of $T_{>j}$.
\end{proof}

\begin{lemma}[The difference for $z_i$]
\label{app:approx:lem:diff}
For every $i$,
\begin{equation}
\Delta_i(p,t)=\sum_{m=0}^{K'-1}\Bigl(\bigl\langle F_{i-1}[m]^{+i},\,H_{i+1}[K'-m-1]\bigr\rangle
-\bigl\langle F_{i-1}[m],\,H_{i+1}[K'-m]\bigr\rangle\Bigr).
\label{app:approx:eq:diff}
\end{equation}
Given the triples, the right-hand side takes $O(K')$ operations.
\end{lemma}

\begin{proof}
In~\eqref{app:approx:eq:Delta}, $S$ is a random subset of $D\setminus\{z_i\}$. Its points before $z_i$ form the set
$S_{\le i-1}$ and its points after $z_i$ form the set $S_{\ge i+1}$; the two sets are independent. Let
$m=|S_{\le i-1}|$.

If $m\ge K'$, Lemma~\ref{app:approx:lem:split} with $j=i-1$ shows that $\NK(S\cup\{z_i\})$ and $\NK(S)$ both equal
$\NK(S_{\le i-1})$. Then $g_t(S\cup\{z_i\})=g_t(S)$, as in case~1 of Lemma~\ref{lem:cases}.

If $m\le K'-1$, we apply Lemma~\ref{app:approx:lem:split} with $j=i$ to $S\cup\{z_i\}$ and to $S$. The first set has the
$m+1\le K'$ points $S_{\le i-1}\cup\{z_i\}$ up to $z_i$, and the second has the $m$ points $S_{\le i-1}$. Hence
\begin{equation*}
\NK(S\cup\{z_i\})=S_{\le i-1}\cup\{z_i\}\cup N_{K'-m-1}(S_{\ge i+1}),
\qquad
\NK(S)=S_{\le i-1}\cup N_{K'-m}(S_{\ge i+1}).
\end{equation*}
By~\eqref{app:approx:eq:add},
\begin{equation*}
g_t(S\cup\{z_i\})=\bigl\langle\chi(S_{\le i-1})^{+i},\,\chi\bigl(N_{K'-m-1}(S_{\ge i+1})\bigr)\bigr\rangle,
\qquad
g_t(S)=\bigl\langle\chi(S_{\le i-1}),\,\chi\bigl(N_{K'-m}(S_{\ge i+1})\bigr)\bigr\rangle .
\end{equation*}
We multiply both by the indicator of $|S_{\le i-1}|=m$ and take expectations. By the independence of the two sets
and the linearity of the two operations, the expectations are the two terms of~\eqref{app:approx:eq:diff} for this $m$.
Summing over $m=0,\dots,K'-1$ gives~\eqref{app:approx:eq:diff}. Each term takes a constant number of operations.
\end{proof}

In words, $z_i$ either joins $m\le K'-1$ nearer neighbors, which leaves one place fewer for the farther points, or
it is absent. The prefix triples describe the nearer points and the suffix triples the farther ones, and the same
two tables serve all $i$. This is what brings the cost per pair $(p,t)$ from $O(N^2K')$, for one pass per point
$z_i$, down to $O(NK')$.

\begin{algorithm}[t]
\caption{Second algorithm for the values $\hat\phi_i$ of Theorem~\ref{app:approx:thm:explicit}, with quadrature
over the inclusion probability; quadratic in $N$. The operation $f^{+j}$ of~\eqref{app:approx:eq:ops} uses the
stored number $E_j$ for $e^{-tw_j}$.}
\label{app:approx:alg:second}
\begin{algorithmic}[1]
\Require weights $w_j$ and targets $y_j$ of the training points in order of distance; $y^{(\val)}$; $y_\emptyset$;
$K$; bounds $w_{\min}$, $w_{\max}$, $B$; tolerance $0<\varepsilon\le B^2$
\Ensure $\hat\phi_1,\dots,\hat\phi_N$ with $|\hat\phi_i-\phi_i|\le\varepsilon$ for every $i$
\State $K'\gets\min(K,N)$;\quad $r_j\gets w_j\,(y_j-y^{(\val)})$ for all $j$;\quad $k_{\min}$, $k_{\max}$, $n$
from~\eqref{app:approx:eq:params}
\State $(t_\ell,\mu_\ell)_{\ell=1}^{L}\gets$ the rule in $t$;\quad $(p_h,\nu_h)_{h=1}^{n_p}\gets$ the rule in $p$
\State $\hat\phi_i\gets(y_\emptyset-y^{(\val)})^2/N$ for all $i$ \Comment{the term $-v(\emptyset)/N$}
\label{app:approx:ln:init}
\For{$\ell=1,\dots,L$}
  \State $E_j\gets\exp(-t_\ell\,w_j)$ for all $j$;\quad $s_i\gets0$ for all $i$ \label{app:approx:ln:exp}
  \For{$h=1,\dots,n_p$}
    \State $H_{N+1}[k]\gets(1,0,0)$ for $k=0,\dots,K'$
    \For{$j=N,N-1,\dots,2$} \Comment{suffix triples~\eqref{app:approx:eq:Hrec}}
      \State $H_j[0]\gets(1,0,0)$
      \State $H_j[k]\gets(1-p_h)\,H_{j+1}[k]+p_h\,H_{j+1}[k-1]^{+j}$ for $k=1,\dots,K'$ \label{app:approx:ln:suffix}
    \EndFor
    \State $F[0]\gets(1,0,0)$;\quad $F[m]\gets(0,0,0)$ for $m=1,\dots,K'-1$ \Comment{prefix triples $F_0$}
    \For{$i=1,\dots,N$}
      \State $\Delta\gets\sum_{m=0}^{K'-1}\bigl(\langle F[m]^{+i},H_{i+1}[K'-m-1]\rangle-\langle
      F[m],H_{i+1}[K'-m]\rangle\bigr)$ \Comment{\eqref{app:approx:eq:diff}} \label{app:approx:ln:delta}
      \State $s_i\gets s_i+\nu_h\,\Delta$ \label{app:approx:ln:sumh}
      \State $F[m]\gets(1-p_h)\,F[m]+p_h\,F[m-1]^{+i}$ for $m=K'-1,\dots,0$, with $F[-1]=(0,0,0)$
      \Comment{\eqref{app:approx:eq:Frec}} \label{app:approx:ln:prefix}
    \EndFor
  \EndFor
  \State $\hat\phi_i\gets\hat\phi_i-\mu_\ell\,t_\ell\,s_i$ for all $i$ \label{app:approx:ln:sumt}
\EndFor
\State \Return $\hat\phi_1,\dots,\hat\phi_N$
\end{algorithmic}
\end{algorithm}

\begin{proposition}[The second algorithm]
\label{app:approx:prop:second}
In exact real arithmetic, Algorithm~\ref{app:approx:alg:second} computes the values $\hat\phi_i$ of the
approximated game~\eqref{app:approx:eq:eta}, the same as Algorithm~\ref{app:approx:alg:linear}, with
$O\bigl(NK'\lceil N/2\rceil L\bigr)$ arithmetic operations and $NL$ evaluations of the exponential function. The
error bound~\eqref{app:approx:eq:bound} therefore holds for it.
\end{proposition}

\begin{proof}
Algorithm~\ref{app:approx:alg:second} evaluates~\eqref{app:approx:eq:output}. For each node $t_\ell$ it computes
the $N$ numbers $e^{-t_\ell w_j}$ once. For each node $p_h$ it then runs a backward pass over the points, which
stores the suffix triples, and a forward pass, which keeps the prefix triples $F_{i-1}$ in the array $F$. By
Lemma~\ref{app:approx:lem:pass} the two passes compute the triples~\eqref{app:approx:eq:F}
and~\eqref{app:approx:eq:H} for $p=p_h$ and $t=t_\ell$, and by Lemma~\ref{app:approx:lem:diff}
line~\ref{app:approx:ln:delta} computes $\Delta_i(p_h,t_\ell)$. Line~\ref{app:approx:ln:sumh} forms the sum over
$h$ and line~\ref{app:approx:ln:sumt} the sum over $\ell$ of~\eqref{app:approx:eq:output}. By
Lemma~\ref{app:approx:lem:pnodes} the result is $\hat\phi_i$.

For one pair $(\ell,h)$ the backward pass updates at most $NK'$ triples, and the forward pass spends $O(K')$
operations per point on line~\ref{app:approx:ln:delta} and on line~\ref{app:approx:ln:prefix}. This is $O(NK')$
operations per pair and $O(NK'\,n_p\,L)$ operations in total, with $n_p=\lceil N/2\rceil$. The exponential
function is evaluated only in line~\ref{app:approx:ln:exp}, $NL$ times.
\end{proof}

\textbf{Memory.} Algorithm~\ref{app:approx:alg:second} stores the $3\,N\,(K'+1)$ numbers of the suffix triples and
$O(N+K')$ further numbers, besides the nodes and weights of the two rules. The suffix triples can instead be stored
at every $\lceil\sqrt N\rceil$-th point only and recomputed block by block during the forward pass, which lowers
the memory to $O(\sqrt N\,K'+N)$ numbers at the price of one more backward pass.

\begin{remark}[The passes do not use the form of the factors]
\label{app:approx:rem:factors}
Lemmas~\ref{app:approx:lem:pass} and~\ref{app:approx:lem:diff} use of the numbers $e^{-tw_j}$ only that $e^{-tA(C)}$ is their
product over $z_j\in C$. Both lemmas remain true for arbitrary real numbers $E_j$ in place of $e^{-tw_j}$ and
arbitrary real numbers $r_j$, when $e^{-tA(C)}$ is read as $\prod_{z_j\in C}E_j$ in~\eqref{app:approx:eq:g}
and~\eqref{app:approx:eq:chi}. Section~\ref{app:approx:float} uses this for the computed values of the exponentials.
\end{remark}

\subsection{Floating-Point Arithmetic}
\label{app:approx:float}

Theorem~\ref{thm:cont} is a statement about exact real arithmetic: the arithmetic operations and the exponential
function are exact and cost one operation each, and the nodes and weights of the Gauss--Legendre rules are exact.
This section first bounds the additional error when the second algorithm, Algorithm~\ref{app:approx:alg:second},
runs in floating-point arithmetic with rules that are themselves computed approximately
(Proposition~\ref{app:approx:prop:float}). It then states the consequence for the running time on a Turing machine
(Corollary~\ref{app:approx:cor:bits}). For Algorithm~\ref{app:approx:alg:linear} we have no such bound. The
section ends with measurements of its rounding error and with the variant that our implementation uses.

\textbf{Model.} We consider binary floating-point arithmetic with unit roundoff $u$ in the standard model
\citep{higham2002accuracy}. Every addition, subtraction, multiplication and division of two machine numbers and
every evaluation of the exponential function returns the exact result times a factor $1+\delta$ with
$|\delta|\le u$, where $\delta$ depends on the operation. For the arithmetic operations this holds in IEEE
arithmetic as long as no operation overflows or underflows, which we assume. For the exponential function it is
a requirement on its implementation. The inputs $w_j$, $y_j$, $y^{(\val)}$ and $y_\emptyset$ are machine numbers.

\textbf{The rules of the machine.} The machine holds nodes $\tilde t_\ell$ and weights $\tilde\mu_\ell$ in $t$ and
nodes $\tilde p_h$ and weights $\tilde\nu_h$ in $p$, all machine numbers. On the piece $[2^k,2^{k+1}]$ they are
$2^{k-1}(3+\tilde\xi_{k,j})$ and $2^{k-1}\tilde\beta_{k,j}$, $j=1,\dots,n$, with $\tilde\xi_{k,j}\in[-1,1]$ and
$\tilde\beta_{k,j}>0$. The nodes in $p$ lie in $[0,1]$ and the weights $\tilde\nu_h$ are positive. Four numbers
measure the distance to the exact rules of Sections~\ref{app:approx:cont} and~\ref{app:approx:second}: for every piece $k$,
\begin{equation}
|\tilde\xi_{k,j}-\xi_j|\le\delta_t,\qquad\sum_{j=1}^{n}|\tilde\beta_{k,j}-\beta_j|\le\varpi_t,
\qquad\text{and}\qquad
|\tilde p_h-p_h|\le\delta_p,\qquad\sum_{h=1}^{n_p}|\tilde\nu_h-\nu_h|\le\varpi_p .
\label{app:approx:eq:ruleerr}
\end{equation}

\textbf{Order of evaluation.} The number of roundings depends on the order of the operations, which we fix as
follows. The algorithm computes $r_j$ as $w_j\cdot(y_j-y^{(\val)})$, the number $E_j$ as the exponential of the
negated product $t_\ell\cdot w_j$, and $1-p_h$, $p_h\cdot E_j$ and $\mu_\ell\cdot t_\ell$ by one operation each. It
evaluates the two operations~\eqref{app:approx:eq:ops} as
\begin{equation}
f^{+j}=E\cdot\bigl(f_0,\;a_1,\;f_2+r_j\,(f_1+a_1)\bigr)\ \ \text{with }a_1=f_1+r_jf_0,
\qquad
\langle f,h\rangle=\bigl(f_2h_0+f_1\,(h_1+h_1)\bigr)+f_0h_2,
\label{app:approx:eq:order}
\end{equation}
where the factor $E$ is $E_j$ in line~\ref{app:approx:ln:delta} and $p_h\cdot E_j$ in lines~\ref{app:approx:ln:suffix}
and~\ref{app:approx:ln:prefix}. The product in line~\ref{app:approx:ln:sumt} is $(\mu_\ell\cdot t_\ell)\cdot s_i$, and all sums
run from left to right.

The next lemma is the usual bookkeeping of rounding errors. Let $\gamma_d=(1+u)^d-1$ for an integer $d\ge0$.

\begin{lemma}[Rounding errors of a computed expression]
\label{app:approx:lem:depth}
Let quantities be computed from given numbers by additions, subtractions and multiplications in the model above.
Give every quantity $x$ a depth $d(x)$ and a majorant $\bar x$ as follows. A given number $x$ whose computed value
$\hat x$ satisfies $|\hat x-x|\le\gamma_d\,|x|$ has depth $d$ and majorant $|x|$. A sum or difference $x\pm y$ has
depth $\max(d(x),d(y))+1$ and majorant $\bar x+\bar y$. A product $xy$ has depth $d(x)+d(y)+1$ and majorant
$\bar x\,\bar y$. Then every quantity satisfies $|x|\le\bar x$ and $|\hat x-x|\le\gamma_{d(x)}\,\bar x$.
\end{lemma}

\begin{proof}
By induction over the operations; $|x|\le\bar x$ is clear. We use that $\gamma_d$ increases with $d$, that
$(1+\gamma_a)(1+\gamma_b)=1+\gamma_{a+b}$, and that $(1+u)\,\gamma_d+u=\gamma_{d+1}$. Let $a=d(x)$ and $b=d(y)$.
The computed sum is $(\hat x\pm\hat y)(1+\delta)$, and its distance to $x\pm y$ is at most
$(1+u)\,(|\hat x-x|+|\hat y-y|)+u\,|x\pm y|\le\gamma_{\max(a,b)+1}\,(\bar x+\bar y)$. For the product,
$|\hat x\hat y-xy|\le(\gamma_a+\gamma_b+\gamma_a\gamma_b)\,\bar x\,\bar y=\gamma_{a+b}\,\bar x\,\bar y$. The computed
product $\hat x\hat y\,(1+\delta)$ therefore differs from $xy$ by at most
$\bigl((1+u)\,\gamma_{a+b}+u\bigr)\,\bar x\,\bar y=\gamma_{a+b+1}\,\bar x\,\bar y$.
\end{proof}

\begin{proposition}[Floating-point arithmetic]
\label{app:approx:prop:float}
Let the hypotheses of Theorem~\ref{app:approx:thm:explicit} hold, with the parameters~\eqref{app:approx:eq:params}. Let
$J=12N+K'+L+n_p+4$ and $\gamma=Ju/(1-Ju)$. Assume the model above, rules of the machine that
satisfy~\eqref{app:approx:eq:ruleerr}, and
\begin{equation}
\text{(i)}\ \ 3u\cdot2^{k_{\max}}K'w_{\max}\le0.9,
\qquad
\text{(ii)}\ \ u\,(J+K')\le0.1,
\qquad
\text{(iii)}\ \ \frac{\varepsilon}{2B^2}+5\,\delta_t+9\,\varpi_t\le\frac12 .
\label{app:approx:eq:hyp}
\end{equation}
Then the values $\hat\phi_i^{\,\mathrm{fl}}$ that Algorithm~\ref{app:approx:alg:second} computes in the order~\eqref{app:approx:eq:order}
satisfy, for every $i$,
\begin{align}
\bigl|\hat\phi_i^{\,\mathrm{fl}}-\phi_i\bigr|\le{}&\varepsilon
+B^2\,\bigl(10\,\delta_t+18\,\varpi_t\bigr)
+B^2\,\bigl(3\,(N-1)\,\delta_p+\tfrac32\,\varpi_p\bigr)\notag\\
&+(1+\varpi_p)\,\bigl(3\gamma+(10K'+154)\,u\bigr)\,B^2
+\gamma\,\frac{(y_\emptyset-y^{(\val)})^2}{N}.
\label{app:approx:eq:flbound}
\end{align}
The statement remains true when $k_{\min}$ is smaller and $k_{\max}$ and $n$ are larger than
in~\eqref{app:approx:eq:params}.
\end{proposition}

The first line of~\eqref{app:approx:eq:flbound} is the error of Theorem~\ref{app:approx:thm:explicit} and the effect of the
inexact rules in $t$ and in $p$. The second line is the effect of rounding. With exact rules,
$\delta_t=\varpi_t=\delta_p=\varpi_p=0$, only $\varepsilon$ and the second line remain. The last term comes from
the term $-v(\emptyset)/N$ of the values. It is large when the default prediction is far from $y^{(\val)}$, and its
rounding error grows with its size.

\begin{proof}
A tilde marks the quantities of Sections~\ref{app:approx:cont} and~\ref{app:approx:second} that are formed with
the rules of the machine: $\tilde q$
and $\tilde\eta$ are~\eqref{app:approx:eq:q} and~\eqref{app:approx:eq:eta} for the rule $(\tilde t_\ell,\tilde\mu_\ell)$, and
$\tilde\phi_i$ is~\eqref{app:approx:eq:output} for both rules of the machine, in exact arithmetic.

\emph{Step 1: the rule in $t$.} Fix $A>0$ and a piece $[2^k,2^{k+1}]$, and let $\alpha=2^kA$ and $G$ be as
in~\eqref{app:approx:eq:G}. As in the proof of Lemma~\ref{app:approx:lem:piece}, the piece contributes
$\sum_j\tilde\beta_{k,j}\,G(\tilde\xi_{k,j})$ to $A^2\tilde q(A)$ and $\sum_j\beta_j\,G(\xi_j)$ to $A^2q(A)$. The
two differ by at most
\begin{equation*}
\sum_j\beta_j\,\bigl|G(\tilde\xi_{k,j})-G(\xi_j)\bigr|+\sum_j|\tilde\beta_{k,j}-\beta_j|\,G(\tilde\xi_{k,j})
\le2\,\delta_t\max_{[-1,1]}|G'|+\varpi_t\max_{[-1,1]}G .
\end{equation*}
On $[-1,1]$ we have $0\le G(\xi)\le\alpha^2e^{-\alpha}\le\psi(\alpha)$. With $x=\tfrac12\alpha\,(3+\xi)\ge\alpha$,
the derivative is $G'(\xi)=\tfrac14\alpha^2\,(1-x)\,e^{-x}$, and $|1-x|\le1+\tfrac x2+\tfrac{x^2}{8}\le e^{x/2}$
gives $|G'|\le\tfrac14\,\psi(\alpha)$. The difference on the piece is thus at most
$(\tfrac12\delta_t+\varpi_t)\,\psi(2^kA)$, and by Lemma~\ref{app:approx:lem:dyadic}
\begin{equation}
\bigl|A^2\tilde q(A)-A^2q(A)\bigr|\le\tfrac92\,\delta_t+9\,\varpi_t\qquad\text{for every }A>0 .
\label{app:approx:eq:qtilde}
\end{equation}
By the proof of~\eqref{app:approx:eq:bound}, the first and the third term of~\eqref{app:approx:eq:etaA} are at most
$\varepsilon/(6B^2)$ for $A\le K'w_{\max}$, the second is at most $\varepsilon/(6B^2)$ for $A\ge w_{\min}$, and
it is at most $1$ for every $A$. With~\eqref{app:approx:eq:qtilde} and~(iii) this gives
\begin{equation}
\tilde\eta\le\frac{\varepsilon}{2B^2}+5\,\delta_t+9\,\varpi_t\le\frac12,
\qquad
A^2\tilde q(A)\le2+\frac{\varepsilon}{3B^2}+5\,\delta_t+9\,\varpi_t\le\frac52
\quad\text{for }\tfrac12w_{\min}\le A\le K'w_{\max}.
\label{app:approx:eq:etatilde}
\end{equation}

\emph{Step 2: the rule in $p$.} For a nonempty $T$ let $\tilde g(T)=R(\NK(T))^2\,\tilde q(A(\NK(T)))$, and let
$\tilde g(\emptyset)=0$. By~\eqref{app:approx:eq:range}, $0\le\tilde g\le\tilde G$ with
$\tilde G=B^2(1+\tilde\eta)\le\tfrac32B^2$. The polynomial
$P(p)=\mathbb{E}_p[\tilde g(S\cup\{z_i\})-\tilde g(S)]$ satisfies
$\tilde\phi_i=-v(\emptyset)/N-\sum_h\tilde\nu_h\,P(\tilde p_h)$. With the exact rule in $p$,
Lemmas~\ref{app:approx:lem:pnodes} and~\ref{app:approx:lem:transfer} give
$\bigl|-v(\emptyset)/N-\sum_h\nu_h\,P(p_h)-\phi_i\bigr|\le2B^2\tilde\eta$. Moreover $|P|\le\tilde G$ on $[0,1]$.
The derivative of $p^{s}(1-p)^{N-1-s}$ is this function times $s/p-(N-1-s)/(1-p)$, and the expectation of
$|S|$ is $(N-1)\,p$, so $|P'|\le\tilde G\;\mathbb{E}_p\bigl[|S|/p+(N-1-|S|)/(1-p)\bigr]=2\,(N-1)\,\tilde G$. Hence
\begin{equation*}
\Bigl|\sum_h\tilde\nu_h\,P(\tilde p_h)-\sum_h\nu_h\,P(p_h)\Bigr|
\le\sum_h\nu_h\,\bigl|P(\tilde p_h)-P(p_h)\bigr|+\sum_h|\tilde\nu_h-\nu_h|\,\bigl|P(\tilde p_h)\bigr|
\le\bigl(2\,(N-1)\,\delta_p+\varpi_p\bigr)\,\tilde G,
\end{equation*}
because the weights $\nu_h$ sum to $1$. With~\eqref{app:approx:eq:etatilde}, the distance $|\tilde\phi_i-\phi_i|$ is at
most the first line of~\eqref{app:approx:eq:flbound}.

\emph{Step 3: the exponentials.} Let $E_{j,\ell}=e^{-\tilde t_\ell w_j}$ and let
$\hat E_{j,\ell}=E_{j,\ell}\,(1+\theta_{j,\ell})$ be its computed value. The product $a=\tilde t_\ell w_j$ is
computed as $a\,(1+\delta)$ and the exponential as $e^{-a(1+\delta)}(1+\delta')$, so
$|\theta_{j,\ell}|\le(e^{au}-1)(1+u)+u$. By~(i) and~(ii), $au\le0.3$ and $u\le0.1$. Then
$e^{au}-1\le au\,e^{0.3}$, and $|\theta_{j,\ell}|\le u\,(1+3\,\tilde t_\ell w_j)$. For a set $C$ of at most $K'$
points, the sum of the $|\theta_{j,\ell}|$ over $z_j\in C$ is at most $u\,(K'+3\,\tilde t_\ell A(C))\le1$ by~(i)
and~(ii). Since $e^x-1\le3x$ for $0\le x\le1$,
\begin{equation}
\epsilon_\ell(C):=\prod_{z_j\in C}\bigl(1+|\theta_{j,\ell}|\bigr)-1\le3u\,\bigl(K'+3\,\tilde t_\ell A(C)\bigr).
\label{app:approx:eq:epsC}
\end{equation}

Now let $T$ be nonempty, $C=\NK(T)$, $A=A(C)$ and $\bar R=\sum_{z_j\in C}|r_j|\le BA$.
By~\eqref{app:approx:eq:etatilde},
$\sum_\ell\tilde\mu_\ell\tilde t_\ell\,\bar R^2e^{-\tilde t_\ell A}=\bar R^2\tilde q(A)\le\tfrac32B^2$. Since
$t\,e^{-tA/2}\le2/(eA)$ for $t\ge0$ and the weights $\tilde\mu_\ell$ are positive,
$\sum_\ell\tilde\mu_\ell\tilde t_\ell^{\,2}e^{-\tilde t_\ell A}\le\frac{2}{eA}\,\tilde q(A/2)\le\frac{20}{eA^3}$,
again by~\eqref{app:approx:eq:etatilde}. With~\eqref{app:approx:eq:epsC} and $60/e<23$,
\begin{equation}
\sum_\ell\tilde\mu_\ell\tilde t_\ell\,\bar R^2e^{-\tilde t_\ell A}\le\tfrac32B^2,
\qquad
\sum_\ell\tilde\mu_\ell\tilde t_\ell\,\bar R^2e^{-\tilde t_\ell A}\,\epsilon_\ell(C)\le3u\,B^2\bigl(\tfrac32K'+23\bigr).
\label{app:approx:eq:twosums}
\end{equation}

Let $\phi_i^{\sharp}$ be the value that the algorithm returns in exact arithmetic when it uses the rules of the
machine and the numbers $\hat E_{j,\ell}$. By Remark~\ref{app:approx:rem:factors}, $\phi_i^\sharp$ is $\tilde\phi_i$ with
$g^\sharp_\ell(T)=R(C)^2\prod_{z_j\in C}\hat E_{j,\ell}$ in place of $g_{\tilde t_\ell}(T)$, and
$|g^\sharp_\ell(T)-g_{\tilde t_\ell}(T)|\le\bar R^2e^{-\tilde t_\ell A}\,\epsilon_\ell(C)$. The values of $T$ are
$S\cup\{z_i\}$ and $S$, and the weights $\tilde\nu_h$ sum to at most $1+\varpi_p$.
By~\eqref{app:approx:eq:twosums},
\begin{equation}
|\phi_i^\sharp-\tilde\phi_i|\le(1+\varpi_p)\cdot3u\,B^2\,(3K'+46).
\label{app:approx:eq:phisharp}
\end{equation}

\emph{Step 4: the other operations.} We apply Lemma~\ref{app:approx:lem:depth} to the run with the numbers
$\hat E_{j,\ell}$ as given numbers of depth $0$, so that its exact result is $\phi_i^\sharp$. We first bound the
depth of the output. The model gives $y_j-y^{(\val)}$ and $1-\tilde p_h$ the depth $1$ and the number
$(y_\emptyset-y^{(\val)})^2/N$ of line~\ref{app:approx:ln:init} the depth $4$. Then $r_j$ has depth $2$, and
$p_h\cdot E_j$ and $\mu_\ell\cdot t_\ell$ have depth $1$. If the entries of a triple $f$ have depth at most $d$,
the last entry of $f^{+j}$ in~\eqref{app:approx:eq:order} has depth at most $d+10$ for the factor $E_j$ and $d+11$ for
the factor $p_hE_j$, and a right-hand side of~\eqref{app:approx:eq:Frec} or~\eqref{app:approx:eq:Hrec} has depth at most
$d+12$. The triples $F_{i-1}[m]$ and $H_{i+1}[k]$ thus have depths at most $12\,(i-1)$ and $12\,(N-i)$. In
line~\ref{app:approx:ln:delta}, a term $\langle F[m]^{+i},H_{i+1}[K'-m-1]\rangle$ has depth at most $12\,(N-1)+13$, and
$\Delta$ has depth at most $12\,(N-1)+13+K'$. Lines~\ref{app:approx:ln:sumh} and~\ref{app:approx:ln:sumt} add at most $n_p+1$
and $L+2$. The output therefore has depth at most $J$.

We then bound the majorant of the output. It is $|v(\emptyset)|/N$ plus the expression
$\sum_\ell\tilde\mu_\ell\tilde t_\ell\sum_h\tilde\nu_h\,\bar\Delta$, where $\bar\Delta$ is
line~\ref{app:approx:ln:delta} with $|r_j|$ in place of $r_j$ and a sum in place of the difference. By the proof of
Lemma~\ref{app:approx:lem:diff} and Remark~\ref{app:approx:rem:factors}, $\bar\Delta$ is at most the expectation
$\mathbb{E}_{\tilde p_h}$ of $\bar g_\ell(S\cup\{z_i\})+\bar g_\ell(S)$, where
$\bar g_\ell(T)=\bar R^2\prod_{z_j\in C}\hat E_{j,\ell}\le\bar R^2e^{-\tilde t_\ell A}\,(1+\epsilon_\ell(C))$.
By~\eqref{app:approx:eq:twosums} the majorant is at most
$|v(\emptyset)|/N+(1+\varpi_p)\,\bigl(3B^2+3u\,B^2(3K'+46)\bigr)$. Moreover
$\gamma_J\le e^{Ju}-1\le\gamma$, and $\gamma\le\tfrac19$ by~(ii). Lemma~\ref{app:approx:lem:depth}
and~\eqref{app:approx:eq:phisharp} therefore give
\begin{equation*}
\bigl|\hat\phi_i^{\,\mathrm{fl}}-\tilde\phi_i\bigr|
\le\gamma\,\frac{|v(\emptyset)|}{N}+(1+\varpi_p)\,\bigl(3\gamma+(1+\gamma)\cdot3u\,(3K'+46)\bigr)\,B^2,
\end{equation*}
and $(1+\gamma)\cdot3\,(3K'+46)\le10K'+154$. Together with Step~2 this is~\eqref{app:approx:eq:flbound}.

The proof used the parameters only through the three bounds on the terms of~\eqref{app:approx:eq:etaA} in Step~1 and
through~(i). The bounds stay valid when $k_{\min}$ decreases and $k_{\max}$ and $n$ increase.
\end{proof}

\textbf{Reading the hypotheses.} Hypothesis~(ii) asks for $u$ to be small against $J$, which bounds the number of
roundings that one computed number can pass through. Hypothesis~(iii) holds for $\varepsilon\le B^2/2$ and rules
that are accurate to one percent. Hypothesis~(i) limits the ratio of the weights: since
$2^{k_{\max}}<4\ln(6B^2/\varepsilon)/w_{\min}$, it holds whenever
$K'w_{\max}/w_{\min}\le0.075/\bigl(u\,\ln(6B^2/\varepsilon)\bigr)$. A larger ratio needs a smaller unit roundoff,
that is, more digits. The assumption that nothing underflows is a further restriction in a format with a fixed
exponent range, because the numbers $e^{-t_\ell w_j}$ can be very small.

\begin{corollary}[Running time on a Turing machine]
\label{app:approx:cor:bits}
Assume that $e^{-a}$ for a machine number $a\ge0$, and the nodes and weights of the $m$-point Gauss--Legendre
rule, can be computed with relative error at most $u$ in time polynomial in $\log(1/u)$, $\log(2+a)$ and $m$
\citep{johansson2018fast}. Let the inputs be binary floating-point numbers. Then all $N$ values
are computed within additive error $\varepsilon$ on a Turing machine in time polynomial in the length of the
input and in $N$, $K$, $\log(w_{\max}/w_{\min})$, $\log B$, $\log(1/\varepsilon)$ and
$\log(1+|y_\emptyset-y^{(\val)}|)$.
\end{corollary}

\begin{proof}
Run Algorithm~\ref{app:approx:alg:second} with the tolerance $\varepsilon/2$ in binary floating-point arithmetic with
$u=2^{-\beta}$ and an unbounded exponent range, so that nothing overflows or underflows. Replace the
parameters~\eqref{app:approx:eq:params} by integers in the direction that Proposition~\ref{app:approx:prop:float} allows, such
that $k_{\max}-k_{\min}$ and $n$ grow by at most a constant factor. Such integers are obtained from the binary
exponents of $\varepsilon$, $B$, $w_{\min}$ and $w_{\max}$ by integer arithmetic. Compute the reference nodes and
weights with relative error at most $u$, and move a node that leaves $[-1,1]$ back to the nearer endpoint. The
multiplications by powers of two are exact. After the roundings in $3+\tilde\xi$ and $\tfrac12(1+\tilde\xi)$, the
rules satisfy~\eqref{app:approx:eq:ruleerr} with $\delta_t,\varpi_t,\delta_p,\varpi_p\le6u$.

Let $\beta$ be at least the number of significand bits of the inputs and so large that~(i) and~(ii) hold and
$u\le\varepsilon\big/\bigl(80\,J\,(B^2+(y_\emptyset-y^{(\val)})^2)\bigr)$. Then~(iii) holds. In~\eqref{app:approx:eq:flbound}
the first term is $\varepsilon/2$, and the others add up to at most
$28\,J\,u\,(B^2+(y_\emptyset-y^{(\val)})^2)\le\varepsilon/2$, because $J\ge27$ and $\gamma\le\tfrac{10}{9}Ju$. The
computed values are therefore within $\varepsilon$ of the $\phi_i$. By~(i) and~\eqref{app:approx:eq:counts}, the smallest
such $\beta$ is at most the number of significand bits of the inputs plus a constant times
\begin{equation*}
1+\log(NK)+\log\frac{w_{\max}}{w_{\min}}+\log\frac{B^2}{\varepsilon}
+\log\Bigl(1+\frac{(y_\emptyset-y^{(\val)})^2}{\varepsilon}\Bigr),
\end{equation*}
which is polynomial in the quantities of the statement. The number of operations is polynomial by
Proposition~\ref{app:approx:prop:second} and~\eqref{app:approx:eq:counts}. Every operation acts on numbers with $\beta$ significand bits and with exponents of
polynomially many bits, because the arguments of the exponential function are at most $2^{k_{\max}}w_{\max}$ in
absolute value. Each operation therefore takes polynomial time, by the assumption for the exponential function
and the rules and by the standard algorithms for the other operations.
\end{proof}

In particular, let the weights, the targets, $y^{(\val)}$ and $y_\emptyset$ be integers of at most $b$ bits. Under
the assumption of Corollary~\ref{app:approx:cor:bits}, the values are then computed within $2^{-b}$ in time polynomial in
$N$ and $b$, while the exact values can need exponentially many bits (Theorem~\ref{thm:bits}).

\textbf{The algorithm of Theorem~\ref{thm:cont} in double precision.} Proposition~\ref{app:approx:prop:float}
does not cover Algorithm~\ref{app:approx:alg:linear}. Its alternating sums~\eqref{app:approx:eq:suffix} subtract
terms that can be much larger than the result when $D_{<e}$ has few points, as Appendix~\ref{app:exact:float} describes for weights on a grid. We report measurements instead of a bound. In
exact arithmetic the two algorithms return the same numbers, so their difference in double precision is rounding
error. We measured it on nine instances with $60\le N\le150$ points, $5\le K\le50$, Gaussian kernel weights and
the tolerances $10^{-6}B^2$ and $10^{-10}B^2$. For Algorithm~\ref{app:approx:alg:linear} as stated, the largest
difference from Algorithm~\ref{app:approx:alg:second} over the points was at most $2.1\cdot10^{-15}B^2$ for
$K\le10$, it was $1.9\cdot10^{-12}B^2$ for $K=20$, and it reached $2.5\cdot10^{-10}B^2$ for $K=30$ and
$1.4\cdot10^{-4}B^2$ for $K=50$. Against the same algorithm run in 80-bit extended
precision with the same double-precision coefficients and quadrature nodes, the implementation described in the next
paragraph differs by at most $9\times10^{-9}$ of the tolerance $10^{-6}B^2$ up to $100{,}000$ points for $K=5$ and $10$
(Appendix~\ref{app:more:enum}).

\textbf{What our implementation does.} Our implementation therefore uses the removal identity only for the
replaced points beyond the $2K$ nearest ones. It takes $e_0=2K$ in Lemma~\ref{app:approx:lem:suffix}, so that the
sums~\eqref{app:approx:eq:UV} run over $e>\max(i,2K)$, and it sums the terms of~\eqref{app:approx:eq:Sigma} with
$K<e\le2K$ directly. For this it keeps, for every $i<2K$, the sums of the set $D_{<e}\setminus\{z_i\}$, to which
the points $z_e$ with $e\le2K$ and $e\ne i$ are added in order of distance by~\eqref{app:approx:eq:Madd}. When
$N\le2K$ it reads $J_i$ from these sums as well. These are at most $2K$ arrays of $K$ triples with at most $2K$
additions of a point each. The direct summation thus adds $O(K^3L)$ arithmetic operations, and $O(K^2L)$
evaluations of the exponential function if the factors $e^{-t_\ell w_j}$ are evaluated again for every array. By
Lemmas~\ref{app:approx:lem:onenode} and~\ref{app:approx:lem:suffix}, the result is the same in exact arithmetic.
On the nine instances above, this variant differed from Algorithm~\ref{app:approx:alg:second} by between
$5.6\cdot10^{-16}B^2$ and $2.1\cdot10^{-15}B^2$, for every $K$.

\subsection{Discretizing Instead: Rounding the Weights}
\label{app:approx:round}

\textbf{Setting.} Let $w_{\min}$ and $w_{\max}$ be the smallest and the largest weight, and let
$0<\delta_w\le w_{\min}$. Rounding replaces every weight $w_j$ by $a_j\,\delta_w$, where $a_j$ is an integer
nearest to $w_j/\delta_w$. Since $w_j/\delta_w\ge1$, every $a_j$ is at least $1$. The rounded game therefore has
discretized weights in the sense of Section~\ref{sec:setup}, with $W=\max_ja_j$ weight levels, and
Theorem~\ref{thm:sq} computes its Shapley values exactly. In this section and the next, $\hat\phi_i$ denotes the
Shapley values of the rounded game, and $\yhat_\delta(S)$ and $v_\delta(S)$ denote its prediction and its
utility. The bounds use three numbers,
\begin{equation}
\rho=\max_j\frac{|a_j\delta_w-w_j|}{w_j},\qquad
Y=\max_j|y_j-y^{(\val)}|,\qquad
R_y=\max_jy_j-\min_jy_j .
\label{app:approx:eq:rhoYR}
\end{equation}
Here $\rho$ is the largest relative rounding error of a weight and $R_y$ is the range of the targets. Rounding to
the nearest multiple moves a weight by at most $\delta_w/2$, so
\begin{equation}
\rho\le\frac{\delta_w}{2\,w_{\min}}\le\frac12 .
\label{app:approx:eq:rhobound}
\end{equation}

The first step is a statement about weighted averages. It says that a relative change of the weights by at most
$\rho$ moves a weighted average by at most $\rho/(1-\rho)$ times the range of the averaged numbers.

\begin{lemma}[A relative change of the weights]
\label{app:approx:lem:avg}
Let $C$ be a nonempty set of training points and let $0\le\rho<1$. Let $\tilde w_j$, for $z_j\in C$, be numbers with
$|\tilde w_j-w_j|\le\rho\,w_j$. Then
\begin{equation*}
\Biggl|\frac{\sum_{z_j\in C}\tilde w_j\,y_j}{\sum_{z_j\in C}\tilde w_j}
-\frac{\sum_{z_j\in C}w_j\,y_j}{\sum_{z_j\in C}w_j}\Biggr|
\le\frac{\rho}{1-\rho}\,\Bigl(\max_{z_j\in C}y_j-\min_{z_j\in C}y_j\Bigr).
\end{equation*}
\end{lemma}

\begin{proof}
All sums are over $z_j\in C$. Let $\bar y=\sum w_jy_j/\sum w_j$ and $\tilde y=\sum\tilde w_jy_j/\sum\tilde w_j$ be
the two averages; the second is well defined because $\tilde w_j\ge(1-\rho)\,w_j>0$. By the definition of $\tilde y$
we have $\tilde y\sum\tilde w_j=\sum\tilde w_jy_j$, and by the definition of $\bar y$ we have $\sum w_j(y_j-\bar
y)=0$. Therefore
\begin{equation*}
(\tilde y-\bar y)\sum\tilde w_j=\sum\tilde w_j\,(y_j-\bar y)=\sum(\tilde w_j-w_j)\,(y_j-\bar y).
\end{equation*}
The average $\bar y$ lies between the smallest and the largest $y_j$ of $C$, so every $|y_j-\bar y|$ is at most
the range $\max_{C}y_j-\min_{C}y_j$. With $|\tilde w_j-w_j|\le\rho\,w_j$ and
$\sum\tilde w_j\ge(1-\rho)\sum w_j$ this gives the claim.
\end{proof}

\begin{proposition}[Rounding the weights]
\label{app:approx:prop:round}
For the squared loss, the Shapley values $\hat\phi_i$ of the rounded game satisfy, for every $i$,
\begin{equation}
|\hat\phi_i-\phi_i|\le 4\,Y\,R_y\,\frac{\rho}{1-\rho}\le 8\,Y^2\,\frac{\delta_w}{w_{\min}} .
\label{app:approx:eq:sharp}
\end{equation}
\end{proposition}

Section~\ref{sec:approx} refers to this proposition. For their hard-label classifier,
\citet{wang2024efficient} show that discretizing the weights does not reverse the order of two values; of the
size of the change they write that it ``is difficult to analytically derive or upper bound'', and they study it
empirically.

\begin{proof}
Rounding the weights changes no distance, so every subset $S$ has the same $K$ nearest neighbors $\NK(S)$ in both
games. The empty subset has the same utility in both games, since its prediction is the default $y_\emptyset$.

Let $S$ be nonempty. The predictions $\yhat(S)$ and $\yhat_\delta(S)$ are the averages of the targets of
$C=\NK(S)$ with the weights $w_j$ and $a_j\delta_w$. By~\eqref{app:approx:eq:rhoYR},
$|a_j\delta_w-w_j|\le\rho\,w_j$, and $\rho<1$ by~\eqref{app:approx:eq:rhobound}. The range of the targets of $C$
is at most $R_y$, so Lemma~\ref{app:approx:lem:avg} gives
\begin{equation}
|\yhat_\delta(S)-\yhat(S)|\le\frac{\rho}{1-\rho}\,R_y .
\label{app:approx:eq:predmove}
\end{equation}
Both predictions are weighted averages of targets, so $|\yhat(S)-y^{(\val)}|\le Y$ and
$|\yhat_\delta(S)-y^{(\val)}|\le Y$. The difference of the two squared errors factors as
\begin{equation*}
|v_\delta(S)-v(S)|
=\bigl|\yhat_\delta(S)-\yhat(S)\bigr|\cdot\bigl|(\yhat_\delta(S)-y^{(\val)})+(\yhat(S)-y^{(\val)})\bigr|
\le\frac{\rho}{1-\rho}\,R_y\cdot 2Y .
\end{equation*}
Lemma~\ref{app:approx:lem:close} with $\Delta=2\,Y R_y\,\rho/(1-\rho)$ gives the first inequality
of~\eqref{app:approx:eq:sharp}.

For the second inequality, every target lies in $[y^{(\val)}-Y,\,y^{(\val)}+Y]$, so $R_y\le 2Y$. The function
$x\mapsto x/(1-x)$ is increasing on $[0,1)$, so~\eqref{app:approx:eq:rhobound} and $\delta_w\le w_{\min}$ give
\begin{equation}
\frac{\rho}{1-\rho}\le\frac{\delta_w}{2\,w_{\min}-\delta_w}\le\frac{\delta_w}{w_{\min}} .
\label{app:approx:eq:rhofrac}
\end{equation}
Hence $4\,Y R_y\,\rho/(1-\rho)\le4\,Y\cdot2Y\cdot\delta_w/w_{\min}$.
\end{proof}

\textbf{The bound is returned with the values.} The numbers $\rho$, $Y$ and $R_y$ of~\eqref{app:approx:eq:rhoYR}
are computed from the inputs in $O(N)$ operations. An implementation can therefore return the first bound
of~\eqref{app:approx:eq:sharp} together with the values $\hat\phi_i$. The bound is the same for all $i$ and holds
on every instance. It is smaller than the second bound when all targets lie on one side of $y^{(\val)}$, where
$R_y\le Y$, and when the realized rounding error $\rho$ is below its worst case $\delta_w/(2w_{\min})$.

\textbf{Choosing the unit for a tolerance.} Let $0<\varepsilon\le 8Y^2$. By
Proposition~\ref{app:approx:prop:round}, the unit
\begin{equation*}
\delta_w=\frac{\varepsilon\,w_{\min}}{8\,Y^2}
\end{equation*}
satisfies $\delta_w\le w_{\min}$ and guarantees $|\hat\phi_i-\phi_i|\le\varepsilon$ for every $i$. The rounded
weights then have $W=\max_ja_j\le w_{\max}/\delta_w+\tfrac12$ levels, and $w_{\max}/\delta_w\ge1$. The number of
weight levels is therefore at most $12\,Y^2w_{\max}/(w_{\min}\,\varepsilon)$, and by Theorem~\ref{thm:sq} the
running time after sorting is
\begin{equation*}
O\bigl(WK^2N\bigr)=O\Bigl(NK^2\,\frac{w_{\max}}{w_{\min}}\cdot\frac{Y^2}{\varepsilon}\Bigr).
\end{equation*}
It is linear in $1/\varepsilon$ and in the ratio $w_{\max}/w_{\min}$, where Theorem~\ref{thm:cont} has the
logarithms of these factors. With the first bound of~\eqref{app:approx:eq:sharp}
and~\eqref{app:approx:eq:rhofrac}, the unit $\delta_w=\varepsilon\,w_{\min}/(4\,Y R_y)$ suffices whenever it is at
most $w_{\min}$.

\subsection{Rounding Weights and Targets for Other Losses}
\label{app:approx:otherloss}

For a loss that is not a polynomial in the prediction, the exact algorithm of Theorem~\ref{thm:pair} needs
discretized targets as well. We round the weights as in Section~\ref{app:approx:round} and every target $y_j$ to
$b_j\,\delta_y$, where $\delta_y>0$ is the target unit and $b_j$ is an integer nearest to $y_j/\delta_y$. The
target $y^{(\val)}$ of the query and the default $y_\emptyset$ are kept. Theorem~\ref{thm:pair} computes the
Shapley values $\hat\phi_i$ of this rounded game exactly.

\begin{proposition}[Rounding weights and targets, absolute loss]
\label{app:approx:prop:abs}
For the absolute loss and every $i$, with $\rho$ and $R_y$ as in~\eqref{app:approx:eq:rhoYR},
\begin{equation}
|\hat\phi_i-\phi_i|\le\delta_y+2\,R_y\,\frac{\rho}{1-\rho}\le\delta_y+2\,R_y\,\frac{\delta_w}{w_{\min}} .
\label{app:approx:eq:absbound}
\end{equation}
For a loss that is Lipschitz in the prediction, the same bounds hold with their right-hand sides multiplied by the
Lipschitz constant.
\end{proposition}

\begin{proof}
As in the proof of Proposition~\ref{app:approx:prop:round}, both games have the same $K$ nearest neighbors in
every subset and the same utility of the empty subset. Let $S$ be nonempty, $C=\NK(S)$, and let all sums be over
$z_j\in C$. We pass from $\yhat(S)$ to the prediction $\yhat_\delta(S)$ of the rounded game in two steps.

First round the weights only, which gives the average $\yhat_w=\sum a_jy_j/\sum a_j$. This is the
step~\eqref{app:approx:eq:predmove} of Section~\ref{app:approx:round}:
$|\yhat_w-\yhat(S)|\le R_y\,\rho/(1-\rho)$.

Then round the targets. The prediction of the rounded game is $\yhat_\delta(S)=\sum a_jb_j\delta_y/\sum a_j$, so
\begin{equation*}
\yhat_\delta(S)-\yhat_w=\frac{\sum a_j\,(b_j\delta_y-y_j)}{\sum a_j}
\end{equation*}
is a weighted average of the numbers $b_j\delta_y-y_j$. Each of them is at most $\delta_y/2$ in absolute value, and
so is their average. Together,
\begin{equation*}
|\yhat_\delta(S)-\yhat(S)|\le\frac{\delta_y}{2}+R_y\,\frac{\rho}{1-\rho} .
\end{equation*}
The absolute loss is $1$-Lipschitz in the prediction:
$\bigl||\yhat_\delta(S)-y^{(\val)}|-|\yhat(S)-y^{(\val)}|\bigr|\le|\yhat_\delta(S)-\yhat(S)|$. The two utilities of
$S$ therefore differ by at most $\delta_y/2+R_y\,\rho/(1-\rho)$, and those of a Lipschitz loss by at most its
Lipschitz constant times this number. Lemma~\ref{app:approx:lem:close} doubles it, because a marginal contribution
is a difference of two utilities. This is the first inequality of~\eqref{app:approx:eq:absbound}; the second
is~\eqref{app:approx:eq:rhofrac}.
\end{proof}

\textbf{Running time for a tolerance.} Let $0<\varepsilon\le R_y$ and take $\delta_y=\varepsilon/2$ and
$\delta_w=\varepsilon\,w_{\min}/(4R_y)$. Then~\eqref{app:approx:eq:absbound} gives
$|\hat\phi_i-\phi_i|\le\varepsilon$. The rounded game has $W\le w_{\max}/\delta_w+\tfrac12$ weight levels, and its
integer targets span $T=\max_jb_j-\min_jb_j\le R_y/\delta_y+1$ units. In the rounded game the total weight and the total
weighted target of Section~\ref{sec:pair} are the integers $A(C)=\sum_{z_j\in C}a_j$ and $M(C)=\sum_{z_j\in C}a_jb_j$. For a set $C$
of at most $K$ points, $A(C)$ takes at most $KW+1$ values. The number
$M(C)-A(C)\min_jb_j=\sum_{z_j\in C}a_j\,(b_j-\min_jb_j)$ is an integer between $0$ and $KWT$, and the pair
$(A(C),M(C))$ is determined by $A(C)$ and this number. The number of pairs in Theorem~\ref{thm:pair} is therefore
$P\le(KW+1)(KWT+1)$, and the running time $O(NKP)$ after sorting is
\begin{equation*}
O\Bigl(NK^3\,\Bigl(\frac{w_{\max}}{w_{\min}}\Bigr)^{2}\Bigl(\frac{R_y}{\varepsilon}\Bigr)^{3}\Bigr).
\end{equation*}
Rounding thus gives, for the absolute loss, a deterministic approximation whose running time is polynomial in
$1/\varepsilon$. By Theorem~\ref{thm:abs}, no deterministic algorithm approximates the values of this loss within
$\varepsilon$ in time polynomial in $N$, the number of bits of the targets and $\log(1/\varepsilon)$, unless
$\mathrm{FP}=\#\mathrm{P}$.

\appsection{Proofs for Section~\ref{sec:lower}}\label{app:lower}
This appendix proves the lower bounds of Section~\ref{sec:lower} and is organized as follows.
\begin{itemize}
\item[$\bullet$] Appendix~\ref{app:lower:bits} proves Theorem~\ref{thm:bits}: under the squared loss the exact
value can take exponentially many bits. The same holds for the soft-label utility
(Proposition~\ref{app:lower:prop:soft}).
\item[$\bullet$] Appendix~\ref{app:lower:precision} shows that deciding how high a power of two divides the
denominator of the exact value is NP-hard.
\item[$\bullet$] Appendix~\ref{app:lower:ss} recalls the counting problem \#\textsc{Subset-Sum} and shows that it
is \#P-complete under reductions that preserve the count.
\item[$\bullet$] Appendix~\ref{app:lower:abs} states Theorem~\ref{thm:abs} in full and proves it by reducing
\#\textsc{Subset-Sum} to the value under the absolute loss.
\item[$\bullet$] Appendix~\ref{app:lower:hard} shows that the Shapley value of the hard-label classifier is
\#P-hard when the weights are written in binary.
\end{itemize}

\textbf{Notation.} The training points $z_1,\dots,z_N$ are sorted by their distance to the query, and
$\NK(S)$ is the set of the $\min(K,|S|)$ nearest points of a subset $S$. The Shapley value $\phi_i$ is the
sum~\eqref{eq:shapley} with the coefficients $c_N(s)=s!\,(N-1-s)!/N!$. An \emph{instance} consists of the number
$K$, the weights and the targets of the training points in the order of their distance to the query, the target
$y^{(\val)}$ of the query and the default prediction $y_\emptyset$. All these numbers are integers written in binary.
For the soft-label and the hard-label utility the targets are class labels, and there is no default prediction.
We call the point $z_i$ whose value is computed the \emph{valued point}, and the $N-1$ points of
$D\setminus\{z_i\}$ the \emph{other points}. In Appendices~\ref{app:lower:bits} and~\ref{app:lower:precision} the
weights are positive integers, $w_j=a_j$ with the weight unit $\delta_w=1$, and $A(C)=\sum_{z_j\in C}a_j$ is the
total weight of a set $C$ of training points, as in~\eqref{eq:AR}.

\textbf{The exponent of a prime.} For a prime $p$ and a nonzero rational number $x$, let $\nu_p(x)$ be the
exponent of $p$ in $x$: the integer $k$ for which $x=p^{k}a/b$ with integers $a$ and $b$ not divisible by $p$. We
use three standard facts. First, $\nu_p(xy)=\nu_p(x)+\nu_p(y)$. Second, if $x+y\neq0$ then
$\nu_p(x+y)\ge\min(\nu_p(x),\nu_p(y))$, with equality when $\nu_p(x)\neq\nu_p(y)$. Third, when $x$ is written as a
reduced fraction, the exponent of $p$ in its denominator is $\max(0,-\nu_p(x))$. The second fact has a consequence
that we use repeatedly: if one term of a sum has a smaller exponent than every other term, then the sum has the
exponent of that term. Indeed, the sum of the other terms is zero or has a larger exponent than that term.

\subsection{Proof of Theorem~\ref{thm:bits}}
\label{app:lower:bits}

\begin{restate}{Theorem~\ref{thm:bits}}
For every $n\ge18$ there is an instance with $N=n+1$ training points, integer weights of at most $n$ bits, targets
in $\{0,1\}$ and $K=N$, for which the Shapley value of one point under the squared loss, in lowest terms, has a
denominator of at least $2^{n-1}$ bits.
\end{restate}

\textbf{The instance.} Fix $n\ge3$. The instance has $N=n+1$ training points and $K=N$. The valued point $z_i$ has
the weight $a_i=1$ and the target $y_i=1$. The other $n$ points have the weights $1,2,4,\dots,2^{n-1}$ and the
target $0$. The query has the target $y^{(\val)}=0$, and the default prediction is $y_\emptyset=0$. Since $K=N$, every
point of a subset is among its $K$ nearest neighbors, so the order of the points plays no role. The proof has three
steps: a formula for the value, the primes in its denominator, and a count of these primes.

The first step is stated for arbitrary integer weights, because Appendix~\ref{app:lower:precision} uses it again.

\begin{lemma}[Value of the gadget]
\label{app:lower:lem:gadget}
Consider the squared loss with positive integer weights $a_1,\dots,a_N$, with $K=N$ and with
$y^{(\val)}=y_\emptyset=0$. Let the valued point $z_i$ have the target $y_i=1$ and let every other point have the
target $0$. Then
\begin{equation}
\phi_i=-\sum_{S\subseteq D\setminus\{z_i\}}c_N(|S|)\,\frac{a_i^2}{\bigl(a_i+A(S)\bigr)^2}.
\label{app:lower:eq:gadget}
\end{equation}
\end{lemma}

In words, the other points predict the target of the query exactly, and $z_i$ is the only point that errs. Its
marginal contribution to a subset $S$ is minus the square of its share $a_i/(a_i+A(S))$ of the total weight.

\begin{proof}
Let $S\subseteq D\setminus\{z_i\}$. Since $K=N$, we have $\NK(S)=S$ and $\NK(S\cup\{z_i\})=S\cup\{z_i\}$. The numbers
$r_j=a_j(y_j-y^{(\val)})$ of~\eqref{eq:AR} are $r_i=a_i$ and $r_j=0$ for $j\neq i$, so $R(S)=0$ and
$R(S\cup\{z_i\})=a_i$. By~\eqref{eq:vAR}, $v(S)=0$ if $S$ is nonempty, and
$v(\emptyset)=-(y_\emptyset-y^{(\val)})^2=0$ as well. Again by~\eqref{eq:vAR},
$v(S\cup\{z_i\})=-a_i^2/(a_i+A(S))^2$. The marginal contribution of $z_i$ to $S$ is therefore
$-a_i^2/(a_i+A(S))^2$, and inserting it into~\eqref{eq:shapley} gives~\eqref{app:lower:eq:gadget}.
\end{proof}

\begin{lemma}[The primes of the denominator]
\label{app:lower:lem:primes}
Let $n\ge3$. In the instance above, $\nu_p(\phi_i)=-2$ for every prime $p$ with $2^{n-1}<p\le2^n$.
\end{lemma}

In words, the square of every prime between $2^{n-1}$ and $2^n$ divides the denominator of $\phi_i$ written as a
reduced fraction, and its cube does not.

\begin{proof}
By Lemma~\ref{app:lower:lem:gadget} with $a_i=1$,
\begin{equation}
\phi_i=-\sum_{S\subseteq D\setminus\{z_i\}}\frac{c_N(|S|)}{\bigl(1+A(S)\bigr)^2}.
\label{app:lower:eq:bits-value}
\end{equation}
The weights of the other points are the powers of two $2^0,2^1,\dots,2^{n-1}$. Every integer between $0$ and $2^n-1$
has exactly one binary representation, so $S\mapsto A(S)$ is a bijection from the $2^n$ subsets of
$D\setminus\{z_i\}$ onto $\{0,1,\dots,2^n-1\}$. Hence $1+A(S)$ takes every value in $\{1,2,\dots,2^n\}$ exactly once.

Fix a prime $p$ with $2^{n-1}<p\le2^n$. Exactly one subset $S_p$ has $1+A(S_p)=p$. For every other subset, $1+A(S)$
is not divisible by $p$, because it differs from $p$ and is at most $2^n<2p$. The coefficients are
$c_N(s)=s!\,(n-s)!/(n+1)!$. Every factor of these factorials is at most $n+1\le2^{n-1}<p$, where the first
inequality holds for $n\ge3$, so $\nu_p(c_N(s))=0$ for every $s$. Therefore the term of $S_p$
in~\eqref{app:lower:eq:bits-value} has the exponent $-2$, and every other term has the exponent $0$. One term has a
smaller exponent than all others, so the sum has the exponent of that term: $\nu_p(\phi_i)=-2$.
\end{proof}

\begin{proof}[Proof of Theorem~\ref{thm:bits}]
Let $n\ge18$ and write $\phi_i$ as a reduced fraction. By Lemma~\ref{app:lower:lem:primes}, its denominator is
divisible by $p^2$ for every prime $p$ with $2^{n-1}<p\le2^n$, and hence by the product of these squares. The
number of bits of a positive integer exceeds its logarithm to base $2$, and each of these primes exceeds
$2^{n-1}$. With $\pi(x)$ the number of primes up to $x$, the number of bits of the denominator is therefore more
than
\begin{equation}
2\sum_{2^{n-1}<p\le2^n}\log_2p\;>\;2\,(n-1)\,\bigl(\pi(2^n)-\pi(2^{n-1})\bigr).
\label{app:lower:eq:bits-count}
\end{equation}
We bound the number of primes with two explicit estimates \citep{rosser1962approximate}:
$\pi(x)>x/\ln x$ for $x\ge17$, and $\pi(x)<1.25506\,x/\ln x$ for $x>1$. Since $\ln(2^n)=n\ln2$, they give
\begin{align}
\pi(2^n)-\pi(2^{n-1})
&>\frac{2^n}{n\ln2}-\frac{1.25506\cdot2^{n-1}}{(n-1)\ln2}
=\frac{2^n}{n\ln2}\Bigl(1-0.62753\,\frac{n}{n-1}\Bigr)\notag\\
&\ge\frac{2^n}{n\ln2}\Bigl(1-0.62753\cdot\frac{18}{17}\Bigr)
>\frac{2^n}{3n}.
\label{app:lower:eq:bits-primes}
\end{align}
The second line uses $n/(n-1)\le18/17$ for $n\ge18$, and then $1-0.62753\cdot18/17>0.335>(\ln2)/3$.
By~\eqref{app:lower:eq:bits-count} and~\eqref{app:lower:eq:bits-primes}, the denominator has more than
\begin{equation*}
2\,(n-1)\,\frac{2^n}{3n}\;\ge\;\frac{2\cdot17}{3\cdot18}\,2^n\;=\;\frac{17}{27}\,2^n\;>\;2^{n-1}
\end{equation*}
bits. The instance is as the theorem states: it has $N=n+1$ training points and $K=N$, its weights are integers of at
most $n$ bits, and its targets are $0$ and $1$.
\end{proof}

\textbf{What this means for Theorem~\ref{thm:sq}.} The input of the instance is short: $n+1$ weights of at most $n$
bits, $n+1$ targets of one bit and the number $K$, which is $O(n^2)$ bits. Its largest weight is $2^{n-1}$, so the
weights lie on $W=2^{n-1}$ levels. Any fraction equal to $\phi_i$ has a denominator that is a multiple of the reduced
one. An algorithm that writes the exact value down as a fraction therefore takes at least $2^{n-1}=W$ steps on this
instance, whatever it does before. Since $N=K=\log_2W+2$ here, the dependence of Theorem~\ref{thm:sq} on $W$ cannot
be improved to a polynomial in $\log W$ for algorithms that write the value down. The bound counts the bits of the
output. It does not limit the number of arithmetic operations of an algorithm that never writes the value in
binary.

\textbf{The soft-label utility.} The exact values of the soft-label utility~\eqref{eq:soft} can be exponentially
long as well. The instance is the one above, with classes in place of targets.

\begin{proposition}[Long values of the soft-label utility]
\label{app:lower:prop:soft}
For every $n\ge18$ there is an instance with $N=n+1$ training points, two classes, integer weights of at most $n$
bits and $K=N$, for which the Shapley value of one point under the soft-label utility~\eqref{eq:soft}, written as a
reduced fraction, has a denominator of at least $2^{n-2}$ bits.
\end{proposition}

\begin{proof}
Let $n\ge3$, $N=n+1$ and $K=N$. The valued point $z_i$ has the weight $a_i=1$ and the class of the query. The other
$n$ points have the weights $1,2,4,\dots,2^{n-1}$ and the other class. Let $S\subseteq D\setminus\{z_i\}$. Since
$K=N$, we have $\NK(S)=S$ and $\NK(S\cup\{z_i\})=S\cup\{z_i\}$. No point of $S$ has the class of the query, so
$v(S)=0$ by~\eqref{eq:soft}, and $v(\emptyset)=0$ as well. In $S\cup\{z_i\}$ only $z_i$ has the class of the query,
so $v(S\cup\{z_i\})=a_i/(a_i+A(S))=1/(1+A(S))$. By~\eqref{eq:shapley},
\begin{equation}
\phi_i=\sum_{S\subseteq D\setminus\{z_i\}}\frac{c_N(|S|)}{1+A(S)}.
\label{app:lower:eq:soft-value}
\end{equation}
Up to the sign, the terms are those of~\eqref{app:lower:eq:bits-value} without the square in the denominator. The
proof of Lemma~\ref{app:lower:lem:primes} therefore applies with the exponent $-1$ in place of $-2$. For every prime
$p$ with $2^{n-1}<p\le2^n$, the term of the one subset $S_p$ with $1+A(S_p)=p$ has the exponent $-1$ and every other
term has the exponent $0$, so $\nu_p(\phi_i)=-1$. Every such prime therefore divides the denominator of $\phi_i$
written as a reduced fraction.

Now let $n\ge18$. Each of these primes exceeds $2^{n-1}$, so by~\eqref{app:lower:eq:bits-primes} the number of
bits of the denominator is more than
\begin{equation*}
\sum_{2^{n-1}<p\le2^n}\log_2p\;>\;(n-1)\,\bigl(\pi(2^n)-\pi(2^{n-1})\bigr)\;>\;(n-1)\,\frac{2^n}{3n}
\;\ge\;\frac{17}{54}\,2^n\;>\;2^{n-2}.
\qedhere
\end{equation*}
\end{proof}

As for the squared loss, the weights of this instance lie on $W=2^{n-1}$ levels, and an algorithm that writes the
exact value down as a fraction takes at least $2^{n-2}=W/2$ steps. The dependence of Corollary~\ref{cor:soft} on $W$
therefore cannot be improved to a polynomial in $\log W$ either, for algorithms that write the value down.

\subsection{Deciding the Precision Is NP-Hard}
\label{app:lower:precision}

Theorem~\ref{thm:bits} shows that the exact value can be long. One might still hope to tell in advance how much
precision the exact value of a given instance needs. We show that a basic form of this question is NP-hard: how
high a power of two divides the denominator.

\textbf{The problem.} An input of \textsc{Precision} is an instance with positive integer weights, a training
point $z_i$ and an integer $t\ge0$. The question is whether $2^t$ divides the denominator of the Shapley value
$\phi_i$ under the squared loss, written as a reduced fraction.

\begin{proposition}
\label{app:lower:prop:precision}
\textup{\textsc{Precision}} is NP-hard, already for instances with $K=N$, targets in $\{0,1\}$ and
$y^{(\val)}=y_\emptyset=0$.
\end{proposition}

In words, unless $\mathrm{P}=\mathrm{NP}$, no deterministic polynomial-time algorithm finds the exponent of two in
the denominator of the exact value. When the denominator is a power of two, this exponent is the number of binary
digits of the value after the point. The proof hides the answer of an NP-complete problem in the exponent.

\begin{proof}
We reduce from \textsc{Subset-Sum}, which is NP-complete \citep{garey1979computers}. Its input consists of positive
integers $g_1,\dots,g_n$ and an integer $T\ge0$. For a set $J\subseteq\{1,\dots,n\}$ write $g(J)=\sum_{j\in J}g_j$,
and let $G=g_1+\dots+g_n$. The question is whether some $J$ has $g(J)=T$.

\emph{Two trivial cases.} If $T=0$ the answer is yes, by the empty set, and if $T>G$ the answer is no. Take the
instance with one training point of weight $1$ and target $1$, with $K=N=1$ and $y^{(\val)}=y_\emptyset=0$. Its
value is $\phi_1=v(\{z_1\})-v(\emptyset)=-1$, with the denominator $1$. The reduction maps the inputs with $T=0$ to
this instance with $t=0$, where the answer is yes, and the inputs with $T>G$ to this instance with $t=1$, where the
answer is no. From now on $1\le T\le G$.

\emph{The instance.} Let $L$ be the number of bits of $G$, so that $G<2^L$, and put
\begin{equation}
V=\lceil\log_2(n!)\rceil+n+1,\qquad t=V+2L.
\label{app:lower:eq:prec-params}
\end{equation}
The instance has $N=n+1$ training points, $K=N$ and $y^{(\val)}=y_\emptyset=0$. For $j=1,\dots,n$ the point $z_j$ has
the weight $a_j=g_j$ and the target $0$. The valued point is $z_i$ with $i=n+1$. It has the target $1$ and the
weight $a_i=2^t-T$, which is positive because $T\le G<2^L\le2^t$. The reduction outputs this instance and the number
$t$. It runs in polynomial time, because $t=O(n\log n+L)$ and so $a_i$ has polynomially many bits.

A subset $S$ of the other points is given by the set $J$ of its indices, and $A(S)=g(J)$. So the input of
\textsc{Subset-Sum} has a solution if and only if some $S\subseteq D\setminus\{z_i\}$ has $A(S)=T$. The weight $a_i$
is chosen such that $a_i+A(S)=2^t$ for exactly these subsets.

\emph{The value.} In Lemma~\ref{app:lower:lem:gadget} we separate the subsets with $A(S)=T$ from the others. Since
$N!\,c_N(s)=s!\,(n-s)!$, the lemma gives
\begin{equation}
\phi_i=-\frac{a_i^2}{N!}\,\bigl(2^{-2t}\,U_{=}+U_{\neq}\bigr),
\label{app:lower:eq:prec-value}
\end{equation}
where both sums below run over subsets $S\subseteq D\setminus\{z_i\}$:
\begin{equation}
U_{=}=\sum_{S:\,A(S)=T}|S|!\,(n-|S|)!,\qquad
U_{\neq}=\sum_{S:\,A(S)\neq T}\frac{|S|!\,(n-|S|)!}{\bigl(a_i+A(S)\bigr)^2}.
\label{app:lower:eq:prec-sums}
\end{equation}
$U_{=}$ is a nonnegative integer, and it is positive exactly when the input of \textsc{Subset-Sum} has a solution.

\emph{Exponents of two.} We need four bounds.
\begin{enumerate}
\item $\nu_2(a_i)=\nu_2(T)\le L-1$. Indeed $1\le T<2^L$ gives $\nu_2(T)\le L-1<t=\nu_2(2^t)$, and the two terms of
$a_i=2^t-T$ have different exponents.
\item $\nu_2(N!)\le n$. By Legendre's formula, $\nu_2(N!)=\sum_{k\ge1}\lfloor N/2^k\rfloor<N=n+1$.
\item $U_{\neq}>0$ and $\nu_2(U_{\neq})\ge-2(L-1)$. Let $A(S)\neq T$. Then $a_i+A(S)=2^t+(A(S)-T)$ with
$0<|A(S)-T|\le G<2^L$, so $\nu_2(A(S)-T)\le L-1<t$ and therefore $\nu_2(a_i+A(S))=\nu_2(A(S)-T)\le L-1$. Every term
of $U_{\neq}$ is thus a positive rational number with an exponent of at least $-2(L-1)$. The sum has at least one
term, that of the empty subset, because $A(\emptyset)=0\neq T$.
\item $U_{=}\le2^{V-1}$, so $\nu_2(U_{=})\le V-1$ whenever $U_{=}\ge1$. Indeed $U_{=}$ has at most $2^n$ terms, each
at most $n!\le2^{\lceil\log_2(n!)\rceil}$.
\end{enumerate}

\emph{No solution.} Then $U_{=}=0$ and $\phi_i=-a_i^2\,U_{\neq}/N!$. By the bounds 1 to 3,
\begin{equation*}
\nu_2(\phi_i)=2\,\nu_2(a_i)-\nu_2(N!)+\nu_2(U_{\neq})\;\ge\;-n-2(L-1).
\end{equation*}
The exponent of two in the denominator of $\phi_i$ is $\max(0,-\nu_2(\phi_i))\le n+2L-2$. This is less than $t$,
because $V\ge n+1$. So $2^t$ does not divide the denominator.

\emph{A solution exists.} Then $U_{=}\ge1$, and by the bounds 3 and 4
\begin{equation*}
\nu_2\bigl(2^{-2t}U_{=}\bigr)\;\le\;V-1-2t\;=\;-V-4L-1\;<\;-2(L-1)\;\le\;\nu_2(U_{\neq}).
\end{equation*}
The two terms in the bracket of~\eqref{app:lower:eq:prec-value} have different exponents, so the bracket has the
smaller one, $\nu_2(U_{=})-2t$. With the bounds 1 and 4,
\begin{equation*}
\nu_2(\phi_i)=2\,\nu_2(a_i)-\nu_2(N!)+\nu_2(U_{=})-2t\;\le\;2(L-1)+(V-1)-2t\;=\;-t-3.
\end{equation*}
The exponent of two in the denominator of $\phi_i$ is therefore at least $t+3$, and $2^t$ divides the denominator.

In both cases the answer to \textsc{Precision} on the constructed input equals the answer to \textsc{Subset-Sum}.
\end{proof}

\subsection{The Counting Problem \#Subset-Sum}
\label{app:lower:ss}

The hardness proof for the absolute loss counts the solutions of \textsc{Subset-Sum}. This section fixes the
counting problem and shows that it is as hard as any counting problem in \#P.

\textbf{Counting classes and reductions.} \#P is the class of functions that count the accepting computations of a
nondeterministic polynomial-time machine \citep{valiant1979permanent}, and FP is the class of functions that a
deterministic algorithm computes in polynomial time. A \emph{parsimonious} reduction from a function $F$ to a
function $F'$ is a polynomial-time map $f$ with $F(x)=F'(f(x))$ for every input $x$. In words, it translates
the inputs and keeps the count. A function in \#P is \emph{\#P-complete under parsimonious reductions} if every
function in \#P has a parsimonious reduction to it.

\begin{definition}[\#\textsc{Subset-Sum}]
\label{app:lower:def:ss}
The input consists of positive integers $g_1,\dots,g_n$ with $n\ge1$ and an integer $T\ge0$, the target sum. For a
set $J\subseteq\{1,\dots,n\}$ write $g(J)=\sum_{j\in J}g_j$, and let $G=g(\{1,\dots,n\})$. The output is the number
of solutions
\begin{equation}
\#\mathrm{SS}(g,T)=\bigl|\{J\subseteq\{1,\dots,n\}:\;g(J)=T\}\bigr|.
\label{app:lower:eq:ss}
\end{equation}
\end{definition}

The count is over sets of indices, so equal numbers at different indices are told apart.

\textbf{Degenerate target sums.} Two cases need no computation, and we fix them once. For $T=0$ the count is
$\#\mathrm{SS}(g,0)=1$: every $g_j$ is positive, so only the empty set has the sum $0$. For $T>G$ the count is
$\#\mathrm{SS}(g,T)=0$. We call these target sums degenerate. The reductions of Appendix~\ref{app:lower:abs} answer
them directly and otherwise assume $1\le T\le G$.

\textbf{Exact covers.} The hardness of \#\textsc{Subset-Sum} is known. We recall the argument, because we need it
for positive integers and with the count preserved. An input of the problem \#X3C consists of the set
$\{1,\dots,3k\}$ and of sets $S_1,\dots,S_n\subseteq\{1,\dots,3k\}$ with three elements each; a set may occur
several times. An \emph{exact cover} is a set of indices $J\subseteq\{1,\dots,n\}$ such that every element lies in
exactly one of the sets $S_j$ with $j\in J$, and the output is the number of exact covers.
\citet{faliszewski2009complexity} report that \citet{hunt1998complexity} showed \#X3C to be \#P-complete under
parsimonious reductions, and they note that the standard reduction from exact covers to subset sums keeps the
count. We rely on their account for the first statement. An input in which no set occurs twice is an input of the
problem as defined here, so allowing repeated sets does not affect the hardness. The next lemma proves the second
statement.

\begin{lemma}[From exact covers to \#\textsc{Subset-Sum}]
\label{app:lower:lem:x3c}
Let $S_1,\dots,S_n$ with $n\ge1$ be sets of three elements of $\{1,\dots,3k\}$, $k\ge1$. Put
\begin{equation*}
g_j=\sum_{b\in S_j}(n+1)^{b}\quad(j=1,\dots,n),\qquad T=\sum_{b=1}^{3k}(n+1)^{b}.
\end{equation*}
The numbers $g_j$ are positive integers, and the number of exact covers equals $\#\mathrm{SS}(g,T)$.
\end{lemma}

\begin{proof}
Every $g_j$ is a sum of three powers of $n+1$ and hence a positive integer.
Let $J\subseteq\{1,\dots,n\}$, and for an element $b$ let $m_b$ be the number of indices $j\in J$ with $b\in S_j$.
Then $g(J)=\sum_{b=1}^{3k}m_b\,(n+1)^{b}$. Since $m_b\le n$, this is the representation of $g(J)$ in base $n+1$,
with the digit $m_b$ at position $b$ and the digit $0$ at position $0$. The target sum $T$ has the digit $1$ at
the positions $1,\dots,3k$. A number has only one representation in base $n+1$, so $g(J)=T$ if and only if $m_b=1$
for every $b$, that is, if and only if $J$ is an exact cover. Both counts are over sets of indices, so sets that
occur several times are treated alike on both sides.
\end{proof}

\begin{proposition}
\label{app:lower:prop:ss}
\#\textup{\textsc{Subset-Sum}} is \#P-complete under parsimonious reductions.
\end{proposition}

\begin{proof}
\#\textsc{Subset-Sum} is in \#P: a nondeterministic machine guesses a set $J$ and accepts if $g(J)=T$, so its number
of accepting computations is $\#\mathrm{SS}(g,T)$. For the hardness, every function in \#P has a parsimonious
reduction to \#X3C, by the result of \citet{hunt1998complexity} as \citet{faliszewski2009complexity} report it.
The map of Lemma~\ref{app:lower:lem:x3c} is computed in polynomial time, because its numbers have at most $3k+1$
digits in base $n+1$ and the input lists the $3k$ elements, and it keeps the count. An input of \#X3C with $k=0$
or $n=0$ has one exact cover (the empty set of indices, when there is no element to cover) or none, and is mapped
to the input $g=(1)$ with $T=1$ or with $T=2$, which has the same count. So \#X3C has a parsimonious reduction to
\#\textsc{Subset-Sum}, and the composition of two parsimonious reductions is parsimonious.
\end{proof}

\subsection{Proof of Theorem~\ref{thm:abs}}
\label{app:lower:abs}

\begin{restate}{Theorem~\ref{thm:abs}}
For unweighted KNN regression ($w_j=1$ for all $j$) with the absolute loss, integer targets written in binary and
$K=N$, computing the Shapley value of a training point is \#P-hard. Consequently, unless
$\mathrm{FP}=\#\mathrm{P}$, no deterministic algorithm computes the value in time polynomial in $N$ and the number
of bits of the targets, or approximates it within $2^{-(N+2\log_2N+1)}$ in such time.
\end{restate}

This section fixes the setting and the kinds of reductions and states the theorem in full
(Theorem~\ref{app:lower:thm:abs}). It then proves the lemmas, first those for the reduction with three queries and
then those for the reduction with one query, and finally the five parts of the full statement.

\textbf{Setting.} Throughout this section the weights are $w_j=1$ and the loss is the absolute loss. The targets
$y_1,\dots,y_N$ and the target $y^{(\val)}$ of the query are integers, and the default prediction is
$y_\emptyset=y^{(\val)}$ unless we say otherwise. Let $b$ be the largest number of bits of the integers
$y_1,\dots,y_N$ and $y^{(\val)}$, so that each of them is smaller than $2^b$ in absolute value. With unit weights,
the quantities of~\eqref{eq:AR} are $r_j=y_j-y^{(\val)}$ and $R(C)=\sum_{z_j\in C}r_j$, and the prediction error of a
nonempty subset $S$ is $R(\NK(S))/|\NK(S)|$. Hence
\begin{equation}
v(S)=-\frac{|R(\NK(S))|}{|\NK(S)|}\quad\text{for }S\neq\emptyset,\qquad
v(\emptyset)=-|y_\emptyset-y^{(\val)}|.
\label{app:lower:eq:abs-utility}
\end{equation}
In particular $v(\emptyset)=0$ when $y_\emptyset=y^{(\val)}$. We write
$\Lambda_N=N!\cdot\operatorname{lcm}(1,\dots,N)$.

\textbf{Kinds of reductions.} Let $F$ and $F'$ be two functions. A \emph{one-query reduction} from $F$ to $F'$
consists of two polynomial-time maps $f$ and $h$ with $F(x)=h(x,F'(f(x)))$ for every input $x$. In words, it builds
one instance of $F'$, asks for its value once, and turns the answer into $F(x)$ in polynomial time. This is known
as a metric reduction \citep{krentel1988complexity}. A \emph{reduction with three non-adaptive queries} builds three
instances of $F'$ before it sees any answer, and computes $F(x)$ from the three answers in polynomial time. A
function is \emph{\#P-hard} under one of these two kinds of reduction if every function in \#P reduces to it by a
reduction of this kind. A parsimonious reduction followed by a reduction of either kind is again of that kind. By
Proposition~\ref{app:lower:prop:ss}, it therefore suffices to reduce \#\textsc{Subset-Sum}. Finally,
$\mathrm{FP}=\#\mathrm{P}$ abbreviates the statement that every function in \#P is in FP.

\begin{theorem}[Full version of Theorem~\ref{thm:abs}]
\label{app:lower:thm:abs}
Consider unweighted KNN regression with the absolute loss, integer targets and an integer $y^{(\val)}$, and let
$(g,T)$ denote an input of \#\textup{\textsc{Subset-Sum}} as in Definition~\ref{app:lower:def:ss}.
\begin{enumerate}
\item[(a)] \emph{The value is short.} Let $y_\emptyset=y^{(\val)}$. For every $K$ and every training point $z_i$,
$\Lambda_N\,\phi_i$ is an integer and $|\phi_i|<2^{b+1}$. Hence $\phi_i$ is a fraction whose numerator and denominator
have at most $2N\log_2N+b+2$ bits each, that is, $O(N\log N+b)$ bits. It is computed in polynomial time with two
queries to a \#P oracle.
\item[(b)] \emph{One query.} \#\textup{\textsc{Subset-Sum}} has a one-query reduction to the value. That is, there
are polynomial-time maps $f$ and $h$ such that, for every input $(g,T)$, $f(g,T)$ is an instance with $K=N$ and
$y^{(\val)}=y_\emptyset=0$ together with one of its training points $z_i$, and
\begin{equation}
\#\mathrm{SS}(g,T)=h\bigl(g,T,\phi_i\bigr)\qquad\text{for the value $\phi_i$ of $z_i$ in $f(g,T)$}.
\label{app:lower:eq:one-query}
\end{equation}
If $1\le T\le G$, the instance has $N=2n'+1$ training points with $n'\in\{n,n+1\}$.
\item[(c)] \emph{Three queries.} Let $1\le T\le G$, $\beta=\max(G+1,3)$ and
$T_{\mathrm e}=T+\sum_{j=1}^{n}\beta^{j}$. For $\theta\in\{-1,0,+1\}$ let $\mathcal{I}_\theta$ be the instance with
$N=2n+1$ training points, $K=N$ and $y^{(\val)}=y_\emptyset=0$, in which $2n$ points have the targets
$g_j+\beta^{j}$ and $\beta^{j}$ for $j=1,\dots,n$, and the valued point $z_i$ has the target
$\theta-T_{\mathrm e}$. The three instances differ only in the target of $z_i$, and its values $\phi_i^{(\theta)}$ in
them satisfy
\begin{equation}
\phi_i^{(+1)}-2\,\phi_i^{(0)}+\phi_i^{(-1)}=-\frac{2\,c_N(n)}{n+1}\;\#\mathrm{SS}(g,T).
\label{app:lower:eq:three-query}
\end{equation}
So \#\textup{\textsc{Subset-Sum}} reduces to the value by a reduction with three non-adaptive queries.
\item[(d)] \emph{Smaller $K$.} Parts (b) and (c) remain true when the valued point is the farthest point $z_N$ and
$K$ is any number with $(N+1)/2\le K\le N$. After padding with dummy points they remain true for much smaller $K$.
Let $c\in(0,1]$ be a constant and let $\kappa$ be an integer-valued function with $N^{c}\le\kappa(N)\le N$ that is
computable in time polynomial in $N$. Then the instances of (b) and (c) can be chosen with $K=\kappa(N)$, where $N$
is their number of training points and is bounded by a polynomial in $n$. The
identity~\eqref{app:lower:eq:three-query} then holds with this $N$.
\item[(e)] \emph{Consequences.} Unless $\mathrm{FP}=\#\mathrm{P}$, the following holds.
(i) No deterministic algorithm computes $\phi_i$ in time polynomial in $N$ and $b$.
(ii) No deterministic polynomial-time algorithm returns a number within $2^{-(N+2\log_2N+1)}$ of $\phi_i$. In
particular, no deterministic algorithm approximates $\phi_i$ within a given $\varepsilon>0$ in time polynomial in
$N$, $b$ and $\log(1/\varepsilon)$.
Both statements hold already for the instances with $K=N$ and $y^{(\val)}=y_\emptyset=0$, and also for the instances
with $K=\kappa(N)$ for any function $\kappa$ as in (d).
\end{enumerate}
\end{theorem}

\textbf{From the full version to Theorem~\ref{thm:abs}.} The two sentences of Theorem~\ref{thm:abs} are parts
(b) and (e), and the hardness is under one-query reductions. Part (d) is the extension to smaller $K$ that
Section~\ref{sec:lower} mentions: the hardness holds when $K$ is any polynomial-time computable function of $N$
with $N^{c}\le K\le N$, for a constant $c>0$. Part (a) says that the value on the hard instances is short
(Remark~\ref{app:lower:rem:short}). An unweighted instance is a weighted
instance whose weights are all equal, so the hardness also holds for weighted KNN regression with the absolute loss.

Part (e)(i) is the sense in which the target grid of Theorem~\ref{thm:pair} is needed. Take unit weights and integer
targets, and let $L_y$ be the number of integers from the smallest target to the largest. A set of $s$ points then
has one of at most $s(L_y-1)+1$ sums of targets, so the number of pairs in Theorem~\ref{thm:pair} is
$P\le(K+1)^2L_y$ and its running time is polynomial in $N$ and $L_y$. Since $\log_2L_y\le b+1$, part (e)(i) excludes a
running time that is polynomial in $N$ and $\log L_y$.

\textbf{The value.} In all our instances the valued point joins the neighbors of every subset that matters, either
because $K=N$ or because it is the farthest point.

\begin{lemma}[The value under the absolute loss]
\label{app:lower:lem:abs-value}
Suppose that $K=N$, or that the valued point is the farthest point, $i=N$. Then, for every default prediction
$y_\emptyset$,
\begin{equation}
\phi_i=-\sum_{\substack{S\subseteq D\setminus\{z_i\}\\ |S|\le K-1}}c_N(|S|)\,\frac{|R(S)+r_i|}{|S|+1}
\;+\sum_{\substack{S\subseteq D\setminus\{z_i\}\\ 1\le|S|\le K-1}}c_N(|S|)\,\frac{|R(S)|}{|S|}
\;+\;\frac{|y_\emptyset-y^{(\val)}|}{N}.
\label{app:lower:eq:abs-value}
\end{equation}
\end{lemma}

The last term is a constant that vanishes for $y_\emptyset=y^{(\val)}$.

\begin{proof}
Let $S\subseteq D\setminus\{z_i\}$. We claim that the marginal contribution of $z_i$ to $S$ is zero if $|S|\ge K$,
and that $\NK(S)=S$ and $\NK(S\cup\{z_i\})=S\cup\{z_i\}$ if $|S|\le K-1$. If $K=N$, then $|S|\le N-1=K-1$, and a set
of at most $K$ points is its own set of $K$ nearest neighbors. If $i=N$, every point of $S$ is nearer to the query
than $z_i$, and the claim is cases~1 and~2 of Lemma~\ref{lem:cases}.

Let $|S|\le K-1$. By~\eqref{app:lower:eq:abs-utility}, $v(S\cup\{z_i\})=-|R(S)+r_i|/(|S|+1)$, and
$v(S)=-|R(S)|/|S|$ if $S$ is nonempty. For the empty subset, $-c_N(0)\,v(\emptyset)=|y_\emptyset-y^{(\val)}|/N$,
because $c_N(0)=1/N$. Inserting this into~\eqref{eq:shapley} gives~\eqref{app:lower:eq:abs-value}.
\end{proof}

\begin{lemma}[Integrality and size]
\label{app:lower:lem:integral}
Let $y_\emptyset=y^{(\val)}$. For every $K$ and every training point $z_i$, at any position among the points,
$\Lambda_N\,\phi_i$ is an integer and $|\phi_i|\le Y$, where $Y=\max_j|y_j-y^{(\val)}|$.
\end{lemma}

\begin{proof}
Let $S'\subseteq D$ be nonempty. By~\eqref{app:lower:eq:abs-utility}, $v(S')$ is an integer divided by
$|\NK(S')|$, which is a number between $1$ and $N$ and therefore divides $\operatorname{lcm}(1,\dots,N)$. Moreover
$v(\emptyset)=0$. For every $s\in\{0,\dots,N-1\}$ we have
$\Lambda_N\,c_N(s)=\operatorname{lcm}(1,\dots,N)\cdot s!\,(N-1-s)!$. Hence $\Lambda_N\,c_N(s)\,v(S')$ is an integer
for every $s$ and every $S'\subseteq D$, and by~\eqref{eq:shapley} $\Lambda_N\,\phi_i$ is a sum of differences of such
integers.

For the bound, the prediction error of a nonempty subset is the average of the numbers $r_j$ over its $K$ nearest
neighbors, so $-Y\le v(S')\le0$ for every $S'\subseteq D$. Every marginal contribution is a difference of two such
numbers and is therefore at most $Y$ in absolute value. The coefficients in~\eqref{eq:shapley} sum to one, because
there are $\binom{N-1}{s}$ subsets of size $s$ and $\binom{N-1}{s}\,c_N(s)=1/N$. Hence $|\phi_i|\le Y$.
\end{proof}

\textbf{Three queries.} The reduction with three queries rests on the kink of the absolute value at zero.

\begin{lemma}[Second difference of the absolute value]
\label{app:lower:lem:kink}
For every integer $u$, $|u+1|-2|u|+|u-1|$ equals $2$ if $u=0$ and $0$ otherwise.
\end{lemma}

\begin{proof}
If $u\ge1$, the three arguments are nonnegative and the expression is $(u+1)-2u+(u-1)=0$. If $u\le-1$, they are
nonpositive and the expression is $-(u+1)+2u-(u-1)=0$. If $u=0$, it is $1-0+1=2$.
\end{proof}

The coefficient $c_N(|S|)$ of a subset depends on its size, so a sum over the solutions of a
\#\textsc{Subset-Sum} input weights solutions of different sizes differently. The next lemma rewrites an input such
that all its solutions have the same size. For a set $I$ of indices we write $e(I)=\sum_{k\in I}e_k$.

\begin{lemma}[Solutions of equal size]
\label{app:lower:lem:equal}
Let $g_1,\dots,g_n$ be positive integers with the sum $G$, let $0\le T\le G$ and $\beta=\max(G+1,3)$. Define $2n$
positive integers and a target sum by
\begin{equation}
e_j=g_j+\beta^{j},\qquad e_{n+j}=\beta^{j}\qquad(j=1,\dots,n),\qquad
T_{\mathrm e}=T+\sum_{j=1}^{n}\beta^{j}.
\label{app:lower:eq:equal}
\end{equation}
A set $I\subseteq\{1,\dots,2n\}$ has the sum $e(I)=T_{\mathrm e}$ if and only if
$I=J\cup\{n+j:\,j\in\{1,\dots,n\}\setminus J\}$ for a set $J\subseteq\{1,\dots,n\}$ with $g(J)=T$. Consequently
there are exactly $\#\mathrm{SS}(g,T)$ sets $I$ with $e(I)=T_{\mathrm e}$, and each of them has $n$ elements.
\end{lemma}

\begin{proof}
Let $I\subseteq\{1,\dots,2n\}$ and $J=I\cap\{1,\dots,n\}$. For $j=1,\dots,n$ let $\eta_j\in\{0,1,2\}$ be the number
of the two indices $j$ and $n+j$ that belong to $I$. Then
\begin{equation*}
e(I)=g(J)+\sum_{j=1}^{n}\eta_j\,\beta^{j}.
\end{equation*}
This is the representation of $e(I)$ in base $\beta$: the digit $g(J)$ at position $0$ is at most $G\le\beta-1$, and
the digits $\eta_j$ at the positions $j\ge1$ are at most $2\le\beta-1$. In the same way, $T_{\mathrm e}$ has the digit
$T\le\beta-1$ at position $0$ and the digit $1$ at the positions $1,\dots,n$. A number has only one representation
in base $\beta$, so $e(I)=T_{\mathrm e}$ if and only if $g(J)=T$ and $\eta_j=1$ for every $j$. The second condition
says that $I$ contains exactly one of $j$ and $n+j$ for every $j$, that is,
$I=J\cup\{n+j:\,j\in\{1,\dots,n\}\setminus J\}$. Such a set has $n$ elements and is determined by $J$, which gives
the count.
\end{proof}

\textbf{The three instances.} Let $1\le T\le G$, and let $e_1,\dots,e_{2n}$ and $T_{\mathrm e}$ be as in
Lemma~\ref{app:lower:lem:equal}. Let $N\ge2n+1$ and $n+1\le K\le N$. For $\theta\in\{-1,0,+1\}$ the instance
$\mathcal{I}_\theta$ has $y^{(\val)}=y_\emptyset=0$. Its valued point is the farthest point $z_N$, with the target
$y_N=\theta-T_{\mathrm e}$. Its other points are, in any order, $2n$ \emph{number points} with the targets
$e_1,\dots,e_{2n}$ and $N-1-2n$ \emph{dummy points} with the target $T_{\mathrm e}+1$. The three instances differ
only in the target of $z_N$. We write $\phi_N^{(\theta)}$ for the value of $z_N$ in $\mathcal{I}_\theta$.

\begin{lemma}[Three queries]
\label{app:lower:lem:three}
The values of $z_N$ in the three instances $\mathcal{I}_{-1}$, $\mathcal{I}_0$ and $\mathcal{I}_{+1}$ satisfy
\begin{equation}
\phi_N^{(+1)}-2\,\phi_N^{(0)}+\phi_N^{(-1)}=-\frac{2\,c_N(n)}{n+1}\;\#\mathrm{SS}(g,T).
\label{app:lower:eq:three-general}
\end{equation}
\end{lemma}

\begin{proof}
Since $y^{(\val)}=0$, we have $r_j=y_j$ for every point. For $S\subseteq D\setminus\{z_N\}$ put
$u(S)=R(S)-T_{\mathrm e}$. This integer is the same in the three instances, and $R(S)+r_N=u(S)+\theta$ in
$\mathcal{I}_\theta$. By Lemma~\ref{app:lower:lem:abs-value} with $i=N$,
\begin{equation*}
\phi_N^{(\theta)}=-\sum_{\substack{S\subseteq D\setminus\{z_N\}\\ |S|\le K-1}}c_N(|S|)\,\frac{|u(S)+\theta|}{|S|+1}
\;+\sum_{\substack{S\subseteq D\setminus\{z_N\}\\ 1\le|S|\le K-1}}c_N(|S|)\,\frac{|R(S)|}{|S|},
\end{equation*}
and the second sum is the same in the three instances. We take the second difference in $\theta$ and apply
Lemma~\ref{app:lower:lem:kink} to every subset:
\begin{equation*}
\phi_N^{(+1)}-2\,\phi_N^{(0)}+\phi_N^{(-1)}
=-2\sum_{\substack{S\subseteq D\setminus\{z_N\},\ |S|\le K-1\\ u(S)=0}}\frac{c_N(|S|)}{|S|+1}.
\end{equation*}
It remains to find the subsets with $u(S)=0$, that is, with $R(S)=T_{\mathrm e}$. A subset that contains a dummy point
has $R(S)\ge T_{\mathrm e}+1$, because the targets of all other points are positive. A subset of number points is given
by the set $I$ of their indices, and $R(S)=e(I)$. By Lemma~\ref{app:lower:lem:equal}, exactly
$\#\mathrm{SS}(g,T)$ of these subsets have $R(S)=T_{\mathrm e}$, and each of them has $n\le K-1$ points. This
gives~\eqref{app:lower:eq:three-general}.
\end{proof}

The identity does not depend on the choice $y_\emptyset=y^{(\val)}$. By Lemma~\ref{app:lower:lem:abs-value}, another
default prediction adds the same constant to the three values, and the constant cancels in the second difference.

\textbf{One query.} A single value already contains the count. The idea is to make the targets of the other points
multiples of a large number $\mu$ and the target of the valued point a multiple of $\mu$ plus one. Modulo $\mu$,
every subset then contributes only a sign to $\Lambda_N\,\phi_N$, by the next lemma. Three further lemmas make all
these signs known in advance, except those that carry the count.

\begin{lemma}[Residue of the absolute value]
\label{app:lower:lem:residue}
Let $u$ and $\mu\ge2$ be integers. Then $|\mu u+1|\equiv1\pmod{\mu}$ if $u\ge0$, and
$|\mu u+1|\equiv-1\pmod{\mu}$ if $u<0$.
\end{lemma}

\begin{proof}
If $u\ge0$, then $\mu u+1>0$ and $|\mu u+1|=\mu u+1\equiv1$. If $u\le-1$, then $\mu u+1\le1-\mu<0$ and
$|\mu u+1|=-\mu u-1\equiv-1$.
\end{proof}

\begin{lemma}[Balanced input]
\label{app:lower:lem:balanced}
Let $1\le T\le G$. If $2T=G$, let $g'=g$, $T'=T$ and $\gamma=1$. Otherwise let
$g'=(g_1,\dots,g_n,|G-2T|)$, $T'=\max(T,G-T)$ and $\gamma=2$. Then $g'$ consists of $n'\in\{n,n+1\}$ positive
integers with the sum $2T'$, $1\le T'\le G$, and $\#\mathrm{SS}(g',T')=\gamma\cdot\#\mathrm{SS}(g,T)$.
\end{lemma}

\begin{proof}
If $2T=G$ there is nothing to prove. Let $2T\neq G$. The sum of $g'$ is $G+|G-2T|=2\max(T,G-T)=2T'$, and
$1\le T\le G$ gives $1\le T'\le G$. The map $J\mapsto\{1,\dots,n\}\setminus J$ is a bijection between the sets with
$g(J)=T$ and the sets with $g(J)=G-T$, so $\#\mathrm{SS}(g,T)=\#\mathrm{SS}(g,G-T)$. A solution of $(g',T')$ that does
not contain the new index $n+1$ is a set $J\subseteq\{1,\dots,n\}$ with $g(J)=T'$. A solution that contains it is
$J\cup\{n+1\}$ with $g(J)=T'-|G-2T|=\min(T,G-T)$. Both $T'$ and $\min(T,G-T)$ are among $T$ and $G-T$, so there are
$\#\mathrm{SS}(g,T)$ solutions of each kind.
\end{proof}

\begin{lemma}[Shift]
\label{app:lower:lem:shift}
Let $e_1,\dots,e_{2m}$ be positive integers and let $T_{\mathrm e}$ be an integer with
$e_1+\dots+e_{2m}=2T_{\mathrm e}$. Put $H=2T_{\mathrm e}+1$ and define
\begin{equation}
e^{+}_j=e_j+H\quad(j=1,\dots,2m),\qquad T_{+}=T_{\mathrm e}+mH,
\label{app:lower:eq:shift}
\end{equation}
and write $e^{+}(I)=\sum_{k\in I}e^{+}_k$. Then every set $I\subseteq\{1,\dots,2m\}$ satisfies $e^{+}(I)<T_{+}$ if
$|I|<m$, $e^{+}(I)>T_{+}$ if $|I|>m$, and $e^{+}(I)-T_{+}=e(I)-T_{\mathrm e}$ if $|I|=m$.
\end{lemma}

In words, once the large number $H$ is added to every number, the size of a set decides how its sum compares with
the target sum. The sets of size $m$ are the exception: for them the comparison does not change.

\begin{proof}
We have $e^{+}(I)-T_{+}=e(I)-T_{\mathrm e}+(|I|-m)\,H$ and $0\le e(I)\le2T_{\mathrm e}$. If $|I|\le m-1$, then
$e^{+}(I)-T_{+}\le T_{\mathrm e}-H<0$. If $|I|\ge m+1$, then $e^{+}(I)-T_{+}\ge-T_{\mathrm e}+H>0$. If $|I|=m$, the
last term vanishes.
\end{proof}

\begin{lemma}[Symmetry]
\label{app:lower:lem:symmetry}
In the setting of Lemma~\ref{app:lower:lem:shift}, among the sets $I\subseteq\{1,\dots,2m\}$ with $m$ elements, as
many have $e^{+}(I)>T_{+}$ as have $e^{+}(I)<T_{+}$.
\end{lemma}

\begin{proof}
The complement $I\mapsto\{1,\dots,2m\}\setminus I$ maps the sets with $m$ elements to themselves. Since all numbers
sum to $2T_{\mathrm e}$, the complement of $I$ has
$e(\{1,\dots,2m\}\setminus I)-T_{\mathrm e}=-(e(I)-T_{\mathrm e})$. By Lemma~\ref{app:lower:lem:shift}, it therefore
exchanges the sets with $e^{+}(I)>T_{+}$ and the sets with $e^{+}(I)<T_{+}$.
\end{proof}

\textbf{The instance of the one-query reduction.} Let $1\le T\le G$. Three steps turn $(g,T)$ into $2n'$ numbers
and a target sum.
\begin{enumerate}
\item Lemma~\ref{app:lower:lem:balanced} gives $(g',T',\gamma)$, where $g'$ consists of $n'$ numbers with the sum
$2T'$.
\item Lemma~\ref{app:lower:lem:equal}, applied to $(g',T')$, gives $e_1,\dots,e_{2n'}$ and $T_{\mathrm e}$. These
numbers sum to $2T_{\mathrm e}$, because by~\eqref{app:lower:eq:equal} their sum is the sum $2T'$ of $g'$ plus
$2\sum_{j}\beta^{j}$.
\item Lemma~\ref{app:lower:lem:shift} with $m=n'$ gives $e^{+}_1,\dots,e^{+}_{2n'}$ and $T_{+}$.
\end{enumerate}
Let $N\ge2n'+1$ and $n'+1\le K\le N$, and put $\mu=2\Lambda_N+1$. The instance has $y^{(\val)}=y_\emptyset=0$. Its
valued point is the farthest point $z_N$, with the target $y_N=1-\mu\,T_{+}$. Its other points are, in any
order, $2n'$ number points with the targets $\mu e^{+}_1,\dots,\mu e^{+}_{2n'}$ and $N-1-2n'$ dummy points with the
target $\mu\,(T_{+}+1)$.

\begin{lemma}[One query]
\label{app:lower:lem:one}
In this instance, let $\alpha_s=\Lambda_N\,c_N(s)/(s+1)$ for $0\le s\le N-1$, and let $Z$ be the integer with
$-\Lambda_N\le Z\le\Lambda_N$ and $Z\equiv-\Lambda_N\,\phi_N\pmod{\mu}$. Then
\begin{equation}
Z=Z_0+\gamma\,\alpha_{n'}\;\#\mathrm{SS}(g,T),\qquad
Z_0=\sum_{s=0}^{K-1}\alpha_s\binom{N-1}{s}-2\sum_{s=0}^{n'-1}\alpha_s\binom{2n'}{s}-\alpha_{n'}\binom{2n'}{n'}.
\label{app:lower:eq:one-general}
\end{equation}
\end{lemma}

In words, the residue of the integer $\Lambda_N\,\phi_N$ modulo $\mu$ is a known number plus a known multiple of the
count. All quantities in~\eqref{app:lower:eq:one-general} other than the count are computed from $(g,T)$, $N$ and $K$.

\begin{proof}
The numbers $\alpha_s=\operatorname{lcm}(1,\dots,N)\cdot s!\,(N-1-s)!/(s+1)$ are positive integers, because
$s+1\le N$ divides $\operatorname{lcm}(1,\dots,N)$. By Lemma~\ref{app:lower:lem:integral}, $\Lambda_N\,\phi_N$ is an
integer. The $2\Lambda_N+1=\mu$ integers from $-\Lambda_N$ to $\Lambda_N$ represent every residue modulo $\mu$ once,
so $Z$ is well defined.

\emph{The value modulo $\mu$.} Since $y^{(\val)}=0$, we have $r_j=y_j$ for every point. The targets of the other
points are multiples of $\mu$. For $S\subseteq D\setminus\{z_N\}$ the number $u(S)=R(S)/\mu-T_{+}$ is therefore
an integer, and $R(S)+r_N=\mu\,u(S)+1$. Lemma~\ref{app:lower:lem:abs-value} with $i=N$, multiplied by $\Lambda_N$,
gives
\begin{equation*}
\Lambda_N\,\phi_N=-\sum_{\substack{S\subseteq D\setminus\{z_N\}\\ |S|\le K-1}}\alpha_{|S|}\,\bigl|\mu\,u(S)+1\bigr|
\;+\sum_{\substack{S\subseteq D\setminus\{z_N\}\\ 1\le|S|\le K-1}}\frac{\Lambda_N\,c_N(|S|)}{|S|}\;|R(S)|.
\end{equation*}
In the second sum, $\Lambda_N\,c_N(s)/s=\operatorname{lcm}(1,\dots,N)\cdot s!\,(N-1-s)!/s$ is an integer and $|R(S)|$
is a multiple of $\mu$, so every term is a multiple of $\mu$. In the first sum we apply
Lemma~\ref{app:lower:lem:residue} to every subset and collect the subsets of equal size:
\begin{equation*}
\Lambda_N\,\phi_N\equiv-\sum_{s=0}^{K-1}\alpha_s\,\zeta_s\pmod{\mu},
\end{equation*}
where $\zeta_s$ is the number of subsets $S\subseteq D\setminus\{z_N\}$ of size $s$ with $u(S)\ge0$ minus the number
of those with $u(S)<0$. Put $Z^{*}=\sum_{s=0}^{K-1}\alpha_s\,\zeta_s$. Since $|\zeta_s|\le\binom{N-1}{s}$ and
$\binom{N-1}{s}\,c_N(s)=1/N$,
\begin{equation*}
|Z^{*}|\;\le\;\sum_{s=0}^{K-1}\alpha_s\binom{N-1}{s}\;=\;\frac{\Lambda_N}{N}\sum_{s=0}^{K-1}\frac{1}{s+1}
\;\le\;\Lambda_N.
\end{equation*}
Hence $Z^{*}$ is the integer $Z$ of the statement.

\emph{The signs.} Fix a size $s\le K-1$. A subset that contains a dummy point has $R(S)/\mu\ge T_{+}+1$ and
so $u(S)\ge1$. There are $\binom{N-1}{s}-\binom{2n'}{s}$ such subsets of size $s$. A subset of number points with the
set $I$ of indices has $R(S)/\mu=e^{+}(I)$ and so $u(S)=e^{+}(I)-T_{+}$, and there are $\binom{2n'}{s}$ of them. By
Lemma~\ref{app:lower:lem:shift}, all of them have $u(S)<0$ if $s<n'$, and all of them have $u(S)>0$ if $s>n'$. Let
$s=n'$. By Lemma~\ref{app:lower:lem:symmetry}, as many of them have $u(S)>0$ as have $u(S)<0$, and these cancel in
$\zeta_{n'}$. By Lemma~\ref{app:lower:lem:shift}, $u(S)=0$ if and only if $e(I)=T_{\mathrm e}$. By
Lemmas~\ref{app:lower:lem:equal} and~\ref{app:lower:lem:balanced}, this holds for exactly
$\#\mathrm{SS}(g',T')=\gamma\,\#\mathrm{SS}(g,T)$ of them. Therefore
\begin{equation*}
\zeta_s=\binom{N-1}{s}-\binom{2n'}{s}+
\begin{cases}
-\binom{2n'}{s} & \text{if } s<n',\\[2pt]
\gamma\,\#\mathrm{SS}(g,T) & \text{if } s=n',\\[2pt]
\binom{2n'}{s} & \text{if } s>n'.
\end{cases}
\end{equation*}
The size $n'$ occurs in $Z^{*}$ because $n'\le K-1$. Multiplying by $\alpha_s$ and summing over $s=0,\dots,K-1$
gives~\eqref{app:lower:eq:one-general}.
\end{proof}

\begin{proof}[Proof of Theorem~\ref{app:lower:thm:abs}]
\emph{Part (a).} By Lemma~\ref{app:lower:lem:integral}, $\Lambda_N\,\phi_i$ is an integer and $|\phi_i|\le Y$. The
targets and $y^{(\val)}$ are smaller than $2^{b}$ in absolute value, so $Y<2^{b+1}$ and $|\phi_i|<2^{b+1}$. Hence
$\phi_i$ is the fraction with the denominator $\Lambda_N$ and the numerator $\Lambda_N\,\phi_i$, which is smaller than
$2^{b+1}\Lambda_N$ in absolute value. Since $\operatorname{lcm}(1,\dots,N)$ divides $N!$ and $N!\le N^N$, we have
$\log_2\Lambda_N\le2N\log_2N$. So the numerator has at most $2N\log_2N+b+2$ bits and the denominator at most
$2N\log_2N+1$, and reducing the fraction does not make them longer.

For the \#P oracle, recall that $v\le0$ and write~\eqref{eq:shapley} as $\Lambda_N\,\phi_i=F_2-F_1$ with
\begin{equation*}
F_1=\sum_{S\subseteq D\setminus\{z_i\}}\Lambda_N\,c_N(|S|)\,\bigl|v(S\cup\{z_i\})\bigr|,\qquad
F_2=\sum_{S\subseteq D\setminus\{z_i\}}\Lambda_N\,c_N(|S|)\,\bigl|v(S)\bigr|.
\end{equation*}
By the proof of Lemma~\ref{app:lower:lem:integral}, every term is a nonnegative integer of at most $\Lambda_NY$. It
is computed from $S$ in polynomial time. Consider a nondeterministic polynomial-time machine that guesses a subset
$S$ and a nonnegative integer of at most $2N\log_2N+b+2$ bits, and accepts if this integer is smaller than the term
of $S$ in $F_1$. It has exactly $F_1$ accepting computations, so $F_1$ is in \#P, and likewise $F_2$. Two queries
give $\Lambda_N\,\phi_i$, and a division by $\Lambda_N$ gives $\phi_i$.

\emph{Part (b).} For $1\le T\le G$, let $f(g,T)$ be the instance of Lemma~\ref{app:lower:lem:one} with $N=2n'+1$
and $K=N$, which has no dummy points. The map $h$ computes $Z$ from the value $\phi_N$ and returns
$(Z-Z_0)/(\gamma\,\alpha_{n'})$, which is the count by~\eqref{app:lower:eq:one-general}. Both maps run in polynomial
time. Indeed $\Lambda_N$, $\mu$, the numbers $\alpha_s$ and the binomial coefficients have $O(N\log N)$ bits. The
base is $\beta=2T'+1\le2G+1$, and $T_{+}$ and the numbers $e^{+}_j$ have $O(n'\log\beta+\log n')$ bits. For the
degenerate target sums, $f(g,T)$ is the instance with a single training point of target $0$, and $h$ returns $1$ if
$T=0$ and $0$ if $T>G$, without using the value.

\emph{Part (c).} This is Lemma~\ref{app:lower:lem:three} with $N=2n+1$ and $K=N$, in which case there are no dummy
points. For $K=N$ the position of the valued point among the points plays no role, by
Lemma~\ref{app:lower:lem:abs-value}. The three instances are computed from $(g,T)$ in polynomial time and before any
value is known: their targets have $O(n\log\beta)$ bits and $\beta\le G+2$. The count is the second difference of the
three values divided by the known number $-2c_N(n)/(n+1)$.

\emph{Part (d).} The instances of Lemmas~\ref{app:lower:lem:three} and~\ref{app:lower:lem:one} have the valued point
farthest and allow every $K$ with $n+1\le K\le N$ and $n'+1\le K\le N$, respectively. Without dummy points,
$N=2n+1$ and $N=2n'+1$, and these conditions read $(N+1)/2\le K\le N$. This is the first claim. For the second
claim, fix an integer $\rho\ge2$ with $\rho\ge1/c$, and take $N=(n+1)^{\rho}$ in Lemma~\ref{app:lower:lem:three} and
$N=(n'+1)^{\rho}$ in Lemma~\ref{app:lower:lem:one}. Then $N\ge2n+1$ and $N\ge2n'+1$ because $\rho\ge2$, and $N$ is
bounded by a polynomial in $n$. The reduction computes $K=\kappa(N)$ for this padded $N$. It satisfies
$K\ge N^{c}=(n+1)^{\rho c}\ge n+1$ in the first case and $K\ge n'+1$ in the second, so the lemmas apply. The maps of
part (b) change accordingly: $f$ outputs the padded instance, and $h$ uses the numbers $\alpha_s$ and $Z_0$ of
Lemma~\ref{app:lower:lem:one} for this $N$ and $K$.

\emph{Part (e).} (i) Suppose that a deterministic algorithm computes $\phi_i$ in time polynomial in $N$ and $b$. For
an input $(g,T)$, the instance $f(g,T)$ of part (b) is computed in polynomial time, so its $N$ and $b$ are bounded
by a polynomial in the size of $(g,T)$. Running the algorithm on it and applying $h$ computes $\#\mathrm{SS}(g,T)$ in
polynomial time. By Proposition~\ref{app:lower:prop:ss}, every function in \#P has a parsimonious reduction to
\#\textsc{Subset-Sum} and is then in FP as well. With the instances of part (d), the same argument applies to an
algorithm that is only required to work when $K=\kappa(N)$.

(ii) Write $\varepsilon_N=2^{-(N+2\log_2N+1)}=1/(2^{N+1}N^2)$. Suppose that a deterministic polynomial-time algorithm
returns, for every instance, a number within $\varepsilon_N$ of $\phi_i$. For the degenerate target sums the count
is known. Let $1\le T\le G$ and run the algorithm on the three instances of part (c), or on those of part (d). Let
$\hat\Delta$ be the second difference of its three answers, formed as on the left
of~\eqref{app:lower:eq:three-query}. The second difference of the three values is $-\chi\,\#\mathrm{SS}(g,T)$ with
$\chi=2c_N(n)/(n+1)$, and $\hat\Delta$ differs from it by at most $4\varepsilon_N$. Now
$c_N(n)=1/\bigl(N\binom{N-1}{n}\bigr)$, $\binom{N-1}{n}<2^{N-1}$ and $n+1\le N$, so
\begin{equation*}
\chi=\frac{2}{(n+1)\,N\binom{N-1}{n}}\;>\;\frac{2}{N^2\,2^{N-1}}\;=\;8\,\varepsilon_N.
\end{equation*}
Hence $-\hat\Delta/\chi$ differs from the integer $\#\mathrm{SS}(g,T)$ by at most $4\varepsilon_N/\chi<1/2$, and
rounding it to the nearest integer gives the count in polynomial time. As in (i), this implies
$\mathrm{FP}=\#\mathrm{P}$. Finally, an algorithm that approximates $\phi_i$ within a given $\varepsilon$ in time
polynomial in $N$, $b$ and $\log(1/\varepsilon)$ runs in polynomial time for $\varepsilon=\varepsilon_N$, because
$\log_2(1/\varepsilon_N)=N+2\log_2N+1$.
\end{proof}

\begin{remark}[The squared loss]
\label{app:lower:rem:squared}
The reductions need the kink of the absolute value at zero, and the squared loss has none. For all integers $u$ and
$\mu\ge2$ we have $(u+1)^2-2u^2+(u-1)^2=2$ and $(\mu u+1)^2\equiv1\pmod{\mu}$, so neither the second difference nor the
residue depends on $u$. For instance, with unit weights and $K=N$ the squared-loss utility of $S\cup\{z_i\}$ is
$-(R(S)+r_i)^2/(|S|+1)^2$. Its second difference in the target of $z_i$ is $-2/(|S|+1)^2$, so that of the value is
$-(2/N)\sum_{s=1}^{N}1/s^2$ on every instance, because $\binom{N-1}{s}\,c_N(s)=1/N$. No other reduction can succeed
either, unless $\mathrm{FP}=\#\mathrm{P}$: with unit weights, Theorem~\ref{thm:sq} with $W=1$ computes the
squared-loss values of all points with $O(K^2N)$ arithmetic operations after sorting, which is polynomial time for
integer targets.
\end{remark}

\begin{remark}[Short values]
\label{app:lower:rem:short}
Theorems~\ref{thm:bits} and~\ref{thm:abs} limit exact computation for different reasons. On the instances of
Theorem~\ref{thm:bits} the exact value is exponentially long, so no algorithm writes it down in polynomial time,
while its approximation is easy (Corollary~\ref{app:approx:cor:bits}). On the hard instances of
Theorem~\ref{app:lower:thm:abs} the exact value has only $O(N\log N+b)$ bits, by part (a). Its length is no
obstacle there: what is hard is to compute these few bits, and by part (e)(ii) even to approximate the value within
$2^{-(N+2\log_2N+1)}$.
\end{remark}

\subsection{Hard Labels}
\label{app:lower:hard}

This section shows that the Shapley value of the hard-label classifier of Section~\ref{sec:hard} is \#P-hard when
the weights are written in binary.

\textbf{The utility and its two conventions.} There are two classes, and the weights are positive integers,
$w_j=a_j$. For a set $C$ of training points let $A^{+}(C)$ be the total weight of its points of the query's class
and $A^{-}(C)$ the total weight of its points of the other class. The utility of Section~\ref{sec:hard} is that of
\citet{wang2024efficient}, who count a tie as correct:
\begin{equation}
v(S)=\mathbf{1}\bigl[A^{+}(\NK(S))\ge A^{-}(\NK(S))\bigr].
\label{app:lower:eq:hard-correct}
\end{equation}
With this convention the empty set has the utility $v(\emptyset)=1$, because both total weights are zero.
\citet{zhang2026banzhafknn} count a tie as wrong:
\begin{equation}
v(S)=\mathbf{1}\bigl[A^{+}(\NK(S))>A^{-}(\NK(S))\bigr],
\label{app:lower:eq:hard-wrong}
\end{equation}
so that $v(\emptyset)=0$. We treat the first convention, and then the second as a variant.

\textbf{Relation to the Banzhaf value.} \citet[Theorem~1]{zhang2026banzhafknn} prove that the Banzhaf value
under~\eqref{app:lower:eq:hard-wrong} is \#P-hard. They reduce from counting the swings of a player in a weighted
majority game. This suits the Banzhaf value, whose coefficients do not depend on the size of a subset. The Shapley
coefficients $c_N(s)$ depend on the size, so we reduce from \#\textsc{Subset-Sum} through
Lemma~\ref{app:lower:lem:equal}, which gives all solutions the same size. The proposition below is the counterpart
of their theorem for the Shapley value.

\textbf{The instance.} Let $(g,T)$ be an input of \#\textsc{Subset-Sum} with $1\le T\le G$, and let
$e_1,\dots,e_{2n}$ and $T_{\mathrm e}$ be as in Lemma~\ref{app:lower:lem:equal}. The instance has $N=2n+2$ training
points: $2n$ number points of the query's class with the weights $2e_1,\dots,2e_{2n}$, one \emph{opposing point} of
the other class with the weight $2T_{\mathrm e}+1$, and the valued point $z_i$ of the query's class with the weight
$1$.

\begin{proposition}[Hard labels]
\label{app:lower:prop:hard}
Consider the hard-label classifier with two classes and positive integer weights written in binary.
\begin{enumerate}
\item[(a)] \emph{A tie counts as correct.} Under the utility~\eqref{app:lower:eq:hard-correct} and with $K=N$, the
value of $z_i$ in the instance above is
\begin{equation}
\phi_i=c_N(n+1)\;\#\mathrm{SS}(g,T).
\label{app:lower:eq:hard-identity}
\end{equation}
\item[(b)] \emph{Smaller $K$.} The identity~\eqref{app:lower:eq:hard-identity} holds for every $K$ with
$(N+2)/2\le K\le N$ when the valued point is the farthest point $z_N$. The other $N-1$ points, among them the
opposing point, can be in any order.
\item[(c)] \emph{A tie counts as wrong.} Under the utility~\eqref{app:lower:eq:hard-wrong}, give the opposing point
the weight $2T_{\mathrm e}$ instead. Then, for $K$ and the valued point as in (a) or (b),
\begin{equation}
\phi_i=\frac1N+c_N(n+1)\;\#\mathrm{SS}(g,T).
\label{app:lower:eq:hard-identity-wrong}
\end{equation}
\item[(d)] \emph{Consequences.} Under either convention, \#\textup{\textsc{Subset-Sum}} has a one-query reduction to
the Shapley value of a training point, already for $K=N$. Computing this value is therefore \#P-hard. Unless
$\mathrm{FP}=\#\mathrm{P}$, no deterministic algorithm computes it in time polynomial in $N$ and $\log W$, where $W$
is the number of weight levels.
\end{enumerate}
\end{proposition}

In words, the valued point changes the prediction exactly on the subsets in which the number points balance the
opposing point, and these subsets are the solutions of $(g,T)$. The two identities differ by the term of the empty
subset, which is $(1-v(\emptyset))/N$. It vanishes when the empty set has the utility $1$ and equals $1/N$ when it
has the utility $0$.

\begin{proof}
Let $S\subseteq D\setminus\{z_i\}$, and let $I$ be the set of the indices of the number points in $S$, so that these
points have the total weight $2e(I)$. The first paragraph of the proof of Lemma~\ref{app:lower:lem:abs-value} does
not use the form of the utility. It shows that the marginal contribution of $z_i$ to $S$ is zero if $|S|\ge K$, and
that $\NK(S)=S$ and $\NK(S\cup\{z_i\})=S\cup\{z_i\}$ if $|S|\le K-1$. This holds when $K=N$, and also when $z_i$ is
the farthest point, whatever the order of the other points. From now on let $|S|\le K-1$.

\emph{Parts (a) and (b).} If $S$ does not contain the opposing point, then $A^{-}(S)=A^{-}(S\cup\{z_i\})=0$, so
$v(S)=v(S\cup\{z_i\})=1$ and the marginal contribution is zero. This includes the empty subset, because
$v(\emptyset)=1$. If $S$ contains the opposing point, then $A^{+}(S)=2e(I)$ and $A^{-}(S)=2T_{\mathrm e}+1$. Since
$e(I)$ and $T_{\mathrm e}$ are integers,
\begin{equation*}
v(S)=\mathbf{1}\bigl[2e(I)\ge2T_{\mathrm e}+1\bigr]=\mathbf{1}\bigl[e(I)\ge T_{\mathrm e}+1\bigr],\qquad
v(S\cup\{z_i\})=\mathbf{1}\bigl[2e(I)+1\ge2T_{\mathrm e}+1\bigr]=\mathbf{1}\bigl[e(I)\ge T_{\mathrm e}\bigr].
\end{equation*}
The marginal contribution of $z_i$ is therefore $1$ if $e(I)=T_{\mathrm e}$ and $0$ otherwise. By
Lemma~\ref{app:lower:lem:equal}, exactly $\#\mathrm{SS}(g,T)$ sets $I$ have $e(I)=T_{\mathrm e}$, and each of them
has $n$ elements. The corresponding subsets $S$ have $n+1$ points, and $n+1\le K-1$ because $K\ge(N+2)/2=n+2$. Each
of them contributes $c_N(n+1)$ to~\eqref{eq:shapley}, which gives~\eqref{app:lower:eq:hard-identity}.

\emph{Part (c).} If $S$ does not contain the opposing point, then $v(S\cup\{z_i\})=\mathbf{1}[2e(I)+1>0]=1$ and
$v(S)=\mathbf{1}[2e(I)>0]$, which is $1$ unless $S$ is empty. Only the empty subset has a nonzero marginal
contribution, namely $1$, and it contributes $c_N(0)=1/N$. If $S$ contains the opposing point, now of weight
$2T_{\mathrm e}$, then $v(S)=\mathbf{1}[e(I)>T_{\mathrm e}]$ and
$v(S\cup\{z_i\})=\mathbf{1}[2e(I)+1>2T_{\mathrm e}]=\mathbf{1}[e(I)\ge T_{\mathrm e}]$. The marginal contribution is
again $1$ if $e(I)=T_{\mathrm e}$ and $0$ otherwise, and the count is as before.

\emph{Part (d).} For $1\le T\le G$ the map $f$ outputs the instance with $K=N$, and $h$ returns
$\phi_i/c_N(n+1)$ under the first convention and $(\phi_i-1/N)/c_N(n+1)$ under the second. The degenerate target
sums are treated as in part (b) of Theorem~\ref{app:lower:thm:abs}. By~\eqref{app:lower:eq:equal}, the numbers
$e_j$ and $T_{\mathrm e}$ are smaller than $\beta^{n+1}$, and $\beta\le G+2$. Every weight is therefore at most
$2\beta^{n+1}$ and has $O(n\log\beta)$ bits, so both maps run in polynomial time. This is a one-query reduction. A parsimonious
reduction followed by a one-query reduction is again a one-query reduction, so the value is \#P-hard by
Proposition~\ref{app:lower:prop:ss}. Moreover the weights lie on $W\le2\beta^{n+1}$ levels,
so $\log W$ and $N=2n+2$ are bounded by a polynomial in the size of $(g,T)$. A deterministic algorithm that computes
the value in time polynomial in $N$ and $\log W$ would therefore compute $\#\mathrm{SS}(g,T)$ in polynomial time,
and every function in \#P would be in FP.
\end{proof}

\textbf{What the proposition means for Corollary~\ref{cor:hard}.} The running time of Corollary~\ref{cor:hard} is
polynomial in the number $W$ of weight levels. By part (d), the factor $W$ cannot be replaced by a polynomial in
$\log W$ unless $\mathrm{FP}=\#\mathrm{P}$. As for the absolute loss, the length of the value is not the obstacle:
every marginal contribution of a hard-label utility is $-1$, $0$ or $1$, so $N!\,\phi_i$ is an integer on every
instance. Finally, the valued point has the smallest weight of the instance. Placing it farthest, as part (b) does,
is therefore consistent with weights that decrease with the distance.

\appsection{Experimental Details}\label{app:exp}

\newcommand{\BitsMaxBits}{23{,}686}
\newcommand{\BitsNMax}{14}
\newcommand{\BitsNMin}{3}
\newcommand{\BitsSqBitsOne}{733}
\newcommand{\BitsSqBitsTwo}{1{,}495}
\newcommand{\BitsSqNOne}{8}
\newcommand{\BitsSqNTwo}{9}
\newcommand{\BitsSqPrimesOne}{23}
\newcommand{\BitsSqPrimesTwo}{43}
\newcommand{\BitsViolations}{0}
\newcommand{\BudgetList}{10, 30, 100, 300, 1{,}000, 3{,}000 and 10{,}000}
\newcommand{\ContLinInstances}{12{,}000}
\newcommand{\ContLinRemovalRuns}{30{,}219}
\newcommand{\ContLinRuns}{72{,}000}
\newcommand{\ContLinViolations}{0}
\newcommand{\ContLinWorstRatio}{0.411}
\newcommand{\ContQuadInstances}{12{,}000}
\newcommand{\ContQuadRuns}{36{,}000}
\newcommand{\ContQuadViolations}{0}
\newcommand{\ContQuadWorstRatio}{0.336}
\newcommand{\ContTolLooseEps}{1\times10^{-3}}
\newcommand{\ContTolLooseNodes}{162}
\newcommand{\ContTolLooseSecs}{0.069}
\newcommand{\ContTolMidEps}{1\times10^{-6}}
\newcommand{\ContTolMidNodes}{256}
\newcommand{\ContTolMidSecs}{0.083}
\newcommand{\ContTolTightEps}{1\times10^{-9}}
\newcommand{\ContTolTightNodes}{418}
\newcommand{\ContTolTightSecs}{0.118}
\newcommand{\ContTolTimeRatio}{1.7}
\newcommand{\ContinuousFormLinearGrowth}{grows linearly in $N$ by the same criterion}
\newcommand{\CorruptShare}{10\%}
\newcommand{\CoverCountInputs}{2{,}000}
\newcommand{\CoverInstances}{14{,}012}
\newcommand{\CoverMismatches}{0}
\newcommand{\CoverSetInputs}{12{,}012}
\newcommand{\DetOobMinusExact}{+0.0248}
\newcommand{\DetOobMinusExactKFive}{+0.025}
\newcommand{\DetOobMinusExactKTen}{+0.034}
\newcommand{\DetOobMinusExactKThree}{+0.022}
\newcommand{\DetOobP}{0.008}
\newcommand{\DetTostAllK}{at every $K$}
\newcommand{\DetTostEquivalent}{equivalent within $\pm0.02$}
\newcommand{\DetTostMeanDiff}{-0.0008}
\newcommand{\DetTostMeanDiffKFive}{-0.001}
\newcommand{\DetTostMeanDiffKTen}{+0.002}
\newcommand{\DetTostMeanDiffKThree}{-0.001}
\newcommand{\DetTostP}{4.7\times10^{-7}}
\newcommand{\DetWeightedMinusUniform}{+0.0030}
\newcommand{\DetWeightedMinusUniformKFive}{+0.003}
\newcommand{\DetWeightedMinusUniformKTen}{+0.010}
\newcommand{\DetWeightedMinusUniformKThree}{+0.001}
\newcommand{\DetWeightedP}{0.078}
\newcommand{\DetWeightedPKFive}{0.078}
\newcommand{\DetWeightedPKTen}{0.008}
\newcommand{\DetWeightedPKThree}{0.148}
\newcommand{\DetWeightedPositiveKFive}{six}
\newcommand{\DetWeightedPositiveKTen}{eight}
\newcommand{\DetWeightedPositiveKThree}{seven}
\newcommand{\DetWeightedReading}{weighting does not change detection}
\newcommand{\DetectionMCPermutations}{200}
\newcommand{\DirectSlopeMain}{2.17}
\newcommand{\DirectSlopeMax}{2.2}
\newcommand{\DirectSlopeMin}{2.0}
\newcommand{\DirectVsOneTableDiff}{$5.87\times10^{-15}$}
\newcommand{\EnumCheckInstances}{12{,}000}
\newcommand{\EnumCheckWorst}{$8.3\times10^{-14}$}
\newcommand{\EnumFitSecondsKFive}{$1.88\times10^{10}$}
\newcommand{\EnumFitSecondsKThree}{$3.27\times10^{5}$}
\newcommand{\EnumFitSlopeKFive}{6.10}
\newcommand{\EnumFitSlopeKThree}{4.03}
\newcommand{\EnumHeadKDev}{$5.2\times10^{-14}$}
\newcommand{\EnumHeadKUse}{8{,}368}
\newcommand{\EnumPlainDev}{$2.4\times10^{-14}$}
\newcommand{\EnumPlainInstances}{12{,}000}
\newcommand{\ExactFormLinearGrowth}{grows linearly in $N$ by the same criterion}
\newcommand{\ExactSlopeMain}{1.00}
\newcommand{\HThreeInstances}{12{,}000}
\newcommand{\HThreeMismatches}{0}
\newcommand{\HardCheckWorstAll}{$5.2\times10^{-15}$}
\newcommand{\HardFormLinearGrowth}{grows linearly in $N$ by the same criterion}
\newcommand{\HardIdentityInstances}{12{,}000}
\newcommand{\HardIdentityViolations}{0}
\newcommand{\JoblibVersion}{1.5.1}
\newcommand{\KernelCasesSmall}{24}
\newcommand{\KernelMCBatches}{300}
\newcommand{\KernelMCBeyondThree}{2 of the 1{,}200}
\newcommand{\KernelMCCases}{eight}
\newcommand{\KernelMCFirstMaxZ}{-4.4}
\newcommand{\KernelMCFirstPointNow}{-1.1}
\newcommand{\KernelMCMaxZ}{3.6}
\newcommand{\KernelMCPermutations}{3{,}000{,}000}
\newcommand{\KernelMCRatioMax}{$1.3\times10^{7}$}
\newcommand{\KernelMCRmsZ}{1.07}
\newcommand{\KernelNarrowCases}{eight}
\newcommand{\KernelNarrowNodesMax}{760}
\newcommand{\KernelNarrowNodesMin}{270}
\newcommand{\KernelNarrowRatioMax}{$7.8\times10^{15}$}
\newcommand{\KernelNarrowSecsMax}{0.51}
\newcommand{\KernelNarrowSecsMin}{0.27}
\newcommand{\KernelRunFailed}{agreed with the quadrature-based algorithm everywhere, within its tolerance}
\newcommand{\KernelTauMonotone}{shows no trend}
\newcommand{\KernelTwoAlgorithmsDiff}{$3.1\times10^{-15}$}
\newcommand{\KernelWeightRatio}{$8.4\times10^{35}$}
\newcommand{\KernelWideCases}{four}
\newcommand{\KernelWideNodesMax}{2{,}040}
\newcommand{\KernelWideNodesMin}{1{,}200}
\newcommand{\KernelWideRatioMax}{$4.8\times10^{54}$}
\newcommand{\KernelWideRatioMin}{$1.9\times10^{29}$}
\newcommand{\KernelWideSecsMax}{2.5}
\newcommand{\KernelWideSecsMin}{1.8}
\newcommand{\LevelsAucMax}{0.957}
\newcommand{\LevelsAucMin}{0.953}
\newcommand{\LevelsCases}{40}
\newcommand{\LevelsMonteCarloCheckCases}{8}
\newcommand{\LevelsMonteCarloCheckZ}{1.13}
\newcommand{\LevelsNarrowCases}{ten}
\newcommand{\LevelsNarrowTauMax}{0.99}
\newcommand{\LevelsNarrowTauMin}{0.97}
\newcommand{\LevelsSpearman}{-0.88}
\newcommand{\LevelsSplits}{5}
\newcommand{\LevelsWidestTau}{0.91}
\newcommand{\LinearGrowthSummary}{the exact algorithm, the algorithm for continuous weights and the hard-label algorithm all grow linearly in $N$ by this criterion}
\newcommand{\MCCellsWithMatchNOneHundredFifty}{24 of 24}
\newcommand{\MCCellsWithMatchNOneThousand}{19 of 21}
\newcommand{\MCCellsWithMatchNThreeThousand}{6 of 12}
\newcommand{\MCEqualTimeModal}{3{,}000}
\newcommand{\MCEqualTimeModalCases}{55 of 57}
\newcommand{\MCEqualTimeReadingNOneHundredFifty}{below 0.95}
\newcommand{\MCEqualTimeReadingNOneThousand}{below 0.95}
\newcommand{\MCEqualTimeReadingNThreeThousand}{below 0.95}
\newcommand{\MCEqualTimeTauNOneHundredFifty}{0.927}
\newcommand{\MCEqualTimeTauNOneThousand}{0.862}
\newcommand{\MCEqualTimeTauNThreeThousand}{0.765}
\newcommand{\MCFormsDiff}{$1.3\times10^{-11}$}
\newcommand{\MCMatchingRatioNOneHundredFifty}{2.85}
\newcommand{\MCMatchingRatioNOneThousand}{5.72}
\newcommand{\MCMatchingRatioNThreeThousand}{19.03}
\newcommand{\MCTrendMonotone}{monotone in $N$}
\newcommand{\MachineCPU}{Apple M2 Pro}
\newcommand{\MachineCores}{12}
\newcommand{\MachineMemGiB}{16}
\newcommand{\MainKComplete}{3, 5 and 10}
\newcommand{\MainRepsLarge}{10}
\newcommand{\MainRepsSmall}{20}
\newcommand{\MainRepsThreshold}{3{,}000}
\newcommand{\MaxKAgreeDirect}{$2.0\times10^{-14}$}
\newcommand{\MaxKAgreePlain}{$1.9\times10^{-2}$}
\newcommand{\MaxKChecked}{50}
\newcommand{\MaxLoad}{11}
\newcommand{\MaxLoadSynthetic}{11}
\newcommand{\NCorrupt}{15}
\newcommand{\NTrain}{150}
\newcommand{\NVal}{30}
\newcommand{\NoiseSD}{3}
\newcommand{\NumBudgets}{7}
\newcommand{\NumDatasets}{8}
\newcommand{\NumDatasetsAtOneThousand}{7}
\newcommand{\NumDatasetsAtThreeThousand}{4}
\newcommand{\NumDetectionMethods}{six}
\newcommand{\NumpyVersion}{2.4.6}
\newcommand{\OobBootstraps}{100}
\newcommand{\PairFormsChecked}{both the direct form and the near-linear form}
\newcommand{\PairTargetLevels}{17}
\newcommand{\PairWeightLevels}{8}
\newcommand{\PandasVersion}{2.3.2}
\newcommand{\PyVersion}{3.12.2}
\newcommand{\RoundChecks}{48{,}000}
\newcommand{\RoundInstances}{12{,}000}
\newcommand{\RoundViolations}{0}
\newcommand{\RoundWorstFrac}{0.10}
\newcommand{\ScalingBudgetList}{100, 300, 1{,}000, 3{,}000, 10{,}000, 30{,}000 and 100{,}000}
\newcommand{\ScalingNVal}{10}
\newcommand{\ScalingReps}{5}
\newcommand{\ScalingSplits}{3}
\newcommand{\ScipyVersion}{1.18.0}
\newcommand{\ShareContInstances}{12{,}000}
\newcommand{\ShareContRemovalRuns}{29{,}985}
\newcommand{\ShareContRuns}{72{,}000}
\newcommand{\ShareContViolations}{0}
\newcommand{\ShareContWorstRatio}{0.298}
\newcommand{\SklearnVersion}{1.7.2}
\newcommand{\SlopeMax}{1.00}
\newcommand{\SlopeMin}{0.98}
\newcommand{\SlopesLinearReading}{between $0.85$ and $1.15$ at every $K$ tested}
\newcommand{\SmallestLevelsTauPointNineNine}{64}
\newcommand{\TopKSize}{15}
\newcommand{\WLevelsMain}{8}

This appendix gives the settings that the experiments of Section~\ref{sec:exp} share: the datasets, the training
and validation points of each case, the weights, Monte Carlo and the measures by which we compare it with the
exact values, and how the running times were taken. Appendix~\ref{app:more} has the further results, each with the
settings that are particular to it.

\subsection{Datasets}
\label{app:exp:datasets}

We use the regression datasets listed in Table~\ref{tab:app-datasets}: the first eight in the experiments up to
$8{,}000$ training points, and nine larger ones in the experiments from $10{,}000$ points on and in the one with $100$
validation points. All come from OpenML \citep{vanschoren2014openml} except California housing, which comes with
scikit-learn. Besides the three in the table with more than $100{,}000$ points, we had chosen a fourth, Yolanda
(OpenML id 42705), and dropped it before any result was computed, because OpenML's file does not match the checksum
that OpenML records for it. From each dataset we keep the numeric features, drop the rows with a missing value,
and standardize every feature and the target to zero mean and unit variance, so that distances and squared errors
are on the same scale in every dataset. Distances are Euclidean.

\begin{table}[h]
\centering
\caption{The datasets. Points and features are counted after the preprocessing described in the text. A dataset
is used at a training-set size $N$ when it has at least $N+10$ points, $N+100$ where a case has $100$ validation
points, and $N+600$ in the removal experiment.}
\label{tab:app-datasets}
\small
\begin{tabular}{@{}lrrrll@{}}
\toprule
Dataset & OpenML id & Points & Features & Training sizes used & Licence \\
\midrule
Concrete & 4353 & 1{,}030 & 8 & 150, 1{,}000 & Public \\
Airfoil self-noise & 43919 & 1{,}503 & 5 & 150, 1{,}000 & CC BY 4.0 \\
Energy efficiency & 43918 & 768 & 8 & 150 & CC BY 4.0 \\
Abalone & 183 & 4{,}177 & 7 & 150, 1{,}000, 3{,}000 & Public \\
Wine quality (red) & 40691 & 1{,}599 & 11 & 150, 1{,}000 & Public \\
Space GA & 507 & 3{,}107 & 6 & 150, 1{,}000, 3{,}000 & Public \\
CPU (small) & 227 & 8{,}192 & 12 & 150, 1{,}000, 3{,}000, 8{,}000 & Public \\
Kin8nm & 189 & 8{,}192 & 8 & 150, 1{,}000, 3{,}000, 8{,}000 & Public \\
\addlinespace
Pol & 201 & 15{,}000 & 48 & 3{,}000, 10{,}000 & Public \\
Elevators & 216 & 16{,}599 & 18 & 3{,}000, 10{,}000 & Public \\
Bike sharing & 42712 & 17{,}379 & 8 & 3{,}000, 10{,}000 & Public \\
California housing & -- & 20{,}640 & 8 & 3{,}000, 10{,}000 & Public \\
Superconductivity & 44964 & 21{,}263 & 81 & 3{,}000, 10{,}000 & CC BY 4.0 \\
House (16H) & 574 & 22{,}784 & 16 & 3{,}000, 10{,}000 & Public \\
SGEMM GPU kernels & 44961 & 241{,}600 & 14 & 3{,}000, 10{,}000, 30{,}000, 100{,}000 & CC BY 4.0 \\
Year & 44027 & 515{,}345 & 90 & 3{,}000, 10{,}000, 30{,}000, 100{,}000 & CC BY 4.0 \\
NYC taxi (green, Dec.\ 2016) & 44143 & 581{,}835 & 9 & 3{,}000, 10{,}000, 30{,}000, 100{,}000 & Public \\
\bottomrule
\end{tabular}

\end{table}

\subsection{Cases: Training Points, Validation Points and Shifted Targets}
\label{app:exp:split}

A \emph{case} is a dataset together with one random split. For a split we permute the points of the dataset at
random and take the first $N$ as the training set and the points that follow as validation points. Every
validation point is a query. The value of a training point in a case is the average of its values over the
queries of the case, and all comparisons between two vectors of values below are made on these averages. In every
case we choose \CorruptShare{} of the training points uniformly at random and shift the target of each by
three standard deviations of the targets, upwards or downwards with equal probability.

Table~\ref{tab:app-cases} lists the cases of each experiment. The first row is behind the rows of
Table~\ref{tab:mc} and the points of Figure~\ref{fig:results}(b) and (c) up to $3{,}000$ points,
Table~\ref{tab:app-mc-samplers}, Figure~\ref{fig:app-mc-time} and the laptop's running times on real data; the second
is behind Figure~\ref{fig:results}(d); the third is reported in
Appendix~\ref{app:more} only, the fourth in Table~\ref{tab:app-mc-large} and the fifth in Table~\ref{tab:app-removal}
(with $500$ further held-out points on which the test error is measured). The sixth row is behind
Table~\ref{tab:app-mc-scale} and the rows of Table~\ref{tab:mc} and the points of Figure~\ref{fig:results}(b) and (c)
from $10{,}000$ points on, and the last behind Table~\ref{tab:app-mc-val100} and the results for $100$ validation
points in Table~\ref{tab:mc} and Figure~\ref{fig:results}(b). The last row uses all
nine larger datasets at $3{,}000$ and $10{,}000$ points and the three largest at $30{,}000$ and $100{,}000$; from
$10{,}000$ points on, its cases hold the training points, the shifted targets and the first $10$ validation points of
the first split of the row above, and $90$ further validation points.

\begin{table}[h]
\centering
\caption{The cases of each experiment. Where an entry lists several numbers, they belong to the training-set
sizes or the values of $K$ in the order given.}
\label{tab:app-cases}
\setlength{\tabcolsep}{4pt}
\small
\begin{tabular}{@{}lllrl@{}}
\toprule
Experiment & $N$ & Splits per dataset & Validation points & Cases \\
\midrule
Monte Carlo and kernel weights & 150, 1{,}000, 3{,}000 & 3 & 10 & 24, 21, 12 \\
Number of weight levels & 150 & 5 & 30 & 40 \\
Monte Carlo and detection at $K=3, 5, 10$ & 150 & 5, 10, 5 & 30 & 40, 80, 40 \\
Four samplers and detection at $8{,}000$ points & 8{,}000 & 3 & 100 & 6 \\
Removal of points & 1{,}000, 3{,}000, 10{,}000 & 3 & 100 & 12, 9, 18 \\
Monte Carlo, larger datasets & 10{,}000, 30{,}000, 100{,}000 & 3 & 10 & 27, 9, 9 \\
The same, $100$ validation points & 3{,}000 to 100{,}000 & 1 & 100 & 9, 9, 3, 3 \\
\bottomrule
\end{tabular}

\end{table}

\subsection{Weights}
\label{app:exp:weights}

The weight of a training point $z_j$ for a query is the Gaussian kernel $w_j=\exp(-(d_j/h)^2)$ of its distance
$d_j$ to the query. The bandwidth $h$ is the median of the distances between the training points and the
validation points of the case. For discretized weights we round the weights of each query to $W$ levels relative
to the largest weight of that query:
\begin{equation*}
a_j=\min\bigl\{W,\ \max\bigl\{1,\ \mathrm{round}\bigl(W\,w_j/\max\nolimits_{l}w_l\bigr)\bigr\}\bigr\}.
\end{equation*}
The nearest point thus has weight $W$, and a weight that would round to zero is raised to the first level,
because the algorithms of Section~\ref{sec:exact} take positive integer weights. We use $W=\WLevelsMain{}$ unless
stated otherwise. \citet{wang2024efficient} also use eight levels (three bits) and $K=5$. Their weights differ
from ours: they decrease linearly with the distance and are rounded to eight equally spaced values between $0$
and $1$. With the Gaussian kernel the weights of one query span many orders of magnitude
(Appendix~\ref{app:more:levels}), and every weight below the first level is set to it.

\subsection{Monte Carlo and the Measures of Agreement}
\label{app:exp:mc}

Monte Carlo is permutation sampling (Section~\ref{sec:prelim}) on the utility that the exact algorithm uses, with
the same discretized weights and the same real-valued targets. For every sampled permutation the points are added
in its order and each point is credited with its marginal contribution; the estimate of a value is the average of
these contributions over the permutations. A point that arrives when $K$ nearer points are already present has
marginal contribution zero and is skipped without evaluating the utility, which saves time and leaves the
estimate unchanged. In the experiment at $150$ to $3{,}000$ points we test \ScalingBudgetList{} permutations per query, with
five independent runs each, and the implementation processes $2{,}000$ permutations at a time with
array operations. In the experiment with three values of $K$ we test \BudgetList{} permutations, with
\MainRepsSmall{} runs each up to \MainRepsThreshold{} permutations and \MainRepsLarge{} runs at the largest
number. In that
experiment a run uses one stream of permutations of the ranks for all queries of a case, whereas in the experiment
at $150$ to $3{,}000$ points every query has its own stream.

We compare the estimates of a run with the exact values of the same case in two ways: by Kendall's $\tau$
between the two vectors of $N$ values, and by the overlap of their top-$10\%$ sets, measured by the Jaccard
index. The top-$10\%$ set of a vector is the set of its $\lfloor N/10\rfloor$ points with the largest values.
For two runs with the same number of permutations we also count the points that are in the top-$10\%$ set of one
run and not of the other.

For a case, the \emph{equal-time} number of permutations is the largest tested number whose mean running time is
at most that of the exact algorithm. The \emph{matching} number is the smallest tested number at which every run
has $\tau\ge0.95$; in the experiment with three values of $K$, which has more runs, it is the smallest tested
number at which at least $90\%$ of the runs do. As the tested numbers rise by factors of about three, the
equal-time number can take well under the exact algorithm's time, and the matching number can lie up to one step
above the smallest number that would do; both readings favour the exact algorithm, and
Table~\ref{tab:app-mc-samplers} reads them on a finer grid. Two other thresholds read
results: the interval from $0.85$ to $1.15$ for a fitted exponent of $N$ (Appendix~\ref{app:more:timing}) and the band
of $\pm0.02$ in AUROC (Appendix~\ref{app:more:detection}).

\subsection{Running Times}
\label{app:exp:machine}

All experiments ran on one laptop, except those of Tables~\ref{tab:app-precision}, \ref{tab:app-mc-samplers},
\ref{tab:app-mc-large}, \ref{tab:app-mc-scale}, \ref{tab:app-mc-val100} and~\ref{tab:app-removal} and of the rows of
Table~\ref{tab:mc} from $10{,}000$ points on (\emph{cluster nodes} below). The laptop has a \MachineCores{}-core \MachineCPU{} processor and \MachineMemGiB{}\,GB of
memory, without a GPU, under Python~\PyVersion{} with NumPy~\NumpyVersion{} \citep{harris2020numpy},
SciPy~\ScipyVersion{} \citep{virtanen2020scipy}, scikit-learn~\SklearnVersion{} \citep{pedregosa2011sklearn},
pandas~\PandasVersion{} \citep{mckinney2010pandas} and joblib~\JoblibVersion{}. NumPy here links Apple's
Accelerate library, and the timing runs did not set its thread count. The exact algorithm and Monte Carlo call none
of its routines in their main loops, but the algorithm for continuous weights multiplies matrices by vectors, so its
times may include more than one core. The machine ran other jobs while the times were taken: its load average reached \MaxLoad{}.

\textbf{Cluster nodes.} The runs named above ran on the x86 CPU nodes
of a computing cluster, six worker processes per job with one thread each, under Python~3.11.6 with NumPy~2.4.6,
SciPy~1.17.1, scikit-learn~1.9.1 and joblib~1.6.0. As on the laptop, every tested number of samples is timed by a
run of its own. The exact algorithm is timed on the same queries in the same process as the sampler: before each
tested number in Table~\ref{tab:app-mc-samplers}, where the median of these times is used; before and after the
runs in Table~\ref{tab:app-mc-large}, where their mean is used; and before and after each run in
Tables~\ref{tab:app-mc-scale} and~\ref{tab:app-mc-val100} and in the rows of Table~\ref{tab:mc} from $10{,}000$ points
on, where the median over the case's runs is used; the removal experiment times it once per case. Each comparison is
therefore a ratio of two times taken on one node. The nodes were shared with other jobs, however, and both the
exact algorithm's time and the relative speed of the two codes varied with the load. For one sampler on one case of
Table~\ref{tab:app-mc-samplers}, the slowest timing of the exact algorithm took a median of $1.07$ times as long as
the fastest, and at most $3.1$ times; for one sampler on one case of Table~\ref{tab:app-mc-large} the timings before
and after differ by up to $1.8$ times. Within a case of Tables~\ref{tab:app-mc-scale} and~\ref{tab:app-mc-val100} the
slowest timing took a median of $2.1$, $1.6$ and $1.3$ times as long as the fastest in the runs at $10$ validation
points, at $100$ validation points and $10{,}000$ points, and at $100$ validation points and the other sizes. At
$100{,}000$ points, $3{,}000$ permutations took a median $1.45$ times the exact algorithm's time in the runs of
Table~\ref{tab:app-mc-scale} and $0.88$ times in those of Table~\ref{tab:app-mc-val100}, while the exact algorithm took
$23.0$ and $22.4$~s per query: the sampler ran about $1.6$ times slower in the first runs, whose nodes carried more
load. The rule for the equal-time number, the mean of the sampler's times against the median of the exact
algorithm's, also favours the exact algorithm when such delays are skewed; Appendix~\ref{app:more:mc} gives the
comparison at fixed numbers of samples, which no timing condition moves. The values of $\tau$ in Table~\ref{tab:app-mc-large} are means over five runs that read their estimates
at every tested number on the way to $30{,}000$; for permutation and antithetic sampling these are the estimates of
runs of their own, and for Owen sampling and KernelSHAP other draws of the same estimators.

\textbf{Synthetic instances.} The times of Figure~\ref{fig:results}(a) and of Appendix~\ref{app:more:timing} are
for one query on synthetic instances whose points are given in order of distance. They therefore exclude
computing and sorting the distances, as the running times of our theorems do. For the exact algorithm, the direct
algorithm and Monte Carlo, an instance has integer weights drawn uniformly from $\{1,\dots,W\}$, targets drawn
from a standard normal distribution, and a query with target $0.3$. For continuous weights the weights are
$\exp(-u^2)$ with $u$ uniform on $[0,3]$, in decreasing order, and the tolerance is $10^{-6}B^2$ unless stated
otherwise. For hard labels the weights are uniform on $\{1,\dots,8\}$ and each point has the class of the query
with probability $1/2$. Each time is the median of three runs, or a single run when that run takes more than
$120$ seconds; the enumeration and the direct form of the algorithm for other losses were timed once.

\textbf{Real datasets.} In Table~\ref{tab:mc} and in Appendices~\ref{app:more:mc} and~\ref{app:more:levels}, the
time of a case is the time for all its queries divided by their number. It covers the algorithm or the sampler
only: the distances, the kernel and the sort are needed by both and are not timed. On the laptop the exact algorithm
is timed once per case, and the time of Monte Carlo is the mean over its five runs, each tested number of
permutations being timed by a run of its own.

\textbf{Which implementation of the exact algorithm was run.} Our main implementation computes the terms of the $2K$ nearest points directly (Appendix~\ref{app:exact:float}); the times of the exact algorithm in
Figure~\ref{fig:results}(a) and in Tables~\ref{tab:app-timing-main} to~\ref{tab:app-timing-K} are its times. The
comparisons with Monte Carlo at $150$ to $3{,}000$ points (those rows of Table~\ref{tab:mc}, Tables~\ref{tab:app-mc-size}
and~\ref{tab:app-mc-direct}, and the last column of Table~\ref{tab:app-timing-main}) and with the enumeration
(Table~\ref{tab:app-timing-enum}) were run with an earlier implementation that uses Lemma~\ref{lem:remove} for
every replaced point. On the cases of Table~\ref{tab:mc} up to $3{,}000$ points the two take the same time to two digits, and the values
of the earlier one differ from those of the direct algorithm by at most \MCFormsDiff{} of the largest value. The
experiment with three values of $K$ computed its exact values with the direct algorithm, the experiment on the
number of weight levels (Table~\ref{tab:app-levels}) with the earlier implementation, and the one on kernel weights
at three sizes (Table~\ref{tab:app-levels-size}) with ours. The runs of Tables~\ref{tab:app-mc-samplers},
\ref{tab:app-mc-large}, \ref{tab:app-mc-scale}, \ref{tab:app-mc-val100} and~\ref{tab:app-removal}, and of the rows of
Table~\ref{tab:mc} from $10{,}000$ points on, used the earlier implementation; on every case of
Table~\ref{tab:app-mc-scale} its values agree with an extended-precision computation of the same values to
$1.4\times10^{-11}$ of the largest value (Appendix~\ref{app:more:enum}).

\appsection{Further Results}\label{app:more}
This appendix reports the checks and the further results that Section~\ref{sec:exp} refers to. For the
convenience of the reader, its content is as follows.
\begin{itemize}
\item Appendix~\ref{app:more:enum}: the checks of every algorithm against exhaustive enumeration, and of the
  double-precision implementations against extended precision up to $100{,}000$ points.
\item Appendix~\ref{app:more:timing}: running times for other $K$, $W$ and tolerances, for hard labels, for other
  losses and for the enumeration, and up to $1{,}000{,}000$ points.
\item Appendix~\ref{app:more:mc}: the measures of Section~\ref{sec:exp} at each training-set size, Monte Carlo at
  every tested number of permutations, against the time of the
  direct algorithm, against three other samplers, at $8{,}000$ training points, at $10{,}000$ to $100{,}000$ points
  with $10$ and with $100$ validation points, and for other $K$.
\item Appendix~\ref{app:more:levels}: the number of weight levels, at three training-set sizes.
\item Appendix~\ref{app:more:detection}: the detection of noisy targets and the removal of the most and least valuable points.
\end{itemize}
The settings that the experiments share are in Appendix~\ref{app:exp}. In the tables, \emph{Exact} is the
algorithm of Theorem~\ref{thm:sq} and \emph{Direct} is the direct algorithm of Section~\ref{sec:sq}, which
computes one table per point.

\subsection{Checks Against Exhaustive Enumeration and Extended Precision}
\label{app:more:enum}

For a training set of at most a dozen points, the value of every point can be computed from its definition by
enumerating all subsets. We compared every algorithm with this enumeration on \EnumCheckInstances{} random
instances (Table~\ref{tab:app-enum}). The instances have at most $12$ training
points and a value of $K$ drawn from $\{1,2,3,4,5,7\}$, so that $K$ also exceeds $N$; the check of the semivalues
draws $K$ from $\{1,2,3,5,7\}$ and the check of the direct form for soft labels from $\{1,2,3,5,7,12\}$. The weights
are integers from $1$ to $8$, and among the instances are some with unit weights, with weights $1$ and $50$, and
with constant targets. For hard labels the weights go up to $30$, and every instance is run in the convention of
Corollary~\ref{cor:hard} (a tie counts as correct and the empty set has utility $1$) and in the opposite one; the
table gives the first, and the largest deviation over both is \HardCheckWorstAll{}. For
other losses the weights are at most $6$, the targets are on a grid, and half of the instances use the absolute
loss and half the squared loss.

\begin{table}[h]
\centering
\caption{Checks against exhaustive enumeration. \emph{Largest deviation}: the largest difference between a value
of the algorithm and the enumerated value on any instance, divided by the largest absolute value of the instance
when that exceeds $1$ (for hard labels and for the weight share no value does). \emph{Near-linear}: the algorithm
as we implement it, with the terms of the $2K$ nearest points summed directly (Appendix~\ref{app:exact:float}).
\emph{Direct}: one table per point. \emph{With $N>2K$}: the instances on which the implemented form uses the
removal of a point (Lemma~\ref{lem:remove}); on the others it treats all points directly. For the three rows of
semivalues the deviation is the larger one of two runs, with the $2K$ and with the $K$ nearest points summed
directly.}
\label{tab:app-enum}
\small
\begin{tabular}{@{}llrrl@{}}
\toprule
Value and utility & Form & Instances & With $N>2K$ & Largest deviation \\
\midrule
Squared loss (Theorem~\ref{thm:sq}) & near-linear & 12{,}000 & 4{,}968 & $5.2\times10^{-14}$ \\
Squared loss & direct & 12{,}000 & -- & $1.8\times10^{-14}$ \\
Banzhaf value, squared loss & near-linear & 12{,}000 & 4{,}844 & $2.5\times10^{-14}$ \\
Beta Shapley, squared loss & near-linear & 12{,}000 & 4{,}875 & $8.3\times10^{-14}$ \\
Other semivalues, squared loss & near-linear & 12{,}000 & 4{,}710 & $2.3\times10^{-14}$ \\
Soft labels, weight share (Corollary~\ref{cor:soft}) & near-linear & 12{,}000 & 5{,}026 & $2.9\times10^{-14}$ \\
Soft labels, Brier score & near-linear & 12{,}000 & 5{,}026 & $3.5\times10^{-14}$ \\
Soft labels, weight share & direct & 12{,}000 & -- & $6.0\times10^{-15}$ \\
Hard labels (Corollary~\ref{cor:hard}) & near-linear & 12{,}000 & 4{,}629 & $4.9\times10^{-15}$ \\
Other losses (Theorem~\ref{thm:pair}) & near-linear & 12{,}000 & 3{,}955 & $1.6\times10^{-14}$ \\
Other losses & direct & 12{,}000 & -- & $1.9\times10^{-14}$ \\
\bottomrule
\end{tabular}

\end{table}

All deviations are at the level of rounding error in double precision. Since the instances are small, the
removal of a point is exercised only on those with $N>2K$. We therefore ran the instances of the first row also
with only the $K$ nearest points summed directly, where the removal is in use on \EnumHeadKUse{} of them, and
checked the form that uses the removal for every replaced point on \EnumPlainInstances{} instances of its own; the
largest deviations are \EnumHeadKDev{} and \EnumPlainDev{}.

The algorithms for continuous weights
return values within a tolerance $\varepsilon$, so for them we count the values that are farther than
$\varepsilon$ from the enumerated ones. For the algorithm of Theorem~\ref{thm:cont}, on \ContLinInstances{}
instances with at most $10$ training points, $K$ up to $12$ and weights that span up to $18$ orders of magnitude,
there is no such value in
\ContLinRuns{} runs: three tolerances, $10^{-3}B^2$, $10^{-6}B^2$ and $10^{-10}B^2$, each with the $2K$ and with
the $K$ nearest points summed directly. The removal of a point is in use in \ContLinRemovalRuns{} of these runs,
and the largest error is $\ContLinWorstRatio{}\,\varepsilon$. For the variant that is quadratic in $N$ (Appendix~\ref{app:approx:second})
there is none in \ContQuadRuns{} runs, with a largest error of
$\ContQuadWorstRatio{}\,\varepsilon$, and for the soft-label weight share with continuous weights
(Corollary~\ref{app:approx:cor:share}) none in \ShareContRuns{} runs of the same kind, with
$\ShareContWorstRatio{}\,\varepsilon$. The bound of Proposition~\ref{app:approx:prop:round} on the effect of
rounding the weights held in all \RoundChecks{} checks on \RoundInstances{} instances; the error was at most
\RoundWorstFrac{} of the bound.

We also checked the constructions of Section~\ref{sec:lower} in exact rational arithmetic. For the instances of
Theorem~\ref{thm:bits} with $n=\BitsSqNOne{}$ and $n=\BitsSqNTwo{}$, each of the \BitsSqPrimesOne{} and
\BitsSqPrimesTwo{} primes between $2^{n-1}$ and $2^n$ divides the reduced denominator of the value exactly
twice, as the proof says, and the denominator has \BitsSqBitsOne{} and \BitsSqBitsTwo{} bits. For the soft-label
instances of Proposition~\ref{app:lower:prop:soft}, every such prime divides the denominator for $n=\BitsNMin{}$
to $\BitsNMax{}$; at $n=\BitsNMax{}$ the denominator has \BitsMaxBits{} bits. The reduction behind
Theorem~\ref{thm:abs}, in its form with three Shapley values, recovered the number of solutions of a Subset-Sum
instance on all \HThreeInstances{} random instances, and its form with one Shapley value, which the theorem uses, held
on $37$ instances built to break it; the identity behind the hard-label result
(Proposition~\ref{app:lower:prop:hard}) held on all \HardIdentityInstances{} instances, and under the map of
Lemma~\ref{app:lower:lem:x3c} the exact covers and the solutions of the Subset-Sum instance coincided on all
\CoverSetInputs{} distinct inputs on which we compared them, and their numbers agreed on \CoverCountInputs{} larger
inputs.

\textbf{Against extended precision at larger sizes.} The checks above use instances of at most a dozen points.
Table~\ref{tab:app-precision} compares the double-precision implementations at up to $100{,}000$ points with the same
code run in 80-bit extended precision (a 64-bit significand, on the cluster nodes of Appendix~\ref{app:exp:machine}),
on synthetic instances drawn as in Appendix~\ref{app:exp:machine}, three for each row. Our main
implementation (Appendix~\ref{app:exp:machine}), which sums the terms of the $2K$ nearest points directly, stays
within $1.5\times10^{-12}$ of the largest value at every size and for $K$ up to $20$; the same code run with
Lemma~\ref{lem:remove} for every replaced point, which is the algorithm of the earlier implementation, stays within
$6.6\times10^{-10}$; and the rounding of the algorithm for continuous weights is at most $9\times10^{-9}$ of its
tolerance ($K=5$ and $10$). The same comparison on the $45$ cases of Table~\ref{tab:app-mc-scale}, with the earlier
implementation itself, whose values the comparisons with Monte Carlo use, finds its values within
$1.4\times10^{-11}$ of the largest value in every case, and Kendall's $\tau$ between them and the extended-precision
values is at least $0.9999995$. Both runs of each comparison use the same coefficients $c_N(s)$ and $c_e(K)$ and,
for continuous weights, the same quadrature nodes, computed in double precision, so these figures measure the rounding
of the algorithms' own arithmetic. Checked against exact rationals for $K=3$, $5$ and $10$ ($c_N(s)$ for $s<K$, and
$c_e(K)$ at six values of $e$), the coefficients agree to a relative $3.3\times10^{-10}$ or better at $100{,}000$
points and $3.3\times10^{-9}$ at $1{,}000{,}000$.

\begin{table}[h]
\centering
\caption{Largest difference between results in double precision and in 80-bit extended precision, on synthetic
instances ($W=8$; the largest over three instances). \emph{Main form} and \emph{Subtraction everywhere}: the
exact algorithm with the terms of the $2K$ nearest points summed directly, and with Lemma~\ref{lem:remove} for every
replaced point, divided by the largest absolute value. \emph{Continuous weights}: the algorithm of
Theorem~\ref{thm:cont} at tolerance $\varepsilon=10^{-6}B^2$, divided by $\varepsilon$; not run at $K=20$.}
\label{tab:app-precision}
\begin{tabular}{@{}rrccc@{}}
\toprule
$N$ & $K$ & Main form & Subtraction everywhere & Continuous weights \\
\midrule
1{,}000 & 5 & $2\times10^{-15}$ & $3\times10^{-14}$ & $2\times10^{-11}$ \\
1{,}000 & 10 & $3\times10^{-15}$ & $7\times10^{-12}$ & $6\times10^{-12}$ \\
1{,}000 & 20 & $6\times10^{-16}$ & $2\times10^{-10}$ & -- \\
10{,}000 & 5 & $1\times10^{-14}$ & $6\times10^{-12}$ & $4\times10^{-11}$ \\
10{,}000 & 10 & $2\times10^{-14}$ & $4\times10^{-11}$ & $8\times10^{-11}$ \\
10{,}000 & 20 & $3\times10^{-14}$ & $2\times10^{-10}$ & -- \\
100{,}000 & 5 & $1\times10^{-12}$ & $2\times10^{-11}$ & $9\times10^{-9}$ \\
100{,}000 & 10 & $1\times10^{-12}$ & $4\times10^{-10}$ & $7\times10^{-10}$ \\
100{,}000 & 20 & $5\times10^{-13}$ & $7\times10^{-10}$ & -- \\
\bottomrule
\end{tabular}

\end{table}
\subsection{Running Times}
\label{app:more:timing}

All times in this section are for one query on the synthetic instances of Appendix~\ref{app:exp:machine}.

\textbf{By training-set size.} Table~\ref{tab:app-timing-main} gives the times at $K=5$ and $W=8$, the setting
of Figure~\ref{fig:results}(a). The direct algorithm was timed up to $N=10{,}000$; its time at that size comes
from an earlier run than the other entries of its column. The last column is the number
of permutations that Monte Carlo completes in the time of the exact algorithm. It lies between $2{,}700$ and
$4{,}800$ at every size, so a comparison at equal time is a comparison with a few thousand permutations.

\begin{table}[h]
\centering
\caption{Time for one query by training-set size $N$, at $K=5$ and $W=8$. \emph{Continuous}: the algorithm of
Theorem~\ref{thm:cont} at tolerance $10^{-6}B^2$. \emph{Hard labels}: the algorithm of
Corollary~\ref{cor:hard} (\emph{Shared}) and one table per point (\emph{Per point}), which has the running time
$O(WK^2N^2)$ of the exact algorithm of \citet{wang2024efficient}; both are our implementations. A dash marks a size
that was not run.}
\label{tab:app-timing-main}
\small
\begin{tabular}{@{}rrrrrrr@{}}
\toprule
 & \multicolumn{5}{c}{Seconds for one query} & Monte Carlo permutations \\
\cmidrule(lr){2-6}
 & & & & \multicolumn{2}{c}{Hard labels} & \\
\cmidrule(lr){5-6}
$N$ & Exact & Direct & Continuous & Shared & Per point & in the time of Exact \\
\midrule
150 & 0.013 & 0.013 & 0.013 & 0.0050 & 0.0095 & 3{,}237 \\
300 & 0.026 & 0.057 & 0.041 & 0.0098 & 0.019 & 3{,}947 \\
1{,}000 & 0.087 & 0.466 & 0.083 & 0.034 & 0.275 & 4{,}598 \\
3{,}000 & 0.256 & 4.772 & 0.254 & 0.099 & 2.429 & 4{,}728 \\
10{,}000 & 0.872 & 71.5 & 0.854 & 0.332 & -- & 4{,}484 \\
30{,}000 & 2.588 & -- & 2.568 & 1.055 & -- & 3{,}829 \\
100{,}000 & 8.534 & -- & 8.453 & 3.314 & -- & 2{,}710 \\
\bottomrule
\end{tabular}

\end{table}

\textbf{Growth with $N$.} For each algorithm and each $K$ we fit a power law to the times at $N\ge1{,}000$, by
least squares on the logarithms. We call the growth linear if the fitted exponent lies between $0.85$ and $1.15$
for every $K$. Table~\ref{tab:app-slopes} gives the exponents: between
\SlopeMin{} and \SlopeMax{} for the three near-linear algorithms, and between \DirectSlopeMin{} and
\DirectSlopeMax{} for the direct algorithm. The time itself grows with $K$: at $N=100{,}000$ the exact algorithm
takes $1.7$ times as long at $K=10$ as at $K=5$.

\begin{table}[h]
\centering
\caption{Exponent of $N$ in a power law fitted to the times at $N\ge1{,}000$, and the time at the largest size,
for three values of $K$ ($W=8$; tolerance $10^{-6}B^2$ for continuous weights).}
\label{tab:app-slopes}
\small
\begin{tabular}{@{}lrrrrrr@{}}
\toprule
 & \multicolumn{3}{c}{Fitted exponent of $N$} & \multicolumn{3}{c}{Seconds at $N=100{,}000$} \\
\cmidrule(lr){2-4}\cmidrule(l){5-7}
Algorithm & $K=3$ & $K=5$ & $K=10$ & $K=3$ & $K=5$ & $K=10$ \\
\midrule
Exact (Theorem~\ref{thm:sq}) & 1.00 & 1.00 & 1.00 & 6.4 & 8.5 & 14.3 \\
Continuous weights (Theorem~\ref{thm:cont}) & 1.00 & 1.00 & 0.98 & 6.8 & 8.5 & 13.0 \\
Hard labels (Corollary~\ref{cor:hard}) & 0.99 & 1.00 & 1.00 & 2.6 & 3.3 & 5.6 \\
Direct & 2.04 & 2.17 & 2.12 & -- & -- & -- \\
\bottomrule
\end{tabular}

\end{table}

\textbf{Up to a million points.} Table~\ref{tab:app-million} extends these times to $N=300{,}000$ and
$1{,}000{,}000$, with the same implementations on the same laptop and instances drawn in the same way, in a run of its
own that also timed $N=100{,}000$ again. Over the three sizes the fitted exponents lie between $0.97$ and $1.01$, so
the growth stays linear: at $K=5$ the exact algorithm takes $83$~s for a query on a million points, $10.2$ times its
time at $100{,}000$ points in the same run. At this size the rounding of one kind of sum shows. The exact algorithm takes the sum over the points
farther than a given point as a total minus a running sum, and over a million points the rounding of these running
sums adds up. Its values, which add up to $v(D)-v(\emptyset)$ in exact arithmetic (here $0.057$, $-0.74$ and $0.081$
at $K=3$, $5$ and $10$), miss it by $5.8\times10^{-7}$, $2.8\times10^{-7}$ and $2.7\times10^{-9}$, against at most
$1.6\times10^{-9}$ up to $300{,}000$ points. The hard-label algorithm has the same structure and misses by up to
$3.6\times10^{-7}$; this was not recorded for the algorithm for continuous weights. A diagnostic run after the timing
run, with a copy of the exact algorithm that accumulates its running sums and their totals with compensated
summation, brings the misses on the same instances down to $8.8\times10^{-10}$, $1.1\times10^{-9}$ and
$1.3\times10^{-9}$, at $1.3$ times the time; its values differ from those of the main implementation by at most
$1.2\times10^{-10}$ of the largest value. The values of the main implementation differ from those of the
plain form, which uses Lemma~\ref{lem:remove} for every replaced point, by at most $4.7\times10^{-10}$,
$1.8\times10^{-10}$ and $1.3\times10^{-9}$ of the largest value. Nothing at this size was compared with extended
precision.

\begin{table}[h]
\centering
\caption{Exponent of $N$ fitted to the times at $N=100{,}000$, $300{,}000$ and $1{,}000{,}000$, and the time at
$1{,}000{,}000$, for three values of $K$ ($W=8$; tolerance $10^{-6}B^2$ for continuous weights). A separate run on
the laptop of Table~\ref{tab:app-slopes}; medians of three runs, or a single run above $120$ seconds.}
\label{tab:app-million}
\begin{tabular}{@{}lrrrrrr@{}}
\toprule
 & \multicolumn{3}{c}{Fitted exponent of $N$} & \multicolumn{3}{c}{Seconds at $N=1{,}000{,}000$} \\
\cmidrule(lr){2-4}\cmidrule(l){5-7}
Algorithm & $K=3$ & $K=5$ & $K=10$ & $K=3$ & $K=5$ & $K=10$ \\
\midrule
Exact (Theorem~\ref{thm:sq}) & 1.00 & 1.01 & 1.01 & 62.6 & 83.2 & 139.6 \\
Continuous weights (Theorem~\ref{thm:cont}) & 0.99 & 0.98 & 0.97 & 65.0 & 79.9 & 121.2 \\
Hard labels (Corollary~\ref{cor:hard}) & 0.99 & 1.01 & 1.01 & 24.8 & 32.4 & 55.5 \\
\bottomrule
\end{tabular}

\end{table}

\textbf{Other $K$ and $W$.} Table~\ref{tab:app-timing-K} varies $K$ up to \MaxKChecked{} and the number of weight
levels $W$ up to $64$, at $N=1{,}000$. The exact algorithm takes $2.3$ seconds at $K=\MaxKChecked{}$, where the
direct algorithm takes $89$. Its operation count is linear in $W$, but the measured time grows much more slowly,
from $0.08$ seconds at $2$ levels to $0.12$ at $64$: the implementation handles a row of the table with array
operations, and at these sizes its time is far from proportional to the length of a row. At $K=\MaxKChecked{}$
the values of the exact algorithm agree with those of the direct algorithm to \MaxKAgreeDirect{} of the largest
value. With Lemma~\ref{lem:remove} used for every replaced point, the difference would be \MaxKAgreePlain{}
(Appendix~\ref{app:exact:float} explains this loss of accuracy and gives the agreement for other $K$).

\begin{table}[h]
\centering
\caption{Time for one query at $N=1{,}000$, for other $K$ (first block) and other numbers of weight levels $W$
(second block).}
\label{tab:app-timing-K}
\small
\begin{tabular}{@{}rrrr@{}}
\toprule
$K$ & $W$ & Exact (s) & Direct (s) \\
\midrule
3 & 8 & 0.065 & 0.243 \\
5 & 8 & 0.087 & 0.466 \\
10 & 8 & 0.144 & 2.347 \\
20 & 8 & 0.312 & 11.9 \\
30 & 8 & 1.107 & 29.2 \\
50 & 8 & 2.277 & 89.0 \\
\addlinespace
5 & 2 & 0.080 & 0.211 \\
5 & 4 & 0.081 & 0.300 \\
5 & 16 & 0.089 & 0.969 \\
5 & 64 & 0.122 & 4.714 \\
\bottomrule
\end{tabular}

\end{table}

\textbf{The enumeration.} The only earlier exact algorithm for these utilities enumerates the sets of at most
$K$ neighbors \citep{jia2019knn}. We timed our own implementation of this enumeration, which takes
$O(KN^{K+1})$ operations for the $N$ values, at sizes where it finishes in seconds
(Table~\ref{tab:app-timing-enum}). Its fitted exponents of $N$ are \EnumFitSlopeKThree{} at $K=3$ and
\EnumFitSlopeKFive{} at $K=5$. Extrapolated by this fit, it would take \EnumFitSecondsKThree{} seconds at
$N=1{,}000$ for $K=3$; this number is an extrapolation and not a measurement.

\begin{table}[h]
\centering
\caption{The enumeration over sets of neighbors against the exact algorithm, at the sizes where the enumeration
was run ($W=8$).}
\label{tab:app-timing-enum}
\small
\begin{tabular}{@{}rrrr@{}}
\toprule
$K$ & $N$ & Enumeration (s) & Exact (s) \\
\midrule
3 & 20 & 0.047 & 0.0012 \\
3 & 30 & 0.237 & 0.0019 \\
3 & 40 & 0.752 & 0.0025 \\
3 & 60 & 3.902 & 0.0040 \\
5 & 10 & 0.012 & 0.0008 \\
5 & 12 & 0.037 & 0.0009 \\
5 & 14 & 0.092 & 0.0012 \\
5 & 16 & 0.215 & 0.0013 \\
\bottomrule
\end{tabular}

\end{table}

\textbf{Other losses.} Table~\ref{tab:app-timing-pair} times the algorithm of Theorem~\ref{thm:pair} for the
absolute loss, with $W=\PairWeightLevels{}$ weight levels and targets on \PairTargetLevels{} levels. Its table
has one entry for every pair of a total weight and a total weighted target, so it is much slower than the
algorithm for the squared loss: at $N=1{,}000$ and $K=5$ the near-linear form takes about one second, against
$0.087$ seconds for the squared loss, and the direct form takes about a minute.

\begin{table}[h]
\centering
\caption{The algorithm for other losses (Theorem~\ref{thm:pair}) with the absolute loss, in its direct and its
near-linear form. The two forms were timed in different runs on instances of the same kind; the times of the
near-linear form are single runs of the check of Appendix~\ref{app:more:enum}. A dash marks a size that was not run.}
\label{tab:app-timing-pair}
\small
\begin{tabular}{@{}rrrr@{}}
\toprule
$K$ & $N$ & Direct form (s) & Near-linear form (s) \\
\midrule
3 & 150 & 0.23 & 0.03 \\
3 & 300 & 1.1 & 0.09 \\
3 & 1{,}000 & 13 & 0.24 \\
3 & 3{,}000 & -- & 0.93 \\
5 & 150 & 1.2 & 0.33 \\
5 & 300 & 6.8 & 0.36 \\
5 & 1{,}000 & 68 & 1.05 \\
\bottomrule
\end{tabular}

\end{table}

\textbf{Other tolerances.} At $N=1{,}000$ and $K=5$, the algorithm for continuous weights uses \ContTolLooseNodes{}
quadrature nodes and \ContTolLooseSecs{} seconds at $\varepsilon=10^{-3}B^2$, \ContTolMidNodes{} nodes
and \ContTolMidSecs{} seconds at $10^{-6}B^2$, and \ContTolTightNodes{} nodes and \ContTolTightSecs{}
seconds at $10^{-9}B^2$. Six more digits thus cost \ContTolTimeRatio{} times the time, in line with the
logarithmic dependence on $1/\varepsilon$ in Theorem~\ref{thm:cont}.

\subsection{Monte Carlo}
\label{app:more:mc}

\textbf{The measures of Section~\ref{sec:exp}.} Table~\ref{tab:mc} gives, at each training-set size, the measures
that Section~\ref{sec:exp} reports and Figure~\ref{fig:results}(b) and (c) plot, with the number of cases and the time
of the exact algorithm. Its column \emph{To $0.95$} gives the time Monte Carlo needs before all five of its runs reach $\tau\ge0.95$: more
than the exact algorithm's time at every size in the median case (at $150$ and $1{,}000$ points on the finer grid of
Table~\ref{tab:app-mc-samplers}), and from $10{,}000$ points on no case gets there within $30{,}000$ permutations. Figure~\ref{fig:app-mc-time} follows Kendall's $\tau$ as Monte Carlo runs longer, at
$150$ to $3{,}000$ training points.

\begin{table}[h]
\caption{Monte Carlo against the exact algorithm. Entries are medians over cases; a case is one dataset with one
split. \emph{Equal time}: $\tau$ with the most permutations tried that fit in the exact algorithm's time, for $10$ and
for $100$ validation points ($100$ only on the first split: $9$, $3$ and $3$ cases). \emph{To $0.95$}: the time, as a
multiple of the exact algorithm's, of the fewest permutations tried with which all five runs reach $\tau\ge0.95$; in
brackets, the cases that get there. A $>$ means that half of the cases or more never do, and gives the time of the
most permutations tried. \emph{Changed}: share of one run's top 10\% points (at $3{,}000$ permutations) that are not
in another run's. Below the line: larger datasets on cluster nodes, whose times are not comparable with the laptop's
(Appendix~\ref{app:exp:machine}).}
\label{tab:mc}
\centering
\small
\begin{tabular}{@{}rccccrc@{}}
\toprule
 & & & \multicolumn{2}{c}{Equal time} & & \\
\cmidrule(lr){4-5}
$N$ & Cases & Exact (s) & $10$ & $100$ & To $0.95$ & Changed \\
\midrule
150 & 24 & 0.013 & 0.93 & -- & $2.9\times$ (24) & 11\% \\
1,000 & 21 & 0.087 & 0.86 & -- & $5.8\times$ (19) & 20\% \\
3,000 & 12 & 0.26 & 0.77 & -- & $>19\times$ (6) & 31\% \\
\midrule
10,000 & 27 & 2.3 & 0.65 & 0.86 & $>6.3\times$ (0) & 40\% \\
30,000 & 9 & 6.9 & 0.56 & 0.80 & $>9.2\times$ (0) & 50\% \\
100,000 & 9 & 23 & 0.34 & 0.64 & $>17\times$ (0) & 62\% \\
\bottomrule
\end{tabular}
\end{table}

\begin{figure}[h]
\centering
\includegraphics[width=0.67\textwidth]{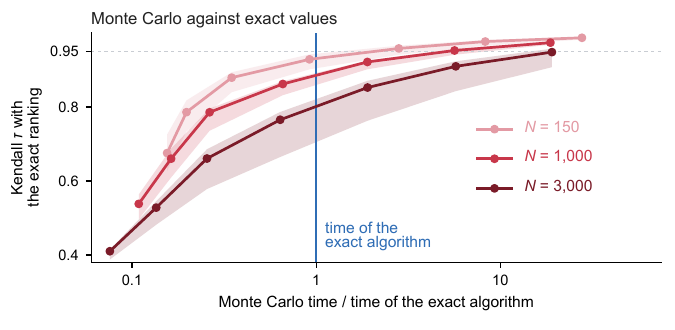}
\caption{Kendall's $\tau$ of Monte Carlo with the exact values against its time relative to the exact algorithm, at
$150$, $1{,}000$ and $3{,}000$ training points ($10$ validation points; laptop; medians over cases, shaded the middle
half).}
\label{fig:app-mc-time}
\end{figure}

\textbf{Every tested number of permutations.} Table~\ref{tab:app-mc-size} extends the rows of Table~\ref{tab:mc} up to $3{,}000$ points: for each
training-set size and each tested number of permutations, it gives the medians over the cases of the four
measures of Appendix~\ref{app:exp:mc}. The equal-time number of permutations is \MCEqualTimeModal{} in
\MCEqualTimeModalCases{} cases and $1{,}000$ in the other two. At $N=3{,}000$, even $100{,}000$ permutations, which
take $19$ times as long as the exact algorithm, give a median $\tau$ of $0.948$, and $17$ of the $300$ top points
still change between two runs.

\begin{table}[h]
\centering
\caption{Monte Carlo against the exact values, by number of permutations (medians over cases; $K=5$, $W=8$).
\emph{Kendall $\tau$}: with the exact values. \emph{Top-$10\%$ overlap}: Jaccard index of the top-$10\%$ sets of
the estimates and of the exact values. \emph{Changed}: top-$10\%$ points that differ between two runs.
\emph{Time / Exact}: running time relative to the exact algorithm.}
\label{tab:app-mc-size}
\small
\begin{tabular}{@{}rrrrrr@{}}
\toprule
$N$ & Permutations & Kendall $\tau$ & Top-10\% overlap & Changed & Time / Exact \\
\midrule
150 & 100 & 0.675 & 0.453 & 7.5 of 15 & 0.16 \\
 & 300 & 0.786 & 0.635 & 4.9 of 15 & 0.20 \\
 & 1{,}000 & 0.879 & 0.789 & 2.9 of 15 & 0.35 \\
 & 3{,}000 & 0.929 & 0.878 & 1.6 of 15 & 0.92 \\
 & 10{,}000 & 0.958 & 0.900 & 1.1 of 15 & 2.81 \\
 & 30{,}000 & 0.977 & 0.939 & 0.6 of 15 & 8.28 \\
 & 100{,}000 & 0.987 & 0.975 & 0.4 of 15 & 27.74 \\
\addlinespace
1{,}000 & 100 & 0.538 & 0.367 & 60.4 of 100 & 0.11 \\
 & 300 & 0.660 & 0.458 & 47.3 of 100 & 0.16 \\
 & 1{,}000 & 0.786 & 0.614 & 32.4 of 100 & 0.26 \\
 & 3{,}000 & 0.862 & 0.748 & 20.3 of 100 & 0.66 \\
 & 10{,}000 & 0.922 & 0.856 & 10.4 of 100 & 1.90 \\
 & 30{,}000 & 0.953 & 0.912 & 7.0 of 100 & 5.64 \\
 & 100{,}000 & 0.974 & 0.953 & 3.3 of 100 & 18.71 \\
\addlinespace
3{,}000 & 100 & 0.410 & 0.253 & 209.4 of 300 & 0.08 \\
 & 300 & 0.528 & 0.333 & 184.7 of 300 & 0.14 \\
 & 1{,}000 & 0.661 & 0.463 & 143.7 of 300 & 0.26 \\
 & 3{,}000 & 0.765 & 0.630 & 94.0 of 300 & 0.64 \\
 & 10{,}000 & 0.853 & 0.769 & 53.5 of 300 & 1.90 \\
 & 30{,}000 & 0.910 & 0.863 & 31.2 of 300 & 5.73 \\
 & 100{,}000 & 0.948 & 0.924 & 16.6 of 300 & 19.06 \\
\bottomrule
\end{tabular}

\end{table}

\textbf{Against the time of the direct algorithm.} Table~\ref{tab:app-mc-direct} repeats the two comparisons of
those rows of Table~\ref{tab:mc} with the time of the direct algorithm in place of that of the exact algorithm. In the time of
the direct algorithm, Monte Carlo completes $3{,}000$ to $100{,}000$ permutations and reaches a median $\tau$ of
$0.93$ to $0.95$. To reach $\tau\ge0.95$ in every run it needs $2.4$ times that time at $N=150$, the same time
at $N=1{,}000$ in the $19$ of $21$ cases that get there, and $0.84$ times that time at $N=3{,}000$ in the $6$ of $12$
cases that get there; in the other six, $100{,}000$ permutations, which fit in that time, do not. Against the direct
algorithm, Monte Carlo is thus about level at the two larger sizes. What puts the exact values ahead in the rows of Table~\ref{tab:mc} up to $3{,}000$ points is the near-linear running time; the direct algorithm was not compared beyond them.

\begin{table}[h]
\centering
\caption{The comparisons of Table~\ref{tab:mc} up to $3{,}000$ points with the time of the direct algorithm (medians over cases).
\emph{Permutations}: the largest tested number that takes no longer than the direct algorithm. \emph{Time to
$0.95$ / Direct}: time of the smallest tested number at which five of five runs have $\tau\ge0.95$, relative to
the direct algorithm, over the cases in which a tested number suffices.}
\label{tab:app-mc-direct}
\small
\begin{tabular}{@{}rrrrrr@{}}
\toprule
 & \multicolumn{2}{c}{Seconds per query} & \multicolumn{2}{c}{In the time of Direct} & \\
\cmidrule(lr){2-3}\cmidrule(lr){4-5}
$N$ & Exact & Direct & Permutations & Kendall $\tau$ & Time to $0.95$ / Direct \\
\midrule
150 & 0.013 & 0.016 & 3{,}000 & 0.929 & 2.40 \\
1{,}000 & 0.087 & 0.51 & 10{,}000 to 30{,}000 & 0.943 & 1.02 \\
3{,}000 & 0.26 & 6.0 & 100{,}000 & 0.948 & 0.84 \\
\bottomrule
\end{tabular}

\end{table}

\textbf{Other samplers.} Permutation sampling is the plainest sampling estimator, and a better one might close
the gap. Table~\ref{tab:app-mc-samplers} adds three samplers on the cases of Table~\ref{tab:mc} up to $3{,}000$ points: antithetic sampling,
which pairs each permutation with its reverse; Owen sampling \citep{okhrati2021multilinear}, which draws each set with
an inclusion probability stratified over $(0,1)$ \citep{owen1972multilinear} and pairs it with its complement; and
KernelSHAP \citep{lundberg2017shap} in the unbiased form of \citet{covert2021improving}, with paired sets and the exact
regression matrix. Each runs five times with each tested number of samples from $100$ to $100{,}000$ (a sample is a
permutation or a set; a pair counts as two), every number a run of its own, on the CPU nodes of a computing cluster
(Appendix~\ref{app:exp:machine}); on small instances each sampler agrees with the exact values in expectation, a check
that two deliberately wrong samplers fail. Permutation sampling has a median $\tau$ of $0.93$ at equal time at $150$
points here, as in Table~\ref{tab:mc}. Owen sampling is the best of the four at every size. From $1{,}000$ points on it
stays below $\tau=0.95$ within the exact algorithm's time, with a median $\tau$ of $0.93$ at $1{,}000$ points and
$0.88$ at $3{,}000$. In the median case it needs more than $2.0$ times the exact algorithm's time for
$\tau\ge0.95$ in all five runs at $1{,}000$ points and more than $8.3$ times at $3{,}000$, the time of the tested
number just below the smallest one that gets there. At $150$ points it gets there in less time than the exact algorithm takes, at
most $0.5$ times that time in the median case. Antithetic sampling gives about
the same $\tau$ as permutation sampling at an equal number of samples but draws half as many random permutations, so
it fits more samples into the exact algorithm's time and comes out a little ahead at equal time. KernelSHAP does much
worse: a set gives it one value of the utility, where it gives Owen sampling every point's marginal contribution. We
ran its unbiased form so that all four samplers are unbiased and checked alike; \citet{covert2021improving} find that
the original form has much lower variance at a negligible increase in bias, so this row is not the best that a
regression estimator can do. A first run of this comparison read each run's estimate at every tested number on its
way to $100{,}000$, which adds the cost of the small early batches to permutation and antithetic sampling only; it
also found Owen sampling the best of the four at every size and below $0.95$ at $1{,}000$ and $3{,}000$ points.

\begin{table}[h]
\centering
\caption{Four samplers against the exact values on the cases of Table~\ref{tab:mc} up to $3{,}000$ points (medians over cases; $K=5$,
$W=8$; cluster nodes). The tested numbers of samples are $100$, $200$, $300$, $500$, $1{,}000$, and so on up to
$100{,}000$, each a run of its own. \emph{Samples}: the largest tested number that takes no longer than the exact
algorithm. \emph{Equal time}: Kendall's $\tau$ at that number; the two medians are taken separately. \emph{To $0.95$}:
as in Table~\ref{tab:mc}, with samples in place of permutations. \emph{Changed}: top-$10\%$ points that differ between two runs at
equal time (Table~\ref{tab:mc} counts them at $3{,}000$ permutations).}
\label{tab:app-mc-samplers}
\begin{tabular}{@{}rlrccc@{}}
\toprule
$N$ & Sampler & Samples & Equal time & To $0.95$ & Changed \\
\midrule
150 & Permutation & 3{,}000 & 0.926 & $2.8\times$ (24) & 1.8 of 15 \\
 & Antithetic & 3{,}000 & 0.928 & $2.3\times$ (24) & 1.4 of 15 \\
 & Owen & 3{,}000 & 0.973 & $0.5\times$ (24) & 0.7 of 15 \\
 & KernelSHAP & 5{,}000 & 0.715 & $>10\times$ (0) & 6.5 of 15 \\
\addlinespace
1{,}000 & Permutation & 3{,}000 & 0.864 & $6.2\times$ (19) & 21 of 100 \\
 & Antithetic & 5{,}000 & 0.892 & $4.8\times$ (19) & 16 of 100 \\
 & Owen & 2{,}000 & 0.929 & $4.1\times$ (21) & 9 of 100 \\
 & KernelSHAP & 5{,}000 & 0.327 & $>12\times$ (0) & 75 of 100 \\
\addlinespace
3{,}000 & Permutation & 5{,}000 & 0.788 & $>18\times$ (6) & 82 of 300 \\
 & Antithetic & 5{,}000 & 0.807 & $>15\times$ (6) & 72 of 300 \\
 & Owen & 2{,}000 & 0.885 & $14.2\times$ (11) & 43 of 300 \\
 & KernelSHAP & 5{,}000 & 0.153 & $>14\times$ (0) & 260 of 300 \\
\bottomrule
\end{tabular}

\end{table}

\textbf{A larger scale.} The comparisons so far use $10$ validation points and at most $3{,}000$ training points.
Table~\ref{tab:app-mc-large} repeats them at $N=8{,}000$ with $100$ validation points, on the two datasets large
enough for it (CPU (small) and Kin8nm, three splits each), on the same cluster nodes, with the four samplers read at
$300$ to $30{,}000$ samples per query. In the timings of this table, the exact algorithm takes $173$ to $317$~s for the $100$
queries on these shared nodes. Within its time, permutation sampling reaches a median $\tau$ of $0.79$, with $3{,}000$ permutations,
and Owen sampling $0.84$, with $1{,}000$ sets; two runs of permutation sampling disagree on $179$ of the $800$ top
points. Given more time, Owen sampling reaches $0.945$ with $10{,}000$ sets, which take $3.9$ times the exact
algorithm's time, and $0.969$ with $30{,}000$, $13$ times; permutation sampling reaches $0.926$ with $30{,}000$
permutations, $5.5$ times. The two datasets differ: with $30{,}000$ permutations, permutation sampling has
$\tau\ge0.95$ in all five runs on all three splits of Kin8nm, and a mean $\tau$ of $0.88$ to $0.90$ on CPU
(small). These numbers do not show how the gap changes with $N$: on the same two datasets at $3{,}000$ points, with
$10$ validation points, permutation sampling's median $\tau$ at equal time is $0.71$ (Table~\ref{tab:app-mc-samplers}'s
cases), and the two settings differ in the number of validation points as well as in $N$. In detecting the shifted
targets, permutation sampling at equal time comes within $0.006$ of the exact values' AUROC in every case (median
$0.962$ against $0.963$, lower in all six), and KernelSHAP is far below ($0.655$).

\begin{table}[h]
\centering
\caption{Four samplers against the exact values at $N=8{,}000$ with $100$ validation points (medians over six
cases: two datasets, three splits each; cluster nodes). For each tested number of samples: Kendall's $\tau$ with the
exact values (mean over five runs) and, in brackets, the time of a run of its own relative to the exact algorithm.
\emph{Changed}: top-$10\%$ points, of $800$, that differ between two runs at equal time, the largest tested number
that takes no longer than the exact algorithm ($1{,}000$ for Owen sampling, $3{,}000$ for the others). \emph{AUROC}:
detection of the shifted targets by the estimates at equal time, and in brackets by the exact values.}
\label{tab:app-mc-large}
\begin{tabular}{@{}lccccrc@{}}
\toprule
Sampler & 1{,}000 & 3{,}000 & 10{,}000 & 30{,}000 & Changed & AUROC (exact) \\
\midrule
Permutation & 0.687 (0.2) & 0.789 (0.6) & 0.876 (1.8) & 0.926 (5.5) & 179 & 0.962 (0.963) \\
Antithetic & 0.688 (0.2) & 0.791 (0.5) & 0.876 (1.5) & 0.926 (4.5) & 182 & 0.961 (0.963) \\
Owen & 0.841 (0.4) & 0.904 (1.2) & 0.945 (3.9) & 0.969 (13.2) & 149 & 0.958 (0.963) \\
KernelSHAP & 0.061 (0.1) & 0.103 (0.4) & 0.170 (1.5) & 0.250 (4.5) & 708 & 0.655 (0.963) \\
\bottomrule
\end{tabular}

\end{table}

\textbf{Up to $100{,}000$ points.} Table~\ref{tab:app-mc-scale} repeats the comparison for permutation and Owen
sampling at $N=10{,}000$, $30{,}000$ and $100{,}000$, on the nine larger datasets of Table~\ref{tab:app-datasets} (at
$30{,}000$ and $100{,}000$ points the three with more than $100{,}010$ points), with $10$ validation points and three
splits each, on the cluster nodes. Every tested number of samples, from $100$ to $30{,}000$, is a run of its own, and
the exact algorithm is timed in the same process before and after each run. The gap keeps widening: permutation
sampling's median $\tau$ at equal time falls from $0.79$ at $3{,}000$ points on the same nodes
(Table~\ref{tab:app-mc-samplers}; $0.77$ on the laptop, Table~\ref{tab:mc}) to $0.65$, $0.56$ and $0.34$ at
$10{,}000$, $30{,}000$ and $100{,}000$ points (over the cases, from $0.53$ to $0.80$, $0.45$ to $0.63$ and $0.28$ to
$0.45$), and Owen sampling's from $0.88$ to $0.66$, $0.60$ and $0.38$. At a fixed number of permutations, where
timing plays no part, $\tau$ falls too: at $3{,}000$ permutations the median over the three largest datasets falls
from $0.66$ at $10{,}000$ points to $0.56$ and $0.45$, and at $30{,}000$ permutations from $0.85$ to $0.80$ and
$0.64$. Case by case, $\tau$ at equal time falls in $8$ of the $9$ cases from $10{,}000$ to $30{,}000$ points and in
$9$ of $9$ from $30{,}000$ to $100{,}000$. At $100{,}000$ points the timing conditions slowed the sampler more in these
runs than in those of Table~\ref{tab:app-mc-val100} (Appendix~\ref{app:exp:machine}), so $0.34$ is the value under the
slowest conditions measured; at a fixed $3{,}000$ permutations it is $0.45$. Within $30{,}000$ samples, all five runs reach $\tau\ge0.95$ in no case for permutation sampling, and for
Owen sampling in $2$ of the $27$ cases at $10{,}000$ points and in none at the two larger sizes. Two runs of
permutation sampling at equal time disagree on $39\%$, $51\%$ and $68\%$ of the top $10\%$ points, and two runs of
$3{,}000$ permutations, Table~\ref{tab:mc}'s measure, on $40\%$, $50\%$ and $62\%$. The columns \emph{Samples} and
\emph{To $0.95$} of Table~\ref{tab:app-mc-scale} are medians taken separately: at $100{,}000$ points, $2{,}000$
permutations take a median $0.99$ times the exact algorithm's time ($0.76$ to $1.22$ over the cases), so they fit
within it in five of the nine cases and $1{,}000$ in the other four, and $30{,}000$ take $16.6$ times. The
exact algorithm takes $23$~s for a query on $100{,}000$ points on these nodes, and its values in every case of the
table agree with an extended-precision computation (Appendix~\ref{app:more:enum}).

\begin{table}[h]
\centering
\caption{Permutation and Owen sampling against the exact values at larger sizes (medians over cases; $10$ validation
points; $K=5$, $W=8$; cluster nodes). \emph{Samples} and \emph{Equal time}: as in Table~\ref{tab:app-mc-samplers}.
\emph{To $0.95$}: the tested numbers go up to $30{,}000$; where half of the cases or more do not get there, $>$ gives
the time of $30{,}000$ samples relative to the exact algorithm, and in brackets the cases that get there.
\emph{Changed}: top-$10\%$ points that differ between two runs at equal time (mean over the ten pairs of runs).}
\label{tab:app-mc-scale}
\begin{tabular}{@{}rlrcccc@{}}
\toprule
$N$ & Sampler & Cases & Samples & Equal time & To $0.95$ & Changed \\
\midrule
10{,}000 & Permutation & 27 & 3{,}000 & 0.647 & $>6.3\times$ (0) & 391 of 1{,}000 \\
 & Owen &  & 1{,}000 & 0.661 & $>15\times$ (2) & 304 of 1{,}000 \\
\addlinespace
30{,}000 & Permutation & 9 & 3{,}000 & 0.563 & $>9.2\times$ (0) & 1{,}529 of 3{,}000 \\
 & Owen &  & 1{,}000 & 0.604 & $>16\times$ (0) & 1{,}240 of 3{,}000 \\
\addlinespace
100{,}000 & Permutation & 9 & 2{,}000 & 0.342 & $>17\times$ (0) & 6{,}797 of 10{,}000 \\
 & Owen &  & 1{,}000 & 0.383 & $>18\times$ (0) & 6{,}089 of 10{,}000 \\
\bottomrule
\end{tabular}

\end{table}

\textbf{With $100$ validation points.} Table~\ref{tab:app-mc-val100} repeats the comparison with $100$ validation
points, on the first split of each dataset: the nine larger datasets at $3{,}000$ and $10{,}000$ training points and
the three largest at $30{,}000$ and $100{,}000$. From $10{,}000$ points on, each case holds the training points, the
shifted targets and the first $10$ validation points of its case in Table~\ref{tab:app-mc-scale}, and $90$ further
validation points. With more validation points Monte Carlo comes closer to the exact ranking. Compared at the same
number of samples, so that the timing conditions of the two runs play no part, $\tau$ is higher with $100$ validation
points than with $10$ in every case: at $3{,}000$ permutations by a median of $0.18$, $0.16$ and $0.20$ at
$10{,}000$, $30{,}000$ and $100{,}000$ points, and for Owen sampling at $2{,}000$ sets by $0.18$, $0.14$ and $0.12$.
The number of samples that fits in the exact algorithm's time differs between the two runs (Tables~\ref{tab:app-mc-scale}
and~\ref{tab:app-mc-val100}), because the two codes ran at different relative speeds (Appendix~\ref{app:exp:machine}).
The gap still widens with $N$. On the three largest datasets, permutation sampling's median $\tau$ at equal time falls
from $0.92$ at $3{,}000$ points to $0.88$, $0.80$ and $0.64$, and on each of the three it falls at every step; Owen
sampling's median falls from $0.94$ to $0.67$, but stays level between $10{,}000$ and $30{,}000$ points ($0.843$ and
$0.845$). These are medians over the three largest datasets; Table~\ref{tab:app-mc-val100} gives the medians over all
nine at $3{,}000$ and $10{,}000$ points. At $3{,}000$ points, all five runs of Owen sampling reach $\tau\ge0.95$ within
the exact algorithm's time in $2$ of the $9$ cases. At $100{,}000$ points two runs of permutation sampling at equal
time disagree on $42\%$ of the top $10\%$ points. The tested numbers stop at $10{,}000$ and $3{,}000$ at the two
larger sizes, so the time to $0.95$ is not read there.

\begin{table}[h]
\centering
\caption{Permutation and Owen sampling against the exact values with $100$ validation points (medians over cases;
first split; $K=5$, $W=8$; cluster nodes). \emph{Samples}, \emph{Equal time} and \emph{Changed}: as in
Table~\ref{tab:app-mc-scale}. \emph{Reach $0.95$}: cases in which all five runs reach $\tau\ge0.95$ within
$30{,}000$ samples; not read where the tested numbers stop at $10{,}000$ ($N=30{,}000$) and $3{,}000$
($N=100{,}000$).}
\label{tab:app-mc-val100}
\begin{tabular}{@{}rlrcccc@{}}
\toprule
$N$ & Sampler & Cases & Samples & Equal time & Reach $0.95$ & Changed \\
\midrule
3{,}000 & Permutation & 9 & 3{,}000 & 0.898 & 8 of 9 & 41 of 300 \\
 & Owen &  & 2{,}000 & 0.934 & 9 of 9 & 27 of 300 \\
\addlinespace
10{,}000 & Permutation & 9 & 3{,}000 & 0.855 & 3 of 9 & 180 of 1{,}000 \\
 & Owen &  & 2{,}000 & 0.844 & 7 of 9 & 171 of 1{,}000 \\
\addlinespace
30{,}000 & Permutation & 3 & 3{,}000 & 0.802 & -- & 520 of 3{,}000 \\
 & Owen &  & 2{,}000 & 0.845 & -- & 453 of 3{,}000 \\
\addlinespace
100{,}000 & Permutation & 3 & 3{,}000 & 0.638 & -- & 4{,}200 of 10{,}000 \\
 & Owen &  & 2{,}000 & 0.667 & -- & 3{,}424 of 10{,}000 \\
\bottomrule
\end{tabular}

\end{table}

\textbf{Other $K$.} Table~\ref{tab:app-mc-budgets} is an experiment at $N=\NTrain{}$ with $K=3$, $5$ and $10$
(third row of Table~\ref{tab:app-cases}). It counts the cases by their matching number of permutations and gives
the mean $\tau$ and the number of changed top-$10\%$ points, of $\TopKSize{}$, at $3{,}000$ and at $10{,}000$
permutations. Monte Carlo needs more permutations as $K$ grows: at $3{,}000$ permutations the mean $\tau$ falls
from $0.946$ at $K=3$ to $0.931$ at $K=10$, and the share of cases that need more than $10{,}000$ permutations
rises from $3$ of $40$ to $8$ of $40$. On the same cases, the exact values of the same utility with uniform
weights, that is, of unweighted KNN regression, rank the points differently from the weighted values: their mean
$\tau$ with them is $0.940$, $0.911$ and $0.869$ at $K=3$, $5$ and $10$ ($0.755$ in the lowest case at $K=5$),
below that of Monte Carlo with $3{,}000$ permutations. This experiment ran with earlier implementations of the exact algorithm
and of the sampler, so we report no times from it; the comparison of times is the one of Table~\ref{tab:mc}.

\begin{table}[h]
\centering
\caption{Monte Carlo at $N=\NTrain{}$ for three values of $K$ ($W=8$). \emph{Cases by matching number}: number
of cases whose smallest tested number of permutations with $\tau\ge0.95$ in at least $90\%$ of the runs is the
one named; \emph{above}: more than $10{,}000$. \emph{Changed}: top-$10\%$ points, of $\TopKSize{}$, that differ
between two runs. Means over cases.}
\label{tab:app-mc-budgets}
\small
\begin{tabular}{@{}rrrrrrrrrr@{}}
\toprule
 & & \multicolumn{4}{c}{Cases by matching number} & \multicolumn{2}{c}{Mean $\tau$ at} & \multicolumn{2}{c}{Changed at} \\
\cmidrule(lr){3-6}\cmidrule(lr){7-8}\cmidrule(l){9-10}
$K$ & Cases & 1{,}000 & 3{,}000 & 10{,}000 & above & 3{,}000 & 10{,}000 & 3{,}000 & 10{,}000 \\
\midrule
3 & 40 & 0 & 17 & 20 & 3 & 0.946 & 0.970 & 1.4 & 0.7 \\
5 & 80 & 1 & 23 & 41 & 15 & 0.938 & 0.965 & 1.7 & 1.0 \\
10 & 40 & 0 & 9 & 23 & 8 & 0.931 & 0.962 & 2.3 & 1.4 \\
\bottomrule
\end{tabular}

\end{table}

\subsection{The Number of Weight Levels}
\label{app:more:levels}

\textbf{At $N=150$.} Table~\ref{tab:app-levels} is the experiment of Figure~\ref{fig:results}(d): the values for
weights on $W$ levels against the values for the kernel weights themselves, in \LevelsCases{} cases. The values
for the kernel weights are computed to tolerance $10^{-8}B^2$ by the variant of the algorithm for continuous
weights that is quadratic in $N$ (Appendix~\ref{app:approx:second}). Kendall's $\tau$ rises from $0.90$ at two
levels to $0.99$ at \SmallestLevelsTauPointNineNine{} levels, the smallest tested number with a mean of at least
$0.99$. The AUROC of the detection task of Appendix~\ref{app:more:detection}, averaged over the cases, barely changes
with the number of levels (\LevelsAucMin{} to \LevelsAucMax{}).

\begin{table}[h]
\centering
\caption{Values for weights on $W$ levels against values for the kernel weights, at $N=150$ and $K=5$.
\emph{Mean} and \emph{lowest $\tau$}: over the \LevelsCases{} cases. \emph{Difference}: the largest difference
between the two values of a point, divided by the range of the values for the kernel weights, averaged over the
cases. \emph{AUROC}: of the detection task, averaged over the cases; the last row is for the kernel weights.}
\label{tab:app-levels}
\small
\begin{tabular}{@{}rrrrr@{}}
\toprule
Levels $W$ & Mean $\tau$ & Lowest $\tau$ & Difference & AUROC \\
\midrule
2 & 0.904 & 0.805 & 0.107 & 0.953 \\
4 & 0.939 & 0.853 & 0.077 & 0.955 \\
8 & 0.962 & 0.890 & 0.049 & 0.955 \\
16 & 0.977 & 0.906 & 0.032 & 0.957 \\
64 & 0.990 & 0.950 & 0.018 & 0.956 \\
\addlinespace
Kernel weights & -- & -- & -- & 0.956 \\
\bottomrule
\end{tabular}

\end{table}

The agreement is lower in the cases whose kernel weights span a wider range. Over the \LevelsCases{} cases, the
rank correlation between $\tau$ at $8$ levels and the largest ratio of two weights of one query is
$\LevelsSpearman{}$. In the \LevelsNarrowCases{} cases with a ratio of at most $100$, $\tau$ lies between
\LevelsNarrowTauMin{} and \LevelsNarrowTauMax{}; in the case with the widest range, a ratio of
\KernelWeightRatio{}, it is \LevelsWidestTau{}. The bound of Proposition~\ref{app:approx:prop:round} does not cover the rule of
Appendix~\ref{app:exp:weights}: it needs a rounding unit no larger than the smallest weight, and the rule's unit is
the largest weight divided by $W$.

\textbf{At three sizes.} Table~\ref{tab:app-levels-size} repeats the comparison on the cases of
Table~\ref{tab:mc} up to $3{,}000$ points, with the values for the kernel weights computed to tolerance $10^{-8}B^2$ by the algorithm of
Theorem~\ref{thm:cont}. At $1{,}000$ and $3{,}000$ points these values have one check: in every query they sum to
$v(D)-v(\emptyset)$, to within $1.1\times10^{-10}B^2$. The median $\tau$ at $8$ levels is $0.97$, $0.97$ and $0.96$ at the three sizes, without a
trend in $N$. The time of the algorithm for continuous weights follows the range of the weights, as the number
of quadrature nodes does. At $N=3{,}000$ it takes \KernelNarrowSecsMin{} to \KernelNarrowSecsMax{} seconds per
query in the \KernelNarrowCases{} cases whose weights span a ratio of at most \KernelNarrowRatioMax{}
(\KernelNarrowNodesMin{} to \KernelNarrowNodesMax{} nodes), and \KernelWideSecsMin{} to \KernelWideSecsMax{}
seconds in the \KernelWideCases{} cases with ratios from \KernelWideRatioMin{} to \KernelWideRatioMax{}
(\KernelWideNodesMin{} to \KernelWideNodesMax{} nodes).

\begin{table}[h]
\centering
\caption{Values for weights on $8$ and on $64$ levels against values for the kernel weights, and the time of the
algorithm for continuous weights, at three training-set sizes ($K=5$; medians over cases unless stated).
\emph{Continuous}: the algorithm of Theorem~\ref{thm:cont} at tolerance $10^{-8}B^2$. \emph{Exact}: the
algorithm of Theorem~\ref{thm:sq} with $8$ levels. \emph{Largest ratio of weights}: between the largest and the
smallest kernel weight of one query. \emph{Most nodes}: the largest number of quadrature nodes used for a query.}
\label{tab:app-levels-size}
\small
\begin{tabular}{@{}rrrrrrrrrrr@{}}
\toprule
 & & \multicolumn{2}{c}{$\tau$ at 8 levels} & \multicolumn{2}{c}{$\tau$ at 64 levels} & \multicolumn{3}{c}{Seconds per query} & Largest ratio & Most \\
\cmidrule(lr){3-4}\cmidrule(lr){5-6}\cmidrule(lr){7-9}
$N$ & Cases & median & lowest & median & lowest & Continuous & slowest case & Exact & of weights & nodes \\
\midrule
150 & 24 & 0.968 & 0.848 & 0.995 & 0.893 & 0.015 & 0.028 & 0.013 & $1.8\times10^{30}$ & 1{,}230 \\
1{,}000 & 21 & 0.969 & 0.934 & 0.996 & 0.967 & 0.11 & 0.70 & 0.087 & $2.0\times10^{53}$ & 2{,}000 \\
3{,}000 & 12 & 0.958 & 0.913 & 0.994 & 0.965 & 0.40 & 2.5 & 0.26 & $4.8\times10^{54}$ & 2{,}040 \\
\bottomrule
\end{tabular}

\end{table}

\textbf{How the values for the kernel weights were checked.} At $150$ training points, no enumeration can check
these values. Three things support them. The error bound of Theorem~\ref{thm:cont} is proved for exact arithmetic, and
we claim no proved bound for the double-precision runs that computed these values. The floating-point form of the bound
(Proposition~\ref{app:approx:prop:float}) assumes that nothing underflows, which the runs do not ensure, and its
hypothesis~(i) fails at this tolerance in the two cases whose weights span the widest ranges (ratios of
$4.9\times10^{14}$ and $8.4\times10^{35}$), the two with the lowest agreement at $8$ levels ($\tau$ of $0.890$ and
$0.912$). The two algorithms
for continuous weights, which sum over the subsets in different ways, differ by at most
\KernelTwoAlgorithmsDiff{}$\,B^2$ in the \KernelCasesSmall{} cases at $N=150$. And Monte Carlo,
which uses neither, agrees with them: for one query in each of \KernelMCCases{} cases we ran
\KernelMCPermutations{} permutations in \KernelMCBatches{} batches and formed, for each of the $150$ points, the
difference between the estimate and the value divided by its standard error over the batches. The root mean
square of these $150$ numbers is at most \KernelMCRmsZ{} in a case, the largest of all is \KernelMCMaxZ{}, and
\KernelMCBeyondThree{} exceed $3$ in absolute value, where a normal distribution predicts $3$. An earlier run of
this check with $30$ batches had given $\KernelMCFirstMaxZ{}$ at one point; in the run reported here that point
is at $\KernelMCFirstPointNow{}$. This last check can only find errors of the size of the Monte Carlo standard
error, and the weights of its queries span ratios up to \KernelMCRatioMax{}, not the widest ones.

\subsection{Detection of Noisy Targets and Removal of Points}
\label{app:more:detection}

\textbf{Task and methods.} The task is to find the training points whose target was shifted
(Appendix~\ref{app:exp:split}): $\NCorrupt{}$ of the $\NTrain{}$ training points of a case. A method gives every
training point a score, and we report the AUROC of ranking the shifted points below the others, averaged over the
splits of a dataset. \citet{wang2024efficient} use such tasks to compare data values, with noise added to the
features. We compare \NumDetectionMethods{} methods: the exact values for kernel weights on $8$ levels; the exact
values of the same utility with uniform weights, that is, of unweighted KNN regression; Monte Carlo with
\DetectionMCPermutations{} permutations per query; Data-OOB \citep{kwon2023dataoob} with \OobBootstraps{}
bootstrap samples of a KNN regressor, which scores a point by its negative absolute out-of-bag residual;
leave-one-out, the change in the mean squared error on the validation points when the point is removed from a
KNN regressor; and a random score. Data-OOB and leave-one-out use the KNN regressor of scikit-learn
\citep{pedregosa2011sklearn} with inverse-distance weights, not the Gaussian kernel.

\begin{table}[h]
\centering
\caption{AUROC of detecting the shifted targets, by dataset, at $K=5$ (means over ten splits). \emph{Exact,
kernel} and \emph{Exact, uniform}: the values of Theorem~\ref{thm:sq} with kernel weights on $8$ levels and with
uniform weights. \emph{Monte Carlo}: \DetectionMCPermutations{} permutations per query.}
\label{tab:app-detection}
\small
\begin{tabular}{@{}lrrrrrr@{}}
\toprule
Dataset & Exact, kernel & Exact, uniform & Monte Carlo & Data-OOB & Leave-one-out & Random \\
\midrule
Abalone & 0.963 & 0.958 & 0.962 & 0.986 & 0.644 & 0.482 \\
Airfoil self-noise & 0.964 & 0.957 & 0.959 & 0.985 & 0.635 & 0.526 \\
Concrete & 0.928 & 0.930 & 0.931 & 0.985 & 0.625 & 0.508 \\
CPU (small) & 0.967 & 0.962 & 0.967 & 0.983 & 0.683 & 0.492 \\
Energy efficiency & 0.997 & 0.995 & 0.998 & 1.000 & 0.742 & 0.495 \\
Kin8nm & 0.961 & 0.961 & 0.963 & 0.990 & 0.588 & 0.492 \\
Space GA & 0.948 & 0.949 & 0.955 & 0.980 & 0.604 & 0.508 \\
Wine quality (red) & 0.946 & 0.938 & 0.946 & 0.964 & 0.629 & 0.534 \\
\addlinespace
Mean & 0.959 & 0.956 & 0.960 & 0.984 & 0.644 & 0.505 \\
\bottomrule
\end{tabular}

\end{table}

\begin{table}[h]
\centering
\caption{AUROC of detecting the shifted targets for three values of $K$ (means over the eight datasets). The
columns are those of Table~\ref{tab:app-detection}.}
\label{tab:app-detection-K}
\small
\begin{tabular}{@{}rrrrrrrr@{}}
\toprule
$K$ & Cases & Exact, kernel & Exact, uniform & Monte Carlo & Data-OOB & Leave-one-out & Random \\
\midrule
3 & 40 & 0.964 & 0.962 & 0.965 & 0.986 & 0.651 & 0.512 \\
5 & 80 & 0.959 & 0.956 & 0.960 & 0.984 & 0.644 & 0.505 \\
10 & 40 & 0.952 & 0.943 & 0.950 & 0.986 & 0.626 & 0.512 \\
\bottomrule
\end{tabular}

\end{table}

\textbf{Results.} Table~\ref{tab:app-detection} gives the AUROC by dataset at $K=5$, and
Table~\ref{tab:app-detection-K} the means at $K=3$, $5$ and $10$. We read them in three steps. The tests are on
the eight dataset means. The first two steps are tested at $K=5$; the test of the third
step is exploratory, and at the two other $K$ we report the same quantities as description.
\begin{itemize}
\item \emph{Exact values and Monte Carlo do equally well.} At $K=5$ the difference is $\DetTostMeanDiff{}$, and the
  two are \DetTostEquivalent{} (two one-sided $t$-tests, $p=\DetTostP{}$); the differences at $K=3$ and $K=10$
  are $\DetTostMeanDiffKThree{}$ and $\DetTostMeanDiffKTen{}$. For this task a few hundred permutations suffice, and
  exact values add nothing.
\item \emph{Kernel weights change detection by less than the band of no change.} At $K=5$ kernel weights score
  $\DetWeightedMinusUniform{}$ above uniform weights (Wilcoxon signed-rank test, $p=\DetWeightedP{}$), inside the
  band of $\pm0.02$ that we take as no change. At $K=10$ the difference is $\DetWeightedMinusUniformKTen{}$
  and positive on all \DetWeightedPositiveKTen{} datasets ($p=\DetWeightedPKTen{}$), still inside the band. We
  therefore claim no improvement from weighting on this task.
\item \emph{Data-OOB detects the shifted targets best.} It scores $\DetOobMinusExactKThree{}$,
  $\DetOobMinusExactKFive{}$ and $\DetOobMinusExactKTen{}$ above the exact values at $K=3$, $5$ and $10$
  (Wilcoxon signed-rank test, $p=\DetOobP{}$ at $K=5$). Leave-one-out is far behind the exact values, Monte
  Carlo and Data-OOB, and above the random score.
\end{itemize}
\textbf{Removing the most and the least valuable points.} The other common downstream use of data values is to
remove the points a method values most or least and refit \citep{ghorbani2019data,kwon2022beta,wang2023banzhaf}.
Table~\ref{tab:app-removal} does so on $12$, $9$ and $18$ cases at $1{,}000$, $3{,}000$ and $10{,}000$ training
points, on the datasets with enough points for $600$ held-out ones (at $10{,}000$ points the six that have fewer than
$100{,}000$): the values are for $100$ validation points, and the change in the mean squared error of weighted KNN
regression, refitted with the kernel weights themselves ($K=5$), is measured on $500$ further held-out points after
$10\%$ of the training points are removed. Monte Carlo runs five times at its equal-time number of
samples. Removing the top $10\%$ by any of the values raises the error far more than removing a random $10\%$, in
every case, and removing the bottom $10\%$, which holds most of the shifted targets, lowers it in every case for the exact
values and Monte Carlo (in $38$ of $39$ for unweighted KNN). The exact values do
not separate from Monte Carlo on either: the exact values' top $10\%$ raises the error more than that of
permutation sampling in $6$ of $12$, $2$ of $9$ and $11$ of $18$ cases, with median ratios of the two increases of
$1.00$, $0.99$ and $1.03$, and the same holds for the bottom $10\%$.

\begin{table}[h]
\centering
\caption{Change in the test error (mean squared error on $500$ held-out points; medians over cases) after removing
the $10\%$ of the training points with the highest (top) or lowest (bottom) values and refitting weighted KNN
regression. \emph{Exact}: the values for kernel weights on $8$ levels. \emph{Permutation} and \emph{Owen}: at equal
time; the change is the mean over removals by each of five runs' values. \emph{Uniform weights}: the exact values of unweighted KNN regression. \emph{Random}:
random scores, mean over five. Without removal the test error has a median of $0.51$, $0.42$ and $0.46$ at the
three sizes.}
\label{tab:app-removal}
\begin{tabular}{@{}rlccccc@{}}
\toprule
$N$ & Removed & Exact & Permutation & Owen & Uniform weights & Random \\
\midrule
1{,}000 & top 10\% & $0.099$ & $0.096$ & $0.100$ & $0.092$ & $0.011$ \\
 & bottom 10\% & $-0.125$ & $-0.125$ & $-0.125$ & $-0.092$ & $0.003$ \\
\addlinespace
3{,}000 & top 10\% & $0.062$ & $0.063$ & $0.064$ & $0.062$ & $-0.002$ \\
 & bottom 10\% & $-0.122$ & $-0.127$ & $-0.122$ & $-0.118$ & $-0.005$ \\
\addlinespace
10{,}000 & top 10\% & $0.083$ & $0.074$ & $0.082$ & $0.080$ & $0.007$ \\
 & bottom 10\% & $-0.108$ & $-0.108$ & $-0.102$ & $-0.108$ & $0.007$ \\
\bottomrule
\end{tabular}

\end{table}

This paper is about computing the values, and it makes no claim about detection or removal. We report the two tasks
because they show where exact values made no detectable difference, which is these downstream uses at the sizes tested, and where
they do, which is the ranking of the points and which points it puts at the top (Section~\ref{sec:exp}).

\end{document}